\documentclass[10pt]{article} 
\usepackage[preprint]{tmlr}

\usepackage{amsmath,amsfonts,bm}

\def\eqref#1{equation~\ref{#1}}

\def\1{\bm{1}}

\DeclareMathAlphabet{\mathsfit}{\encodingdefault}{\sfdefault}{m}{sl}
\SetMathAlphabet{\mathsfit}{bold}{\encodingdefault}{\sfdefault}{bx}{n}

\newcommand{\cX}{\mathcal{X}}
\newcommand{\cC}{\mathcal{C}}
\newcommand{\cO}{\mathcal{O}}
\newcommand{\cP}{\mathcal{P}}
\newcommand{\cI}{\mathcal{I}}
\newcommand{\cA}{\mathcal{A}}
\newcommand{\cM}{\mathcal{M}}
\newcommand{\cT}{\mathcal{T}}

\newcommand{\ES}{$\emptyset$}

\newcommand{\citenump}[1]{[\citenum{#1}]}

\newcommand{\NumSamples}{3,028} 

\renewcommand{\citenump}[1]{\citep{#1}}

\usepackage{url}

\usepackage{microtype}
\usepackage{graphicx}
\usepackage{tabularray}
\usepackage{booktabs} 
\usepackage{hyperref}
\usepackage{wrapfig}

\usepackage{amsfonts}       
\usepackage{nicefrac}       
\usepackage{xcolor}         

\usepackage{amsmath}
\usepackage{amssymb}
\usepackage{mathtools}
\usepackage{amsthm}
\usepackage{thmtools} 
\usepackage{bbm}

\usepackage{enumitem}
\usepackage[bottom]{footmisc}
\usepackage{bm}
\usepackage{algorithm2e}
\usepackage{arydshln}

\usepackage[capitalize,noabbrev]{cleveref}
\usepackage{tikz}
\usetikzlibrary{shapes,arrows,backgrounds,fit,math,positioning,tikzmark,graphs}

\theoremstyle{plain}
\newtheorem{theorem}{Theorem}[section]
\newtheorem{proposition}[theorem]{Proposition}
\newtheorem{lemma}[theorem]{Lemma}
\newtheorem{corollary}[theorem]{Corollary}
\theoremstyle{definition}
\newtheorem{definition}[theorem]{Definition}

\theoremstyle{remark}
\newtheorem{remark}[theorem]{Remark}

\DeclareMathOperator{\tup}{Tupl}

\title{On the Architectural Complexity of Neural Networks}

\author{
	  \hspace{4.5em} \name Nicholas J. Cooper$^{\diamond,\dag,}$\thanks{Correspondence to: \texttt{nick@combinatoriallabs.com}} \hspace{1em}
      \name François G. Meyer$^{\dag}$ \hspace{1em}
      Michael L. Roberts$^{\diamond}$
      \AND
      \hspace{7em} \name Carlos Zapata-Carratalá$^{\S}$ \hspace{1em}
      \name Lijun Chen$^{\dag}$ \hspace{1em}
      \name Danna Gurari$^{\dag}$
}

\def\month{MM}  
\def\year{YYYY} 
\def\openreview{\url{https://openreview.net/forum?id=XXXX}} 

\begin{document}

\tikzset{hfit/.style={rounded rectangle, inner xsep=0pt, fill=#1!30},
	vfit/.style={rounded corners, fill=#1!30}}

\maketitle

\vspace{-2.75em}
\begin{center}
	$^{\diamond}$\href{https://combinatoriallabs.com}{Combinatorial Labs} \hspace{1em} $^{\dag}$University of Colorado Boulder \hspace{1em} $^{\S}$\href{https://semf.org.es/}{SEMF}/Independent Researcher
\end{center}
\vspace{0.5em}

\begin{abstract}
	\vspace{-1em}
	We introduce a unified theoretical framework for the rigorous analysis and systematic construction of deep neural networks (DNNs). This framework addresses a gap in existing theory by explicitly modeling the \textit{structure of tensor operations}---lower level information that is often abstracted. Our framework enables two novel objectives: (1) analysis of the evolution of \textit{architectural complexity} over deep learning history, and (2) automatic construction of novel architectures based on new types of tensor operations.  Our study of DNNs introduced over the past 40 years reveals a connection between groundbreaking architectures and increases in different types of architectural complexity. Moreover, we identify several large classes of higher complexity architectures that have not yet been explored. We then collect a dataset of 3,000+ higher complexity architectures, which we publicly release on {\color{blue}\href{https://github.com/combinatoriallabs/ArchitecturalComplexity}{GitHub}}.
\end{abstract}

\section{Introduction}
Deep neural networks (DNNs) have proliferated across diverse applications \citenump{DLGood1, DLGood2}, demonstrating strong empirical performance. A key driver of

\begin{wrapfigure}[20]{r}{0.575\textwidth}
	\raisebox{0pt}[\dimexpr\height-0.5\baselineskip\relax]{\begin{tikzpicture}
			\node at (3,3) (E) {Elements};  \node at (7,3) {Base set $\mathcal X^0$};
			
			\node at (3,2) (S) {Slices};  \node at (7,2) {$\mathcal X^1 \subsetneq \mathcal P(\mathcal X^0)$};
			
			\node at (3,1) (M) {Modes};  \node at (7,1) {$\mathcal X^2 \subsetneq \mathcal P(\mathcal X^1)$};
			
			\node at (3,0) (T) {Tensors};
			\node at (1,0) (C) {Couplings};
			\node at (5,0) (MMC) {MMCs};  \node at (7,0) {$\mathcal X^3 \subsetneq \mathcal P(\mathcal X^2)$};
			
			\node at (1.75,-1.25) (TO) {Operations};
			\node at (4.25,-1.25) (MM) {Mode Maps};  \node at (7,-1.25) {$\mathcal X^4 \subsetneq \mathcal P(\mathcal X^3)$};
			
			\node at (3,-2.5) (TE) {Neural Networks};  \node at (7,-2.5) {$\mathcal X^5 \subsetneq \mathcal P(\mathcal X^4)$};
			
			\draw[blue, very thick, <->] (E.south) -- (S.north);
			
			\draw[black, very thick, ->] (S.south) -- (M.north);
			
			\draw[black, very thick, ->] (M.south) -- (T.north);
			\draw[black, very thick, ->] ([xshift=-15]M.south) -- (C.north);
			\draw[black, very thick, ->] ([xshift=15]M.south) -- (MMC.north);
			
			\draw[black, very thick, ->] (C.south) -- ([xshift=-5]TO.north);
			\draw[black, very thick, ->] ([xshift=-7]T.south) -- ([xshift=5]TO.north);
			\draw[blue, very thick, <->] ([xshift=7]T.south) -- ([xshift=-5]MM.north);
			\draw[black, very thick, ->] (MMC.south) -- ([xshift=5]MM.north);
			
			\draw[black, very thick, ->] ([xshift=5]TO.south) -- ([xshift=-5]TE.north);
			\draw[black, very thick, ->] ([xshift=-5]MM.south) -- ([xshift=5]TE.north);
			
			\begin{pgfonlayer}{background}
				\node[fit=(E), vfit=cyan!60!white] {};
				\node[fit=(S), vfit=cyan!60!white] {};
				\node[fit=(M), vfit=cyan!60!white] {};
				\node[fit=(C)(T)(MMC), vfit=cyan!60!white] {};
				\node[fit=(TO)(MM), vfit=cyan!60!white] {};
				\node[fit=(TE), vfit=cyan!60!white] {};
			\end{pgfonlayer}
			
			\coordinate (A) at (1.2, -0.2);
			\coordinate (B) at (2.5, -0.2);
			\coordinate (C) at (1.75, -0.9);
			\draw[fill=purple!60!white] (A) -- (B) -- (C) -- cycle;
			
			\coordinate (D) at (4.8, -0.2);
			\coordinate (E) at (3.5, -0.2);
			\coordinate (F) at (4.25, -0.9);
			\draw[fill=cyan!50!white] (D) -- (E) -- (F) -- cycle;
			
			\coordinate (G) at (2.1, -1.45);
			\coordinate (H) at (3.9, -1.45);
			\coordinate (I) at (3, -2.15);
			\draw[fill=orange!60!white] (G) -- (H) -- (I) -- cycle;
			
	\end{tikzpicture}}
	\caption{Proposed hierarchical framework. Arrows denote incidence relation type: ({\color{blue} $\leftrightarrow$}) many-to-many relation; ($\rightarrow$) many-to-one relation. \textit{Cell rank} indicated with superscripts. MMCs are \textit{mode map components} (Def. \ref{def:MMs})}
	\label{fig:RankDiagram}
\end{wrapfigure}
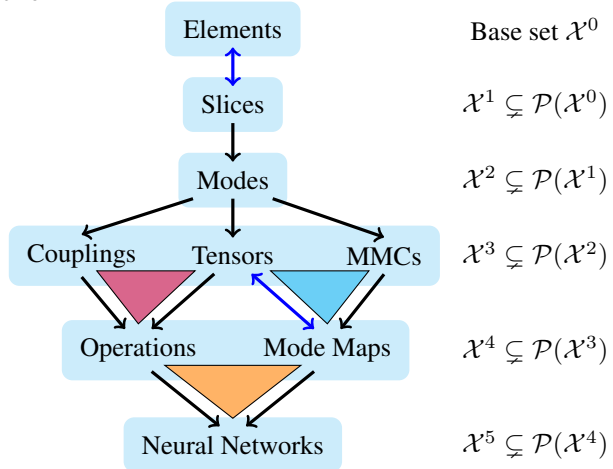
\vspace{-0.5em}
\noindent
this success has been the development of new types of models, leading to a growing diversity of complex architectures~\citenump{MultiLinOpNets, Mamba, DTTenNets}. In this paper, we introduce a unified framework which provides a hierarchical characterization of the \textit{architectural complexity} of neural networks, uncovering retrospective insights into their evolution and providing a foundation for constructing new models.

While unified frameworks for neural networks already exist~\citenump{GeometricDL, CategoricalDL, NeuroAlgebraicDL, CopresheafNNs}, all rely on \textit{high-level abstractions} of tensor operations. For example, Categorical Deep Learning \citenump{CategoricalDL} models architectures as abstract parameterized functions, while Copresheaf networks \citenump{CopresheafNNs} models layers as abstract linear transformations. 

We instead propose a framework that \textit{explicitly models} the structure of tensor operations, enabling the rigorous analysis of architectural complexity and the \textit{systematic construction of new architectures}.  Summarized in \cref{fig:RankDiagram}, it is based on hierarchical combinatorial complexes.  The construction starts with a novel generalization of tensors, which reinterprets multidimensional arrays as a specific type of rank $3$ cells.  These are built from a base set of real-valued variables, called \textit{elements}, and form the foundation for the framework.  {\color{purple}\textit{Tensor operations}} ({\color{purple}$\blacktriangledown$}), such as matrix multiplication, Hadamard product, and multi-head projection, form one type of rank $4$ cells, while tensor transformations such as flattening, unfolding, and patch-ifying, form a second type that we call {\color{cyan}\textit{mode maps}} ({\color{cyan}$\blacktriangledown$}). Neural networks follow immediately as rank $5$ cells composed of mode maps and operations.  This framework parameterizes a highly expressive class of architectures and formalizes many intuitive concepts; e.g., an {\color{orange}\textit{architecture diagram}} ({\color{orange}$\blacktriangledown$}) corresponds to the incidence relationships between operations and mode maps. The framework also highlights the boundary of known architectures; {\color{cyan}mode maps} have received relatively little attention \citenump{DefCNNs, SpiralMLP} and \textit{higher arity}\footnote{A higher arity operation is one that requires more than $2$ operands as input.} {\color{purple}tensor operations} have been largely ignored by both practitioners and theorists alike \citenump{Plexes}.

We apply our hierarchical framework in two ways. \textit{First}, we formally analyze eight types of architectures introduced over the past 40 years in \cref{sec:ComplexityAnalysis}. We find that groundbreaking architectures corresponded to increases in various types of architectural complexity and that many classes of higher complexity architectures have not yet been explored. \textit{Second}, we systematically generate a dataset of $\NumSamples$ novel higher complexity architectures built from new types of tensor operations in \cref{sec:EmpiricalAnalysis}. We find that many exhibit remarkable parameter and depth efficiency. For example, a $5$-layer model outperforms the $30$-layer efficiency-focused MobileNetV2 \citenump{MobileNet} using just $10\%$ of the parameters.

\section{Related Work}
\paragraph{Unified Frameworks for Deep Neural Networks.}
Numerous mathematical frameworks characterize existing architectures, providing insight into their strengths and weaknesses.  For example, the pioneering \textbf{Geometric deep learning} (GDL) ~\citenump{GeometricDL} explains the success of specific layer types by analyzing the \emph{equivariance constraints} they satisfy; e.g., convolutional layers satisfy translational equivariance --- they respond consistently regardless of a pattern's position in a grid~\citenump{EquivariantCNNs}.  \textbf{Categorical deep learning} (CDL) \citenump{CategoricalDL, CatTypesOfLayers, CatToposSurvey} generalizes GDL with more expressive machinery --- homomorphisms of monad algebras. This category-theoretic abstraction enables CDL to describe a broader range of constraints that layers should satisfy; e.g., invariance to non-invertible transformations such as down-sampling a feature map.  \textbf{Neuro-algebraic geometry} \citenump{NeuroAlgebraicDL} complements GDL and CDL by studying the geometric structure of the function spaces parameterized by neural networks. These function spaces are derived from fixed, existing architectures and are known as neuro-manifolds \citenump{SingularityBias, CriticalPointsinMLPs}. Neuro-manifolds are useful for analyzing the expressiveness of existing architectures \citenump{NAGExpress, NAGExpress2}.  Complementing prior work, we introduce a combinatorial framework which is \textit{the first to model the intricate hierarchical structure of neural networks}. Unlike prior work, our framework enables \textbf{exploring new spaces of architectures} built from novel tensor operations.
\paragraph{Neural Architecture Search.} Neural Architecture Search (NAS) aims to discover new architectures \citenump{DARTS, NASSurvey} and to collect datasets of architectures \citenump{NASBench101, NASBench201}. Our work is most similar to einsum based NAS \citenump{Einspace, LinearLayerSearch}, as we focus on the fundamental building blocks of neural networks. However, these techniques are limited to linear transformations whereas our framework extends to higher arity multi-linear tensor operations. Generally, NAS algorithms focus on the top level of the hierarchical structure of neural networks; they discover architectures by varying the connections between a fixed set of existing operations. As a result, existing NAS techniques are special cases of our framework in the sense that any architecture they can discover is expressible in our framework whereas the converse is false. Rather than proposing a search algorithm for an existing architecture space, we advance NAS by constructing a novel architecture space based on high-complexity tensor operations. We also publicly release the first dataset of architectures drawn from this space.
\paragraph{Higher Order Deep Learning.}
Two separate research areas incorporate higher order concepts into deep learning.  \textit{First}, the \textbf{model-centric} approach aims to improve generalization performance by enabling deep neural networks to express higher order polynomial functions of their inputs~\citenump{MultiLinOpNets, DTTenNets}. This area has produced architectures that do not require non-linear activation functions~\citenump{P-Nets, PolyFramework, HadamardSurvey}. \textit{Second}, the \textbf{data-centric} topological deep learning (TDL) \citenump{TDLBook, CopresheafNNs} extends standard deep learning architectures to support higher order data formats. For example, \cite{TDLArchSurvey} demonstrated that combinatorial complexes (Hajij et al., 2022) provide a mathematical framework for characterizing existing neural network architectures. We extend this perspective by leveraging a variant of combinatorial complexes to represent deep neural network architectures as higher order structures, independently of the data domain on which they operate. Crucially, our framework not only describes existing architectures, it enables constructing novel architectures built from higher-complexity tensor operations.

\paragraph{Higher Arity Tensor Algebra.} 
Outside of deep learning, a general tensor operation calculus has been developed based on a correspondence between tensor operations and hypergraphs \citenump{Plexes}. To our knowledge, our work is the first to apply this correspondence to the study of deep neural networks; we leverage it in our combinatorial model of tensor operations. This enables the first ever exploration of neural network architectures built from higher arity tensor operations.

\section{Preliminaries}
We start by introducing some background required to understand the hierarchical framework.

\paragraph{Notation.}

We denote sets by $\{ \}$ and \textit{ordered lists} (strict weak orders\footnote{Think of a strict weak order as a list which may have multiple elements with the same index.}) by $[ \ ]$. For example, $\{A, B, C\} = \{B, C, A\}$, but $[A, B, C] \neq [B, C, A]$.  We write $|S|$ for the cardinality of a set $S$, $\|L\|$ for the length of an ordered list $L$, and $S_1 \setminus S_2$ for the set difference of $S_1$ and $S_2$.

$\mathcal{P}(S)$ denotes the power-set (set of all subsets) and we indicate \textit{power-set rank} with superscripts; e.g., $s^1\in \mathcal{P}(S)$ is rank $1$, $s^2\in \mathcal{P}(\mathcal{P}(S))$ is rank $2$, and so on. A useful example of rank $2$ objects are \textit{partitions}---collections of subsets of a set $S$ that are pairwise disjoint and cover $S$. These subsets are called \textit{parts}, and a partition is \textit{binary} if it consists of exactly two parts; e.g., $P = \big\{ \{A, B\}, \{C, D\} \big\}$ is a binary partition of $\{A, B, C, D\}$. Partitions are always $2$ ranks above their base elements. We write $\{s^r_i\}_{i \in I} \sqsubset s^{r+2}$ to indicate that $\{s^r_i\}_{i \in I}$ is a \textit{transversal of the partition} $s^{r+2}$ --- a selection of exactly one $s^r_i$ from each part of $s^{r+2}$. For example, $\{A, D\}, \{B, C\}$ are transversals of $P$.
\paragraph{Hierarchical Combinatorial Complexes.}

\begin{wrapfigure}[11]{r}{0.375\textwidth}
	\centering
	\raisebox{0pt}[\dimexpr\height-2.0\baselineskip\relax]{\includegraphics[width=0.375\textwidth]{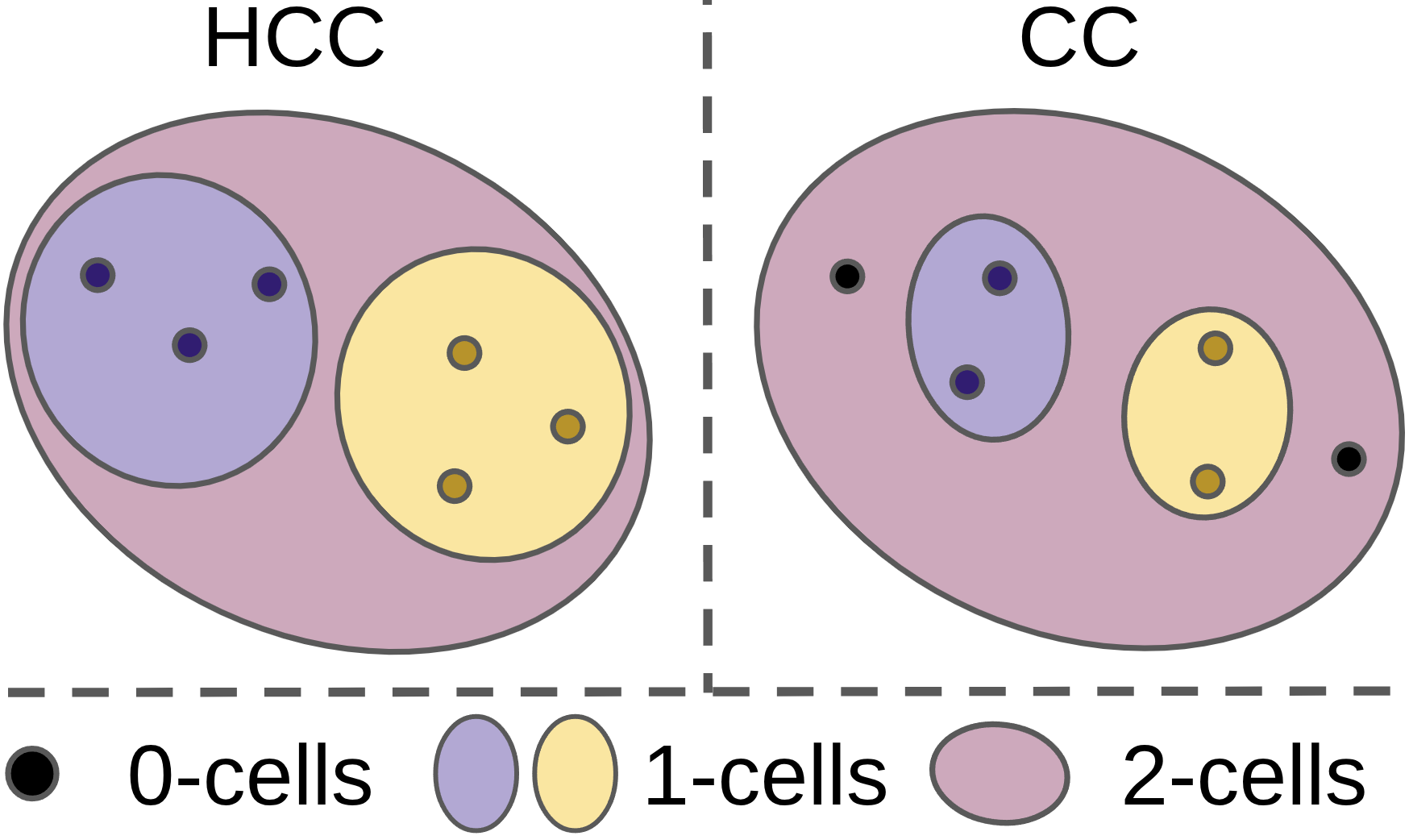}}
	\caption{HCCs vs. CCs. Right complex is not an HCC because the $2$-cell is composed of $1$-cells \textit{and} $0$-cells. $0$-cell shading indicates the sub-HCCs $\mathcal{X}^0|_{s^1_1}$ and $\mathcal{X}^0|_{s^1_2}$.}
	\label{fig:SubHCCEx}
\end{wrapfigure}
Our framework is built on \textit{hierarchical combinatorial complexes} (HCCs).

\begin{definition}
	\label{def:HCCs}
	A \textit{rank-}$r$ hierarchical combinatorial complex $\cX$ is a set $S$ along with a ranked family of collections of sets $\{\mathcal{X}^0, \mathcal{X}^1, ... \mathcal{X}^r\}$ defined recursively:
	\begin{equation*}
		\begin{aligned}
			&\mathcal{X}^0 = \{\{s\} \text{: } s \in S\}\\
			&\mathcal{X}^k \subseteq \mathcal{P}(\mathcal{X}^{k-1}) \text{, for } 0 < k \leq r
		\end{aligned}
	\end{equation*}
	where each $s^k \in \mathcal{X}^k$ is called a rank-$k$ cell, or a $k$-cell.
\end{definition}

HCCs are \textit{combinatorial complexes} (CCs)~\citenump{TDLBook} with an additional property: all $k$-cells are composed only of $(k-1)$-cells. This requires complexes to have strict hierarchies of cells, as shown in \cref{fig:SubHCCEx}. We denote by $\mathcal{X}^l|_{s^k}$ the set of $l$-cells contained in $s^k$ for $k > l$. Similarly, we indicate \textit{the sub-HCC defined by} $s^k$ by $\mathcal{X}|_{s^k} := \langle \cX^0|_{s^k}, ..., \cX^{k-1}|_{s^k}, s^k \rangle$.
\paragraph{Tensors.}
In this paper, tensors are multidimensional arrays (MDAs).
\begin{definition}
	\label{def:Tensor}
	A tensor $T$ is a function defined on a finite Cartesian product of finite index sets:
	\begin{equation*}
		T: [[M_1]] \times [[M_2]] \times ... \times [[M_{\cO}]] \rightarrow V
	\end{equation*}
	where $[[M_i]] = [1, 2, ..., M_i]$. $\cO$ is called the \textit{order} of the tensor, and $V$ is the \textit{value set} of the tensor.
\end{definition}

The value of the function $T$ at \textit{multi-index} $(i_1, i_2, ..., i_{\cO})$ can be seen as the value located at the $(i_1, i_2, ..., i_{\cO})^{th}$ position of an $\cO$-dimensional array. $[[M_i]]$ is called the $i^{th}$\textit{-mode} of the tensor.

\begin{wrapfigure}[8]{r}{0.195\textwidth}
	\centering
	\raisebox{0pt}[\dimexpr\height+0.0\baselineskip\relax]{\includegraphics[width=0.195\textwidth]{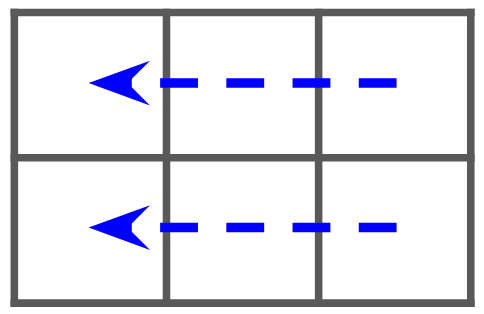}}
	\caption{A $1$-slice space of a matrix.}
	\label{fig:SSEx}
\end{wrapfigure}
Various partitions of tensors such as the pixel space of an image are \textbf{\textit{slice spaces}}, \textit{higher order functions}\footnote{Higher order functions are functions $f$ of the form $f: S_1 \rightarrow (g: S_2 \rightarrow S_3).$} obtained from binary partitions of a tensor's modes.
\begin{definition}
	\label{def:SliceSpaces}
	A $k$-slice space of a tensor $T$ is a tensor-valued tensor:
	\begin{equation*}
		\begin{aligned}
			T_{\lambda}: [[M_1]] \times ... \times [[M_{\cO-k}]] \rightarrow \Big( T_{\sigma}: [[M_{\cO-k+1}]] \times ... \times [[M_{\cO}]] \rightarrow V \Big).
		\end{aligned}
	\end{equation*}
\end{definition}
Slice spaces are function-valued functions which accept $(\cO-k)$ input arguments and return functions of $k$ input arguments. The image of each $T_{\sigma}$ is called a $k$\textit{-slice}. Given a subset of modes $M' = \{M_{i_1}, ..., M_{i_k}\}$ the $M'$-\textit{slice space} is the slice space for which each $T_{\sigma}$ is a function defined on $\bigotimes_{j = 1}^{k}[[M_{i_j}]]$. We write $\sigma_k \vartriangleleft_{M'} T$ to indicate that $\sigma_k$ is a $k$-slice of the $M'$-slice space of $T$. For example, the pixel space of an $[[M_H]]\times [[M_W]] \times [[M_C]]$ color image is the $1$-slice space defined by $M' = \{M_C\}$; the attention maps of a multi-head self-attention layer (with $M_{head}$ heads) form a $2$-slice space of an order $3$ tensor, where $T_{\lambda}$ is defined on $[[M_{head}]]$. Different slice spaces provide different perspectives on the information encoded in a tensor.
\section{Theoretical Framework\label{sec-proposed-framework}}

We now describe our hierarchical framework for deep neural networks. All proofs are provided in the appendix. \cref{fig:RankDiagram} displays the types of $k$-cells we will construct in order to model neural networks as rank $5$ HCCs. We start with a base set of \textit{elements}, which are real-valued variables. We then construct a generalized version of multidimensional arrays as $3$-cells built from these elements. Next, we construct {\color{cyan}\textit{mode maps}} and {\color{purple}\textit{tensor operations}} as $4$-cells. Finally, we formalize the intuitive concept of {\color{orange}\textit{architecture diagrams}} as $5$-cells. These $5$-cells describe a broad class of architectures.


\subsection{Generalized Tensors}

We now express multidimensional arrays (\cref{def:Tensor}) as rank-$3$ cells. This will provide a more general definition that encompasses \textit{jagged tensors} --- incomplete arrays used in popular libraries \citenump{Pytorch, TensorFlow}. The definitional translation relies on an important property of multidimensional arrays: \textit{each element is organized with respect to multiple ordered sets simultaneously}. This is how tensors encode such rich information about their elements. We define tensors in the language of HCCs as follows:

\begin{definition}
	\label{def:GenTensor}
	A \textit{generalized tensor} is a $3$-cell $s^3$ of an HCC $\mathcal{X}$ that partitions $\mathcal{X}^1|_{s^3}$ and satisfies:
	\begingroup\setlength{\jot}{-1ex}
	\begin{equation*}
		\begin{aligned}
			&\textit{1.} \ \text{Each } 1\text{-cell } s^1 \in \mathcal{X}^1|_{s^3} \text{ is equipped with a strict weak order.}\\
			&\textit{2.} \ \text{For each } x \in \mathcal{X}^0|_{s^3}, \text{ there exists a transversal } \{s^1_i\}_{i\in I} \sqsubset s^3 \text{ such that } x \in \bigcap_{i=1}^{|s^3|} s^1_i\\
			&\textit{3. The slice ordering compatibility conditions.}\\
		\end{aligned}.
	\end{equation*}
	\endgroup
\end{definition}

\begin{wrapfigure}[10]{r}{0.47\textwidth}
	\centering
	\raisebox{0pt}[\dimexpr\height-2.75\baselineskip\relax]{\includegraphics[width=0.625\linewidth]{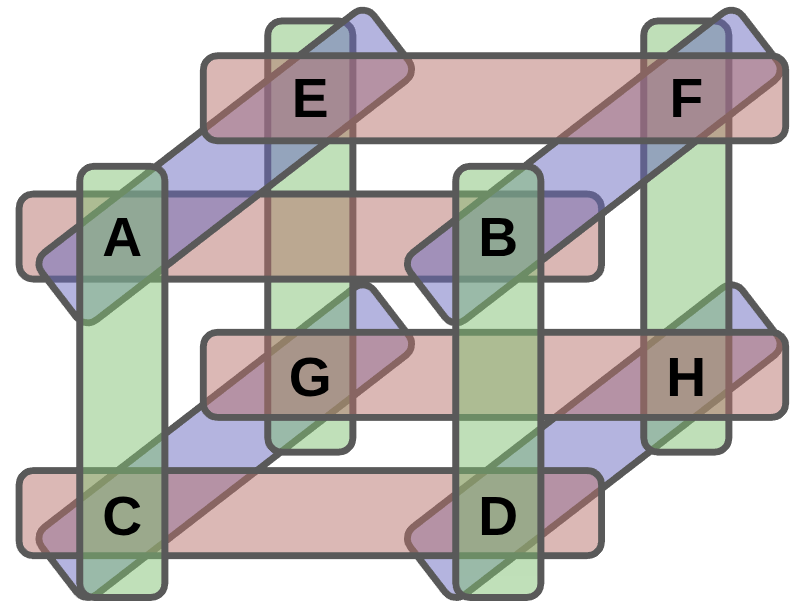}}
	\caption{Visualization of the generalized tensor associated to a $2\times 2\times 2$ multidimensional array. Colors indicate the partition of the $1$-cells.}
	\label{fig:GenTensorExample}
\end{wrapfigure}
The slice ordering compatibility conditions are technical requirements of the ordered $1$-cells that ensure each element of a generalized tensor can be unambiguously assigned a multi-index. As they are not required to understand the key ideas of our framework, we provide full details in \cref{sec:GenTenSOCC_Apdx}.

Generalized tensors are effectively partitions of collections of ordered sets. Their $2$-cells correspond to the modes of a multidimensional array and their $1$-cells to $1$-slices. \cref{fig:GenTensorExample} provides an example. The HCC for this tensor is given below:
\begin{equation*}
		\Big\{{\color{red}r} = \big\{[A,B], [C,D], [E,F], [G,H]\big\},
		{\color{green}g} = \big\{[A,C], [B,D], [E,G], [F,H]\big\},
		{\color{blue}b} = \big\{[A,E], [B,F], [C,G], [D,H]\big\} \Big\}
\end{equation*}
Note that the $2$-cells form a partition of the ordered $1$-cells and that each element is contained in the intersection of some transversal of this partition.

\paragraph{Generalized Tensors vs. Multidimensional Arrays.} Generalized tensors are strictly more expressive objects than injective multidimensional arrays (i.e., arrays with no duplicate elements). The following theorem precisely describes the connection between \cref{def:GenTensor,def:Tensor}.
\begin{theorem}
	All injective multidimensional arrays can be represented as generalized tensors. Moreover, if $s^3$ is a finite generalized tensor which also satisfies the following two conditions:
	\begin{equation*}
		\begin{aligned}
			&\textit{1. } \forall s^2_i \in s^3, \exists \text{ a positive constant } c_i \in \mathbb{N}_+ \text{ such that } ||s^1|| = c_i, \ \forall s^1 \in \cX^1|_{s^2_i},\\
			&\textit{2. } \text{For each transversal } \{s^1_i\}_{i\in I} \sqsubset s^3, \text{ we have that } \bigg| \bigcap_{i=1}^{|s^3|} s^1_i \bigg| \leq 1,\\
		\end{aligned}
	\end{equation*}
	
	then $s^3$ can be represented as an injective multidimensional array.
	\label{thrm:GenTensors<=>MDAs}
\end{theorem}
Removing requirement (1) leads to the jagged tensors of PyTorch \citenump{Pytorch}. Removing requirement (2) leads to the notion of \textit{hyper-tensors} --- multidimensional arrays that map multi-indices to multiple elements. Hyper-tensors that map each multi-index to exactly $\alpha$ elements are called $\alpha$-\textit{regular}. Hyper-tensors will be useful for establishing important properties of tensor operations. As discussed next, mode maps are required to model non-injective multidimensional arrays.
\subsection{Mode Maps \label{sec:MMs}}
\begin{wrapfigure}[13]{r}{0.471\textwidth}
	\tikzset{
		CircleStyle/.style={
			circle,
			inner sep=0pt,
			text width=5mm,
			align=center,
			draw=black,
			fill=black!20
		}
	}
	\centering
	\raisebox{0pt}[\dimexpr\height-4.25\baselineskip\relax]{
		\includegraphics[width=0.45\textwidth,height=0.31\textwidth]{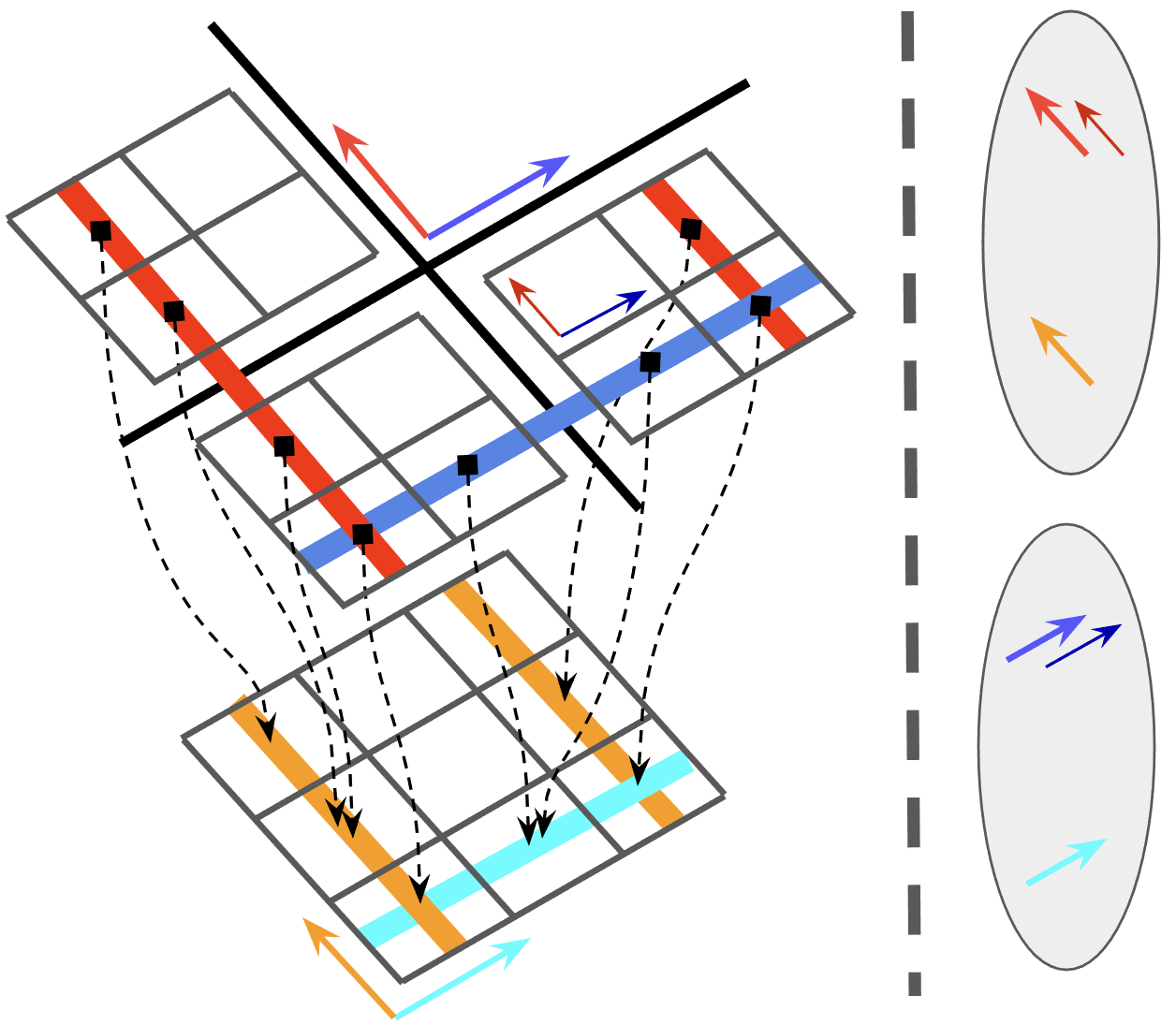}
	}
	\caption{Example of a mode map. Left side: a jagged tensor of order $4$ (top) maps into a matrix (bottom). Right side: the mode map components. Shaded arrows indicate the modes.}
	\label{fig:ModeMapsEx}
\end{wrapfigure}
Mode maps are functions between tensors that preserve their slice space structures. As tensors can have \textit{many} slice spaces, there is not a unique choice of structure to preserve, so mode maps must be defined with respect to collections of modes, i.e., $3$-cells. This makes mode maps themselves $4$-cells.

\vspace{0.75em}
\begin{definition}
	A \textit{mode map} is a $4$-cell $s^4$ of an HCC $\mathcal{X}$ that contains two generalized tensors $t_1, t_2 \in s^4$ and satisfies the following conditions:
	\begin{equation*}
		\begin{aligned}
			\textit{1. } &\text{There exists a unique } 2\text{-cell } f \in \cX^2|_{s^4}\\ 
			&\text{that defines a function } f: \cX^0|_{t_1} \rightarrow \cX^0|_{t_2}.\\
			\textit{2. } &\text{For each } s^3 \in s^4 \setminus \{t_1, t_2\}:\\
			&\textit{i. } \text{The set } \big\{ p_1 = s^3 \cap t_1, \ p_2 = s^3 \cap t_2 \big\} \text{ is a binary partition of } s^3,\\
			&\textit{ii. } \text{The image of each } k\text{-slice } \sigma_k \vartriangleleft_{p_1} t_1 \text{ under $f$ is an } l\textit{-sub-slice } \text{ of } t_2, \text{ that is, }\\
			&\quad f(\sigma_k) \subseteq \sigma_l \vartriangleleft_{p_2} t_2 \text{, where } k = |p_1|, \ l = |p_2|.\\
			&\ \text{The } 3 \text{-cells } s^3 \in s^4 \setminus \{t_1, t_2\} \text{ are called }\textit{mode map components}\text{ or MMCs.}
		\end{aligned}
	\end{equation*}
	\label{def:MMs}
\end{definition}
\vspace{-1em}
Mode map components describe which slice spaces are preserved by a function between two generalized tensors. \cref{fig:ModeMapsEx} provides an illustration of the familiar \textit{unfolding} mode map that is used in convolutional layers. It is important to notice that each $2$-slice of the {\color{red}red slice-space} maps into a $1$-slice of the {\color{orange}orange slice-space}. An analogous statement is true of the {\color{blue}dark} and {\color{cyan}light} blue slice-spaces.

Beyond unifying the tensor transformations necessary to construct neural networks, mode maps describe the structure of these transformations, explaining how they fit into the hierarchy of \cref{fig:RankDiagram}. Similarly, mode maps between generalized tensors enable the translation of non-injective multidimensional arrays into the hierarchical language of the HCC framework. 

\begin{lemma}
	Any multidimensional array can be represented as a mode map.
	\label{lem:MMs=>NonInjMDAs}
\end{lemma}

Mode maps are one of two types of $4$-cells that comprise neural networks. The second type are \textit{tensor operations}, $4$-cells which describe how multiple tensors can be combined to produce new tensors.
\subsection{Tensor Operations}
Tensor operations are at the core of neural networks. As tensor operations involve tensors, they are $4$-cells. Similar to mode maps, tensor operations involve a second type of $3$-cell called \textit{couplings}.
%

\begin{definition}
	\label{def:1SM}
	An $\alpha$\textit{-ary tensor operation} is a $4$-cell $s^4$ of an HCC $\mathcal{X}$ that contains generalized tensors $t_1, t_2, ..., t_{\alpha} \in s^4$ and satisfies:
	\begin{equation*}
		\begin{aligned}
			&\textit{1. } \text{The set } \{t_1, t_2, ..., t_{\alpha}\} \ \text{partitions } \mathcal{X}^2|_{s^4}.\\
			&\textit{2. } \text{For each } s^3_i \in s^4 \setminus \{t_1, t_2, ..., t_{\alpha}\}:\\
			&\qquad \textit{i. }|s^3_i \cap t_j| \leq 1, \ \forall j \in [[\alpha]],\\
			&\qquad \textit{ii. }\exists c_i \in \mathbb{N}_+ \text{ such that each } s^2 \in s^3_i \text{ has } \textit{tensor length } c_i.\\
			&\quad \text{The } 3 \text{-cells } s^3_i \in s^4 \setminus \{t_1, t_2, ..., t_{\alpha}\} \text{ are called }\textit{tensor couplings.}
		\end{aligned}
	\end{equation*}
\end{definition}
The precise definition of \textit{tensor length} is somewhat technical and is provided in the appendix. Intuitively, condition $2.ii$ is a generalization of the familiar rule that $m \times n_1$ matrices may be multiplied with $n_2 \times p$ matrices exactly when $n_1 = n_2$. The technicalities are necessary to ensure that jagged tensors (i.e., tensors which do not satisfy requirement 1 of \cref{thrm:GenTensors<=>MDAs}) can always be used as operands in tensor operations to produce well-defined results.

Given a tensor operation, the number of operands ($\alpha$) is called the operation's \textit{arity} and the size of the largest coupling is called the \textit{coupling-arity}. The number of modes (counted only once per coupling) is called the \textit{order complexity}. Sometimes (e.g., \citenump{ConeProduct}), the term \textit{arity} is used to mean \textit{coupling-arity} instead, as the two coincide for the majority of ternary (arity $3$) operations on order three tensors \citenump{Plexes}.

Condition $2.i$ of \cref{def:1SM} ensures that tensor operations admit useful binary-valued matrix representations that we call \textit{tensor operation matrices} (TOMs). These representations are effectively incidence matrices 
\begin{wrapfigure}[16]{r}{0.375\textwidth}
	\centering
	\raisebox{0pt}[\dimexpr\height-0.5\baselineskip\relax]{
	\includegraphics[width=0.375\textwidth]{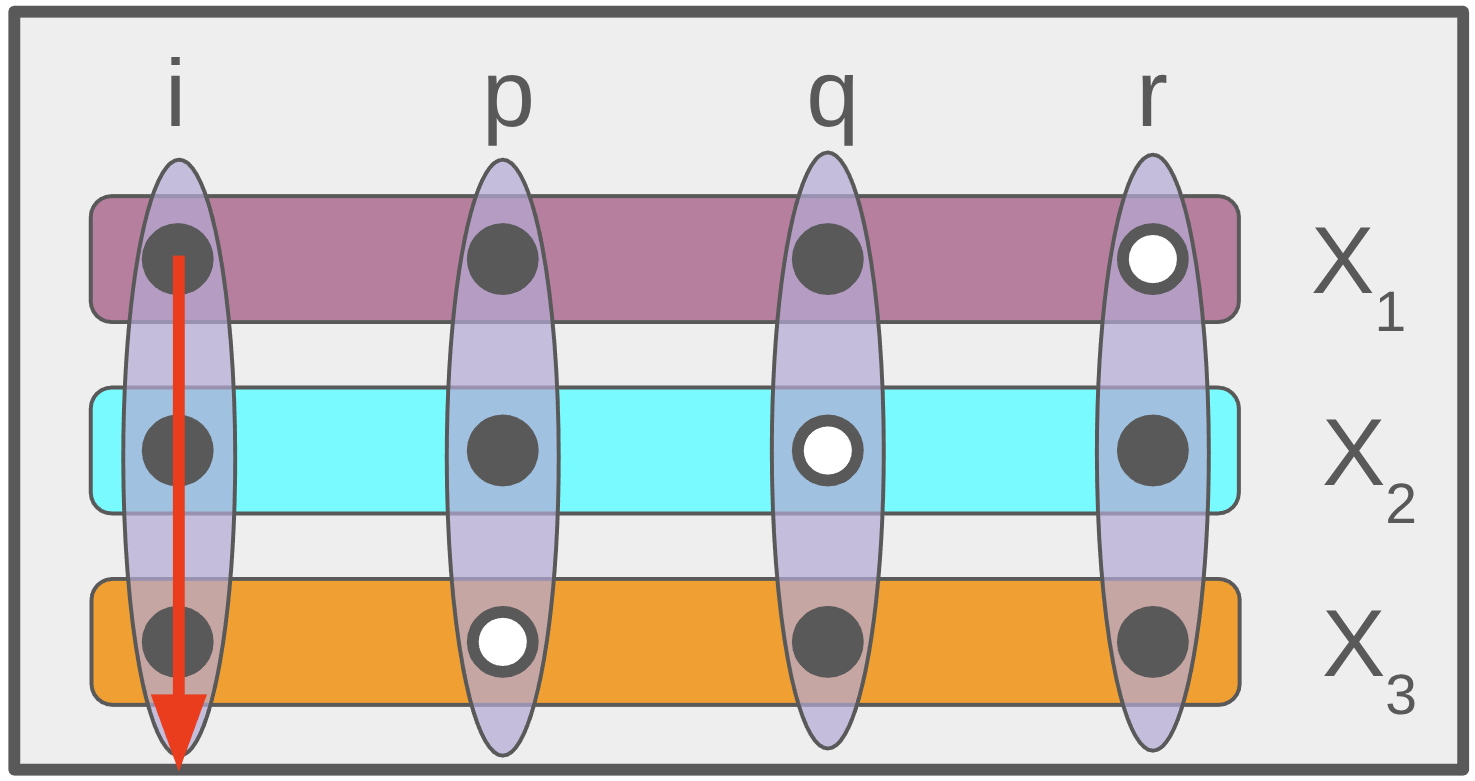}}
\caption{TOM representation of the cone product. Rows correspond to input tensors; columns correspond to modes of the output tensor. ({\color{darkgray}$\bm{\bullet}$/$\bm{\circ}$}) indicates which modes are/are not part of which inputs. ({\color{red}$\bm{\downarrow}$}) denotes a contracted mode.}
\label{fig:1SMsEX}
\end{wrapfigure}
(of arity many rows and order complexity many columns) that describe the relational structure between couplings and tensors ({\color{purple}purple triangle} of \cref{fig:RankDiagram}). As an example, the TOM for the \textit{Bhattacharya-Mesner product} \citenump{ConeProduct} is shown in \cref{fig:1SMsEX}. This ternary tensor operation is defined by:
\begin{equation*}
Y[p, q, r] = \sum_{i=1}^{n}{\color{purple}\bm{X_1}}[i, p, q]{\color{cyan}\bm{X_2}}[i, p, r]{\color{orange}\bm{X_3}}[i, q, r]
\end{equation*}
for order $3$ tensors $X_1, X_2, X_3$. All the algebraic information of the above equation, i.e., which indices are used for which input tensors, is encoded in the matrix representation shown in \cref{fig:1SMsEX}. Note that the $i$-mode is \textit{contracted}. Contracted modes are summations along the corresponding $1$-slice spaces of a tensor. Any tensor operation can be specified by its couplings (which describe how the input tensors align) and its contractions (which describe the summations).

When supplied with two \textit{base operations} --- functions $\times, +: \mathbb{R}^{2} \rightarrow \mathbb{R}$ --- tensor operations produce new tensors from their inputs by \textit{evaluating} the operation. The language of HCCs elucidates that this conversion of operations ($4$-cells) to tensors ($3$-cells) relies on an intermediate hyper-tensor.
\clearpage

\begin{figure}[t!]
	\centering
	\includegraphics[width=\textwidth]{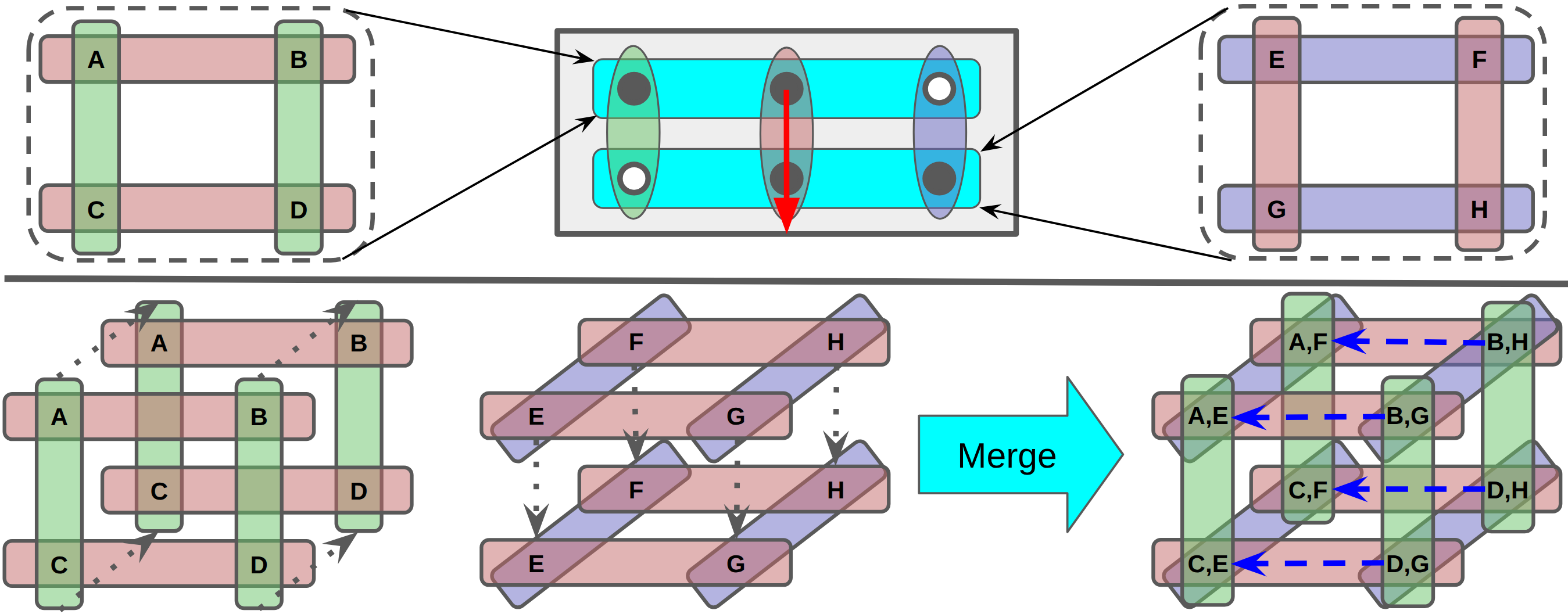}
	\caption{Illustration of \cref{thrm:1-SMs=>Tensors}. \textbf{Top}: input tensors to the TOM for matrix multiplication; modes are color-coded. \textbf{Bottom}: operands are broadcasted ({\color{gray}dotted arrows}) according to the coupling structure and then merged to create a $2$-regular hyper-tensor (bottom right). The result is given by applying one base operation along the tuples of the hyper-tensor and then applying the second base operation along the $1$-slice space ({\color{blue}$\bm{\dashleftarrow}$}) defined by the contraction ({\color{red}$\bm{\downarrow}$}).}
	\label{fig:TensorOpEx}
\end{figure}
\begin{theorem}
\label{thrm:1-SMs=>Tensors}
Tensor operations are evaluated in two distinct steps. First, every tensor operation determines a $k$-slice-space of a hyper-tensor, where $k$ is the number of contractions of the operation. Second, given \text{any} two base operations, every hyper-tensor slice-space then determines a non-hyper tensor. Moreover, if each operand satisfies the conditions of \cref{thrm:GenTensors<=>MDAs}, then the hyper-tensor is $\alpha$-regular, where $\alpha$ is the operation's arity.
\end{theorem}
This result describes how tensor operations are defined by two independent pieces of data: the tensor structure of the operation, and the choice of base operations on the underlying set $\mathbb{R}$. \cref{fig:TensorOpEx} demonstrates how \cref{thrm:1-SMs=>Tensors} works for the familiar case of matrix multiplication. It is important to observe how the operands are \textit{broadcasted} (i.e., copied) across the slice-spaces of the output hyper-tensor defined by the filled circles ({\color{darkgray}$\bm{\bullet}$}) of each row of the tensor operation matrix. \Cref{thrm:1-SMs=>Tensors} provides an algorithm for evaluating \textit{any tensor operation} because this broadcasting method of operation evaluation \textit{extends uniformly} to tensor operations of arbitrarily high arity and order complexity. Moreover, \cref{thrm:1-SMs=>Tensors} can be used to show (\cref{sec:ArityDecomp}) that when the base operations are multiplication and addition on $\mathbb{R}$, compatible sequences of binary tensor operations are equivalent to higher arity operations. We will use this fact frequently in our architectural complexity analysis.

\subsection{Neural Networks}
\begin{wrapfigure}[10]{r}{0.375\textwidth}
	\centering
	\raisebox{0pt}[\dimexpr\height-2.75\baselineskip\relax]{
		\includegraphics[width=0.375\textwidth]{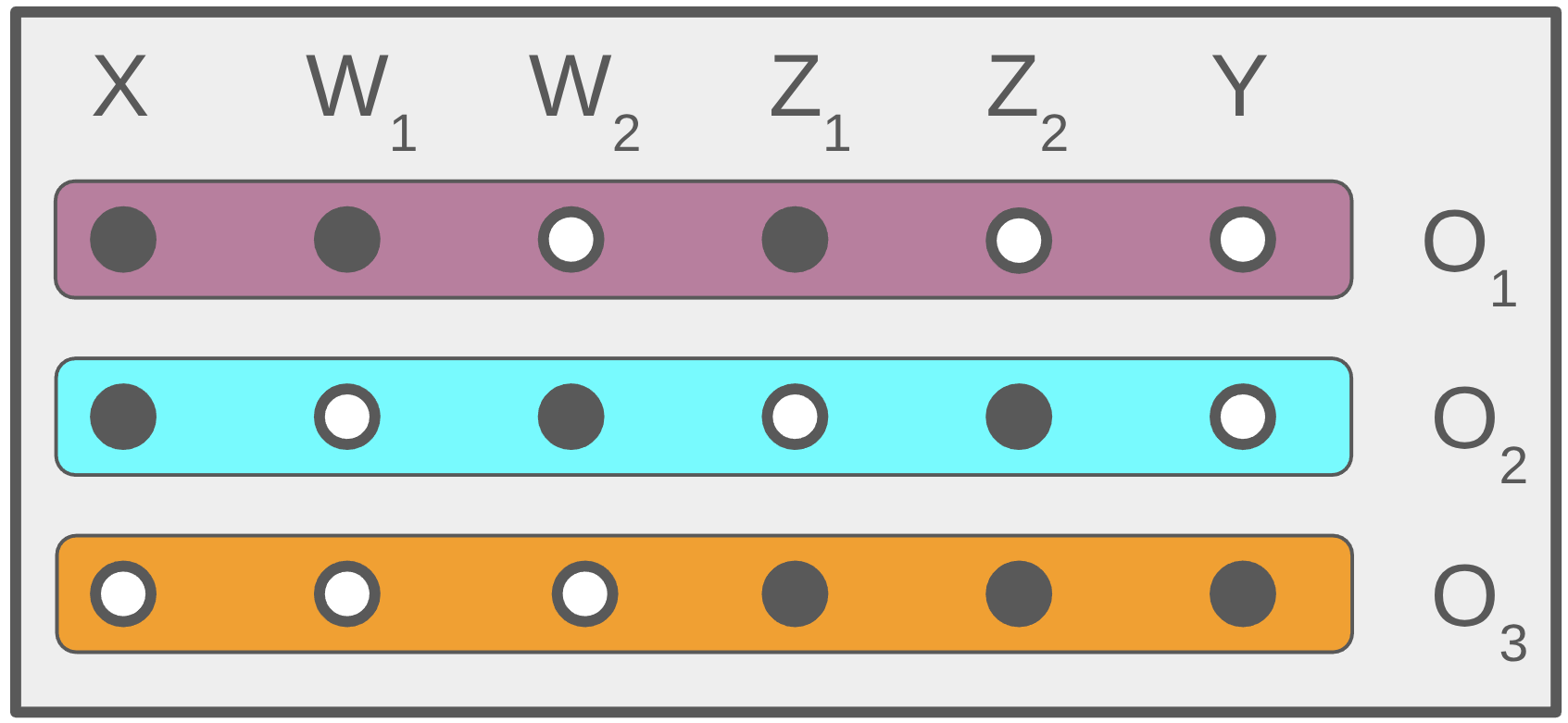}}
	\caption{TEM for the PolyNets core block. Rows correspond to operations; columns correspond to tensors. ({\color{darkgray}$\bm{\bullet}$/$\bm{\circ}$}) indicates which tensors are/are not part of which operations.}
	\label{fig:2SMEx}
\end{wrapfigure}
The foundation we have developed allows us to formalize neural networks.
\begin{definition}
A \textit{neural network} is a $5$-cell that is composed of mode maps and tensor operations.
\end{definition}
Analogously to how tensor operations can be represented with TOMs, the top level relational structure in a neural network ({\color{orange}orange triangle} of \cref{fig:RankDiagram}) can be represented with binary-valued matrices that we call \textit{tensor equation matrices} (TEMs). These incidence matrices describe the connections between tensors and operations. As an example, the TEM for the core building block of PolyNets \citenump{P-Nets} is visualized in \cref{fig:2SMEx}. We recall this system of operations:
\begin{equation*}
\begin{aligned}
	Z_1 &{\color{purple}\bm{=}} X \times W_1\\
	Z_2 &{\color{cyan}\bm{=}} X \times W_2\\
	Y &{\color{orange}\bm{=}} Z_1 \odot Z_2
\end{aligned}
\end{equation*}
where $\times$ is matrix multiplication and $\odot$ is Hadamard product. While TOMs encode modes along the columns and \textit{tensors along the rows}, TEMs are one rank up, so they instead encode \textit{tensors along the columns} and operations along the rows. \textit{Operation structure} is invisible from this level; we cannot discern between $\times$ and $\odot$. However, TEMs provide a useful representation of tensor \textit{equation structure}.

TOMs and TEMs (\cref{fig:1SMsEX,fig:2SMEx}) provide a convenient combinatorial representation of the complete hierarchical structure of neural networks. We leverage this parameterization to create the first \textit{hierarchical neural network engine} that is capable of processing \textit{any} system of \textit{any} tensor operations. For completeness, we define non-linear activations to be unary (arity $1$) tensor operations with particular choices of unary base operations. Our publicly shared codebase contains an efficient GPU implementation of this neural network engine to facilitate future research.
\section{Applications}
We now apply our framework to analyze architectural complexity trends and to create novel models.
\subsection{\label{sec:ComplexityAnalysis}Architectural Complexity Analysis}
Our framework enables a variety of rigorous characterizations of the architectural complexity of neural networks via the cardinalities of various sub-HCCs:
\begin{itemize}[leftmargin=*, itemindent=0pt]
	\itemsep0em 
	\item \textbf{Operation complexity} ($\cC_{op}$) is the number of operations contained in a neural network.
	\item \textbf{Tensor complexity} ($\cC_{T}$) is the number of tensors contained in a neural network.
	\item \textbf{Arity complexity} ($\cC_{\alpha}$) is the maximum arity of any single operation.
	\item \textbf{Order complexity} ($\cC_{O}$) is the maximum order complexity of any single operation.
	\item \textbf{Coupling-arity complexity} ($\cC_{A}$) is the maximum size of any single coupling.
\end{itemize}
The above architectural complexity measures are organized by decreasing rank of the cells involved, with $\cC_{op}$ the number of $4$-cells in a $5$-cell, $\cC_T$ the number of $3$-cells in a $5$-cell, $\cC_{\alpha}$ the number of $3$-cells in a $4$-cell, $\cC_O$ the number of $2$-cells in a $4$-cell, and $\cC_A$ the number of $2$-cells in a $3$-cell. 

To analyze the evolution of architectural complexity and to demonstrate the generality of our framework , we compute the \textit{complexity signature}---the collection $(\cC_{op}, \cC_T, \cC_{\alpha}, \cC_O, \cC_A)$---for eight types of architectures introduced over the past 40 years.  We consider four \textit{fundamental} architectures---fully-connected neural network (FCNN), convolutional neural network (CNN), Residual Network (ResNet), and Transformer---and four \textit{post-Transformer}---Poly-Net, MO-Net, ViM, and TT-Net.  We focus on the core aspects of each architecture rather than the entire model; e.g., the reported signatures for ResNet and Transformer are for residual blocks and self-attention layers, respectively. All complexities are computed from the simplest possible HCC encoding of each core block, where lower rank cells are considered simpler. We provide the full details of all architecture encodings along with discussion of alternative encoding strategies in the appendix.
\paragraph{Results.}
The computed complexity signatures for all architectures are shown in \cref{tab:ComplexityHistory}.

\begin{table*}[h!]
	\resizebox{\textwidth}{!}{
		\begin{tblr}{
				colspec={c||c|c|c|c||c|c|c|c||},
				column{1}={halign=r},
				row{6}={bg=cyan!30!white, fg=black},
				row{8}={bg=orange!30!white, fg=black}
			}
			\SetCell[r=2]{c}
			& \SetCell[c=4]{c} Fundamental Architectures 
			& & & & \SetCell[c=4]{c} Post-Transformer Architectures \\\cline{2-9}
			& FCNN & CNN & ResNet & Transformer & Poly-Net & MO-Net & ViM & TT-Net\\\hline\hline
			Year & 1986 & 1998, 2012 & 2016 & 2017 & 2020 & 2024 & 2024 & 2025\\\hline
			$\cC_{op}$ & $1$ & $1$ \tikzmark{op1} & \tikzmark{op2} $3$ & $3$ & $4$ & $4$ & $27$ & $5$ \\\hline
			$\cC_{T}$    & $2$ & $2$ \tikzmark{t1}  & \tikzmark{t2} $6$ & $7$ & $9$ & $9$ & $45$ & $11$ \\\hline
			$\cC_{\alpha}$    & $2$ & $2$ & $2$ \tikzmark{a1} & \tikzmark{a2} $3$ & $3$ & $4$ & $3$ & $3$ \\\hline
			$\cC_{O}$     & $3$ \tikzmark{oc1} & \tikzmark{oc2} $6$ & $6$ & $4$ & $6$ & $4$ & $4$ & $6$ \\\hline
			$\cC_A$ & $2$ & $2$ & $2$ & $2$ & $2$ & $2$ & $2$ \tikzmark{A1} & \tikzmark{A2} $3$ \\\hline
		\end{tblr}
	}
	\begin{tikzpicture}[overlay,remember picture]
		\draw[->,blue, very thick] ([yshift=0,xshift=4]pic cs:oc1) -- ([yshift=0,xshift=4]pic cs:oc2);
		\draw[->,blue, very thick] ([yshift=4.0,xshift=8]pic cs:op1) -- ([yshift=4.0,xshift=6]pic cs:op2);
		\draw[->,blue, very thick] ([yshift=3.2,xshift=8]pic cs:t1) -- ([yshift=3.2,xshift=6]pic cs:t2);
		\draw[->,blue, very thick] ([yshift=1.6,xshift=14]pic cs:a1) -- ([yshift=1.6,xshift=12]pic cs:a2);
		\draw[->,blue, very thick] ([yshift=-1.6,xshift=-12]pic cs:A1) -- ([yshift=-1.6,xshift=-15]pic cs:A2);
	\end{tikzpicture}
	\vspace{-0.5em}
	\caption{Evolution of architectural complexity over the history of neural networks, highlighting the trend of increasingly complex designs (with each {\color{blue} $\bm{\rightarrow}$} indicating the first known increase in each of the five complexity types). All complexity values are for the `core blocks' of each architecture.} 
	\label{tab:ComplexityHistory}
	\vspace{-1em}
\end{table*}

As visualized with blue arrows ({\color{blue} $\bm{\rightarrow}$}), there is an overall trend of increasing architectural complexity.  From left-to-right, fully-connected networks \citenump{BackpropMLPs} are the least complex, convolution \citenump{CNNs, AlexNet, VGG} is a higher order complexity operation, skip connections \citenump{ResNet} require additional tensors and operations, and the minimum complexity encoding of self-attention \citenump{Transformers} requires a ternary operation. Following transformers, the complexity of core blocks has exploded \citenump{P-Nets, MultiLinOpNets, VisionMamba, DTTenNets}.  Notably, this precise characterization of architectural complexity is possible because our framework provides a unified language for the consistent description of the `core blocks' used in all these architectures.

Examining the evolution of the ``Fundamental Architectures" (\cref{tab:ComplexityHistory}, left half), we observe a direct correspondence between major groundbreaking architectures and the first observed increase in each complexity type. For example, the step from ResNet's residual blocks to Transformer's self-attention corresponds to the first instance of a higher-arity ($\cC_{\alpha}$) tensor operation ({\color{cyan}blue row}). This finding suggests that previously overlooked types of complexity can provide guidance for developing new architectures, highlighting the value of exploring higher complexity signature neural networks.

Examining the evolution of the ``Post-Transformer Architectures" (\cref{tab:ComplexityHistory}, right half), we observe the boundary of known architectures.  For example, until 2025~\citenump{DTTenNets}, there was \textit{no exploration} of higher coupling arity ($\cC_A > 2$) architectures ({\color{orange} orange row}). Moreover, all known examples of $\cC_{\alpha} > 2$ architectures have been discovered \textit{accidentally} --- they are always described with HCCs of $\cC_{\alpha} = 2$, but happen to admit simpler HCC encodings with $\cC_{\alpha} > 2$. Such simplifications are possible because, as we prove in \cref{sec:ArityDecomp}, when tensor operations are evaluated with real number multiplication and addition, certain sequences of binary operations are equivalent to higher-arity operations.

\subsection{\label{sec:EmpiricalAnalysis}Generating New High Complexity Architectures}
Our hierarchical framework described in \cref{sec-proposed-framework} \textit{greatly simplifies the process of creating new architectures}. In particular, a core block of any desired complexity signature can be obtained by first sampling a $0/1$-valued matrix of size $\cC_{op} \times \cC_T$ to define a TEM, then sampling a sequence of $0/1$-valued matrices of size at most $\cC_{\alpha} \times \cC_O$ to define the TOMs. We leverage this expressive hierarchical representation of neural networks to extend the boundary of known architectures highlighted in \cref{sec:ComplexityAnalysis}, conducting the first systematic exploration of higher-complexity (up to $\cC_{\alpha}, \cC_A = 4, \cC_O \leq 11$) architectures.  Following \citenump{NASBench101,NASBench201}, we focus on the foundational image classification task.

\paragraph{Overview.}
\begin{wrapfigure}[7]{r}{0.375\textwidth}
	\centering
	\raisebox{0pt}[\dimexpr\height-1.25\baselineskip\relax]{\includegraphics[width=0.375\textwidth]{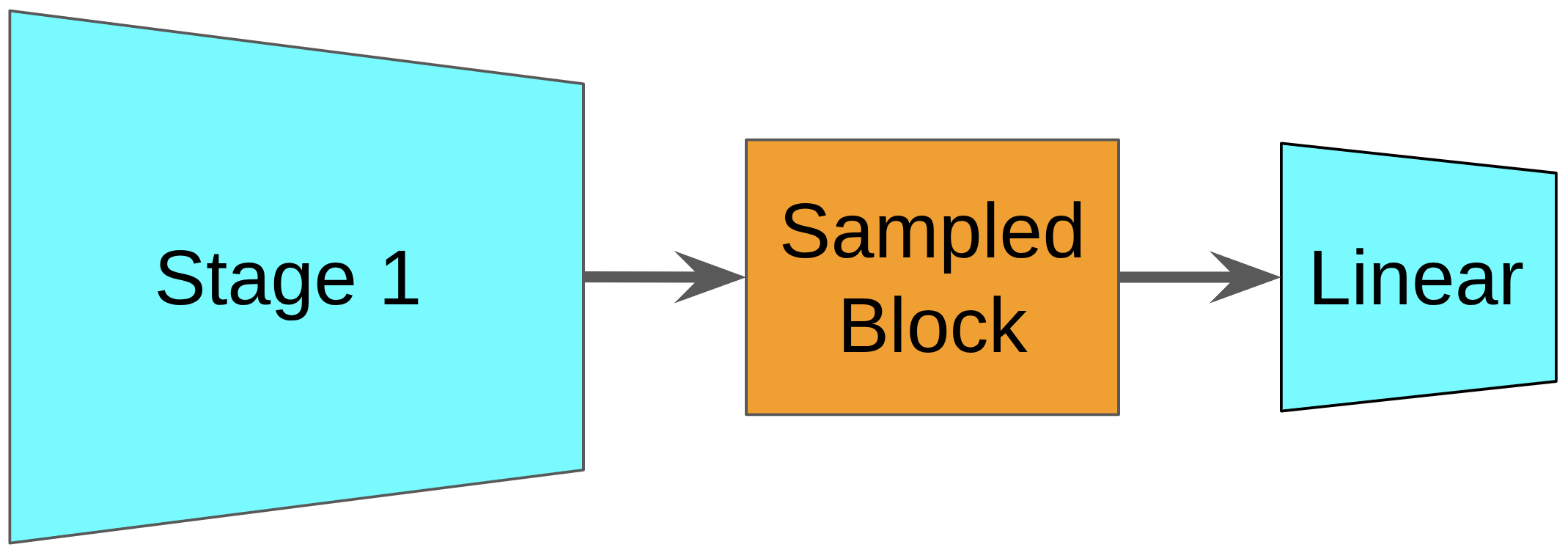}}
	\caption{Architecture search space.}
	\label{fig:NNSearch}
\end{wrapfigure}
We collect a dataset of $\NumSamples$ architectures sampled from spaces of the type illustrated in \cref{fig:NNSearch}. We include {\color{cyan}stage 1} (i.e, all initial layers up to and including the first pooling layer) to provide uniform baselines for the sampled blocks. We use the first stage of ResNet34 \citenump{ResNet} and Swin\_T \citenump{Swin} to collect samples in the context of the fundamental convolution and self-attention operations.

\paragraph{Baselines.}
We evaluate each stage 1 alone and report the parameter-accuracy curves (computed by scaling the channel/embedding dimensions) for the complete ResNet/Swin architectures, controlling for the effects of depth and width in the baseline architectures.
\paragraph{Novel Architecture Sampling Strategy.}
To ensure dataset consistency, the {\color{orange}sampled blocks} are constrained to have approximately the same `depth' as a single `layer'. Given the lack of a consistent definition of `depth', we set complexity constraints for the sampled blocks as follows:
\begin{equation*}
	\begin{aligned}
		2 \leq \cC_{op} \leq 5, \qquad \qquad 5 \leq \cC_T \leq 16, \qquad \qquad 2 \leq \cC_{\alpha}, \cC_A \leq 4, \qquad \qquad 2 \leq \cC_O \leq 11
	\end{aligned}
\end{equation*}
These ranges ensure that each sampled block adds about as much architectural complexity as a small to average sized core block. As a result, each architecture in the dataset consists of around $5-7$ total `layers', depending on how one chooses to define `layer'.

\begin{wrapfigure}[26]{r}{0.425\textwidth}
	\centering
	\raisebox{0pt}[\dimexpr\height-2.25\baselineskip\relax]{\includegraphics[width=0.425\textwidth,]{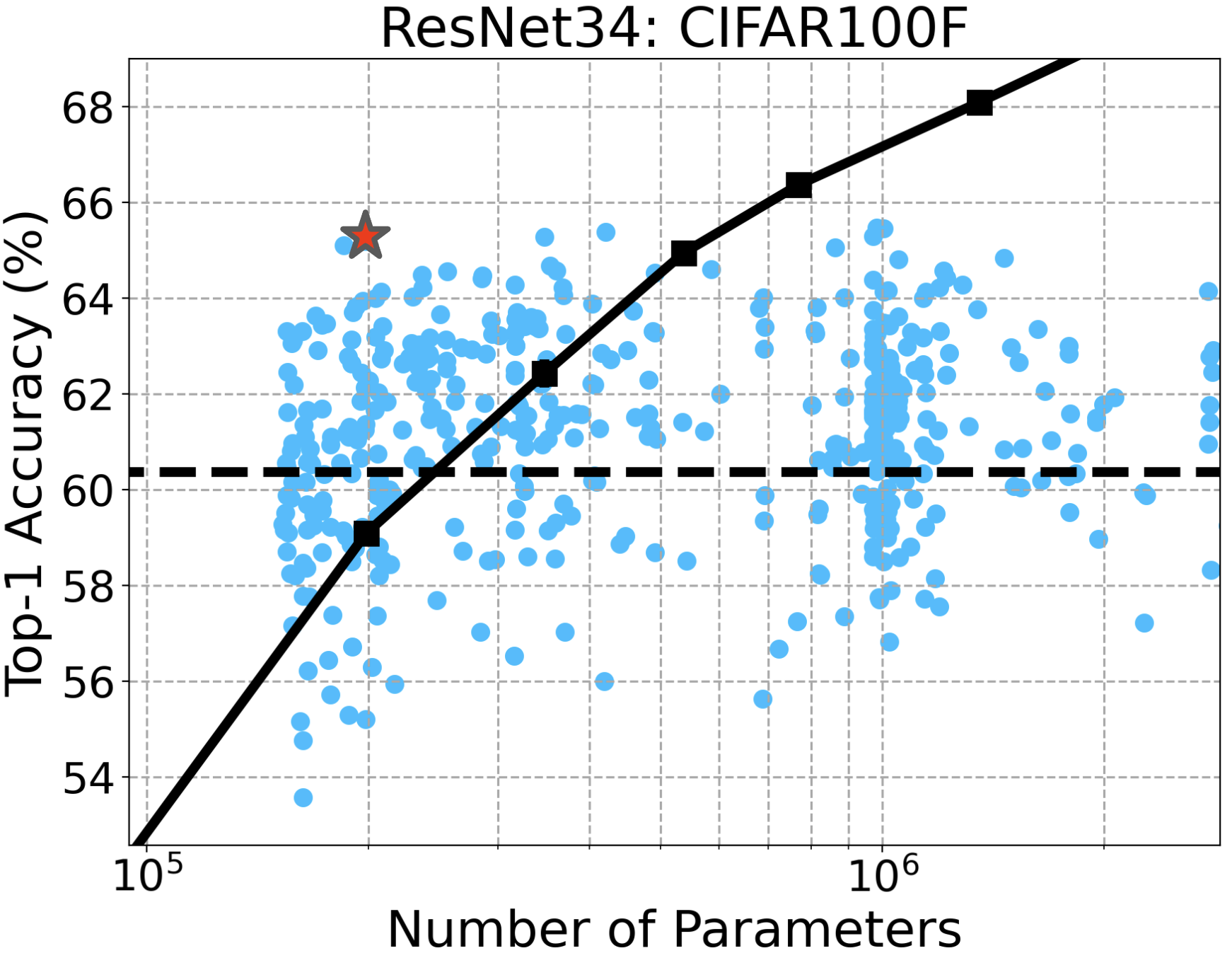}}\\\vspace{0.75em}
	\includegraphics[width=0.425\textwidth,]{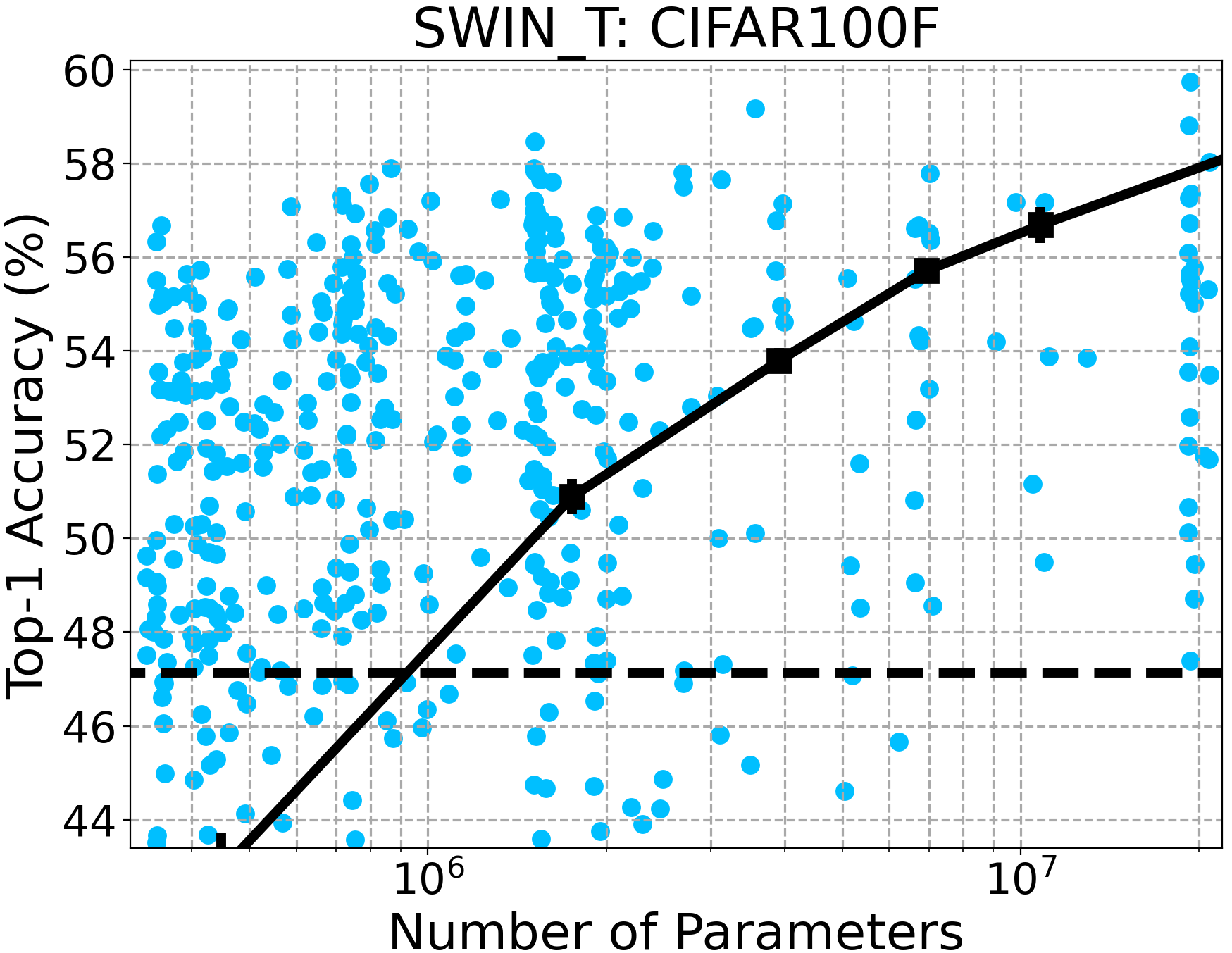}
	\caption{({\color{cyan} $\bullet$},{\color{red}$\bm{\star}$}) sampled models. ($\bm{- -}$) stage 1 baselines. ($\blacksquare$) full ResNet/Swin models at different width-scales.}
	\label{fig:C100Results}
\end{wrapfigure}
After each tensor operation, some randomly selected non-linear activations (from a pool of: Leaky ReLU, ReLU6, Layer Norm, and Softmax) are added with $50\%$ probability. Shapes for all intermediate tensors are randomly sampled.
\paragraph{Datasets.}
Following standard NAS practice \citenump{NASBench101,NASBench201}, we sample architectures for CIFAR-10, CIFAR-100 \citenump{CIFAR}, and Tiny Imagenet \citenump{TIM}. These datasets span a range of task difficulties with 10, 100, and 200 classes, respectively. Resolutions are: $32\times32$, $32\times32$, and $64\times64$, respectively.

\paragraph{Optimization.} To maximize the size of our contributed dataset of architectures, we train all models \textit{efficiently} with AdamW \citenump{Adam} and the OneCycle learning rate scheduler \citenump{OneCycle} to enable convergence in $50$ epochs. All sampled models are trained end-to-end from random initialization on the \textit{original resolution images} which we augment with random crops and flips.
\paragraph{Results.} 
The sampled CIFAR-100 architectures are shown in \cref{fig:C100Results}, with the remaining (similar) results shown in the appendix due to space constraints. 

The blue dots ({\color{cyan} $\bullet$}) show a wide spread in performance of the sampled architectures.  We suspect this spread stems from the immense number of possible higher-arity tensor operations; there are thousands of distinct operations at arity $3$ and tens of thousands at arity $4$ \citenump{Plexes}. The size of these operation spaces prevents exploring \textit{every} architecture, highlighting the value of our diverse collection of samples for understanding which tensor operation structures lead to strong performance.  We publicly release all 3,028 models alongside comprehensive diagnostic data to facilitate future explorations of higher complexity architecture space here: {\color{blue}\href{https://github.com/combinatoriallabs/ArchitecturalComplexity}{https://github.com/combinatoriallabs/ArchitecturalComplexity}}.

Upper left quadrants (dots ({\color{cyan} $\bullet$}) above both the dashed and solid lines in Fig. \ref{fig:C100Results}) contain many architectures that achieve remarkable parameter efficiency. As all sampled architectures use \textit{exactly the same first stage} as the original models, these improvements over the baselines come from single blocks, many of which have only around $10,000$ parameters. 
Some of the ResNet-based models even \textit{surpass} the performance of MobileNetV2 \citenump{MobileNet} ($64.29\%\pm 0.04$ on CIFAR-100).

Specifically, the red star sample ({\color{red}$\bm{\star}$}) achieves $65.52\%\pm0.22$ with less than 200,000 parameters (152,000 are from the ResNet stage 1), an order of magnitude fewer than MobileNetV2's $2.5$ million. This is notable because MobileNetV2 is still widely used today as a lightweight architecture \citenump{LogitStandardizationKD, DivDataDistill, MNV2Numbers}. 

\begin{wrapfigure}[10]{r}{0.375\textwidth}
	\centering
	\raisebox{0pt}[\dimexpr\height-1.0\baselineskip\relax]{\includegraphics[width=0.375\textwidth]{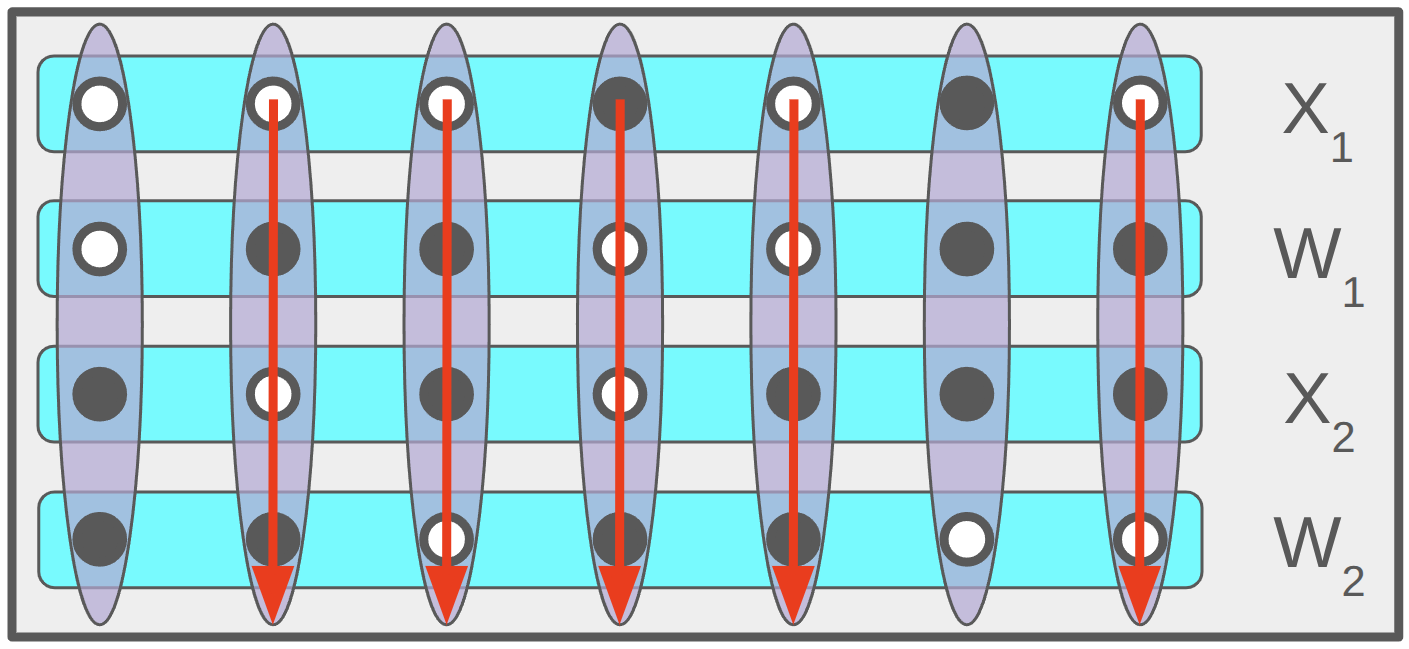}}
	\caption{A new tensor operation from the ({\color{red}$\bm{\star}$}) architecture. $X_1, X_2$ are input tensors; $W_1, W_2$ are learned weight tensors. $\cC_{\alpha} = 4$, $\cC_O = 7$, $\cC_A = 3$.}
	\label{fig:SpicyTOM}
\end{wrapfigure}
An example tensor operation matrix from the ({\color{red}$\bm{\star}$}) sample is shown in \cref{fig:SpicyTOM}. This high complexity operation is novel both in the context of deep learning and in the theory of tensor algebra.  Many other top performing samples utilize high complexity operations as well; more examples are provided in \cref{sec:TensorOps_Examples}.

\section{Discussion}

\paragraph{Limitations.} The hierarchical framework developed in this paper was purpose-built to study the \textit{architectural complexity} of neural networks. It does \textit{not} provide insight on their functional properties nor their training dynamics. Other frameworks, e.g., Categorical Deep Learning \citenump{CategoricalDL} and Neuro-algebraic Geometry \citenump{NeuroAlgebraicDL}, are better equipped to study these topics. Importantly, architectures are generated randomly; the red star is likely not optimal and can only be interpreted as an indication of the promise of higher complexity architectures. Moreover, the architecture dataset explores only a limited architectural complexity range.

\paragraph{Future Work.} We believe that our formal construction of complexity-based structural properties for deep neural networks will help outline both paths towards new architectures and improved understanding of existing architectures. We hope to develop methods for analyzing HCC representations to proactively estimate the potential of arbitrary architectures in the future. Establishing connections between our proposed framework and other theories of deep learning will likely assist in this effort. In particular, we imagine neuro-algebraic methods will be useful in studying the expressiveness of higher complexity tensor operations.

\section{Conclusion}
In this paper, we developed the first framework for neural networks that \textit{explicitly models the structure of tensor operations}.  We applied our framework to analyze architectural complexity, revealing new insights into the evolution of deep learning and highlighting the boundary of known architectures. We then used our framework to systematically construct thousands of novel architectures, demonstrating the potential of higher complexity tensor operations for constructing next-generation architectures.

\section*{Acknowledgments}
The authors extend their thanks to Keith Kearnes and Joshua Grochow for many insightful conversations and to Zhuoheng Li and Nolan Brady for their feedback on this manuscript.

\bibliography{main}

\begin{thebibliography}{51}
\providecommand{\natexlab}[1]{#1}
\providecommand{\url}[1]{\texttt{#1}}
\expandafter\ifx\csname urlstyle\endcsname\relax
  \providecommand{\doi}[1]{doi: #1}\else
  \providecommand{\doi}{doi: \begingroup \urlstyle{rm}\Url}\fi

\bibitem[Abadi et~al.(2015)Abadi, Agarwal, Barham, Brevdo, Chen, Citro,
  Corrado, Davis, Dean, Devin, Ghemawat, Goodfellow, Harp, Irving, Isard, Jia,
  Jozefowicz, Kaiser, Kudlur, Levenberg, Man\'{e}, Monga, Moore, Murray, Olah,
  Schuster, Shlens, Steiner, Sutskever, Talwar, Tucker, Vanhoucke, Vasudevan,
  Vi\'{e}gas, Vinyals, Warden, Wattenberg, Wicke, Yu, and Zheng]{TensorFlow}
Mart\'{i}n Abadi, Ashish Agarwal, Paul Barham, Eugene Brevdo, Zhifeng Chen,
  Craig Citro, Greg~S. Corrado, Andy Davis, Jeffrey Dean, Matthieu Devin,
  Sanjay Ghemawat, Ian Goodfellow, Andrew Harp, Geoffrey Irving, Michael Isard,
  Yangqing Jia, Rafal Jozefowicz, Lukasz Kaiser, Manjunath Kudlur, Josh
  Levenberg, Dandelion Man\'{e}, Rajat Monga, Sherry Moore, Derek Murray, Chris
  Olah, Mike Schuster, Jonathon Shlens, Benoit Steiner, Ilya Sutskever, Kunal
  Talwar, Paul Tucker, Vincent Vanhoucke, Vijay Vasudevan, Fernanda Vi\'{e}gas,
  Oriol Vinyals, Pete Warden, Martin Wattenberg, Martin Wicke, Yuan Yu, and
  Xiaoqiang Zheng.
\newblock {TensorFlow}: Large-scale machine learning on heterogeneous systems,
  2015.
\newblock URL \url{https://www.tensorflow.org/}.
\newblock Software available from tensorflow.org.

\bibitem[Avval et~al.(2025)Avval, Eskue, Groves, and Yaghoubi]{NASSurvey}
Sasan Salmani~Pour Avval, Nathan~D. Eskue, Roger~M. Groves, and Vahid Yaghoubi.
\newblock Systematic review on neural architecture search.
\newblock \emph{Artificial Intelligence Review}, 2025.

\bibitem[Bassam et~al.(2025)Bassam, Zhu, and Bian]{MNV2Numbers}
Ejafa Bassam, Dawei Zhu, and Kaigui Bian.
\newblock Pld: A choice-theoretic list-wise knowledge distillation.
\newblock In \emph{NeurIPS}, 2025.

\bibitem[Bretto(2013)]{HGT}
Alain Bretto.
\newblock \emph{Hypergraph Theory, an Introduction.}
\newblock Springer, 2013.

\bibitem[Bronstein et~al.(2021)Bronstein, Bruna, Cohen, and
  Veličković]{GeometricDL}
Michael~M. Bronstein, Joan Bruna, Taco Cohen, and Petar Veličković.
\newblock Geometric deep learning: Grids, groups, graphs, geodesics, and
  gauges.
\newblock 2021.

\bibitem[Cheng et~al.(2024)Cheng, Chrysos, Georgopoulos, and
  Cevher]{MultiLinOpNets}
Yixin Cheng, Grigorios~G. Chrysos, Markos Georgopoulos, and Volkan Cevher.
\newblock Multilinear operator networks.
\newblock In \emph{ICLR}, 2024.

\bibitem[Chrysos et~al.(2020)Chrysos, Moschoglou, Bouritsas, Panagakis, Deng,
  and Zafeiriou]{P-Nets}
Grigorios~G. Chrysos, Stylianos Moschoglou, Giorgos Bouritsas, Yannis
  Panagakis, Jiankang Deng, and Stefanos Zafeiriou.
\newblock P-nets: Deep polynomial neural networks.
\newblock In \emph{CVPR}, 2020.

\bibitem[Chrysos et~al.(2022)Chrysos, Georgopoulos, Deng, Kossaifi, Panagakis,
  and Anandkumar]{PolyFramework}
Grigorios~G. Chrysos, Markos Georgopoulos, Jiankang Deng, Jean Kossaifi, Yannis
  Panagakis, and Anima Anandkumar.
\newblock Augmenting deep classifiers with polynomial neural networks.
\newblock In \emph{ECCV}, 2022.

\bibitem[Chrysos et~al.(2025)Chrysos, Wu, Pascanu, Torr, and
  Cevher]{HadamardSurvey}
Grigorios~G. Chrysos, Yongtao Wu, Razvan Pascanu, Philip~H.S. Torr, and Volkan
  Cevher.
\newblock Hadamard product in deep learning: Introduction, advances and
  challenges.
\newblock In \emph{IEEE PAMI}, 2025.

\bibitem[Dai et~al.(2017)Dai, Qi, Xiong, Li, Zhang, Hu, and Wei]{DefCNNs}
Jifeng Dai, Haozhi Qi, Yuwen Xiong, Yi~Li, Guodong Zhang, Han Hu, and Yichen
  Wei.
\newblock Deformable convolutional networks.
\newblock In \emph{ICCV}, 2017.

\bibitem[Dong \& Yang(2020)Dong and Yang]{NASBench201}
Xuanyi Dong and Yi~Yang.
\newblock Nas-bench-201: Extending the scope of reproducible neural
  architecture search.
\newblock In \emph{ICLR}, 2020.

\bibitem[Dosovitskiy et~al.(2021)Dosovitskiy, Beyer, Kolesnikov, Weissenborn,
  Zhai, Unterthiner, Dehghani, Minderer, Heigold, Gelly, Uszkoreit, , and
  Houlsby]{ViT}
Alexey Dosovitskiy, Lucas Beyer, Alexander Kolesnikov, Dirk Weissenborn,
  Xiaohua Zhai, Thomas Unterthiner, Mostafa Dehghani, Matthias Minderer, Georg
  Heigold, Sylvain Gelly, Jakob Uszkoreit, , and Neil Houlsby.
\newblock An image is worth 16x16 words: Transformers for image recognition at
  scale.
\newblock In \emph{ICLR}, 2021.

\bibitem[Ericsson et~al.(2024)Ericsson, Espinosa1, Yang, Antoniou, Storkey,
  Cohen, McDonagh, and Crowley]{Einspace}
Linus Ericsson, Miguel Espinosa1, Chenhongyi Yang, Antreas Antoniou, Amos
  Storkey, Shay~B. Cohen, Steven McDonagh, and Elliot~J. Crowley.
\newblock einspace: Searching for neural architectures from fundamental
  operations.
\newblock In \emph{NeurIPS}, 2024.

\bibitem[Fan et~al.(2023)Fan, Li, Wang, Lai, and Wang]{NAGExpress2}
Feng-Lei Fan, Mengzhou Li, Fei Wang, Rongjie Lai, and Ge~Wang.
\newblock On expressivity and trainability of quadratic networks.
\newblock In \emph{IEEE NNLS}, 2023.

\bibitem[Gavranović et~al.(2024)Gavranović, Lessard, Dudzik, von Glehn,
  Araujo, and Velicković]{CategoricalDL}
Bruno Gavranović, Paul Lessard, Andrew Dudzik, Tamara von Glehn, Joao~G.M.
  Araujo, and Petar Velicković.
\newblock Position: Categorical deep learning is an algebraic theory of all
  architectures.
\newblock In \emph{ICML}, 2024.

\bibitem[Gu \& Dao(2024)Gu and Dao]{Mamba}
Albert Gu and Tri Dao.
\newblock Mamba: Linear-time sequence modeling with selective state spaces,
  2024.
\newblock URL \url{https://arxiv.org/abs/2312.00752}.

\bibitem[Hajij et~al.(2022)Hajij, Zamzmi, Papamarkou, Miolane, Guzmán-Sáenz,
  Ramamurthy, Birdal, Dey, Mukherjee, Samaga, Livesay, Walters, Rosen, and
  Schaub]{TDLBook}
Mustafa Hajij, Ghada Zamzmi, Theodore Papamarkou, Nina Miolane, Aldo
  Guzmán-Sáenz, Karthikeyan~Natesan Ramamurthy, Tolga Birdal, Tamal~K. Dey,
  Soham Mukherjee, Shreyas~N. Samaga, Neal Livesay, Robin Walters, Paul Rosen,
  and Michael~T. Schaub.
\newblock Topological deep learning: Going beyond graph data.
\newblock 2022.

\bibitem[Hajij et~al.(2025)Hajij, Bastian, Osentoski, Kabaria, Davenport,
  Dawood, Cherukuri, Kocheemoolayil, Shahmansouri, Lew, Papamarkou, and
  Birdal]{CopresheafNNs}
Mustafa Hajij, Lennart Bastian, Sarah Osentoski, Hardik Kabaria, John~L.
  Davenport, Sheik Dawood, Balaji Cherukuri, Joseph~G. Kocheemoolayil, Nastaran
  Shahmansouri, Adrian Lew, Theodore Papamarkou, and Tolga Birdal.
\newblock Copresheaf topological neural networks: A generalized deep learning
  framework.
\newblock In \emph{NeurIPS}, 2025.

\bibitem[He et~al.(2016)He, Zhang, Ren, , and Sun]{ResNet}
Kaiming He, Xiangyu Zhang, Shaoqing Ren, , and Jian Sun.
\newblock Deep residual learning for image recognition.
\newblock In \emph{CVPR}, 2016.

\bibitem[Jia et~al.(2025)Jia, Peng, Yang, and Chen]{CatToposSurvey}
Yiyang Jia, Guohong Peng, Zheng Yang, and Tianhao Chen.
\newblock Category-theoretical and topos-theoretical frameworks in machine
  learning: A survey.
\newblock \emph{Axioms}, 2025.

\bibitem[Kileel et~al.(2019)Kileel, Trager, and Bruna]{NAGExpress}
Joe Kileel, Matthew Trager, and Joan Bruna.
\newblock On the expressive power of deep polynomial neural networks.
\newblock In \emph{NeurIPS}, 2019.

\bibitem[Kingma \& Ba(2015)Kingma and Ba]{Adam}
D.~Kingma and J.~Ba.
\newblock Adam: A method for stochastic optimization.
\newblock In \emph{ICLR}, 2015.

\bibitem[Krizhevsky et~al.(2009)Krizhevsky, Nair, and Hinton]{CIFAR}
A.~Krizhevsky, V.~Nair, and G.~Hinton.
\newblock Cifar-10 and cifar100 datasets.
\newblock 2009.
\newblock URL \url{https://www.cs.toronto.edu/kriz/cifar.html}.

\bibitem[Krizhevsky et~al.(2012)Krizhevsky, Sutskever, and Hinton]{AlexNet}
Alex Krizhevsky, Ilya Sutskever, and Geoffrey~E. Hinton.
\newblock Imagenet classification with deep convolutional neural networks.
\newblock In \emph{NeurIPS}, 2012.

\bibitem[Le \& Yang(2015)Le and Yang]{TIM}
Ya~Le and Xuan Yang.
\newblock Tiny imagenet visual recognition challenge.
\newblock 2015.
\newblock URL \url{https://api.semanticscholar.org/CorpusID:16664790}.

\bibitem[Lecun et~al.(1998)Lecun, Bottou, Bengio, and Haffner]{CNNs}
Y.~Lecun, L.~Bottou, Y.~Bengio, and P.~Haffner.
\newblock Gradient-based learning applied to document recognition.
\newblock \emph{Proceedings of the IEEE}, 86\penalty0 (11):\penalty0
  2278--2324, 1998.
\newblock \doi{10.1109/5.726791}.

\bibitem[Li et~al.(2025)Li, Zhou, Gu, Li, and Wang]{DivDataDistill}
Hongcheng Li, Yucan Zhou, Xiaoyan Gu, Bo~Li, and Weiping Wang.
\newblock Diversity-enhanced distribution alignment for dataset distillation.
\newblock In \emph{ICCV}, 2025.

\bibitem[Liu et~al.(2019)Liu, Simonyan, and Yang]{DARTS}
Hanxiao Liu, Karen Simonyan, and Yiming Yang.
\newblock Darts: Differentiable architecture search.
\newblock In \emph{ICLR}, 2019.

\bibitem[Liu et~al.(2021)Liu, Lin, Cao, Hu, Wei, Zhang, Lin, and Guo]{Swin}
Ze~Liu, Yutong Lin, Yue Cao, Han Hu, Yixuan Wei, Zheng Zhang, Stephen Lin, and
  Baining Guo.
\newblock Swin transformer: Hierarchical vision transformer using shifted
  windows.
\newblock In \emph{ICCV}, 2021.

\bibitem[Marchetti et~al.(2025)Marchetti, Shahverdi, Mereta, Trager, and
  Kohn]{NeuroAlgebraicDL}
Giovanni~Luca Marchetti, Vahid Shahverdi, Stefano Mereta, Matthew Trager, and
  Kathlen Kohn.
\newblock Algebra unveils deep learning an invitation to neuroalgebraic
  geometry.
\newblock In \emph{ICML}, 2025.

\bibitem[Mesner \& Bhattacharya(1988)Mesner and Bhattacharya]{ConeProduct}
Dale~M. Mesner and Prabir Bhattacharya.
\newblock Association schemes on triples and a ternary algebra.
\newblock \emph{JOURNAL OF COMBINATORIAL THEORY}, 1988.

\bibitem[Mienye \& Swart(2024)Mienye and Swart]{DLGood2}
Ibomoiye~Domor Mienye and Theo~G. Swart.
\newblock A comprehensive review of deep learning: Architectures, recent
  advances, and applications.
\newblock In \emph{Information}, 2024.

\bibitem[Mu et~al.(2024)Mu, Tayyab, and Chua]{SpiralMLP}
Haojie Mu, Burhan~Ul Tayyab, and Nicholas Chua.
\newblock Spiralmlp: A lightweight vision mlp architecture.
\newblock In \emph{WACV}, 2024.

\bibitem[Nguyen \& Wu(2022)Nguyen and Wu]{CatTypesOfLayers}
Minh Nguyen and Nicholas Wu.
\newblock Folding over neural networks, 2022.

\bibitem[Nie et~al.(2025)Nie, Chen, and Chen]{DTTenNets}
Chang Nie, Junfang Chen, and Yajie Chen.
\newblock Deep tree tensor networks for image recognition.
\newblock In \emph{NeurIPS}, 2025.

\bibitem[Noor \& Ige(2025)Noor and Ige]{DLGood1}
Mohd Halim~Mohd Noor and Ayokunle~Olalekan Ige.
\newblock A survey on state-of-the-art deep learning applications and
  challenges.
\newblock In \emph{Engineering Applications of Artificial Intelligence}, 2025.

\bibitem[Papillon et~al.(2023)Papillon, Sanborn, Hajij, and
  Miolane]{TDLArchSurvey}
Mathilde Papillon, Sophia Sanborn, Mustafa Hajij, and Nina Miolane.
\newblock Architectures of topological deep learning: A survey of
  message-passing topological neural networks.
\newblock In \emph{arXiv}, 2023.

\bibitem[Paszke et~al.(2017)Paszke, Gross, Chintala, Chanan, Yang, DeVito, Lin,
  Desmaison, Antiga, and Lerer]{Pytorch}
Adam Paszke, Sam Gross, Soumith Chintala, Gregory Chanan, Edward Yang, Zachary
  DeVito, Zeming Lin, Alban Desmaison, Luca Antiga, and Adam Lerer.
\newblock Automatic differentiation in pytorch.
\newblock In \emph{NIPS-W}, 2017.

\bibitem[Potapczynski et~al.(2024)Potapczynski, Qiu, Finzi, Ferri, Chen,
  Goldblum, Bruss, Sa, and Wilson]{LinearLayerSearch}
Andres Potapczynski, Shikai Qiu, Marc Finzi, Christopher Ferri, Zixi Chen,
  Micah Goldblum, Bayan Bruss, Christopher~De Sa, and Andrew~Gordon Wilson.
\newblock Searching for efficient linear layers over a continuous space of
  structured matrices.
\newblock In \emph{NeurIPS}, 2024.

\bibitem[Rumelhart et~al.(1986)Rumelhart, Hinton, and Williams]{BackpropMLPs}
David Rumelhart, Geoffrey Hinton, and Ronald Williams.
\newblock Learning representations by backpropagating errors., 1986.

\bibitem[Sandler et~al.(2018)Sandler, Howard, Zhu, Zhmoginov, and
  Chen]{MobileNet}
M.~Sandler, A.~G. Howard, M.~Zhu, A.~Zhmoginov, and L.~Chen.
\newblock Mobilenetv2: Inverted residuals and linear bottlenecks.
\newblock In \emph{CVPR}, 2018.

\bibitem[Shahverdi et~al.(2026)Shahverdi, Marchetti, and Kohn]{SingularityBias}
Vahid Shahverdi, Giovanni~Luca Marchetti, and Kathlen Kohn.
\newblock Learning on a razor’s edge: the singularity bias of polynomial
  neural networks.
\newblock In \emph{ICLR}, 2026.

\bibitem[Simonyan \& Zisserman(2015)Simonyan and Zisserman]{VGG}
K.~Simonyan and A.~Zisserman.
\newblock Very deep convolutional networks for large-scale image recognition.
\newblock In \emph{ICLR}, 2015.

\bibitem[Smith \& Topin(2017)Smith and Topin]{OneCycle}
Leslie~N. Smith and Nicholay Topin.
\newblock Super-convergence: Very fast training of neural networks using large
  learning rates, 2017.

\bibitem[Sun et~al.(2024)Sun, Ren, Li, Wang, and Cao]{LogitStandardizationKD}
Shangquan Sun, Wenqi Ren, Jingzhi Li, Rui Wang, and Xiaochun Cao.
\newblock Logit standardization in knowledge distillation.
\newblock In \emph{CVPR}, 2024.

\bibitem[Trager et~al.(2020)Trager, Kohn, and Bruna]{CriticalPointsinMLPs}
Matthew Trager, Kathlen Kohn, and Joan Bruna.
\newblock Pure and spurious critical points: A geometric study of linear
  networks.
\newblock In \emph{ICLR}, 2020.

\bibitem[Vaswani et~al.(2017)Vaswani, Shazeer, Parmar, Uszkoreit, Jones, Gomez,
  Kaiser, and Polosukhin]{Transformers}
Ashish Vaswani, Noam Shazeer, Niki Parmar, Jakob Uszkoreit, Llion Jones,
  Aidan~N. Gomez, Lukasz Kaiser, and Illia Polosukhin.
\newblock Attention is all you need.
\newblock In \emph{NeurIPS}, 2017.

\bibitem[Weiler et~al.(2023)Weiler, Forré, Verlinde, and
  Welling]{EquivariantCNNs}
Maurice Weiler, Patrick Forré, Erik Verlinde, and Max Welling.
\newblock Equivariant and coordinate independent convolutional networks, 2023.

\bibitem[Ying et~al.(2019)Ying, Klein, Real, Christiansen, Murphy, and
  Hutter]{NASBench101}
Chris Ying, Aaron Klein, Esteban Real, Eric Christiansen, Kevin Murphy, and
  Frank Hutter.
\newblock Nas-bench-101: Towards reproducible neural architecture search.
\newblock In \emph{ICML}, 2019.

\bibitem[Zapata-Carratalá et~al.(2024)Zapata-Carratalá, Arsiwalla, and
  Beynon]{Plexes}
Carlos Zapata-Carratalá, Xerxes~D. Arsiwalla, and Taliesin Beynon.
\newblock Diagrammatic calculus and generalized associativity for higher-arity
  tensor operations.
\newblock \emph{Theoretical Computer Science}, 2024.

\bibitem[Zhu et~al.(2024)Zhu, Liao, Zhang, Wang, Liu, and Wang]{VisionMamba}
Lianghui Zhu, Bencheng Liao, Qian Zhang, Xinlong Wang, Wenyu Liu, and Xinggang
  Wang.
\newblock Vision mamba: Efficient visual representation learning with
  bidirectional state space model.
\newblock In \emph{ICML}, 2024.

\end{thebibliography}
\bibliographystyle{tmlr}

\clearpage
\appendix
\section*{Table of Contents}
\begin{enumerate}
	\item [A] \textbf{Theoretical Appendix.}
	\begin{enumerate}
		\item[\textbf{\ref{sec:GenTens_Apdx}}] Generalized Tensors.
		\begin{enumerate}
			\item[\ref{sec:Machinery_Apdx}] Technical Background.
			\item[\ref{sec:GenTenSOCC_Apdx}] Complete definition of the slice ordering compatibility conditions.
			\item[\ref{sec:GenTenIsos_Apdx}] Isomorphisms and canonical representatives of generalized tensors.
			\item[\ref{sec:GenTenHCCs_Apdx}] HCC Encodings of generalized tensors.
			\item[\ref{sec:GenTenDisc_Apdx}] Discussion of the HCC Construction of generalized tensors.
			\item[\ref{sec:GenTenProofs_Apdx}] Proofs of \cref{thrm:GenTensors<=>MDAs} and \cref{lem:MMs=>NonInjMDAs}.
		\end{enumerate}
		\item[\textbf{\ref{sec:TensorOps}}] Tensor Operations.
		\begin{enumerate}
			\item[\ref{sec:TensorOps_Defs}] Complete definition of tensor couplings and contractions.
			\item[\ref{sec:TensorOps_Eval}] Evaluating tensor operations with \cref{thrm:1-SMs=>Tensors}.
			\item[\ref{sec:TensorOps_Decomp}] Arity decomposition of tensor operations.
			\item[\ref{sec:TensorOps_Examples}] Additional examples of tensor operations.
			\item[\ref{sec:TensorOps_Connections}] Connections to other frameworks.
		\end{enumerate}
		\item[\textbf{\ref{sec:ArchDerivs}}] Complete derivations of the architectural complexity numbers in \cref{tab:ComplexityHistory}.
	\end{enumerate}
	
	\item [B] \textbf{Empirical Appendix.}
	\begin{enumerate}
		\item [\textbf{\ref{sec:DataCollection}}] Information on the dataset collection.
		\begin{enumerate}
			\item [\ref{sec:DataCollection_Sampling}] Details of sampling strategy.
			\item [\ref{sec:DataCollection_DiagData}] Diagnostic data provided.
			\item [\ref{sec:DataCollection_Optim}] Details of optimization recipe.
		\end{enumerate}
		\item [\textbf{\ref{sec:DataResults}}] Results and statistics of the complete architecture dataset.
		\begin{enumerate}
			\item [\ref{sec:DataResults_PA}] Parameter efficiency plots for all image classification datasets.
			\item [\ref{sec:DataResults_Stats}] Statistics of the sampled architectures.
		\end{enumerate}
		\item [\textbf{\ref{sec:RedStarArch}}] Complete descriptions of the red star architecture ({\color{red}$\bm{\star}$}).
	\end{enumerate}
\end{enumerate}

\clearpage
\section{\label{sec:TheoryAppendix}Theoretical Appendix}
In this appendix, we discuss in full theoretical detail the hierarchical framework for deep neural networks provided in the main paper. 

In part 1, we expand on the novel construction of tensors introduced in \cref{def:GenTensor}. This requires some graph-theoretic and order-theoretic machinery. We properly define the slice ordering compatibility conditions and prove \cref{thrm:GenTensors<=>MDAs,lem:MMs=>NonInjMDAs}. A discussion of how this novel construction relates to existing constructions of multidimensional arrays is also provided.

In part 2, we provide the full definition of tensor length and expand on the definition of tensor operation given in \cref{def:1SM}. We then prove \cref{thrm:1-SMs=>Tensors}, and provide additional examples of tensor operations including those involving jagged tensors. We also develop some machinery for decomposing and merging tensor operations. This is necessary for the complexity analysis of architectures. A discussion of connections to other frameworks is also provided.

In part 3, we provide complete derivations of all the architectural complexity numbers summarized in \cref{tab:ComplexityHistory}.

\clearpage

\subsection{\label{sec:GenTens_Apdx}Generalized Tensors}
\paragraph{Overview.} As articulated in the main paper, there are some necessary conditions --- the slice ordering compatibility conditions (SOCCs) --- that must be enforced of the $1$-cells of generalized tensors. The slice ordering compatibility conditions are necessary because they ensure that the definition of generalized tensor exhibits the correct theoretical properties. In particular, without the SOCCs, it is possible to assign orders to the $1$-cells in a manner which prevents well-defined multi-indices from being associated to the elements. Additionally, it is not obvious that the orders for the $1$-cells can be constructed within the set hierarchy of \cref{fig:RankDiagram}. In this part, we address these two topics in full. 

To demonstrate the necessity of the slice ordering compatibility conditions, consider the following example of what can go wrong:
\begin{equation}
	\begin{aligned}
		\cX^1 &= \big\{ [A, B], [C, D], [A, C], [D, B] \big\}\\
		\cX^2 &= \bigg\{ \big\{ [A, B], [C, D] \big\}, \big\{ [A, C], [D, B] \big\} \bigg\}
	\end{aligned}
	\label{eq:DegenerateOrdersEx}
\end{equation}
Any attempt to decode a multidimensional array from this HCC will result in a nonsensical object. Indeed, extracting multi-indices by starting at $A$ and counting the ``list distances" (we will formally define this notion shortly) of paths to the other elements results in the following:
\begin{equation*}
	\begin{bmatrix}
		A & D\\
		C & B
	\end{bmatrix}
\end{equation*}
A contradiction arises from the fact that $A$ and $B$ are supposed to belong to the same $1$-slice!

As mentioned in the main paper, the \textit{slice ordering compatibility conditions} (SOCCs) will prevent such organizational contradictions. To formulate these conditions, we must reinterpret the set hierarchical definition of generalized tensors in the language of graph theory. Doing so will lead us to the notion of a \textit{partitioned and weakly ordered hypergraph}. These hypergraphs greatly simplify the process of formalizing the SOCCs. However, the increased mathematical power of this reformulation comes at a cost; such hypergraphs are rank $5$ objects instead of rank $3$. To address this issue, we will prove that the axioms of generalized tensors are strong enough to guarantee that a rank $3$ encoding exists for all generalized tensors except a specific sub-type of hyper-tensors. This encoding can be extended to cover all generalized tensors, and is therefore sufficient to justify \cref{def:GenTensor} from the main paper.

Throughout this part, it is important to keep in mind that the construction of generalized tensors assumes nothing about the nature of the base set of elements. It is not until \cref{sec:TensorOps} that the base set must be a set of variables; here the base set may be a collection of \textit{any} mathematical objects.

\subsubsection{\label{sec:Machinery_Apdx}Order and Graph Theory Machinery}

\paragraph{Order Theory Background.} An order on a set is, in general, a binary relation. That is, an order $R$ is a subset of the cartesian product of a set with itself; $R \subseteq S \times S$. We say that $s \in S$ is \textit{$R$-minimal} if there exists no $t \in S$ such that $(t, s) \in R$. Similarly, $s$ is \textit{$R$-maximal} if there exists no $t$ with $(s, t) \in R$. It is important to recall that minimal and maximal elements are \textit{not} the same as minimum and maximum elements, which are those $s \in S$ such that for every other $t \in S$, $(s, t) \in R$ or $(t, s) \in R$, respectively. In this part, we will be frequently interested in minimal and maximal elements of various orders.

For completeness, we recall the definition of a \textit{strict weak order}.
\begin{definition}
	A \textit{strict weak order}, $<$, on a set $S$ is a binary relation $< \ \subsetneq S \times S$ satisfying the following four conditions:
	\begin{equation*}
		\begin{aligned}
			\textit{Irreflexivity.} \qquad &\forall a \in S \ , \ a \nless a\\
			\textit{Asymmetry.} \qquad &\forall a, b \in S \ , \ a < b \implies b \nless a\\
			\textit{Transitivity.} \qquad &\forall a, b, c \in S \ , \ (a < b) \wedge (b < c) \implies a < c\\
			\textit{Transitivity of Incomparability.} \qquad &\forall a, b, c \in S \ , \ (a \sim b) \wedge (b \sim c) \implies a \sim c\\
		\end{aligned}
	\end{equation*}
	where $\sim$ is the incomparability relation defined by the rule: $a \sim b \iff (a \nless b) \wedge (b \nless a)$.
\end{definition}
The fourth condition is the defining property of strict weak orders. It implies that $\sim$ is an \textit{equivalence relation}, i.e., $\sim$ is reflexive, symmetric, and transitive. We denote by $[a]$ the equivalence class of $a \in S$ with respect to $\sim$. Importantly, these equivalence classes are ``constant" with respect to the original order $<$, meaning that whenever $a \sim b$, we have that $x < a < y \iff x < b < y$. We therefore use $[a] < [b]$ to indicate that each $a \in [a]$ is less than each $b \in [b]$. Moreover, the order defined by $[a] < [b]$ is in fact a linear order on the equivalence classes of $\sim$.

For our purposes, all weak orders are finite, and therefore always admit (possibly non-injective) order preserving maps to finite subsets of $\mathbb{Z}$ with the standard $<$ relation. When necessary, we will refer to these maps as \textit{index} functions of the strict weak order in question. Moreover, for each finite weak order there is a special index function onto $\langle \{1, 2, ..., L\}, < \rangle$. We will refer to such index functions as \textit{the} index function of a strict weak order $\langle S, < \rangle$. $L$ is called the \textit{length} of the weak order, which is equivalent to the total number of equivalence classes with respect to $\sim$, i.e., $L$ is the cardinality of the quotient of $S$ by the incomparability relation:
\begin{equation*}
	L = |S / \sim| = |Im(index)|
\end{equation*}
where $Im(index)$ denotes the image of the index function. This definition of length coincides with the idea of the "length of a list" discussed in the main paper.

Weak orders will continue to be written as lists with (possibly) elements which are sets. For example:
\begin{equation*}
	\begin{aligned}
		&[A, B, C]\\
		&[A, \{B, D\}, C]\\
		&[\{A, B\}, \{C, D\}, \{E, F\}]
	\end{aligned}
\end{equation*}
Each of the above represents a strict weak order with index functions of codomain $\langle \{1, 2, 3\}, <\rangle$ given by the list positions. The elements of each ``set" have the same index in the list, e.g., $A \sim B$, $C \sim D$, and $E \sim F$ in the third example.

\paragraph{Graph Theory Background.}
We now introduce some necessary machinery from elementary hypergraph theory.
\begin{definition}
	A \textit{hypergraph} is a pair $\big\langle V, E \big\rangle$. $V$ is the vertex set of the hypergraph, and $E \subset \mathcal{P}(V)$ is the set of hyper-edges. Equivalently, a hypergraph is a rank $1$ HCC with $\cX^0 = V$.
	\vspace{-0.5em}
\end{definition}
Hypergraphs are generalizations of (undirected and unweighted) graphs, which require each $e \in E$ to contain exactly $2$ vertices from $V$.

As with traditional graphs, paths and cycles determine many important properties of hypergraphs.
\begin{definition}
	A \textit{path} on a hypergraph $H = \big\langle V, E \big\rangle$ is an ordered sequence of vertices and hyper-edges $[v_1, e_1, v_2, e_2, v_3, ...,  e_{n-1}, v_n]$ such that each consecutive pair of vertices are connected by the corresponding hyper-edge. That is, $v_1, v_2 \in e_1 \in E$, $v_2, v_3 \in e_2 \in E$, and so on.
\end{definition}
A hypergraph is called \textit{connected} if any two vertices can be joined by a path. We denote by $\gamma_{a,b}$ a path joining the vertices $a$ and $b$. A \textit{cycle} is a path for which the first and last vertices coincide, that is, $v_n = v_1$. This definition of cycle coincides with that of a cycle on normal graphs.

Paths can be reversed to obtain new paths. Given a path $\gamma_{a,b} = [a = v_1, e_1, v_2, e_2, v_3, ...,  e_{n-1}, v_n = b]$ denote by $\bar{\gamma}_{b,a}$ the reserve path given by: $\bar{\gamma}_{b,a} = [b = v_n, e_{n-1}, v_{n-2}, ...,  v_2, e_1, v_1 = a]$.

Paths can be composed if the endpoint of the first is the starting point of the second. Explicitly, given paths $\gamma_{a,b} = [a = v_1, e_1, v_2, e_2, v_3, ...,  e_{n-1}, v_n = b]$ and $\gamma_{b,c} = [b = w_1, f_1, w_2, f_2, w_3, ...,  f_{m-1}, w_m = c]$, define their composition $\gamma_{a,c} = \gamma_{b, c} \circ \gamma_{a,b}$ as follows:
\begin{equation*}
	\gamma_{a,c} = [a = v_1, e_1, v_2, e_2, v_3, ...,  e_{n-1}, v_n = b = w_1, f_1, w_2, f_2, w_3, ..., f_{m-1}, w_m = c].
\end{equation*}

We are now ready to introduce the key ingredient of the slice ordering compatibility conditions: partitioned and weakly ordered hypergraphs, or PWO-HGs for short.
\begin{definition}
	A \textit{partitioned and weakly ordered hypergraph} $H = \big\langle V, E, \{<_e\}_{e \in E}, P \big\rangle$ is a hypergraph equipped with two additional pieces of data:
	\begin{enumerate}
		\item A partition of $E$, $P \in \mathcal{P}(\mathcal{P}(E))$.
		\item An edge-indexed family of strict weak orders, $\big\{<_e \ \ \subsetneq e \times e\big\}_{e \in E}$.
	\end{enumerate}
	\label{def:PWOHGs}
\end{definition}
A PWO-HG without the partition is called a \textit{weakly ordered hypergraph}.

Intuitively, weakly ordered hypergraphs are hypergraphs which have lists for edges. As such, they are a strict generalization of the well-studied interval hypergraphs (see \citet{HGT}, section 4.1). We recall that an interval hypergraph comes equipped with a single linear order for $V$, whereas weakly ordered hypergraphs (as we have defined them) come with distinct orders for each hyper-edge. It is straightforward to see that any interval hypergraph is by definition also a particular type of weakly ordered hypergraph. The figure below is an example of a weakly ordered hypergraph:
\vspace{0.5em}
\begin{center}
	\begin{tikzpicture}
		\draw[rounded corners=6pt]
		(-2.5, -0.3)  -- 
		(-2.5, 0.3)  -- 
		(2.5, 0.3)  -- 
		(2.5, -0.3)  -- 
		node[midway, below, anchor=north] {} cycle;
		
		\draw[rounded corners=6pt]
		(2.3, -0.5)  -- 
		(2.3, 1.5)  -- 
		(-2.3, 1.5)  -- 
		(-2.3, -0.5)  -- 
		(-1.7, -0.5)  -- 
		(-1.7, 1.0)  -- 
		(1.7, 1.0)  -- 
		(1.7, -0.5)  -- 
		node[midway, below, anchor=north] {} cycle;
		
		\node at (-2,0) (A) {$A$};
		\node at (-1,0) (A) {$<$};
		\node at (0,0) (B) {$B$};
		\node at (1,0) (A) {$<$};
		\node at (2,0) (C) {$C$};
		\node at (0,1.25) (A) {$>$};
	\end{tikzpicture}
\end{center}
It is useful to observe that in the above example, there is a \textit{cycle} $(A, B, C, A)$ that only passes through the edge orders $<$ in the increasing direction. As such, it is unclear if $B$ is ``one index away" from $A$, or ``two indices away". This type of degenerate cycle is exactly the source of the problems that the SOCCs will address. We (sadly) do require a bit more machinery in order to formalize this idea.

The weak orders $<_e$ are useful for defining a \textit{signed path distance} on the underlying hypergraph.
\begin{definition}
	The \textit{signed path distance} $d$ of a path $\gamma = [v_1, e_1, v_2, e_2, v_3, ..., e_{n-1}, v_n]$ on a weakly ordered hypergraph $H = \langle V, E, \{<_e\}_{e \in E} \rangle$ is the signed sum of the number of elements ``in between" each pair $(v_i, v_{i+1})$ with respect to the relation $<_{e_i}$. That is:
	\begin{equation*}
		\begin{aligned}
			d(\gamma) &= \sum_{i = 1}^{n - 1} d_{e_i}(v_i, v_{i+1})\\
			d_{e_i}(v_i, v_{i+1}) &= \bigg| \big\{ [v] \in e_i / \sim_{e_i} \ : \ [v] <_{e_i} [v_{i+1}] \big\} \bigg| - \bigg| \big\{ [v] \in e_i / \sim_{e_i} \ : \ [v] <_{e_i} [v_{i}] \big\} \bigg|\\
			&= index_{e_i}(v_{i+1}) - index_{e_i}(v_i)
		\end{aligned}
	\end{equation*}
	where $index_{e_i}$ is the index function for the strict weak order $e_i$.
\end{definition}
Signed path distances can be thought of as the familiar notion of path distance on a weighted (hyper)graph where the weights vary depending on both $e$, and the vertices' positions in the ``list" $<_e$. It is important to note that the summands $d_{e_i}(v_i, v_{i+1})$ may be \textit{negative}. For brevity, we will refer to signed path distances simply as distances, when clear from context.

The \textit{unsigned path length} of a path is the summand-wise absolute value of its signed path distance.
\begin{definition}
	The \textit{unsigned path length} $l$ of a path $\gamma = [v_1, e_1, v_2, e_2, v_3, ..., e_{n-1}, v_n]$ on a weakly ordered hypergraph $H = \langle V, E, \{<_e\}_{e \in E} \rangle$ is the unsigned sum of the number of elements ``in between" each pair $(v_i, v_{i+1})$ with respect to the relation $<_{e_i}$. That is:
	\begin{equation*}
		\begin{aligned}
			l(\gamma) &= \sum_{i = 1}^{n - 1} |d_{e_i}(v_i, v_{i+1})|\\
		\end{aligned}
	\end{equation*}
\end{definition}

We will use $d(a,b)$ to denote the distance of a minimal length path joining $a$ and $b$.

Hyper-edge partitions induce partitions of path distance in the obvious way.
\begin{definition}
	Given a partitioned ordered hypergraph $H = \langle V, E, \{<_e\}_{e \in E}, P \rangle$ and $p \in P$, the $p$-\textit{distance} $d_p$ of a path $\gamma = [v_1, e_1, v_2, e_2, v_3, ..., e_{n-1}, v_n]$ is:
	\begin{equation*}
		d_p(\gamma) = \sum_{i = 1}^{n - 1} \mathbbm{1}_p(e_i) d_{e_i}(v_i, v_{i+1})
	\end{equation*}
	where $\mathbbm{1}_p: E \rightarrow \{0, 1\}$ is the indicator function of $p$, i.e., $\mathbbm{1}_p(e) = 1$ exactly when $e \in p$.
\end{definition}
Partitioned lengths $l_p$ are defined analogously.

Here are some basic, but important, properties of path distances.
\begin{proposition}
	Let $H = \langle V, E, \{<_e\}_{e \in E}, P \rangle$ be a partitioned ordered hypergraph, and let $a, b, c \in V$. We have the following:
	\begin{equation*}
		\begin{aligned}
			&\textit{1. For any path } \gamma, \ \sum_{p \in P}d_p(\gamma) = d(\gamma)\\
			&\textit{2. For any path } \gamma, \ d_p(\gamma) = -d_p(\bar{\gamma})\\
			&\textit{3. For any pair of composable paths } \gamma_{a,b}, \gamma_{b,c}, \ d_p(\gamma_{b,c} \circ \gamma_{a,b}) = d_p(\gamma_{b,c}) + d_p(\gamma_{a,b})
		\end{aligned}
	\end{equation*}
	\label{prop:DistanceProperties}
\end{proposition}
\begin{proof}
	These identities follow immediately from the definition of $p$-distance and the fact that $P$ is a partition of $E$.
\end{proof}

As PWO-HGs will be used to study hyper-tensors, we now introduce the notion of a path between \textit{tuples}.
\begin{definition}
	A \textit{tuple} of a PWO-HG $H = \langle V, E, \{<_e\}_{e \in E}, P \rangle$ is a $\subsetneq$-maximal element of the set:
	\begin{equation*}
		\big\{ \{v_1, v_2, ..., v_n\} \in \cP(V) \ : \ d(v_i, v_j) = 0, \ \forall i,j \in [[n]] \big\}
	\end{equation*}
	We denote by $\tup(H)$ the set of all tuples of $H$.
\end{definition}
Tuples are effectively the ``elements" of hyper-tensors. For example, consider the irregular hyper-tensor:
\begin{equation*}
	\begin{bmatrix}
		\{A,B\} & C\\
		D & \{E,F\}
	\end{bmatrix}
\end{equation*}
Each ``entry" of the above matrix is a tuple of the PWO-HG representation of the hyper-tensor. It is apparent from the definition that the set $\tup(H)$ covers (but does \textit{not necessarily} partition) $V$.

Tuples are useful for refining the definition of a path.
\begin{definition}
	A \textit{tuple path} on a PWO-HG $H = \langle V, E, \{<_e\}_{e \in E}, P \rangle$ is an ordered sequence of tuples and hyper-edges $[t_1, e_1, ..., e_{n-1}, t_n]$ such that for each $v_i \in t_1$, $v_j \in t_n$, there is a path between $v_i$ and $v_j$ which consists of exactly the same hyper-edges. The distance of a tuple path is defined to be:
	\begin{equation*}
		d(v_i^*, v_j^*), \ \text{where } v_i^*, v_j^* = argmax_{v_i, v_j} |d(v_i, v_j)|
	\end{equation*}
\end{definition}
The length of a tuple path is defined analogously. Tuple cycles are closed tuple paths. 

\subsubsection{\label{sec:GenTenSOCC_Apdx}The Slice Ordering Compatibility Conditions}
We are now ready to state the SOCCs and give the fully complete definition of a generalized tensor:
\begin{definition}
	A \textit{generalized tensor} is a $3$-cell $s^3$ of an HCC $\cX$ for which the following hold:
	\begin{enumerate}
		\item The $2$-cells of $s^3$ form a partition of the $\subsetneq$-maximal elements of $\cX^1|_{s^3}$.
		\item Each $x \in \cX^0|_{s^3}$ is contained in the intersection of some transversal of this partition. 
		\begin{equation*}
			\forall x \in \mathcal{X}^0|_{s^3}, \text{ } \exists \{s^1_i\}_{i\in I} \sqsubset s^3 \text{ with } x \in \bigcap_{i=1}^{|s^3|} s^1_i
		\end{equation*}
		\item The structure $H = \big\langle \cX^0|_{s^3}, \cX^1|_{s^3}, \cX^2|_{s^3} \big\rangle$ forms a partitioned and weakly ordered hypergraph which satisfies the \textit{slice ordering compatibility conditions}:
		\begin{enumerate}
			\item $H$ is connected.
			\item All tuple cycles on $H$ have zero $p$-distance, for each $p \in s^3$.
		\end{enumerate}
	\end{enumerate}
	\label{def:GenTensorAppendix}
\end{definition}
The only difference from \cref{def:GenTensor} is the specification of condition 1 to the subset-maximal elements of $s^3$, which is necessary to accommodate the extra $1$-cells that will be used to encode the orders in the rank $3$ representation discussed in \cref{sec:EncodingGenTensors}. If all cycles (not just tuple cycles) on a PWO-HG $H$ have zero $p$-distances, we say that $H$ satisfies the \textit{strong slice ordering compatibility condition}.

Intuitively, the slice ordering compatibility conditions state that the path distances of the PWO-HG associated to $s^3$ must behave like a conservative vector field. This prevents degenerate structures such as that of \cref{eq:DegenerateOrdersEx} and guarantees that given any arbitrary generalized tensor, one may extract something that is ``essentially" a multidimensional array. Here ``essentially" means that the resulting MDA may be jagged and/or contain multiple elements with the same multi-index, but will be free of any organizational contradictions. We formalize this important idea in the following statement.
\begin{lemma}
	For any generalized tensor $s^3$, the associated PWO-HG $H = \big\langle V, E, \{<_e\}_{e \in E}, P \big\rangle$ satisfies:
	\begin{equation*}
		\begin{aligned}
			\text{Whenever } \gamma_{t_1,t_2}, \ \delta_{t_1,t_2} &\text{ are two different tuple paths between } t_1,t_2 \in Tup(H)\\
			d_p(\gamma_{t_1,t_2}) &= d_p(\delta_{t_1,t_2}), \ \forall p \in s^3
		\end{aligned}
	\end{equation*}
	\label{lem:TuplePathProp}
\end{lemma}
\begin{proof}
	Suppose $\gamma_{t_1,t_2}$ and $\delta_{t_1,t_2}$ are two different tuple paths between $a, b \in \cX^0|_{s^3}$ and let $p \in s^3$. Then $\gamma_{t_1,t_1} = \bar{\delta}_{t_1,t_2} \circ \gamma_{t_1,t_2}$ is a tuple cycle. By the SOCCs, $d_p(\gamma_{t_1,t_1}) = 0$. By \cref{prop:DistanceProperties}, we have that $d_p(\gamma_{t_1,t_2}) = -d_p(\bar{\delta}_{t_1,t_2}) = d_p(\delta_{t_1,t_2})$.
\end{proof}
This lemma describes how path distance is completely determined by the start and end \textit{vertex tuples} of the PWO-HG. Therefore, we will use $d(t_1, t_2)$ to refer to the distance of any tuple path between $t_1$ and $t_2$. To simplify notation going forward, we will use vertex variables such as $v, u$ to denote vertex tuples.

\Cref{lem:TuplePathProp} can be used to assign multi-indices to the tuples of a SOCC-satisfying PWO-HG in a consistent manner by first selecting any tuple to serve as the origin (i.e., the multi-index $(1, ..., 1)$) and then setting multi-indices for all other tuples to be the $p$-distances of any tuple path from the origin. The SOCCs ensure that the index systems obtained from each choice of origin are simply offsets of one-another. This is formalized in the following statement.
\begin{lemma}
	Let $s^3$ be a generalized tensor with associated PWO-HG $H = \big\langle V, E, \{<_e\}_{e \in E}, P \big\rangle$. Then for each $v \in \tup(H)$, there exists a well defined function $T_v$ given by: 
	\begin{equation*}
		\begin{aligned}
			T_v: \tup(H) &\rightarrow \mathbb{Z}^{|P|}\\
			u &\mapsto (d_{p_1}(v, u), ..., d_{p_{|P|}}(v, u))
		\end{aligned}
	\end{equation*}
	Furthermore, for each $v, w, u \in \tup(H)$, $T_v(u) - T_w(u)= T_v(w)$. That is, for any two origins, the resulting coordinate systems are translations of each other by the path distance between the origins.
	\label{lem:CoTensors}
\end{lemma}
\begin{proof}
	By \cref{lem:TuplePathProp}, each such $T_v$ is well-defined. By the connectivity of $H$, $T_v$ is defined for every $u \in \tup(H)$ regardless of $v$. Next, for $v, w, u \in \tup(H)$, observe:
	\begin{equation*}
		\begin{aligned}
			T_v(u) - T_w(u) &= (d_{p_1}(v, u), ..., d_{p_{|P|}}(v, u)) - (d_{p_1}(w, u), ..., d_{p_{|P|}}(w, u))\\
			&= (d_{p_1}(v, u), ..., d_{p_{|P|}}(v, u)) + (d_{p_1}(u, w), ..., d_{p_{|P|}}(u, w))\\
			&= (d_{p_1}(v, w), ..., d_{p_{|P|}}(v, w)) = T_v(w)
		\end{aligned}
	\end{equation*}
	These equalities hold by \cref{lem:TuplePathProp} and \cref{prop:DistanceProperties}. 
\end{proof}

\paragraph{On the Rank Complexity of Generalized Tensors.} The powerful tools of this hypergraphical formulation are not free. Indeed, partitioned and weakly ordered hypergraphs are, in general, rank $5$ objects (see \cref{fig:GenTenNaive}). This naturally begs the question: how can we tell if a rank $3$ cell defines a PWO-HG? The complete answer to this question is given in \cref{sec:EncodingGenTensors}, but, to summarize: PWO-HGs which satisfy the SOCCs already admit rank $4$ representations as opposed to their general PWO-HG counterparts. To address the remaining rank discrepancy, we will show that a natural notion of equivalence of generalized tensors leads to ``nicer" forms which admit rank $3$ representations.

\subsubsection{\label{sec:GenTenIsos_Apdx}Isomorphisms of Generalized Tensors}
To derive properties of generalized tensors, it is sometimes necessary to consider their ``canonical" representatives. This in turn necessitates the notion of isomorphisms between generalized tensors. We define this as follows:
\begin{definition}
	An \textit{isomorphism} of generalized tensors $s^3_1$ and $s^3_2$ is a bijective map $f: \cX^0|_{s^3_1} \leftrightarrow \cX^0|_{s^3_2}$, and an injective map $g: \cX^2|_{s^3_1} \rightarrow \cX^2|_{s^3_2}$ which preserve all path distances. That is, for each pair of elements $a, b \in \cX^0|_{s^3_1}$ and each path $\gamma_{a,b}$, the following holds:
	\begin{equation*}
		d_p(\gamma_{a,b}) = d_{g(p)}(\gamma_{f(a), f(b)}), \ \forall p \in s^3_1
	\end{equation*}
\end{definition}
It is important to notice that isomorphisms of generalized tensors assert no explicit requirements on $\cX^1|_{s^3_1}$ or $\cX^1_{s^3_2}$. Rather, the compatibility of the $1$-cells is determined by the preservation of path distances. An \textit{equivalence} of generalized tensors is an isomorphism such that $f(\cdot)$ is the identity map on a common base set. That is, $\cX^0|_{s^3_1} = \cX^0_{s^3_2}$, and $f: \cX^0|_{s^3_1} \rightarrow \cX^0|_{s^3_2}$ is given by $f(x) = x$. Two tensors are called isomorphic (equivalent) if there exists an isomorphism (equivalence) between them. A demonstrative example of equivalent generalized tensors is given below:
\begin{equation*}
	\begin{aligned}
		\cX^0_1 = \cX^0_2 &= \{A, B, C, D\}\\
		\cX^1_1 &= \big\{ [A, B], [C, D], [A, C], [B, D] \big\}\\
		\cX^1_2 &= \big\{ [A, B], [C, D], [B, D] \big\}\\
		\cX^2_1 &= \bigg\{ \big\{ [A, B], [C, D] \big\}, \big\{ [A, C], [B, D] \big\} \bigg\} = s^3_1\\
		\cX^2_2 &= \bigg\{ \big\{ [A, B], [C, D] \big\}, \big\{ [A], [C], [B, D] \big\} \bigg\} = s^3_2
	\end{aligned}
\end{equation*}
The equivalence is given by the identity map on $\cX^0$ and the map associating $2$-cells based on the order in which they are written above. It is evident that the omission of the $[A, C]$ $1$-cell from $s^3_2$ does not disconnect the underlying hypergraph. Similarly, all $p$-distances are identical across the two tensors. This example of a ``missing" $1$-cell leads naturally to the notion of a maximal generalized tensor.
\begin{definition}
	A \textit{maximal} generalized tensor $s^3$ is a generalized tensor which, for every $a, b \in s^3$, satisfies:
	\begin{equation*}
		\begin{aligned}
			\textit{1. } & d(\gamma_{a, b}) = 0 \implies a, b \in s^1 \in s^2_p \in s^3, \forall s^2_p\\
			\textit{2. } & |d(\gamma_{a,b})| = 1 \implies \exists s^2_p \text{ such that } a, b \in s^1 \in s^2_p \in s^3
		\end{aligned}
	\end{equation*}
\end{definition}
Given any generalized tensor, it is straightforward to produce a maximal tensor which is equivalent to it. Such maximal tensors are called \textit{maximal representatives} of the original tensor. The existence of maximal representatives is described in the following statement:
\begin{lemma}
	Let $s^3$ be a generalized tensor. There exists a maximal generalized tensor $s^3_m$ which is equivalent to $s^3$.
	\label{prop:MaximalGenTensors}
\end{lemma}
\begin{proof}
	Starting from a copy of $s^3_1$, we construct a maximal generalized tensor $s^3_2$ with the following rules:
	\begin{equation*}
		\begin{aligned}
			\text{1. Whenever } a,b \in \cX^0|_{s^3_1} \text{ satisfy: }& d(\gamma_{a, b}) = 0 \text{,}\\
			&\text{add } [\{a, b\}] \text{ to each } s^2 \in s^3\\
			\text{2. Whenever } a,b \in \cX^0|_{s^3_1} \text{ satisfy: }& |d(\gamma_{a, b})| = 1 \text{,}\\
			&\text{add } [a, b] \text{ to } s^2 \text{ if } d_{s^2}(\gamma_{a,b}) = 1 \text{, otherwise, add } [b, a] \text{ to } s^2,\\
			&\text{where $s^2$ is the $2$-cell with } |d_{s^2}(\gamma_{a,b})| = 1.
		\end{aligned}
	\end{equation*}
	The generalized tensor $s^3_2$ produced by this construction is maximal because the added $1$-cells ensure the required implications. Furthermore, as none of the added $1$-cells introduce cycles of non-zero distance, $s^3_2$ preserves the path distances of $s^3_1$, making the two tensors equivalent.
\end{proof}

There is a second way in which generalized tensors can be ``degenerate". Namely, $1$-cells may overlap, as demonstrated by the below example:
\begin{equation}
	\begin{aligned}
		\cX^0_1 = \cX^0_2 &= \{A, B, C, D, E, F\}\\
		\cX^1_1 &= \big\{ [A, B, C], [D, E, F], [A, D], [B, E], [C, F] \big\}\\
		\cX^1_2 &= \big\{ [A, B], [B, C], [D, E, F], [A, D], [B, E], [C, F] \big\}\\
		\cX^2_1 &= \bigg\{ \big\{ [A, B, C], [D, E, F] \big\}, \big\{ [A, D], [B, E], [C, F] \big\} \bigg\} = s^3_1\\
		\cX^2_2 &= \bigg\{ \big\{ [A, B], [B, C], [D, E, F] \big\}, \big\{ [A, D], [B, E], [C, F] \big\} \bigg\} = s^3_2
	\end{aligned}
	\label{eq:NonCanonicalRep}
\end{equation}
$s^3_1$ and $s^3_2$ are clearly equivalent. These extraneous $1$-cells allow us to make precise this notion of ``canonical" representatives.
\begin{definition}
	A \textit{canonical representative} of a generalized tensor $s^3$ is a maximal representative with the least number of weakly ordered hyper-edges.
	\label{def:CanonicalRep}
\end{definition}
It is straightforward to see that canonical representatives always exist, as by definition their hyper-edge sets are $\subseteq$-minimal elements of $\cP(\cP(V))$ --- a finite set --- of minimal cardinality. Moreover, if the strong slice ordering compatibility condition holds, they can be explicitly constructed by first building a maximal representative with \cref{prop:MaximalGenTensors} and then merging $1$-cells whenever applicable, as described in the following statement:
\begin{theorem}
	Let $s^3$ be a generalized tensor which satisfies the strong SOCCs. There exists a procedure to produce a canonical representative $s^3_c$ for $s^3$. Furthermore, $s^3_c$ is unique up to permutations of its $2$-cells.
	\label{thrm:CanonicalRepsExist}
\end{theorem}
\begin{proof}
	Let $s^3$ be a generalized tensor, and let $s^3_m$ be a maximal representative for $s^3$. We will construct a canonical representative of $s^3$ by combining overlapping $1$-cells from each $s^2 \in s^3_m$. That is, we are interested in cases when:
	\begin{equation*}
		\begin{aligned}
			s^1_i, s^1_j \in s^2 \in s^3_m \text{ satisfy: }& \big( \nexists s^1_k \text{ with } s^1_i \subsetneq s^1_k \text{ or } s^1_j \subsetneq s^1_k \big) \ \wedge \ \big( s^1_i \cap s^1_j \neq \emptyset \big)\\
		\end{aligned}
	\end{equation*}
	We recall that such $\subsetneq$-maximal $1$-cells correspond to overlapping hyper-edges of the associated PWO-HG.
	Claim: if there are no such $s^1_i, s^1_j \in s^2 \in s^3_m$, then $s^3_m$ is already a canonical representative for $s^3$.\\
	We will show that no $1$-cell can be removed from $s^3_m$. To see this, recall that each $0$-cell of a generalized tensor must be contained in the intersection of a transversal of the partition of $\subsetneq$-maximal $1$-cells. In particular, this means that the claim's hypothesis implies that each $2$-cell must itself be a partition of the $0$-cells. It is then clear that no $1$-cell can removed from $s^3_m$, as otherwise some $2$-cell would no longer cover the $0$-cells. This proves the claim.
	
	We now assume that there exist some overlapping $\subsetneq$-maximal $s^1_i, s^1_j \in s^2 \in s^3_m$ and consequently we must produce a procedure for merging the cells. So, fix $b \in s^1_i \cap s^1_j$ and let $index_i(\cdot): s^1_i \rightarrow \big \langle \{1, 2, ..., |s^1_i|\}, < \big \rangle$ be the index function for $s^1_i$. We extend $index_i(\cdot)$ to a new function $index()$ defined on $s^1_i \cup s^1_j$ by the following rule:
	\begin{equation*}
		index(a) = \begin{cases}
			index_i(a), \ a \in s^1_i\\
			index_i(b) + d_{s^1_j}(\gamma_{b, a}), \text{ otherwise}
		\end{cases}
	\end{equation*}
	Claim: $index(\cdot)$ is a well-defined function on $s^1$.\\
	Suppose not. The only way $index(\cdot)$ could assign two different values to some $a \in s^1$ is if $a \in s^1_i \cap s^1_j$ and the two cases of the definition disagree. This means that $index_i(a) \neq index_i(b) + d_{s^1_j}(\gamma_{b, a})$. So, $index_i(a) - index_i(b) \neq d_{s^1_j}(\gamma_{b, a})$. But, $index_i(a) - index_i(b) = d_{s^1_i}(\gamma_{b, a})$, meaning that $a$ and $b$ have paths of different distances within the same $2$-cell. This contradicts the strong SOCC, proving the claim.
	
	Next, define a relation on $s^1 := s^1_i \cup s^1_j$ based on $index(\cdot)$, that is:
	\begin{equation*}
		\forall a, c \in s^1, a <_{s^1} c \iff index(a) < index(c)
	\end{equation*}
	Claim: $<_{s^1}$ is a strict weak order.\\
	By virtue of being defined by a function to the integers with $<$, $<_{s^1}$ is automatically transitive and has a transitive incomparability relation. We must now verify that $<_{s^1}$ is irreflexive and asymmetric. So, assume for the sake of contradiction that $<_{s^1}$ is not asymmetric. Then, we can find some $a, b \in s^1$ with $a <_{s^1} b$ and $b <_{s^1} a$. By construction, this means that $index(a) < index(b)$ and $index(b) < index(a)$. This is impossible because $index(\cdot)$ is well-defined. An analogous argument shows that $<_{s^1}$ is irreflexive. This proves the claim.
	
	Claim: Replacing $s^1_i$ and $s^1_j$ with $s^1$ does not change any path distances on the PWO-HG associated to $s^3$.\\
	Let $\gamma_{a, c}$ be a path between $a, c \in s^1$. By construction, if $a, c \in s^1_i$, then $d_{s^1}(\gamma_{a, c}) = d_{s^1_i}(\gamma_{a, c})$. Similarly, if $a, c \in s^1_j$, $d_{s^1}(\gamma_{a, c}) = d_{s^1_j}(\gamma_{a, c})$. So, we now assume that $a \in s^1_i$ and $c \in s^1_j$. Because $b \in s^1_i \cap s^1_j$, we can decompose $\gamma_{a, c}$ into two paths $\gamma_{a, c} = \gamma_{b, c} \circ \gamma_{a, b}$. Because $a, b \in s^1_i$, $d_{s^1}(\gamma_{a, b}) = d_{s^1_i}(\gamma_{a, b})$. Because $b, c \in s^1_j$, $d_{s^1}(\gamma_{b, c}) = d_{s^1_j}(\gamma_{b, c})$. So, $d_{s^1}(\gamma_{a, c}) = d_{s^1_i}(\gamma_{a, b}) + d_{s^1_j}(\gamma_{b, c})$. We conclude that the generalized tensor formed from $s^3_m$ by deleting $s^1_i$ and $s^1_j$ and then adding $s^1$ is equivalent to $s^3_m$. This shows that this process produces a canonical representative $s^3_c$ for $s^3$. This proves the claim.
	
	We must now show that $s^3_c$ is unique up to mode permutations. So, let $s^3_{c'}$ be a distinct canonical representative for $s^3_m$. Because equivalence of generalized tensors is transitive, $s^3_c$ and $s^3_{c'}$ are equivalent so they cannot differ in their $0$-cells. Let $g$ be the map between the $2$-cells of $s^3_c$ and $s^3_{c'}$. By a previous argument, the $2$-cells of each representative must partition the $0$-cells. As we need only prove uniqueness up to permutations of $2$-cells, there must exist some $s^2_c \in s^3_c$ such that $s^2_{c'} = g(s^2_c) \neq s^2_c$. In order for $s^2_c$ and $s^2_{c'}$ to be distinct partitions of the same set, there must be $0$-cells $a, b$ which belongs to the same $s^1 \in s^2_c$ and such that $a \in s^1_1 \in s^2_{c'}$ and $b \in s^1_2 \in s^2_{c'}$. By connectivity, let $\gamma_{a, b}$ be a path between $a$ and $b$ in $s^3_{c'}$. There are two possibilities: either $d_{s^2}(\gamma_{a, b}) \neq 0$ for some $s^2 \neq s^2_{c'}$, or $d_{s^2}(\gamma_{a, b}) = 0$ $\forall s^2 \neq s^2_{c'}$. The first case is impossible because there is a direct path between $a$ and $b$ in $s^3_c$ which only passes through $s^2_c$, so this would violate the fact that $s^3_c$ and $s^3_{c'}$ are equivalent. The second case decomposes into two more cases. If $|d_{s^2_{c'}}(\gamma_{a, b})| \leq 1$, then by the maximiality of $s^3_{c'}$, $a, b$ would belong to the same $1$-cell in $s^2_{c'}$ which is a contradiction. So, we now assume that $|d_{s^2_{c'}}(\gamma_{a, b})| > 1$ and $d_{s^2}(\gamma_{a, b}) = 0$ for all $s^2 \neq s^2_{c'}$. Because $s^3_c$ and $s^3_{c'}$ are equivalent, this means that there must exist some $0$-cells $x_1, x_2, ..., x_n$ such that $[a]_{s^1} <_{s^1} [x_1]_{s^1} <_{s^1} [x_2]_{s^1} ... <_{s^1} [x_n]_{s^1} <_{s^1} [b]_{s^1}$. By the maximality of $s^3_{c'}$, this means that $x_1 \in s^1_1$ and $x_n \in s^1_2$. Analogously, $x_2 \in s^1_1$ and $x_{n-1} \in s^1_2$, and so on. We can continue this process until we reach some $i$ such that $d_{s^1}(\gamma_{x_i, x_{i+1}}) = 1$ and $x_i \in s^1_1$ and $x_{i+1} \in s^1_2$. By the previous argument, this is impossible.
	
	We conclude that $s^3_c$ is unique up to mode permutations.
\end{proof}

Canonical representatives are useful for reducing the rank complexity of the HCC encodings of generalized tensors.

\subsubsection{\label{sec:GenTenHCCs_Apdx}HCC Encodings of Generalized Tensors\label{sec:EncodingGenTensors}}
\paragraph{Overview.} Taken as face value, \cref{def:GenTensorAppendix} is incompatible with the structure diagram of \cref{fig:RankDiagram}. Indeed, hyper-edges are by definition $1$-cells, but weak orders are collections of $2$-cells, and partitions always require an additional 2 ranks of cells. Therefore, directly translating a PWO-HG into a set hierarchy results in a rank $5$ HCC, as expressed in the structure diagram of \cref{fig:GenTenNaive}. It is therefore not obvious when generalized tensors can be represented with rank $3$ HCCs. In this section, we explain how this is possible.

\begin{figure}[h!]
	\begin{tikzpicture}
		\node at (3,3) (V) {Vertices};  \node at (7,3) {Base set $\mathcal X^0$};
		\node at (3,2) (EL) {Edges};  \node at (7,2) {$\mathcal X^1 \subset \mathcal P(\mathcal X^0)$};
		\node at (3,1) (E) {Ordered Pairs};  \node at (7,1) {$\mathcal X^2 \subset \mathcal P(\mathcal X^1)$};
		\node at (3,0) (OE) {Ordered Edges};   \node at (7,0) {$\mathcal X^3 \subset \mathcal P(\mathcal X^2)$};
		\node at (3,-1) (P) {Parts};   \node at (7,-1) {$\mathcal X^4 \subset \mathcal P(\mathcal X^3)$};
		\node at (3,-2) (T) {PWO-HGs};   \node at (7,-2) {$\mathcal X^5 \subset \mathcal P(\mathcal X^4)$};
		
		\draw[blue, very thick, <->] (V.south) -- (EL.north);
		\draw[black, very thick, ->] (EL.south) -- (E.north);
		\draw[black, very thick, ->] (E.south) -- (OE.north);
		\draw[black, very thick, ->] (OE.south) -- (P.north);
		\draw[black, very thick, ->] (P.south) -- (T.north);
		
		\begin{pgfonlayer}{background}
			\node[fit=(V), vfit=cyan!60!white] {};
			\node[fit=(EL), vfit=cyan!60!white] {};
			\node[fit=(E), vfit=cyan!60!white] {};
			\node[fit=(OE), vfit=cyan!60!white] {};
			\node[fit=(P), vfit=cyan!60!white] {};
			\node[fit=(T), vfit=cyan!60!white] {};
		\end{pgfonlayer}

		\node at (11,3) (V2) {Elements};
		\node at (11,2) (E2) {Ordered Slices};
		\node at (11,1) (M2) {Modes};
		\node at (11,0) (T2) {Tensors};
		\node at (11,-1) (Z) {};
		\node at (11,-2) (Z2) {};
		
		\draw[blue, very thick, <->] (V2.south) -- (E2.north);
		\draw[black, very thick, ->] (E2.south) -- (M2.north);
		\draw[black, very thick, ->] (M2.south) -- (T2.north);
		
		\begin{pgfonlayer}{background}
			\node[fit=(V2), vfit=cyan!60!white] {};
			\node[fit=(E2), vfit=cyan!60!white] {};
			\node[fit=(M2), vfit=cyan!60!white] {};
			\node[fit=(T2), vfit=cyan!60!white] {};
		\end{pgfonlayer}
	\end{tikzpicture}
	\centering
	\caption{Structure diagrams for the default formulation of PWO-HGs from \cref{def:PWOHGs} (left) and the reduced rank representations of their canonical representatives from \cref{def:CanonicalRep} (right).}
	\label{fig:GenTenNaive}
\end{figure}
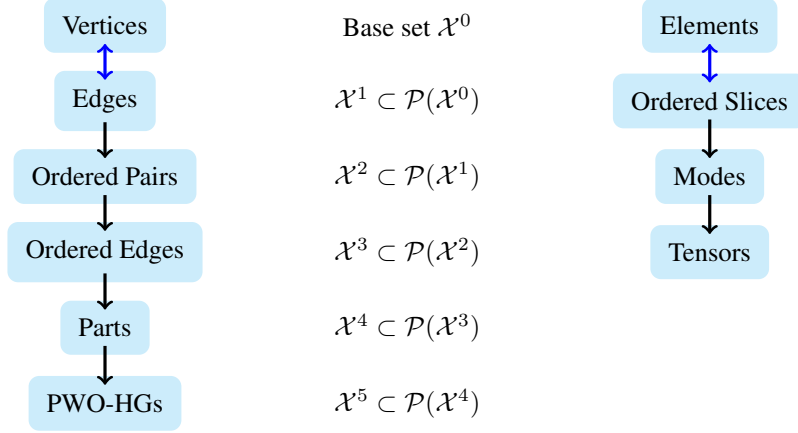

\paragraph{The Problem.} As an example of the straightforward but rank-inefficient HCC encoding of partitioned and weakly ordered hypergraphs, consider the below structure:
\begin{center}
	\begin{tikzpicture}
		\draw[rounded corners=6pt, fill = orange!70!white, opacity = 0.6]
		(-2.5, -0.25)  -- 
		(-2.5, 0.25)  -- 
		(0, 0.5) -- 
		(2.5, 0.25)  -- 
		(2.5, -0.25)  -- 
		(0, -0.5) -- 
		node[midway, below, anchor=north] {} cycle;
		
		\draw[rounded corners=6pt, fill = cyan!70!white, opacity = 0.6]
		(-2.3, 0.5)  -- 
		(-1.7, 0.5)  -- 
		(-1.7, -1.75)  -- 
		(-2.3, -1.75)  -- 
		node[midway, below, anchor=north] {} cycle;
		
		\draw[rounded corners=6pt, fill = cyan!70!white, opacity = 0.6]
		(2.3, -0.5)  -- 
		(1.7, -0.5)  -- 
		(1.7, 1.75)  -- 
		(2.3, 1.75)  -- 
		node[midway, below, anchor=north] {} cycle;
		
		\node at (-2,0) (A) {$A$};
		\node at (-1,0) (A) {$<$};
		\node at (0,0.25) (B) {$B$};
		\node at (0,-0.25) (C) {$C$};
		\node at (1,0) (A) {$<$};
		\node at (2,0) (D) {$D$};
		\node at (-2,-0.65) (A) {\rotatebox[origin=c]{90}{$<$}};
		\node at (-2,-1.25) (E) {$E$};
		\node at (2,0.65) (A) {\rotatebox[origin=c]{90}{$<$}};
		\node at (2,1.25) (F) {$F$};
	\end{tikzpicture}
\end{center}
Directly translating \cref{def:PWOHGs} to an HCC, the above partitioned and ordered hypergraph leads to the following HCC:
\begin{equation*}
	\begin{aligned}
		\cX^0 &= \{A, B, C, D, E, F\}\\
		\cX^1 &= \big\{ \{A, E\}, \{A, B, C, D\}, \{D, F\}\big\} \cup (A, E) \cup (A, B) \cup (A, C) \cup (B, D) \cup (C, D) \cup (D, F)\\
		\cX^2 &= \Big\{ e_1 = \big\{ \{A, E\}\big\}, e_2 = \big\{ \{A, B, C, D\}\big\},\\
		&\qquad \qquad \qquad e_3 = \big\{ \{D, F\}\big\}, (A, E), (A, B), (A, C), (B, D), (C, D), (D, F) \Big\}\\
		\cX^3 &= \Big\{ E_1 = \big\{e_1, (A, E)\big\}, E_2 = \big\{e_2, (A, B), (A, C), (B, D), (C, D)\big\}, E_3 = \big\{e_3, (D, F)\big\} \Big\}\\
		\cX^4 &= H = \big\{ \{E_1, E_3\}, \{E_2\} \big\}\\
	\end{aligned}
\end{equation*}
Here we have used the shorthand $(A, B)$ to refer to the ordered pair of $A$ and $B$ which, expressed with the standard Kuratowski encoding, is equivalent to: $(A, B) = \big\{ \{A\}, \{A, B\} \big\}$. In contrast to the reduced rank encoding will we construct shortly, the default HCC encoding of PWO-HGs requires that edges be $2$-cells. This is necessary in order to explicitly associate orders to their edges. The edges themselves are easily distinguished from their orders by the fact that they are the only singleton $2$-cells. As we will see, the SOCCs eliminate the need for this explicit association of orders to edges.

It is important to observe that the literal interpretation of \cref{def:PWOHGs} is indeed the minimal rank possible for arbitrary partitioned and ordered hypergraphs. This is demonstrated in the next example.
\begin{center}
	\begin{tikzpicture}
		\draw[rounded corners=6pt]
		(-2.5, -0.3)  -- 
		(-2.5, 0.3)  -- 
		(2.5, 0.3)  -- 
		(2.5, -0.3)  -- 
		node[midway, below, anchor=north] {} cycle;
		
		\draw[rounded corners=6pt]
		(2.3, -0.5)  -- 
		(2.3, 1.5)  -- 
		(-2.3, 1.5)  -- 
		(-2.3, -0.5)  -- 
		(-1.7, -0.5)  -- 
		(-1.7, 1.0)  -- 
		(1.7, 1.0)  -- 
		(1.7, -0.5)  -- 
		node[midway, below, anchor=north] {} cycle;
		
		\node at (-2,0) (A) {$A$};
		\node at (-1,0) (A) {$<$};
		\node at (0,0) (B) {$B$};
		\node at (1,0) (A) {$<$};
		\node at (2,0) (C) {$C$};
		\node at (0,1.25) (A) {$>$};
	\end{tikzpicture}
\end{center}
The above example is a perfectly valid (trivially) partitioned and weakly ordered hypergraph, despite the fact that it is \textit{not} a generalized tensor as it does not satisfy any form of the slice ordering compatibility conditions. The corresponding HCC representation is given by:
\begin{equation*}
	\begin{aligned}
		\cX^0 &= \{A, B, C\}\\
		\cX^1 &= \big\{ \{A, B, C\}, \{A, C\} \big\} \cup (A, B) \cup (B, C) \cup (C, A) \cup (C, A)\\
		\cX^2 &= \Big\{ e_1 = \big\{ \{A, B, C\} \big\}, e_2 = \big\{ \{A, C\} \big\}, (A, B), (B, C), (A, C), (C, A) \Big\}\\
		\cX^3 &= \Big\{ E_1 = \big\{e_1, (A, B), (B, C), (A, C)\big\}, E_2 = \big\{e_2, (C, A)\big\} \Big\}\\
		\cX^4 &= H = \big\{ \{E_1, E_2\}\big\}\\
	\end{aligned}
\end{equation*}
$\cX^1$ is necessary as the $1$-cells are by definition the edges of the hypergraph. $\cX^2$ is necessary to unambiguously represent the ordered pairs as they cannot be determined from only $\cX^1$. To see this, consider the expansion:
\begin{equation*}
	\begin{aligned}
		\cX^1 &= \big\{ \{A, B, C\}, \{C, A\},\{A\}, \{A, B\}, \{B\}, \{B, C\}, \{C\}, \{A, C\} \big\}\\
	\end{aligned}
\end{equation*}
The problem is clear, it is impossible to determine if $B < A$ or $B \nless A$. This question is disambiguated by $\cX^2$. $\cX^3$ is necessary to explicitly match edges with orders, otherwise, it would be impossible to determine if $A < C$ or $C < A$ in $e_2$. $\cX^4$ and $\cX^5$ are necessary to define a partition of the ordered edges.

\paragraph{Lower Rank Encodings.} The default HCC representation for partitioned and ordered hypergraphs uses the standard encoding of ordered pairs and is a direct translation of \cref{def:PWOHGs} into a set hierarchical form. However, if the strong slice ordering compatibility conditions hold, some possible sources of ambiguity are removed, which allows for a more ``efficient" rank $4$ representation. Moreover, the canonical representatives of such PWO-HGs allow for an even more efficient rank $3$ encoding, as strict weak orders can be represented succinctly with an extended version of the standard Kuratowski encoding of ordered pairs. This in turn further simplifies the task of representing generalized tensors with HCCs. This compressed encoding leads to the following theorem.
\begin{theorem}
	The canonical representative of any partitioned and weakly ordered hypergraph which satisfies the strong slice ordering compatibility condition can be represented with a rank $3$ HCC.
	\label{thrm:PWOHG_Encoding}
\end{theorem}
\begin{proof}
	Let $H = \big\langle V, E, \{<_e\}_{e \in E}, P \big\rangle$ be the canonical representative of a partitioned and weakly ordered hypergraph which has no cycle with non-zero $p$-distance, $\forall p \in P$. We construct a rank $3$ HCC from $H$ as follows:
	\begin{equation*}
		\begin{aligned}
			\cX^0 &:= V\\
			\cX^1 &:= \Bigg\{ \bigcup_{j = i}^{|e / \sim_e|} {}_j[v]_e \ : \ e \in E, \ i \in \big\{1, 2, ..., |e / \sim_e|\big\} \Bigg\}\\
			\cX^2 &:= \Bigg\{ \Big\{ \bigcup_{j = i}^{|e / \sim_e|} {}_j[v]_e \ : \ e \in p, \ i \in \big\{1, 2, ..., |e / \sim_e|\big\} \Big\} \ : \ p \in P \Bigg\}\\
		\end{aligned}
	\end{equation*}
	
	Claim 1: The original hypergraph $H$ can be determined from $\langle \cX^0, \cX^1, \cX^2 \rangle$.\\
	$\cX^1$ forms an expanded hypergraph $\mathcal{H} = \langle V, \cX^1 \rangle$ which, by construction, contains $H$ as a sub-hypergraph. Furthermore, each additional hyper-edge is a strict subset of an original hyper-edge, meaning that $E$ corresponds exactly to the $\subsetneq$-maximal $1$-cells of $\cX^1$. So the original partitioned hypergraph can be always be unambiguously reconstructed from $\langle \cX^0, \cX^1, \cX^2 \rangle$. This proves claim 1.
	
	Similarly, $P$ can be recovered by restricting $\cX^2$ to $\subsetneq$-maximal elements of $\cX^1$.
	
	To reconstruct the strict weak orders for each hyper-edge, we will construct a series of non-injective maps $index_e(\cdot): V \rightarrow \mathbb{N}$ based on each $\subsetneq$-maximal hyper-edge. Given some $e \in \cX^1$ for which $\nexists$ any $s^1$ with $e \subsetneq s^1$, define:
	\begin{equation*}
		\begin{aligned}
			index_e: \cX^0|_e &\rightarrow \mathbb{Z}\\
			x &\mapsto \bigg| \big\{ s^1 \in \cX^1 \ : \ x \in s^1 \subsetneq e \big\} \bigg|
		\end{aligned}
	\end{equation*}
	
	The weak orders are then given by these index functions.
	
	Claim 2: Each $index_e$ is well-defined and agrees with all index functions of the corresponding $<_e$.\\
	Because $H$ is a canonical representative, each $\subsetneq$-maximal $s^2 \in \cX^2$ is a partition of $\cX^0$. So, the $e$ can only overlap on $\subsetneq$-minimal sets. Therefore, each $\subsetneq$-chain induced by the unions used in the construction of $\cX^1$ is distinct apart from their bases. By construction, $\subsetneq$-minimal sets will only be counted once by $index_e(\cdot)$. This means that $index_e(\cdot)$ is completely determined by the strict subsets of the original hyper-edge of $H$. In turn, $index_e(\cdot)$ is completely determined by the equivalence classes of $\sim_e$, meaning that the image of $index_e(\cdot)$ is order isomorphic to any index function of $<_e$, which are well defined functions. This proves claim 2.
	
	Claims $1-2$ demonstrate that the original PWO-HG $H$ can be uniquely recovered from $\langle \cX^0, \cX^1, \cX^2 \rangle$.
\end{proof}

We can now be more precise about what it means for a rank $3$ HCC $s^3$ to form a partitioned and weakly ordered hypergraph. As in \cref{def:GenTensorAppendix}, $\cX^2|_{s^3}$ must partition the $\subsetneq$-maximal $1$-cells. The above construction shows that for $s^1_1, s^1_2 \in \cX^1|_{s^3}$, $s^1_1 \cap s^1_2$ must either be empty or $\subsetneq$-minimal in $\cX^1|_{s^2}$, for each $s^2 \in s^3$. If these conditions hold, a PWO-HG $H$ can be uniquely constructed from $s^3$, however, $H$ will in general not satisfy the SOCC. This is demonstrated by the following example:
\begin{equation*}
	\begin{aligned}
		\cX^0 &= \{A, B, C, D, E, F\}\\
		\cX^1 &= \big\{ \{A, B, C\}, \{B, C\}, \{C\}, \{D\}, \{F, E\}, \{E\}, \{A, F\}, \{F\}, \{B\}, \{C, D, E\}, \{D, E\} \big\}\\
		\cX^2 &= \Big\{ \big\{ \{A, B, C\}, \{B, C\}, \{C\}, \{D\}, \{F, E\}, \{E\} \big\}, \big\{ \{A, F\}, \{F\}, \{C, D, E\}, \{D, E\}, \{E\} \big\} \Big\}\\
	\end{aligned}
\end{equation*}
This HCC uniquely determines the below PWO-HG:
\begin{center}
	\begin{tikzpicture}
		\draw[rounded corners=6pt]
		(-0.5, -0.3)  -- 
		(-0.5, 0.3)  -- 
		(4.5, 0.3)  -- 
		(4.5, -0.3)  -- 
		node[midway, below, anchor=north] {} cycle;
		
		\draw[rounded corners=6pt]
		(-0.5, -2.3)  -- 
		(-0.5, -1.7)  -- 
		(4.5, -1.7)  -- 
		(4.5, -2.3)  -- 
		node[midway, below, anchor=north] {} cycle;
		
		\draw[rounded corners=6pt]
		(4.3, 0.5)  -- 
		(3.7, 0.5)  -- 
		(3.7, -2.5)  -- 
		(4.3, -2.5)  -- 
		node[midway, below, anchor=north] {} cycle;
		
		\draw[rounded corners=6pt]
		(0.3, 0.5)  -- 
		(-0.3, 0.5)  -- 
		(-0.3, -2.5)  -- 
		(0.3, -2.5)  -- 
		node[midway, below, anchor=north] {} cycle;
		
		\node at (0,0) (A) {$A$};
		\node at (1,0) (A) {$<$};
		\node at (2,0) (B) {$B$};
		\node at (3,0) (A) {$<$};
		\node at (4,0) (C) {$C$};
		\node at (4,-1) (D) {$D$};
		\node at (4,-2) (E) {$E$};
		\node at (0,-2) (F) {$F$};
		\node at (2,-2) (A) {$<$};
		\node at (0,-1) (A) {\rotatebox[]{90}{$<$}};
		\node at (4,-0.5) (A) {\rotatebox[]{90}{$<$}};
		\node at (4,-1.5) (A) {\rotatebox[]{90}{$<$}};
	\end{tikzpicture}
\end{center}
While this is indeed a valid PWO-HG, it fails to satisfy the SOCC. This demonstrates the necessity of each of the three requirements in \cref{def:GenTensorAppendix}.

\paragraph{The Strong SOCC and Hyper-Tensors.} The encoding given in \cref{thrm:PWOHG_Encoding} fails when the strong SOCC does not hold. This can occur for certain types of hyper-tensors, namely, those with overlapping tuples. In such situations, the same encoding can be applied to produce a rank-$4$ HCC by first explicitly collecting tuples together into $1$-cells. Alternatively, one could redefine the base set to be a set of tuples of elements instead, as this by definition causes both versions of the SOCCs to coincide. In any situation, the same $3$ ranks of cells used to construct slices, modes, and tensors are required in order to represent all generalized tensors. Interestingly, some hyper-tensors (in particular those who satisfy the strong SOCC) do not require the base set to be a collection of explicit tuples.

\subsubsection{\label{sec:GenTenDisc_Apdx}Benefits of the PWO-HG Construction.}
The foundation for generalized tensors built in the previous parts may seem like an unnecessarily complicated description of multidimensional arrays. However, in addition to naturally extending to jagged and hyper-tensors, the PWO-HG construction offers two key benefits. \textit{First}, it offers an important new way to conceptualize tensors, providing insight on how to extend the study of tensor operations into new directions. \textit{Second}, it is an \textit{intrinsic definition} of tensors which is also \textit{fixed rank}. Here, intrinsic means that the HCC encoding contains only the elements and no superfluous $0$-cells. Fixed rank means that the HCC encoding can express tensors of any order with a constant number of ranks. We now compare and contrast the PWO-HG construction of generalized tensors to traditional descriptions of multidimensional arrays. This will explain why a fixed rank and intrinsic definition of tensors is valuable.

\begin{wrapfigure}[13]{r}{0.475\textwidth}
	\centering
	\raisebox{0pt}[\dimexpr\height-0.25\baselineskip\relax]{\includegraphics[width=0.75\linewidth]{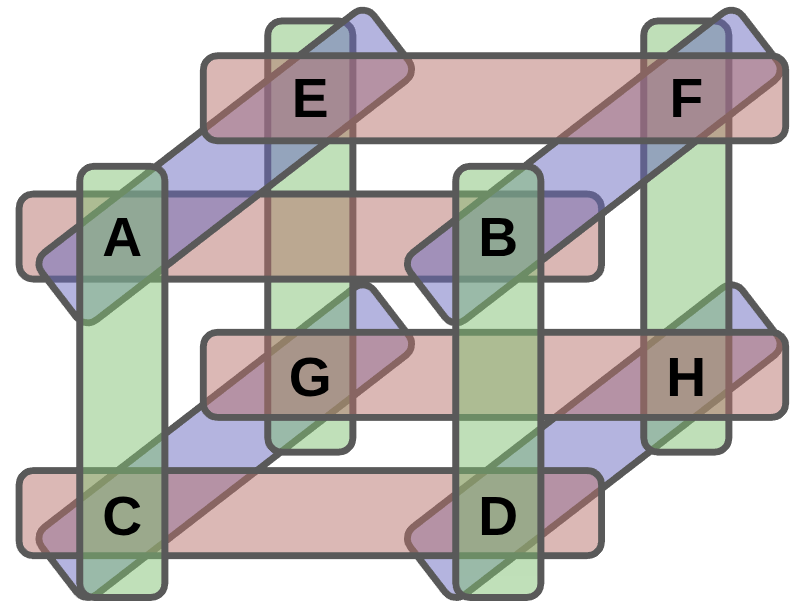}}
	\caption{Copy of \cref{fig:GenTensorExample}.}
	\label{fig:GenTenEx_Appendix}
\end{wrapfigure}
For the sake of uniform analysis, we focus our attention to generalized tensors which admit multidimensional array representations. As demonstrated in the previous sections, the PWO-HG definition requires a rank $3$ HCC to represent such objects. Furthermore, there are no restrictions on the order of the tensors represent-able in this way. So it is a fixed rank $3$ construction of multidimensional arrays. Throughout, let us consider the $2 \times 2 \times 2$ example from \cref{fig:GenTensorExample} (copied here for convenience). We recall the PWO-HG based encoding here:
\begin{equation*}
	\begin{aligned}
		s^3 = \Big\{r &= \big\{[A,B], [C,D], [E,F], [G,H]\big\},\\
		g &= \big\{[A,C], [B,D], [E,G], [F,H]\big\},\\
		b &= \big\{[A,E], [B,F], [C,G], [D,H]\big\} \Big\}\\
	\end{aligned}
\end{equation*}

\begin{figure}
	\begin{tikzpicture}
		\node at (1.5,3) (V) {Vertices}; \node at (4.5,3) (I) {Indices};  \node at (7,3) {Base set $\mathcal X^0$};
		\node at (1.5,1.5) (OC) {Order Componenets}; \node at (4.5,1.5) (VI) {V vs. I}; \node at (7,1.5) {$\mathcal X^1 \subset \mathcal P(\mathcal X^0)$};
		\node at (1.5,0) (FC) {Function Components}; \node at (4.5,0) (P) {Parts}; \node at (7,0) {$\mathcal X^2 \subset \mathcal P(\mathcal X^1)$};
		\node at (3,-1.5) (MDAs) {Multidimensional Arrays};   \node at (7,-1.5) {$\mathcal X^3 \subset \mathcal P(\mathcal X^2)$};
		\node at (7,-3) {$\mathcal X^4 \subset \mathcal P(\mathcal X^3)$};
		
		\draw[black, very thick, ->] (V.south) -- (OC.north);
		\draw[blue, very thick, <->] ([xshift=-5]I.south) -- ([xshift=10]OC.north);
		\draw[black, very thick, ->] (I.south) -- (VI.north);
		\draw[black, very thick, ->] (OC.south) -- (FC.north);
		\draw[black, very thick, ->] (VI.south) -- (P.north);
		\draw[black, very thick, ->] (FC.south) -- ([xshift=-5]MDAs.north);
		\draw[black, very thick, ->] (P.south) -- ([xshift=5]MDAs.north);
		
		\begin{pgfonlayer}{background}
			\node[fit=(V), vfit=cyan!60!white] {};
			\node[fit=(I), vfit=cyan!60!white] {};
			\node[fit=(VI), vfit=cyan!60!white] {};
			\node[fit=(OC), vfit=cyan!60!white] {};
			\node[fit=(P), vfit=cyan!60!white] {};
			\node[fit=(FC), vfit=cyan!60!white] {};
			\node[fit=(MDAs), vfit=cyan!60!white] {};
		\end{pgfonlayer}

		\node at (11,3) (V2) {Elements};
		\node at (11,1.5) (E2) {Vectors};
		\node at (11,0) (M2) {Matrices};
		\node at (11,-1.5) (T2) {Order $3$ Tensors};
		\node at (11,-3) (Z) {Order $4$ Tensors};
		\node at (11,-4.5) (Z2) {...};
		
		\draw[black, very thick, ->] (V2.south) -- (E2.north);
		\draw[black, very thick, ->] (E2.south) -- (M2.north);
		\draw[black, very thick, ->] (M2.south) -- (T2.north);
		\draw[black, very thick, ->] (T2.south) -- (Z.north);
		\draw[black, very thick, ->] (Z.south) -- (Z2.north);
		
		\begin{pgfonlayer}{background}
			\node[fit=(V2), vfit=cyan!60!white] {};
			\node[fit=(E2), vfit=cyan!60!white] {};
			\node[fit=(M2), vfit=cyan!60!white] {};
			\node[fit=(T2), vfit=cyan!60!white] {};
			\node[fit=(Z), vfit=cyan!60!white] {};
			\node[fit=(Z2), vfit=cyan!60!white] {};
		\end{pgfonlayer}
	\end{tikzpicture}
	\centering
	\caption{Structure diagrams for the functional definition of multidimensional arrays (left) and the list-of-lists definition (right). ``V vs. I'' is shorthand for $1$-cells which are required to distinguish between the vertices and index variables used to construct MDAs.}
	\label{fig:StructDiags}
\end{figure}
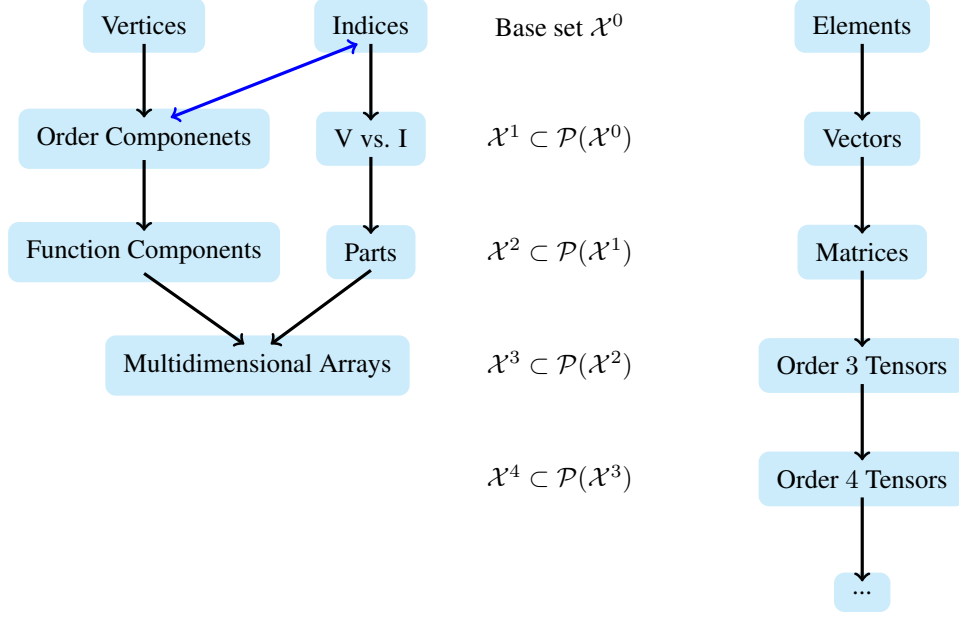

Next, we consider the standard functional definition of multidimensional array given in \cref{def:Tensor}. To represent such a function $T$ as an HCC, we start from the set $\cX^0 := V \bigcup_i [[M_i]]$, where $V$ is the value set and $[[M_i]]$ is the $i^{th}$ mode. Moving to rank $1$ cells, we can separate out $V$ and the components necessary to assemble the ordered pairs of $\bigotimes_i [[M_i]]$ Note that as with the naive encoding of PWO-HGs, these ordered pair components are \textit{not} ordered pairs yet, for this we need $2$-cells. Moreover, at the same time as we collect these components into ordered pairs, we can associate to them their images under $T$. Finally, we must collect all these function components together into the parts of a partition of the index variables. This results in a rank $3$ encoding, as shown in the left structure diagram of \cref{fig:StructDiags}. Explicitly, the encoding for the $2 \times 2 \times 2$ MDA is:
\begin{equation*}
	\begin{aligned}
		\cX^0 &= \{ A, B, C, D, E, F, G, H, r_1, r_2, g_1, g_2, b_1, b_2\}\\
		\cX^1 &= \big\{ \{r_1\}, \{r_1, r_2\}, \{g_1\}, \{g_1, g_2\}, \{b_1\}, \{b_1, b_2\}, \{ A, B, C, D, E, F, G, H \}, \{A\}, \{B\}, ... \big\}\\
		\cX^2 &= \Big\{ \big\{ \{r_1\}, \{g_1\}, \{b_1\}, \{A\} \big\}, \big\{ \{r_1, r_2\}, \{g_1\}, \{b_1, b_2\}, \{F\} \big\}, ..., \big\{\{r_1, r_2\} \big\}, \big\{\{b_1, b_2\} \big\}, ... \Big\}
	\end{aligned}
\end{equation*}
The MDA is given by $s^3 = \cX^2$. We can see that this encoding, while also rank $3$, is not intrinsic because \textit{many} more elements needed to be added to the base set $\cX^0$, unlike the PWO-HG construction. The functional definition of multidimensional arrays has to explicitly separate the elements from their multi-indices. In contrast, the PWO-HG construction of generalized tensors is intrinsic and extracts multi-indices for the elements from their relationships to other elements.

Next, we consider the list-of-lists construction of multidimensional arrays. The idea of this encoding is to first form lists (vectors), then form lists of lists (matrices), lists of lists of lists (order $3$ tensors), etc. This is how tensors are stored in computer memory. Explicitly, the encoding for the $2 \times 2 \times 2$ MDA is:
\begin{equation*}
	\begin{aligned}
		\cX^0 &= \{ A, B, C, D, E, F, G, H\}\\
		\cX^1 &= \big\{ [A,B], [C,D], [E,F], [G,H]\big\}\\
		\cX^2 &= \Big\{ \big[ [A,B], [C,D] \big], \big[ [E,F], [G,H] \big]\Big\}\\
		\cX^3 &= \Bigg\{ \Big[ \big[ [A,B], [C,D] \big], \big[ [E,F], [G,H] \big]\Big] \Bigg\}
	\end{aligned}
\end{equation*}
This encoding is intrinsic like the PWO-HG construction, requiring no additional elements to be introduced into the base set. However, it is not constant rank; an additional rank of cells are required to encode each additional order of multidimensional arrays.

It is interesting to note that out of these three encodings of multidimensional arrays, only the PWO-HG construction introduced in this paper simultaneously requires no extra elements \textit{and} can express tensors of any order with rank $3$ cells. It is in this sense that the PWO-HG construction is the ``best of both worlds" with respect to the standard functional definition and the computationally useful list-of-lists definition.

\subsubsection{\label{sec:GenTenProofs_Apdx}Proofs of \cref{thrm:GenTensors<=>MDAs,lem:MMs=>NonInjMDAs}}
We have now constructed enough new machinery to prove \cref{thrm:GenTensors<=>MDAs}. For the convenience of the reader, we recall the statement (with a minor technical update) here:

\begin{theorem}
	All injective multidimensional arrays can be represented as generalized tensors. Moreover, if $s^3$ is a finite generalized tensor whose canonical representative also satisfies the following two conditions:
	\begin{equation*}
		\begin{aligned}
			&\textit{1. } \forall s^2_i \in s^3, \exists \text{ a positive constant } c_i \in \mathbb{N}_+ \text{ such that } ||s^1_j|| = c_i, \ \forall s^1 \in \cX^1|_{s^2_i},\\
			&\textit{2. } \text{For each transversal } \{s^1_i\}_{i\in I} \sqsubset s^3, \text{ we have that } \bigg| \bigcap_{i=1}^{|s^3|} s^1_i \bigg| \leq 1\\
		\end{aligned}
	\end{equation*}
	\vspace{-0.5em}
	then $s^3$ can be represented as an injective multidimensional array.
	\label{thrm:GenTensors<=>MDAs_Appendix}
\end{theorem}
The ``minor technical update'' as compared to \cref{thrm:GenTensors<=>MDAs} is the restriction of hypotheses 1 to the canonical representative for $s^3$. This is necessary because hypothesis 1 (the \textit{equal length} condition) is meaningless in the presence of redundant slices (as shown in \cref{eq:NonCanonicalRep}).


A previous argument shows that finite generalized tensors must necessarily have canonical representatives, making the restriction of hypothesis 1 well-defined.

\noindent\textit{Proof.}
	There are two components to the proof. First, we will show how to construct a generalized tensor from an arbitrary injective multidimensional array. Second, we will use the properties of canonical representatives to show that generalized tensors satisfying hypothesis 1 must either be infinite or admit representations as multidimensional arrays.
	
	Throughout, let $H = \big\langle V, E, \{<_e\}_{e \in E}, P \big\rangle$ denote the PWO-HG associated to the generalized tensor $s^3$. 
	
	\paragraph{Part 1.}
	We first show that all injective multidimensional arrays can be represented as generalized tensors. Let $T: [[M_1]] \times [[M_2]] \times ... \times [[M_O]] \rightarrow V$ be an injective multidimensional array. Without loss of generality, assume $T$ is surjective (if not, simply restrict $V$ to the image of $T$). Denote by $v_{i_1, i_2, ..., i_O}$ the unique element of $V$ corresponding to the multi-index $(i_1, i_2, ..., i_O)$ under the function $T$. We construct a rank $3$ hierarchical combinatorial complex from $T$ as follows:
	\begin{equation*}
		\begin{aligned}
			\cX^0 &= \bigg\{ \{v\} \ : \ v \in V\bigg\}\\
			\cX^1 &= \bigg\{ \{ v_{1, k_2, ..., k_O} \}, \{ v_{1, k_2, ..., k_O}, v_{2, k_2, ..., k_O} \}, ..., \{ v_{1, k_2, ..., k_O}, v_{2, k_2, ..., k_O}, ..., v_{|M_1|, k2, ..., k_O} \},\\
			& \qquad \{ v_{k_1, 1, ..., k_O} \}, \{ v_{k_1, 1, ..., k_O}, v_{k_1, 2, ..., k_O} \}, ..., \{ v_{k_1, 1, ..., k_O}, v_{k_1, 2, ..., k_O}, ..., v_{k_1, |M_2|, ..., k_O} \},\\
			& \qquad \quad .\\
			& \qquad \quad .\\
			& \qquad \quad .\\
			& \qquad \{ v_{k_1, ..., k_{O-1}, 1} \}, \{ v_{k_1, ..., k_{O-1}, 1}, v_{k_1, ..., k_{O-1}, 2} \}, ..., \{ v_{k_1, ..., k_{O-1}, 1}, ..., v_{k_1, ..., k_{O-1}, |M_O|} \}\\
			& \qquad : \ \forall (k_1, k_2, ..., k_O) \in [[M_1]] \times [[M_2]] \times ... \times [[M_O]] \bigg\}\\
			\cX^2 &= s^3 = \bigg\{ \big\{ \{ v_{1, k_2, ..., k_O}, v_{2, k_2, ..., k_O}, ..., v_{|M_1|, k2, ..., k_O} \} \\
			& \qquad \qquad \qquad  : \ \forall (k_2, k_3, ..., k_O) \in [[M_2]] \times [[M_3]] \times ... \times [[M_O]] \big\},\\
			& \qquad \big\{ \{ v_{k_1, 1, ..., k_O}, v_{k_1, 2, ..., k_O}, ..., v_{k_1, |M_2|, ..., k_O} \} \\
			&\qquad \qquad \qquad : \ \forall (k_1, k_3, ..., k_O) \in [[M_1]] \times [[M_3]] \times ... \times [[M_O]] \big\},\\
			& \qquad \quad .\\
			& \qquad \quad .\\
			& \qquad \quad .\\
			& \qquad \big\{ \{ v_{k_1, ..., k_{O-1}, 1}, v_{k_1, ..., k_{O-1}, 2}, ..., v_{k_1, ..., k_{O-1}, |M_O|} \} \\
			&\qquad \qquad \qquad : \ \forall (k_1, k_2, ..., k_{O-1}) \in [[M_1]] \times [[M_2]] \times ... \times [[M_{O-1}]] \big\} \bigg\}
		\end{aligned}
	\end{equation*}
	Note that we have omitted the innermost braces around each singleton set of $\cX^0$ for legibility. Similarly, we have not explicitly encoded the order of the $2$-cells. This can be trivially accomplished with Kuratowski encoding. We now claim that $\cX^2$, when interpreted as a $3$-cell $s^3$, is indeed a generalized tensor. To verify this claim, we must check the following:
	\begin{enumerate}
		\item $\cX^2$ is a partition of the $\subsetneq$-maximal elements of $\cX^1$.
		\item Each $\{v\} \in \cX^0$ is contained in a transversal of this supposed partition.
		\item $\cX^1$ satisfies the slice ordering compatibility conditions.
	\end{enumerate}
	To show $1$, let $s^1 \in \cX^1$ be a $\subsetneq$-maximal $1$-cell. By construction, $s^1$ is a set of $v \in V$ which corresponds under $T$ to a set of multi-indices with $O-1$ of the entries fixed. Let $i$ denote the free index. As $\cX^2$ was constructed such that all $1$-cells with the $i^{th}$ index free belong to the same $2$-cell, we conclude there can be at most one $2$-cell $s^2$ which contains $s^1$. By construction, each $2$-cell contains all possible $1$-cells with the $i^{th}$ index free, so we conclude that there is indeed a $2$-cell which contains $s^1$. As $s^1$ was an arbitrary $\subsetneq$-maximal $1$-cell, we further conclude that property $1$ holds.
	
	To show $2$, let $v \in \cX^0$ be a $0$-cell. Let $[j_1, j_2, ..., j_O]$ denote the multi-index associated to $v$ by $T$. To show that $\{v\}$ is contained in a transversal of $\cX^2$, it suffices to show that an arbitrary $2$-cell contains some $1$-cell which contains $\{v\}$. So, let $s^2$ be an arbitrary $2$-cell. By construction, all $1$-cells contained in $s^2$ have the same $O-1$ indices fixed. So, let $s^1 \in s^2$ be the $1$-cell with fixed indices that agree with the corresponding $j_i$'s. Such a $1$-cell must exist because we constructed $\cX^2$ by forming all possible combinations of fixed indices. Then, as $s^1$ contains elements corresponding to all values of the remaining free index, it must contain $\{v\}$. We conclude that all $2$-cells contain some $1$-cell containing $\{v\}$ which in turn means there must exist a transversal of $\cX^2$ containing $\{v\}$.
	
	To show $3$, first observe that by construction, we have:
	\begin{equation*}
		d(v_{i_1, i_2, ..., i_O}, v_{j_1, j_2, ..., j_O}) = \sum_{n=1}^{O} j_n - i_n
	\end{equation*}
	Note that because $T$ is a multidimensional array, $s^3$ is not a hyper-tensor, meaning all tuples of $H$ are singletons. Because the identity vector field on $R^n$ is conservative, there are no cycles of non-zero distance on $H$. As no tuples are of size $ > 1$, there are no tuple cycles of non-zero distance on $H$.
	
	We conclude that the above construction does indeed produce a generalized tensor from $T$.
	
	\paragraph{Part 2.} We now show that if $s^3$ satisfies:
	\begin{equation*}
		\begin{aligned}
			&\textit{Hypothesis 1. } \forall s^2_i \in s^3, \exists \text{ a positive constant } c_i \in \mathbb{N}_+ \text{ such that } ||s^1_j|| = c_i, \ \forall s^1 \in \cX^1|_{s^2_i},\\
			&\textit{Hypothesis 2. } \text{For each transversal } \{s^1_i\}_{i\in I} \sqsubset s^3, \text{ we have that } \bigg| \bigcap_{i=1}^{|s^3|} s^1_i \bigg| \leq 1\\
		\end{aligned}
	\end{equation*}
	then $s^3$ can be represented as an injective multidimensional array.
	
	There are two key ingredients to the proof. First, we will show that $H$ satisfies a stronger type of connectivity. Second, we will show that $H$ has a zero element. Together with hypothesis 2, these facts are sufficient to completely determine a unique multidimensional array. Throughout, we will use \cref{cor:TupSOCC=>SOCC} to restrict our attention to paths/cycles instead of their tuple-variants.
	
	We start by handling the case when $|P| = 1$. Here, the only way for $H$ to be connected and to satisfy hypothesis \#1 is if there exists some $e \in E$ such that $e = V$. This is because $H$ is a canonical representative so the single $p \in P$ must partition $V$, so if there were disjoint edges of equal length $H$ would be disconnected. We can construct an order $1$ multidimensional array from $H$ by inverting the index function from $e$ to $\langle \{1,2,...,|e|\}, < \rangle$. This is possible because of hypothesis 2.
	
	From here on, we assume that $|P| \geq 2$.
	
	We pause to observe another important consequence of hypothesis \#1, namely, $|e|, ||e|| \geq 2$ for each $e \in E$. This must be the case because if some $e \in p \in P$ was such that $|e| = 1$, then $||e|| = 1$ so by the hypotheses, every edge in $p$ also has size $1$. This makes it impossible for any path to have non-zero $p$-distance, as no vertices are connected by any edge of $p$. So, $p \in P$ can be removed to produce an equivalent generalized tensor, contradicting the fact that $H$ is the canonical representative for $s^3$.
	
	\textbf{Claim \#1}: For each $a, b \in e \in p \in P$, and for each permutation $(p_1, p_2,...,p_{|P|})$ of $P$, there exists a path between $a$ and $b$ which intersects each $p_i$ at most once in exactly the order given by the permutation. For brevity, we will refer to this property as \textit{array connectivity}.
	
	\textit{Overview.} We will show that if this claim is false, then $V$ must necessarily be infinite. This shows that array connectivity holds for any finite $H$ satisfying hypothesis 1. We accomplish this by first finding an ``exterior corner'' of $H$ which will use to produce an infinite path in $H$.
	
	\textit{Proof.} For the sake of contradiction, assume that $H$ is not array connected. Then, there exists some permutation $\sigma_P = (p_1p_2p_3...p_{|P|})$ of $P$, such that some pair of vertices cannot be joined by a path which intersects each $p_i$ only once in the specified order. 
	Let $a \in V$. We construct a directed rooted tree $T$ as a sub-hypergraph of $H$ as follows:
	\begin{equation*}
		\begin{aligned}
			&L_0 := \{a\}\\
			&E_0 := \{(a, a)\}\\
			&\textit{For each $i \in (1, 2, ..., |P|)$}:\\
			& \qquad L_i := \{v \in V \ : \ \exists w \in L_{i - 1} \text{ such that } w, v \in e \in p_{i}\} \supseteq L_{i-1}\\
			& \qquad E_i = \{ (w, v) \in L_i \times L_i \ : \ w \in L_{i-1} \wedge v \notin L_{i-1}\}\\
			&V_T := \bigcup_{i = 0}^{|P|}L_i = L_{|P|}\\
			&E_T := \bigcup_{i = 0}^{|P|}E_i\\
		\end{aligned}
	\end{equation*}
	This construction produces a tree because if any two distinct paths starting at $r$ intersect, they would produce a cycle of non-zero $p$-distance for some $p \in P$ because each path only intersects each $p$ once. This is impossible because $H$ satisfies the SOCCs, so $T$ is acyclic.
	
	By construction, all paths in $T$ intersect each $p_i$ at most once. By the contradiction assumption, $\exists b \in V$ such that $b \notin V_T$, that is, $V \setminus V_T \neq \emptyset$. Without loss of generality, we may assume that $b$ is distance $1$ away from some $a' \in V_T$. This is because $H$ is connected, so if $b$ were farther away from $a'$ we could take the last vertex on a path from $b$ to $a'$ which is outside of $V_T$. By the maximality $H$, there exists some $e$ with $b, a' \in e$. So, let $J$ be the permutation index of this $b, a'$ edge, i.e., $b, a' \in e_J \in p_J$. By the construction of $T$, $a'$ cannot occur in $L_1$, otherwise $b \in V_T$. So, let $j$ be the smallest index such that $a' \in L_j$, and let $e_j \in p_j$ be the edge containing $a'$ in the $j^{th}$ part of $P$. The construction of $T$ ensures that $J > j$, as otherwise we would find that $b \in V_T$.
	
	Next, we show that WLOG, $a'$ is an endpoint of its containing $p_j$ edge, i.e., $a'$ is either $<_{e_j}$-maximal or $<_{e_j}$-minimal. To see this, suppose that $a'$ is not an endpoint. By a previous argument, $|e| \geq 2$ for each $e \in E$. So, let $c \in V$ be a $j$-neighbor of $b$, i.e., $c$ is a vertex such that $b, c \in e'_j \in p_j$. Importantly, $c \notin V_T$ because $b \notin V_T$. Without loss of generality, we can pick $c$ such that it is $j$-distance $1$ away from $b$.
	\begin{wrapfigure}[17]{r}{0.45\textwidth}
		\centering
		\raisebox{0pt}[\dimexpr\height-1.0\baselineskip\relax]{\includegraphics[width=0.45\textwidth]{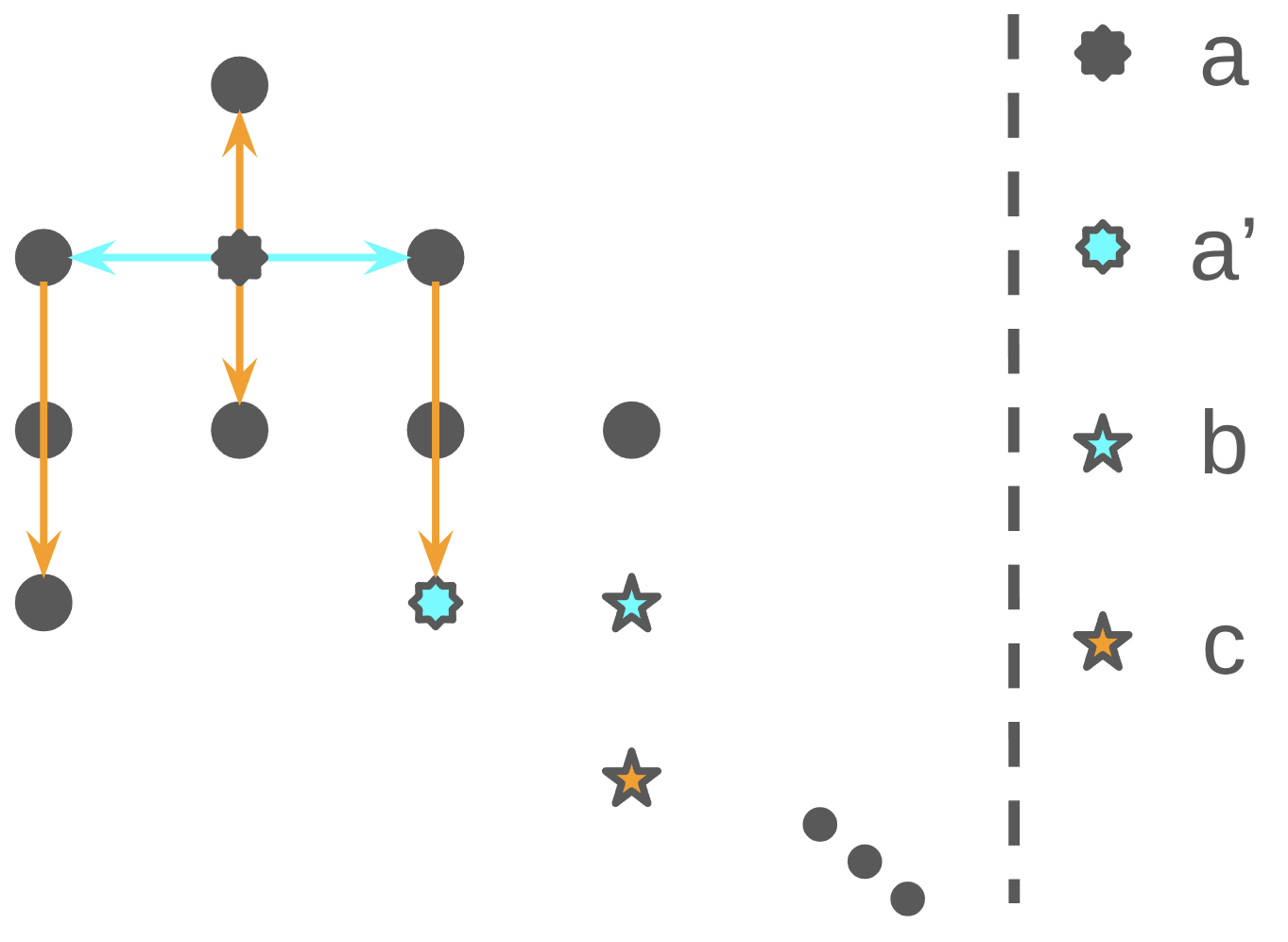}}\\
		\caption{Visualization of an example ``exterior'' corner. ({\color{cyan}$\bm{\rightarrow}$}) edges of $T$ added by $E_1$. ({\color{orange}$\bm{\rightarrow}$}) edges of $T$ added by $E_2$}
		\label{fig:InfiniteSequence}
	\end{wrapfigure}
	Next, consider the path $\gamma_{a', c} := [a', e_J, b, e_j, c]$, and let $\varsigma_j \in \{+1, -1\}$ be the $e_j$-distance between $b$ and $c$. Because $a'$ is not an endpoint of $e_j$, we can find a $c' \in V_T$ such that the path $\gamma_{c', c} := [c', e_j, a', e_J, b, e'_j, c]$ has $j$-distance zero. By the maximality of $H$, $c', c \in e \in E_J$. So $d$ and $c$ satisfy the same properties as $a'$ and $b$, i.e., $c' \in V_T$, $c \notin V_T$, and $|d_J(\gamma_{c', c})| = 1$. Because $e_j$ and $e'_j$ have the same length, we can repeat this construction to conclude that there must exist some choice of $a' \in V_T$ which is an endpoint of $e_j$.
	
	Next, we show that WLOG, $a'$ is an endpoint of $e_j$ \textit{and} $b$ is not a $e'_j$ endpoint of the same type, that is, $a'$ and $b$ cannot be simultaneously both maximal or both minimal with respect to $<_{e_j}$ and $<_{e'_j}$. To see this, observe that the equal length property and the maximality of $H$ show that if $a'$ and $b$ were endpoints of the same type, then every vertex of $e_j$ would be $J$-distance $1$ from a vertex of $e'_j$. This means that each $a' \in e_j$ has a corresponding $b \in e'_j$ which does not belong to $V_T$. Because this holds for the entirety of $e_j$, it holds in particular for the predecessor of $e_j$, that is, the vertex from some $e_{j-1}$ which is connected to each $a' \in e_j$ in $T$. So, we can take $a'$ to be this predecessor and replace $j$ with $j - 1$. This process cannot continue indefinitely, as $J > j$. So, we must eventually find some $j$ such that $a'$ is a $j$ endpoint and $b$ is not a $j$ endpoint of the same type.
	
	We are now ready to iteratively construct an infinite path on $H$. A recap of what we have built so far is provided in \cref{fig:InfiniteSequence}. We will use the fact that $c$ is ``off of'' the vertices covered by the tree $T$ to build an infinite sequence ``in the ({\color{darkgray}$\bullet \ \bullet \ \bullet$}) direction''.
	
	The above arguments show that we can find a $c$ such that $b, c \in e'_j \in p_j$ with $d_j(b, c) = \pm 1$ and $\nexists \ c' \in e_j$ with $d_j(\gamma_{a', c'}) = d_j(\gamma_{b, c})$. Therefore, $c$ is a $J$ endpoint pointing away from $a'$, i.e., every $J$ neighbor of $c$ lies in the same direction as $a', b \in e_J$. To put it another way, if $\varsigma_J = d_J(a', b)$, then $\forall d$ such that $c, d \in e'_J \in p_J$, $\varsigma_J \times d_J(c, d) \geq 0$. We now argue that $d$ can be chosen such that $d$ is a $j$ endpoint pointing in the direction of $\varsigma_j$. To see this, suppose that for every $d$, we could find some $f$ such that $d, f \in e''_j$ and $d_j(d, f) = -d_j(b, c)$. Then, the path $[b, e'_j, c, e'_J, d, e''_j, f]$ has only non-zero $J$-distance. So, by iteratively setting $|d_{e'_J}(c, d)| = 1, 2, ...$, the canonicity of $H$ implies that such $f$ must belong to the same $J$ edge as $a'$ and $b$. This immediately violates hypothesis 1 as we have found two $J$ edges of different lengths, namely, $||e_J|| = 1 + ||e'_J||$. So, we have found vertices $b, c, d$ such that $d_J(a', b) \times d_J(c, d) > 0$ and $d_j(b, c) \times d_j(d, f) > 0$, for all $f$ which are $j$ neighbors of $d$. If there are multiple such $d$, select the one which minimizes $J$-distance from $c$. This process can be continued indefinitely. Indeed, consider a path $(v_1, e^j_1, v_2, e^J_1, v_3, ..., v_{2N-1}, e^j_{2N-1}, v_{2N})$ such that $v_{2N}$ is a $J$-endpoint of minimal distance from $v_{2N-1}$. Suppose it is not possible to select a $v_{2N+1}$ which is a $j$-endpoint pointing away from $v_{2N-1}$. Then each $J$ neighbor of $v_{2N}$ admits a step towards $v_{2N-1}$, meaning either there was another $J$-endpoint closer to $v_{2N-1}$ or hypothesis \#1 fails. As this argument is symmetric with respect to $j$ and $J$, this shows that we can build an infinite path wherein the steps alternate between two parts of the partition $P$ and each step is in the same direction as all others from the part. By the SOCCs, such a path cannot contain a cycle, so we conclude that it must pass through infinitely many vertices. In turn, this contradicts the finiteness of $H$, so we deduce that $H$ must be array connected.
	
	\textbf{Claim \#2}: There exists some $z \in V$ such that for each $e$ with $z \in e \in E$, $z$ is $<_e$-minimal.\\
	\textit{Proof.} First, we show that the same type of infinite sequence used to prove claim \#1 can be constructed to show another important property of $H$. This property is: if $a, b \in e_i \in p_i$, then $index_j(a) = index_j(b)$ for all $e_j \in p_j$, $j \neq i$. Suppose this was not true. Then, observe that WLOG $a$ and $b$ are $i$-distance $1$ apart. If not, then iteratively move from $a$ to $b$ in $e_i$. If at any point in this process another $index_j$ differing vertex is encountered, then stop there. Eventually this process arrives at $b$, so $|d_i(a, b)| = 1$. Take $j$ such that $index_j(a) \neq index_j(b)$. Because the $i$ edges containing $a$ and $b$ are of equal length, we can shift $a$ and $b$ to find the same condition used above, i.e., $a$ a $j$-endpoint and $b$ not a $j$-endpoint of the same type. From here, $|V| \geq \aleph_0$ can be established with the same path construction.
	
	To prove claim \#2, consider the tree $T$ constructed in the proof of claim \#1. If $a$ is not minimal with respect to the single hyper-edge $e_1$ which defines $L_1$, move to a $<_{e_1}$-minimal element of $L_1$. Call this element $z_1$. Continuing the tree construction from $z_1$ must still cover $V$ by claim \#1. So, consider the $p_2$ edge with $z_1 \in e_2$. If $z_1$ is not $<_{e_2}$-minimal, move to a $z_2$ which is. By the previous argument, $index_{e_1}(z_1) = index_{e_1}(z_2)$, so $z_2$ is minimal with respect to both $<_{e_1}$ and $<_{e_1}$. Continuing this process for each $p_i$ produces an element which is $<_{e_1}$-minimal for all $i$.
	
	We can now construct a multidimensional array from $H$. Let $z$ be the zero element guaranteed by claim \#2. We will define a function $M$ on $I := [[c_1]] \times [[c_2]] \times ... \times [[c_{|P|}]]$. So, let $(i_1, i_2, ..., i_{|P|}) \in I$. Starting from $z \in e_1 \in p_1$, find an element in $e_1$ which has $1$-distance $i_1$ from $z$. Such an element must exist by hypothesis \#1. Call this element $z_1$. Next, consider the $p_2$ edge $e_2$ which contains $z_1$. By a previous argument, each $z_2 \in e_2$ has the same $1$-index as $z_1$. By construction, $z_1$ is minimal in $e_2$, so we can find a $z_2$ which is $i_2$ away from $z_1$. By repeating this process for each $p_i$, we can find an element whose $p$-distances to $z$ are given exactly by $(i_1, i_2, ..., i_{|P|})$. By hypothesis \#2, this element is unique. By claim \#1, there are no $v \in V$ which are not covered by some choice of $(i_1, i_2, ..., i_{|P|})$. So, we have constructed an injective multidimensional array from $H$ that is unique up to mode permutations.
	
	\label{prf:MDAs->GenTensors}
\null\hfill $\Box$

\begin{remark}
	In the construction of part 1 of \cref{thrm:GenTensors<=>MDAs_Appendix}, the $\subsetneq$-maximal elements of $\cX^1$ are exactly the $1$-slices of $T$, and $\cX^2$ is exactly the set of all $1$-slice spaces of $T$.
\end{remark}

\begin{remark}
	Part 2 of the proof of \cref{thrm:GenTensors<=>MDAs_Appendix} reveals that the cartesian product is inevitable in the sense that any finite PWO-HG which admits well-defined multi-indices and has equal length hyper-edges must give rise to a cartesian product. This is interesting because the equal length hypothesis is a local property, whereas the cartesian product is a global statement about all vertices of the PWO-HG. In this sense, the equal length condition ``rigidifies'' the underlying hypergraph.
\end{remark}

We conclude this part with the proof of \cref{lem:MMs=>NonInjMDAs}. Recall the statement:
\begin{lemma}
	Any multidimensional array can be represented as a mode map between two tensors.
\end{lemma}
\begin{proof}
	Let $T: [[M_1]] \times [[M_2]] \times ... \times [[M_{\cO}]] \rightarrow V$ be a multidimensional array. As before, we WLOG assume that $T$ is surjective. If $T$ is injective the result holds vacuously.
	
	We now assume that $T$ is non-injective. Create an order $1$ generalized tensor based on the image of $T$, i.e., a vector containing the elements of $V$. Create a second generalized tensor from any injective multidimensional array defined on $[[M_1]] \times [[M_2]] \times ... \times [[M_{\cO}]]$. Finally, express $T$ itself as a mode map from the second generalized tensor to the first. There is only a single mode map component which is the union of these tensors. By construction, the single trivial slice of the second tensor is a subset (in fact, equal to) the trivial slice of the first.
\end{proof}

\begin{remark}
	For the categorically inclined reader, we note that the proof of \cref{lem:MMs=>NonInjMDAs} demonstrates how mode maps can reduce to general functions between sets. In this sense, one may interpret the category of generalized tensors (with mode maps as morphisms) as a refinement of the category of sets which includes much additional structure.
\end{remark}

\clearpage

\subsection{Tensor Operations\label{sec:TensorOps}}
\paragraph{Overview.} The definition of tensor operation given in the main paper reveals how the tensor structure of an operation may be separated from the underlying operations on the tensors' value set. Additionally, it allows for jagged tensors to be used in operations of any complexity. In this part, we unpack the details surrounding this definition. We start by giving the complete definition of a tensor coupling. Then, we leverage our constructions to prove several useful statements about tensor operations. In particular, we give sufficient conditions for the \textit{arity decomposition} of tensor operations --- when a higher arity operation is equivalent to the composition of lower arity operations. We conclude with a collection of examples and discussions of connections to other frameworks.

Unlike the previous part where we constructed generalized tensors as abstract containers, in this section we assume that the base set of each generalized tensor is a collection of \textit{variables} valued in some set $\mathbb{F}$.

\subsubsection{\label{sec:TensorOps_Defs}Definition of Tensor Coupling.} We now provide the full definition of a tensor coupling. As stated in the main paper, the purpose of this definition is to generalize the rule that $m \times n_1$ matrices can only be multiplied with $n_2 \times p$ matrices when $n_1 = n_2$. This requires the notion of tensor length.
\begin{definition}
	Let $s^3$ be a generalized tensor and $s^2 \in s^3$. Then, $s^2$ has \textit{tensor length} $c$ when the maximum $s^2$-distance of any tuple path, on the PWO-HG associated to $s^3$, is $c$.
	\label{def:TensorLength}
\end{definition}
In the case when $s^3$ satisfies the hypotheses of \cref{thrm:GenTensors<=>MDAs_Appendix}, maximum $s^2$-distance coincides with the constant length of all hyper-edges of $s^2$. Equal tensor length is then exactly the familiar rule described above.

As articulated in the main paper, \cref{def:TensorLength} is necessary to enable the use of jagged tensors in tensor operations. As an example, consider the below example of ``matrix multiplication":
\begin{equation*}
	\begin{bmatrix}
		A & B\\
		C & 
	\end{bmatrix} \times
	\begin{bmatrix}
		E & F\\
		& H
	\end{bmatrix}
\end{equation*}
Clearly, the coupled modes contain hyper-edges ($1$-slices) of different lengths. However, the maximum $p$-distances of the modes agree, which is sufficient to guarantee that the broadcasting method of operation evaluation can produce a well-defined result. This result is given below:
\begin{equation*}
	\begin{bmatrix}
		AE + B & AF + BH\\
		CE & CF + H
	\end{bmatrix}
\end{equation*}

Contractions are simply distinguished couplings, that is, couplings which indicate explicitly that they are to be summed out. Technically, we accomplish this as follows:
\begin{definition}
	A \textit{tensor contraction} is a $3$-cell $s^3$ of an HCC $\cX$ such that $s^3 = C \cup \Big\{\big\{\{ \emptyset \}\big\}\Big\}$ where $C$ is a coupling.
\end{definition}
The purpose of the ``indicator $2$-cell" $\big\{\{ \emptyset \}\big\}$ is to ensure that contractions are distinguishable. This works because this set prevents any cell which contains it from satisfying the axioms of generalized tensors, mode map components, and couplings. As detailed in the next section, contractions are used to define slice-spaces of hyper-tensors during the computation of tensor operations.

\subsubsection{\label{sec:TensorOps_Eval}Evaluating Tensor Operations}
\paragraph{Overview.}
The hyper-tensor slice-spaces obtained from \cref{thrm:1-SMs=>Tensors} can be thought of as ``un-evaluated" versions of the final result tensors. Indeed, given a hyper-tensor slice-space, there are many possible ways to produce a ``collapsed" non-hyper tensor. Therefore, additional data is required, namely, a choice of two binary operations defined on the value set of the elements (recall that elements are \textit{variables}). These operations are called \textit{base operations}: one is used to collapse the tuples, and the other to collapse the slice space. They are therefore called the \textit{tuple-operation} and \textit{slice-operation}.

As an example, the familiar case of $\odot$ and $\oplus$ are ``the same" as tensor operations in the sense that given the same inputs they produce the same hyper-tensor slice-spaces. They do, of course, differ in their choice of base operations: $\odot$ uses multiplication as the tuple-operation whereas $\oplus$ uses addition. The slice-operation is irrelevant to this example, as neither operation produces non-trivial slice-spaces. The distinction between tensor operations of different base operations will be of critical importance to the architecture derivations conducted in \cref{sec:ArchDerivs}.

We will next prove a slightly more general version of \cref{thrm:1-SMs=>Tensors} stated in the main paper. This will be useful for establishing the machinery necessary to properly conduct the complexity analysis in \cref{tab:ComplexityHistory}. Before that, we pause to re-emphasize that \textit{elements}, the objects of the base set of generalized tensors, are \textit{variables}. Until now, there has been no reason to pay particular attention to them. We now assume that these variables are valued in a particular set $\mathbb{F}$. For completeness, we mention that a \textit{binary operation on $\mathbb{F}$} (not to be confused with a tensor operation!) is a function from $\mathbb{F}^2$ back to $\mathbb{F}$.

\paragraph{Setup.} We start our exploration of tensor operations in their full generality by briefly mentioning the PWO-HG interpretation of slice spaces.
\begin{proposition}
	Let $T: [[M_1]] \times ... \times [[M_{\cO}]] \rightarrow V$ be an injective multidimensional array with HCC representation $s^3$ and associated PWO-HG $H = \big\langle V, E, \{<_e\}_{e \in E}, P \big\rangle$. The slice space defined by $M' = \{M_1, ..., M_k\}$ for $k \leq \cO$ corresponds to the $\subsetneq$-maximal elements of $\varSigma_k \subsetneq \cP(V)$, where $\varSigma_k$ is given by:
	\begin{equation*}
		\varSigma_k = \big\{ \{v_1, ..., v_n\} \in \cP(V) \ : \ d_p(v_i, v_j) = 0, \ \forall i,j \in [[n]], \ p \in [k+1, ..., \cO]\big\}
	\end{equation*}
	\label{prop:HG-SSs}
\end{proposition}
\begin{proof}
	It follows from the construction used to prove part 1 of \cref{thrm:GenTensors<=>MDAs_Appendix} that $v_1, v_2$ belong to the same $k$-slice exactly when $d_p(v_1, v_2) = 0$ for $k+1 \leq p \leq \cO$. Therefore, each $\subsetneq$-maximal element of $\varSigma_k$ corresponds to a $k$-slice, making $\varSigma_k$ exactly the collection of all $k$-slices.
\end{proof}
\begin{remark}
	The above hypergraph formulation of slice spaces extends uniformly to arbitrary jagged hyper-tensors.
\end{remark}

We are now ready to give the strengthened version of \cref{thrm:1-SMs=>Tensors}.
\begin{theorem}
	Tensor operations are evaluated in two distinct steps. First, every tensor operation $s^4$ determines a $k$-slice-space of a hyper-tensor $s^3_h$, where $k$ is the number of contractions of the operation. Second, given a fixed ordering of the tensors $t_1, ..., t_{\alpha} \in s^4$, and two binary operations ($\star$ and $\diamond$) defined on $\mathbb{F}$, $s^3_h$ then determines a non-hyper tensor $s^3$. Moreover, if each operand satisfies the conditions of \cref{thrm:GenTensors<=>MDAs} and each operand's base set is distinct, then $s^3_h$ is $\alpha$-regular, where $\alpha$ is the arity of $s^4$.
	\label{thrm:1-SMs=>Tensors_Appendix}
\end{theorem}

\begin{proof}
	Let $s^4$ be a tensor operation containing tensors $T_1, T_2, ..., T_{\alpha}$ and couplings $C_1, C_2, ..., C_n$. Let $k \leq n$ denote the number of couplings that are obtained from contractions. Without loss of generality, we assume that the contracted couplings are $C_1, ..., C_k$ (if not, simply reorder the modes). Much like the evaluation of tensor operations, the proof consists of two distinct parts. First, we will construct a hyper-tensor $s^3_h$ from $s^4$. Second, we will construct from this hyper-tensor a strictly non-hyper tensor.
	
	\textbf{Part 1.}
	The construction of $s^3_h$ consists of three steps. First, build a tuple-valued multidimensional array from each $T_i$. Second, form a hyper-tensor by taking the element-wise union of these MDAs after broadcasting according to the couplings. Third, build a $k$-slice space of this hyper-tensor with the contractions.
	
	\textit{Step 1.} For each $T_i$, build a multidimensional array valued in the set $V_i := \cP(\cX^0|_{T_i}) \cup \emptyset$ as follows:
	\begin{enumerate}
		\item Let $H_i$ be the PWO-HG (\cref{def:PWOHGs}) associated to $T_i$, and pick an origin $z_i \in \cX^0|_{T_i}$ to maximize the minimum distance path from $z_i$. That is, pick $z_i = argmax_{z} \big\{ min_{x \in \cX^0|_{T_i}} d(z, x) \big\}$.
		\item Construct multi-indices for each tuple of $H_i$ with \cref{lem:CoTensors}, setting the origin to $z_i$. 
		\item If necessary, introduce offsets to these indices to ensure that no multi-index contains any non-positive entries. Call the resulting set of multi-indices $\cI_i \subseteq [[M_{i,1}]] \times ... \times [[M_{i,\cO_i}]]$, where $\cO_i = |\cX^2|_{T_i}|$ and $M_{i,j}$ is the tensor length of the $j^{th}$ mode of $T_i$.
		\item Define a $V_i$-valued multidimensional array $\cT_i$ by assigning to each multi-index of $\cI_i$ the corresponding tuple if one exists. Otherwise, assign $\emptyset$. By the SOCCs, this is a well-defined function.
		\item Collapse the tuples of $\cT_i$ to their products with the rule: $t \in image(\cT_i) \mapsto \prod_{x \in t} x$, where $\prod$ is shorthand for iterated application of $\star$.
	\end{enumerate}
	This process produces a multidimensional array $\cT_i: [[M_{i,1}]] \times ... \times [[M_{i,\cO_i}]] \rightarrow V_i$ for each tensor $T_i$. While these multidimensional arrays may be jagged, they are by construction not hyper.
	
	\textit{Step 2.} Denote by $\cM$ the set of all modes of the arrays constructed in step 1, i.e., $\cM = \bigcup_{i = 1}^{\alpha}\{ [[M_{i,1}]], ..., [[M_{i,\cO_i}]]\}$. Of course, we require a reduced set of modes for the output tensor based on the couplings. So, let $\sim_C$ be the equivalence relation defined by the couplings, i.e., $M \sim_C N$ for $M, N \in \cM$ exactly when the corresponding $2$-cells $s^2_M, s^2_N$ belong to the same coupling of $s^4$. Next, form $\cM' = \cM / \sim_C$, the set of modes modulo the couplings. By the definition of couplings and the construction in step 1, we know that all modes of the same equivalence class of $\sim_C$ have the same tensor length and therefore correspond to the same index set. Next, let $\cC_{O} = |\cM'|$, and let $M'_1, ..., M'_{\cC_O} \in \cM'$. We define another tuple-valued multidimensional array on $[[M'_1]]  \times .. \times [[M'_{\cC_O}]]$ as follows:
	\begin{equation*}
		\begin{aligned}
			&\cA: [[M'_1]]  \times .. \times [[M'_{\cC_O}]] \rightarrow V := \cP\bigg(\bigcup_{a = 1}^{\alpha}\cX^0|_{T_a}\bigg) \cup \emptyset\\
			&\text{where }(i_1, ..., i_{\cC_O}) \mapsto \bigcup_{a = 1}^{\alpha} \cT_a[i_{j_1}, ..., i_{j_{\cO_l}}],\\
			&\text{and for a fixed } a, \ \{j_1, ..., j_{\cO_i}\} := \big\{j \in [[\cC_O]] \ : \ [\sim_C]_j \cap T_a \neq \emptyset \big\}
		\end{aligned}
	\end{equation*}
	By \cref{lem:MMs=>NonInjMDAs}, $\cA$ can be represented as either a generalized tensor, or a mode map between two generalized tensors.
	
	\textit{Step 3.} Finally let $M'_{i_1}, ..., M'_{i_k}$ be the modes associated to the contracted couplings $C_1, ..., C_k$. The hyper-tensor slice-space is given by the $\{M'_{i_1}, ..., M'_{i_k}\}$-slice space of $\cA$ which can be well-defined in the language of generalized tensors by \cref{prop:HG-SSs}.
	
	We now assume in addition, that each $T_i$ satisfies the hypothesis of \cref{thrm:GenTensors<=>MDAs_Appendix}. Therefore, each $\cI_i = dom(T_i)$ is equal to $[[M_{i,1}]] \times ... \times [[M_{i,\cO_i}]]$. Furthermore, each $\cT_i$ cannot map any multi-index to $\emptyset$. In fact, each $\cT_i$ maps each multi-index to a singleton set. By assumption, the operands have distinct base sets. Therefore, each union used in the construction of step 2 reduces to a union of $\alpha$ distinct singleton sets. This shows that each tuple of the PWO-HG associated to $\cA$ consists of exactly $\alpha$ elements. We conclude that under this additional assumption, $\cA$ is an $\alpha$-regular hyper-tensor.
	
	\textbf{Part 2.}
	We now construct a non-hyper tensor from $\cA$ using $\star$ and $\diamond$. Keep the input tensors $T_1, ..., T_{\alpha}$ in their provided fixed ordering. We now collapse the tuples of variables to concrete values in $\mathbb{F}$, as follows:
	\begin{equation}
		\begin{aligned}
			&\cT: [[M'_1]]  \times .. \times [[M'_{\cC_O}]] \rightarrow \mathbb{F}\\
			&\text{where }(i_1, ..., i_{\cC_O}) \mapsto  \prod_{a=1}^{\alpha} \Bigg[ \big(\cA[i_1, ..., i_{\cC_O}] \big)[a] \Bigg],\\
			&\text{and } \big(\cA[i_1, ..., i_{\cC_O}] \big)[a] \text{ is the element from the } a^{th} \text{ operand,}\\
			&\text{and where }\prod \text{ is shorthand for iterated application of } \star
		\end{aligned}
		\label{eq:EvalOfA}
	\end{equation}
	$\cT$ is now, by construction, a non-hyper tensor. We may, of course, reduce its order by applying $\diamond$ along a slice space. So, after possibly re-indexing the modes, let $[[M'_1]]  \times .. \times [[M'_{\sigma}]]$ be the multi-index set defined by the non-contracted modes $\cM' \setminus \{M'_{i_1}, ..., M'_{i_k}\}$. Define another non-hyper tensor as follows:
	\begin{equation*}
		\begin{aligned}
			&\cT_{\sigma}: [[M'_1]]  \times .. \times [[M'_{\sigma}]] \rightarrow \mathbb{F}\\
			&\text{where }(i_1, ..., i_{\sigma}) \mapsto  \sum_{(i_{\sigma + 1},...,i_{\cC_O}) \in [[M'_{\sigma + 1}]]  \times .. \times [[M'_{\cC_O}]]} \cT[i_1, ..., i_{\cC_O}],\\
			&\text{and }\sum \text{ is shorthand for iterated application of } \diamond
		\end{aligned}
	\end{equation*}
	The multidimensional array $\cT_{\sigma}$ is exactly the result tensor.
\end{proof}
\begin{remark}
	\cref{thrm:1-SMs=>Tensors_Appendix} makes explicit the separation of tensor structure and base operations in the evaluation of tensor operations. Indeed, for any tensor operation, no information about the base operations is required to construct the hyper-tensor slice-space. Similarly, for any hyper-tensor slice-space, no information about the original tensor operation is required to construct the final result tensor.
\end{remark}
\begin{remark}
	\cref{thrm:1-SMs=>Tensors_Appendix} clarifies why the number of columns in a tensor operation matrix is called the order complexity. Indeed, order complexity is exactly the order of the multidimensional array $\cA$ constructed in the above proof.
\end{remark}
To establish connections to tensor \textit{broadcasting}, we pause to properly define this concept.
\begin{definition}
	Let $T$ be a generalized tensor which can be represented as a multidimensional array $\cT$ defined on the multi-index set $[[M_1]] \times ... \times [[M_{\cO}]]$. A \textit{broadcasted tensor over} $T$ is a multidimensional array $\cT_b$ of the form:
	\begin{equation*}
		\begin{aligned}
			\cT_b: [[M_1]] \times ... \times [[M_{\cO}]] \times [[M_{\cO + 1}]] \times ... \times [[M_{\cO + n}]] &\rightarrow V\\
			(i_1, ..., i_{\cO + n}) & \mapsto \cT[i_1, ..., i_{\cO}]
		\end{aligned}
	\end{equation*}
\end{definition}
This definition makes formal the familiar intuition that broadcasting just means copying a tensor along ``dummy" modes. We can now state the connection between broadcasting and the hyper-tensors obtained from tensor operations.
\begin{proposition}
	The index selection used in part 1, step 2 of the proof of \cref{thrm:1-SMs=>Tensors_Appendix} can be equivalently stated as accessing the $(i_1, ..., i_{\cC_O})^{th}$ element of appropriately broadcasted versions of $T_i$. That is, one may introduce dummy modes to each operand based on the tensor operation matrix to see that the hyper-tensor is given as an element-wise union of equal shape tensors.
\end{proposition}
\begin{proof}
	The rule:
	\begin{equation*}
		(i_1, ..., i_{\cC_O}) \mapsto \cT_a[i_{j_1}, ..., i_{j_{\cO_l}}], \qquad \{j_1, ..., j_{\cO_i}\} := \big\{j \in [[\cC_O]] \ : \ [\sim_C]_j \cap T_a \neq \emptyset \big\}
	\end{equation*}
	is, by definition, equivalent to forming a broadcasted tensor over $\cT_a$ with dummy modes $\{M'_1, ..., M'{\cO}\} \setminus \{M'_{j_1}, ..., M'_{j_{\cO_l}}\}$.
\end{proof}
It is important to underscore that \cref{thrm:1-SMs=>Tensors_Appendix} illustrates how the tensor structure of a tensor operation is completely independent of the choice of specific base operations used to evaluate it. Effectively, it is a decomposition of the notion of a ``tensor operation" into a ``tensor structure part" and an ``$\mathbb{F}$ operation part". As a result, this general statement explains how to compute \textit{any} tensor operation given \textit{any} two base operations.

\clearpage
\subsubsection{\label{sec:TensorOps_Decomp}Arity Decomposition of Tensor Operations\label{sec:ArityDecomp}}
The separation of tensor structure and base operations provided by \cref{thrm:1-SMs=>Tensors_Appendix} allows us to establish a useful statement about when high arity tensor operations may be \textit{decomposed}, as alluded to in \cref{sec:EmpiricalAnalysis} of the main paper. This will explain when and why high arity operations are equivalent to compositions of lower arity operations.

We recall that for binary operations $\star$ and $\diamond$ defined on $\mathbb{F}$, \textit{associativity} is the property that $(a \diamond b) \diamond c = a \diamond (b \diamond c)$ always holds. $\star$ \textit{distributes over} $\diamond$ if $a \star (b \diamond c) = (a \star b) \diamond (a \star c)$ and  $(b \diamond c) \star a = (b \star a) \diamond (c \star a)$ always hold.

\cref{thrm:ArityDecomp} is an extremely helpful tool and will be useful for architecture analysis, as it allows us to split and re-combine tensor operations of different arity provided the base operations are well-enough behaved. We now make this precise.
\begin{theorem}
	Let $s^4$ be an $\alpha$-arity tensor operation with $\alpha \geq 3$, and let $\star$ and $\diamond$ be two binary operations defined on $\mathbb{F}$. If the following three conditions hold:
	\begin{enumerate}
		\item Each operand of $s^4$ can be represented as an injective multidimensional array
		\item $\star$ distributes over $\diamond$
		\item $\diamond$ is associative
	\end{enumerate}
	then, after evaluation using $\star$ and $\diamond$, $s^4$ is equal to the composition of an $(\alpha - 1)$-arity tensor operation and a binary tensor operation.
	\label{thrm:ArityDecomp}
\end{theorem}

Before we discuss the somewhat technical proof, we provide a fundamental corollary as further motivation.
\begin{corollary}
	Assume the entire setup of \cref{thrm:ArityDecomp}. Then $s^4$ can be decomposed into a sequence of binary tensor operations.
	\label{cor:DecompIntoBinary}
\end{corollary}
\begin{proof}
	Simply apply \cref{thrm:ArityDecomp} $(\alpha - 2)$ times to produce a sequence of binary tensor operations.
\end{proof}
\cref{cor:DecompIntoBinary} Is of great practical utility, as it explains how to construct an algorithm for evaluating high complexity operations. This algorithm is given in \cref{alg:TensorOpEval}. We now give the proof of \cref{thrm:ArityDecomp}.

\noindent\textit{Proof.}
Assume all notation from the proof of \cref{thrm:1-SMs=>Tensors_Appendix}, and let $\cT$ denote the multidimensional array produced by the evaluation of $s^4$. We first compute the effect of hypothesis 1 on the formula for $\cT$.
\begin{equation*}
	\begin{aligned}
		\cT: [[M'_1]]  \times .. \times [[M'_{\cC_O}]] &\rightarrow \mathbb{F}\\
		(i_1, ..., i_{\cC_O}) &\mapsto  \prod_{a=1}^{\alpha} \big(\cA[i_1, ...,  i_{\cC_O}] \big)[a]
	\end{aligned}
\end{equation*}
\begin{equation*}
	\begin{aligned}
		\cT_{\sigma}: [[M'_1]]  \times .. \times [[M'_{\sigma}]] &\rightarrow \mathbb{F}\\
		(i_1, ..., i_{\sigma}) &\mapsto  \sum_{(i_{\sigma + 1},...,i_{\cC_O}) \in [[M'_{\sigma + 1}]]  \times .. 	\times [[M'_{\cC_O}]]} \cT[i_1, ..., i_{\sigma}, i_{\sigma + 1},..., i_{\cC_O}]
	\end{aligned}
\end{equation*}
Simplifying (and dropping the domain of the contracted indices for legibility):
\begin{equation}
	\begin{aligned}
		\cT_{\sigma}: [[M'_1]]  \times .. \times [[M'_{\sigma}]] &\rightarrow \mathbb{F}\\
		(i_1, ..., i_{\sigma}) &\mapsto  \sum_{(i_{\sigma + 1},...,i_{\cC_O})} \Bigg[ \prod_{a=1}^{\alpha} \big(\cA[i_1, ..., i_{\sigma}, i_{\sigma + 1},..., i_{\cC_O}] \big)[a] \Bigg]
	\end{aligned}
	\label{eq:TOrig}
\end{equation}

Next, we construct two lower arity operations from $s^4$. Let $C$ be the set of couplings of $s^4$. Define an $(\alpha - 1)$-arity tensor operation $s^4_1$ as follows:
\begin{equation*}
	s^4_1 = \{T_1, ..., T_{\alpha - 1}\} \cup C_1
\end{equation*}
Where $T_1, ..., T_{\alpha -1}$ are the first $\alpha - 1$ tensors from the original operation $s^4$, and $C_1$ is the set of couplings $C$ restricted to the set $\cM_1 = \bigcup_{i = 1}^{\alpha -1} T_i$. Critically, it is necessary to remove contractions from any coupling which intersects $T_{\alpha}$, that is:
\begin{equation*}
	\begin{aligned}
		C_1 &= \big\{ c \setminus T_{\alpha} \ : \ c \in C, \ c \text{ is not contracted} \big\} \\
		&\bigcup \big\{ c \setminus T_{\alpha} \ : \ c \in C, \ (c \text{ is contracted}) \wedge (c \cap T_{\alpha} = \emptyset) \big\}\\
		&\bigcup \big\{ (c \setminus T_{\alpha}) \setminus \{ \{ \{ \emptyset \} \} \} \ : \ c \in C, \ (c \text{ is contracted}) \wedge (c \cap T_{\alpha} \neq \emptyset) \big\}
	\end{aligned}
\end{equation*}
$s^4_1$ is, by construction, an ($\alpha - 1$)-arity tensor operation.

Next, let $\cT_{\beta}$ be the multidimensional array obtained from evaluating $s^4_1$, and let $T_{\beta}$ be its associated generalized tensor. We define a binary tensor operation $s^4_2$ as follows:
\begin{equation*}
	s^4_2 = \{T_{\beta}, T_{\alpha}\} \cup C_2
\end{equation*}
Form $C_2$ from $C$ by removing no contractions from $C$ and for each $c \in C$, interpreting its modes as modes of $T_{\beta}$ when applicable. This is perfectly well-defined, because the modes of $T_{\beta}$ are in one-to-one correspondence with some modes of $\cT$.

Let $\cT_{\alpha}$ be the multidimensional array associated to the generalized tensor $T_{\alpha}$. 

Next, we need to evaluate the composition of $s^4_1$ and $s^4_2$ to check against \cref{eq:TOrig}. The important data to keep track of are: which indices are contracted and \textit{when}, meaning contracted in $s^4_1$ or $s^4_2$. To facilitate this process, we now perform some index management:
\begin{equation*}
	\begin{aligned}
		(i_1, ..., i_{\sigma}) &\text{ --- non-contracted indices of } \cA\\
		(i_{\sigma + 1}, ..., i_{\cC_O}) &\text{ --- contracted indices of } \cA\\
		(j_1, ..., j_{\rho}) &\text{ --- non-contracted indices of } \cT_{\beta} \text{ from } s^4_1\\
		(j_{\rho + 1}, ..., j_{\cO_{\beta}}) &\text{ --- contracted indices of } \cT_{\beta} \text{ from } s^4_1\\
		(k_1, ..., k_{\varrho}) &\text{ --- non-contracted indices of } \cT_{\alpha} \text{ in } s^4_2\\
		(k_{\varrho + 1}, ..., k_{\cO_{\alpha}}) &\text{ --- contracted indices of } \cT_{\alpha} \text{ in } s^4_2\\
	\end{aligned}
\end{equation*}

Our next objective is to map the $j$ and $k$ indices to their relative positions in the $i$ indices. Throughout this process, we will use ``index" and ``mode" interchangeably. We start by noticing that the construction of $s^4_1$ and $s^4_2$ ensures that we may reorder the modes such that $(j_1, ..., j_{\cO_{\beta}})$ is a sub-sequence of $(i_1, ..., i_{\cC_O})$ with $j_1 = i_n$ for some $1 \leq n \leq \cC_O$ and $j_{\rho} = i_{\sigma + \Delta}$, for some $\Delta \geq 0$. $\Delta$ is the number of contracted modes of $\cA$ that are not contracted in $s^4_1$. $\Delta$ is non-negative because we only (possibly) removed contractions in the creation of $s^4_1$.

\begin{wrapfigure}[14]{r}{0.6\textwidth}
	\centering
	\raisebox{0pt}[\dimexpr\height-1.0\baselineskip\relax]{\includegraphics[width=0.6\textwidth]{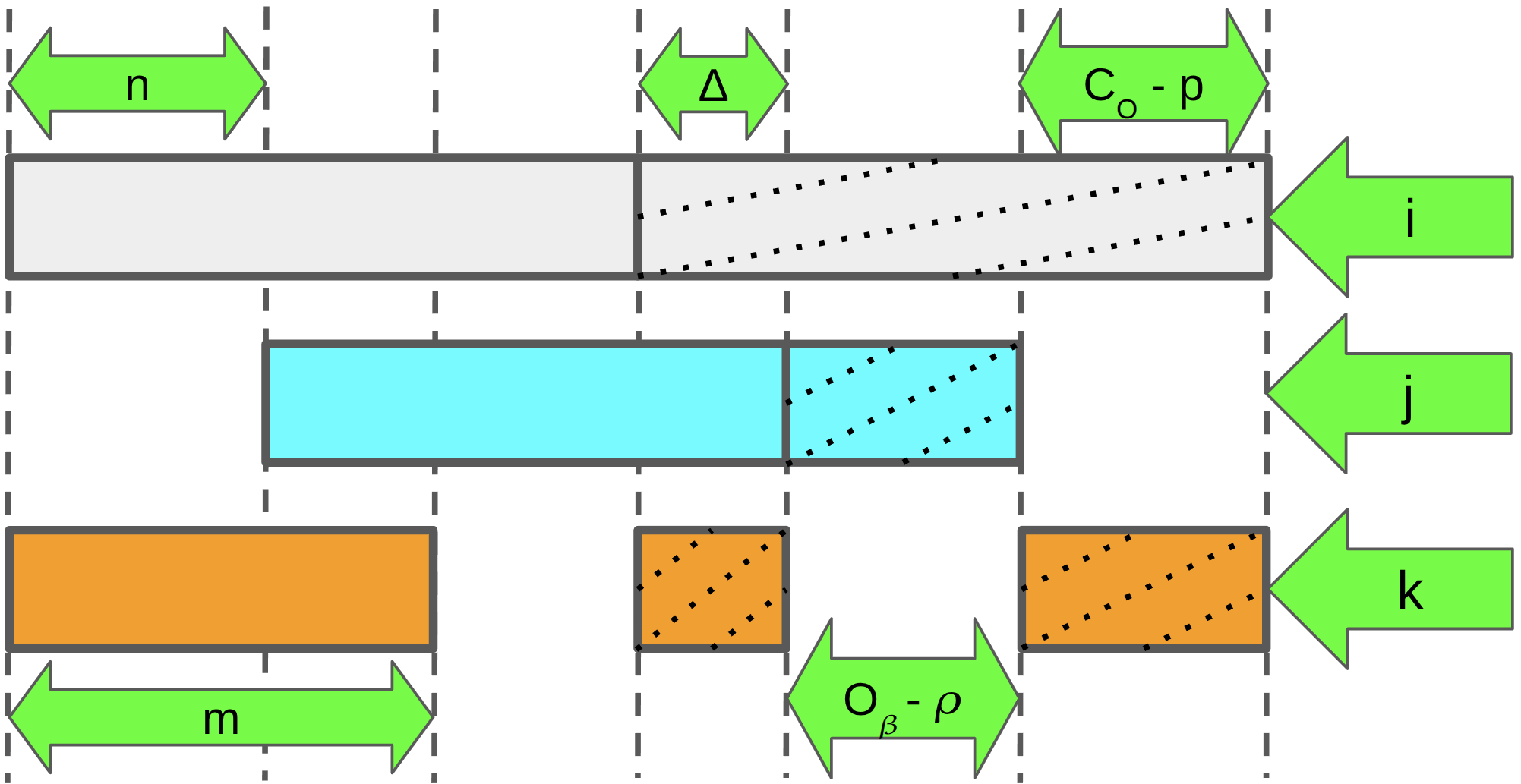}}
	\caption{Index map. Dotted boxes are contracted modes.}
	\label{fig:IndexMadness}
\end{wrapfigure}
Next, observe that for each contraction removed in the creation of $s^4_1$, there must exist some $k$ index which is part of the corresponding coupling. This is because of how $C_1$ is defined. Dually, for each contracted mode of $s^4_1$, there cannot exist any $k$ index which corresponds to the same original coupling of $C$.

Finally, observe that there may exist non-contracted $k$ indices which are coupled to some of the non-contracted $j$ indices. Moreover, there may exist both contracted or non-contracted $k$ indices which are not coupled to any $j$ modes. A visualization of this index map is given in \cref{fig:IndexMadness}. To recap, we have created an ($\alpha - 1$)-arity tensor operation $s^4_1$ ({\color{cyan}blue}) from the original $\alpha$-arity operation $s^4$ (top). The rule used to construct $s^4_1$ possibly removes contractions of $s^4$ ($\Delta$), but only from couplings which intersect $T_{\alpha}$ ({\color{orange}orange}).

Next, we need to translate this picture into equations. So, notice that $(k_1, ..., k_{\varrho}) = (i_1, ..., i_m)$, where $n \leq m \leq \sigma$. This holds because any non-contracted mode of $\cA$ that is broadcasted in $s^4_1$ must, by construction, belong to $T_{\alpha}$. Next, let $p := \cO_{\alpha} - \varrho$ be the number of contracted modes of $\cT_{\alpha}$ in $S^4_2$. The complete system of index equivalences is given below:
\begin{equation*}
	\begin{aligned}
		(i_1, ..., i_m) &= (k_1, ..., k_{\varrho})\\
		(j_1, ..., j_{m-n}) &= (k_n, ..., k_{\varrho})\\
		(j_1, ..., j_{\rho}) &= (i_n, ..., i_{\sigma + \Delta})\\
		(i_{\sigma + 1}, ..., i_{\sigma + \Delta}) &= (k_{\varrho + 1}, ..., k_{\varrho + \Delta})\\
		(j_{\rho + 1}, ..., j_{\cO_{\beta}}) &= (i_{\sigma + \Delta + 1}, ..., i_{\cC_O - p})\\
	\end{aligned}
\end{equation*}

Let $[[M'_n]]  \times .. \times [[M'_{n + \rho}]]$ be the multi-index set defined by the non-contracted modes of $\cA_1$, where $\cA_1$ is the hyper-tensor defined by $s^4_1$. From here on, we write $(i_1 : i_n)$ as shorthand for $(i_1, ..., i_n)$. We can now evaluate $s^4_1$:
\begin{equation}
	\begin{aligned}
		&\cT_{\beta}: [[M'_n]]  \times .. \times [[M'_{n + \rho}]] \rightarrow \mathbb{F}\\
		&(j_1 : j_{\rho}) = (i_n : i_{\sigma + \Delta}) \mapsto \sum_{(j_{\rho + 1} : j_{\cO_{\beta}})} \Bigg[ \prod_{a=1}^{\alpha - 1} \big(\cA_1[j_1 : j_{\cO_{\beta}}] \big)[a] \Bigg]\\
		&= \sum_{(j_{\rho + 1} : j_{\cO_{\beta}})} \Bigg[ \prod_{a=1}^{\alpha - 1} \big(\cA[\emptyset = i_1 : i_{n-1}, j_1 : j_{\cO_{\beta}}, i_{\sigma + \Delta + 1} : i_{\cC_O} = \emptyset] \big)[a] \Bigg]
	\end{aligned}
	\label{eq:S41}
\end{equation}
Here we have used an empty index to indicate broadcasting, i.e., indexing into dummy modes. This works for each $a \leq \alpha - 1$ because, by construction, none of the broadcasted modes intersect any $T_a$ with $a \leq \alpha - 1$.

Let $\cA_2$ be the hyper-tensor defined by $s^4_2$, and let $\cT'_{\sigma}$ be the corresponding multidimensional array. We are now finally ready to evaluate $s^4_2$:
\begin{equation*}
	\begin{aligned}
		&\cT'_{\sigma}: [[M'_1]]  \times .. \times [[M'_{\sigma}]] \rightarrow \mathbb{F}\\
		&(i_1 : i_{\sigma}) \mapsto \sum_{(i_{\sigma + 1} : i_{\sigma + \Delta}, i_{\cC_O - p + 1} : i_{\cC_O})} \Bigg[ \cT_{\beta}[i_n : i_{\sigma + \Delta}] \star \cT_{\alpha}[i_1 : i_m, i_{\sigma + 1} : i_{\sigma + \Delta}, i_{\cC_O - p + 1} : i_{\cC_O}] \Bigg]
	\end{aligned}
\end{equation*}
expanding out the above gives:
\begin{equation*}
	\begin{aligned}
		= \sum_{(i_{\sigma + 1} : i_{\sigma + \Delta}, i_{\cC_O - p + 1} : i_{\cC_O})} \Bigg( \cT_{\beta}[i_n : i_{\sigma + \Delta}] &\star \cT_{\alpha}[i_1 : i_m, i_{\sigma + 1} : i_{\sigma + \Delta}, i_{\cC_O - p + 1} : i_{\cC_O}] \Bigg)\\
		= \sum_{(i_{\sigma + 1} : i_{\sigma + \Delta}, i_{\cC_O - p + 1} : i_{\cC_O})} \Bigg( \Bigg[ \sum_{(i_{\sigma + \Delta + 1} : i_{\cC_O - p})} &\Bigg[ \prod_{a=1}^{\alpha - 1} \big(\cA_1[i_n : i_{\cC_O - p}] \big)[a] \Bigg] \Bigg]\\
		& \star \cT_{\alpha}[i_1 : i_m, i_{\sigma + 1} : i_{\sigma + \Delta}, i_{\cC_O - p + 1} : i_{\cC_O}] \Bigg)
	\end{aligned}
\end{equation*}
By hypothesis 2, $\star$ distributes over $\diamond$. Furthermore, the indices used in $\cT_{\alpha}$ are independent of those in the innermost sum (recall $\sum$ is shorthand for iterated $\diamond$). We can therefore factor the $\cT_{\alpha}$ term inside the parentheses:
\begin{equation*}
	\begin{aligned}
		= \sum_{(i_{\sigma + 1} : i_{\sigma + \Delta}, i_{\cC_O - p + 1} : i_{\cC_O})} \Bigg(  \sum_{(i_{\sigma + \Delta + 1} : i_{\cC_O - p})} \Bigg[&\Bigg[ \prod_{a=1}^{\alpha - 1} \big(\cA_1[i_n : i_{\cC_O - p}] \big)[a] \Bigg] \\
		& \star \cT_{\alpha}[i_1 : i_m, i_{\sigma + 1} : i_{\sigma + \Delta}, i_{\cC_O - p + 1} : i_{\cC_O}] \Bigg] \Bigg)
	\end{aligned}
\end{equation*}
Because the broadcasted indices of \cref{eq:S41} align, by construction, with the indices of $\cT_{\alpha}$, we can combine the $\cT_{\alpha}$ term into the product (recall $\prod$ is shorthand for iterated $\star$) and combine:
\begin{equation*}
	\begin{aligned}
		&= \sum_{(i_{\sigma + 1} : i_{\sigma + \Delta}, i_{\cC_O - p + 1} : i_{\cC_O})} \Bigg(  \sum_{(i_{\sigma + \Delta + 1} : i_{\cC_O - p})} \Bigg[ \prod_{a=1}^{\alpha} \big(\cA[i_1 : i_{\cC_O}] \big)[a] \Bigg] \Bigg)
	\end{aligned}
\end{equation*}
By hypothesis 3, $\diamond$ is associative. Furthermore, the indices of each summation are independent. So, we can combine them:
\begin{equation*}
	\begin{aligned}
		&= \sum_{(i_{\sigma + 1},..., i_{\cC_O})} \Bigg[ \prod_{a=1}^{\alpha} \big(\cA[i_1 : i_{\cC_O}] \big)[a] \Bigg]
	\end{aligned}
\end{equation*}
This is exactly the value obtained from the higher arity operation in \cref{eq:TOrig}. We conclude that the composition of $s^4_1$ and $s^4_2$ is equivalent to the original tensor operation $s^4$.
\null\hfill $\Box$
\paragraph{Algorithmic form of \cref{cor:DecompIntoBinary}} \Cref{alg:TensorOpEval} demonstrates that any tensor operation of multidimensional arrays valued in $\mathbb{R}$, when evaluated using standard real number multiplication and addition, is equivalent to a sequence of Hadamard products followed by summation along a slice space. We note that the proof of \cref{thrm:ArityDecomp} further demonstrates that this algorithm can be made more efficient by processing applicable contractions inside the for loop.

\begin{algorithm}
	\KwData{$s^4$ a tensor operation of multidimensional arrays $\cT_1, ..., \cT_{\alpha}$, couplings $C$, and base operations real number $\cdot$ and $+$}
	\KwResult{A result multidimensional array $\cT_r$}
	$\{M_{i,1}, ..., M_{i,\cO_i}\} \gets dom(\cT_i)$\;
	$\cM \gets \bigcup_{i} \{M_{i,1}, ..., M_{i,\cO_i}\}$\;
	$\{M'_1, ..., M'_{\cC_O}\} = \cM' \gets \cM / \sim_C$\;
	$\cT_r' : [[M'_1]] \times ... \times [[M'_{\cC_O}]] \rightarrow \mathbb{R}$ given by\;
	$(i_1,...,i_{\cC_O}) \mapsto \cT_1[i_{j_1}, ..., i_{j_{\cO_1}}]$, where $\{j_1, ..., j_{\cO_1}\} = \big\{ j \in [[\cC_O]] \ : \ [\sim_C]_j \cap dom(\cT_1) \neq \emptyset \big\}$\;
	\For{$2 \leq i \leq \alpha$}{
	$\cT'_i: [[M'_1]] \times ... \times [[M'_{\cC_O}]] \rightarrow \mathbb{R}$ given by\;
	$(i_1,...,i_{\cC_O}) \mapsto \cT_i[i_{j_1}, ..., i_{j_{\cO_i}}]$, where $\{j_1, ..., j_{\cO_i}\} = \big\{ j \in [[\cC_O]] \ : \ [\sim_C]_j \cap dom(\cT_i) \neq \emptyset \big\}$\;
	$\cT'_r \gets \cT'_r \odot \cT'_i$\;
	}
	$C' \gets \{ c' \in C \ : \ c' \text{ is contracted}\}$\;
	$\cM'' \gets \{ M' \in \cM' \ : \ M' = [\sim_C]_{c'}, \ c' \in C'\}$\;
	$k \gets |\cM''|$\;
	$\cT_r : [[M'_{k+1}]] \times ... \times [[M'_{\cC_O}]] \rightarrow \mathbb{R}$ given by\;
	$\cT_r[i_{k+1}, ..., i_{\cC_O}] \gets \sum_{(i_1, ..., i_k) \in [[M''_1, ..., M''_k]]}\cT'_r[i_1, ..., i_{\cC_O}]$\;
	\KwRet{$\cT_r$}
	\caption{Procedure for computing results of tensor operations evaluated using real number multiplication and addition.}
	\label{alg:TensorOpEval}
\end{algorithm}

\begin{corollary}
	Assume hypotheses 2 and 3 from \cref{thrm:ArityDecomp}. Let $s^4_1$ be an $\alpha$-arity tensor operation, and let $s^4_2$ be an $\alpha'$-arity tensor operation and assume all operands of $s^4_1$ and $s^4_2$ can be represented as injective multidimensional arrays. Furthermore, assume that the result of $s^4_1$ has the same tensor lengths as the first operand of $s^4_2$. Then, $s^4_1$ and $s^4_2$ can be combined into an equivalent tensor operation of arity $\alpha + \alpha' - 1$.
	\label{cor:OperationMerging}
\end{corollary}
\begin{proof}
	These hypotheses are simply a restatement of the properties satisfied by the tensor operations $s^4_1$ and $s^4_2$ constructed in the proof of \cref{thrm:ArityDecomp}. Therefore, they are equivalent to the tensor operation defined by:
	\begin{equation*}
		s^4 := \{T \in s^4_1 \ : \ T \text{ is a tensor}\} \cup \{T \in s^4_2 \setminus T_1 \ : \ T \text{ is a tensor}\} \cup C_1 \cup C_2
	\end{equation*}
	where $C_1$, $C_2$ are the couplings from $s^4_1$, $s^4_2$, respectively, with $C_2$ modified by replacing modes of $T_1$ with the modes of the result of $s^4_1$. The tensor lengths align by assumption. 
\end{proof}

\clearpage
\subsubsection{\label{sec:TensorOps_Examples}Additional Examples of Tensor Operations}
We now provide more examples of tensor operations to facilitate understanding of the main paper. Throughout this section, we assume that $\mathbb{F} = \mathbb{R}$, and that the base operations of all tensor operations are real number multiplication and addition.

\paragraph{Jagged Matrix Multiplication.}
As explained in the previous sections, our definition of tensor operations is sufficiently general to allow jagged tensors to participate in any operation. We recall that jagged tensors are \textit{incomplete} in the sense that the multi-index sets defined by \cref{lem:CoTensors} will be strict subsets of a cartesian product. \cref{fig:JaggedMatMul} provides an illustration in the style of \cref{fig:TensorOpEx} of the evaluation of matrix multiplication with jagged matrices by \cref{thrm:1-SMs=>Tensors_Appendix}.
\begin{figure}[h!]
	\centering
	\includegraphics[width=\textwidth]{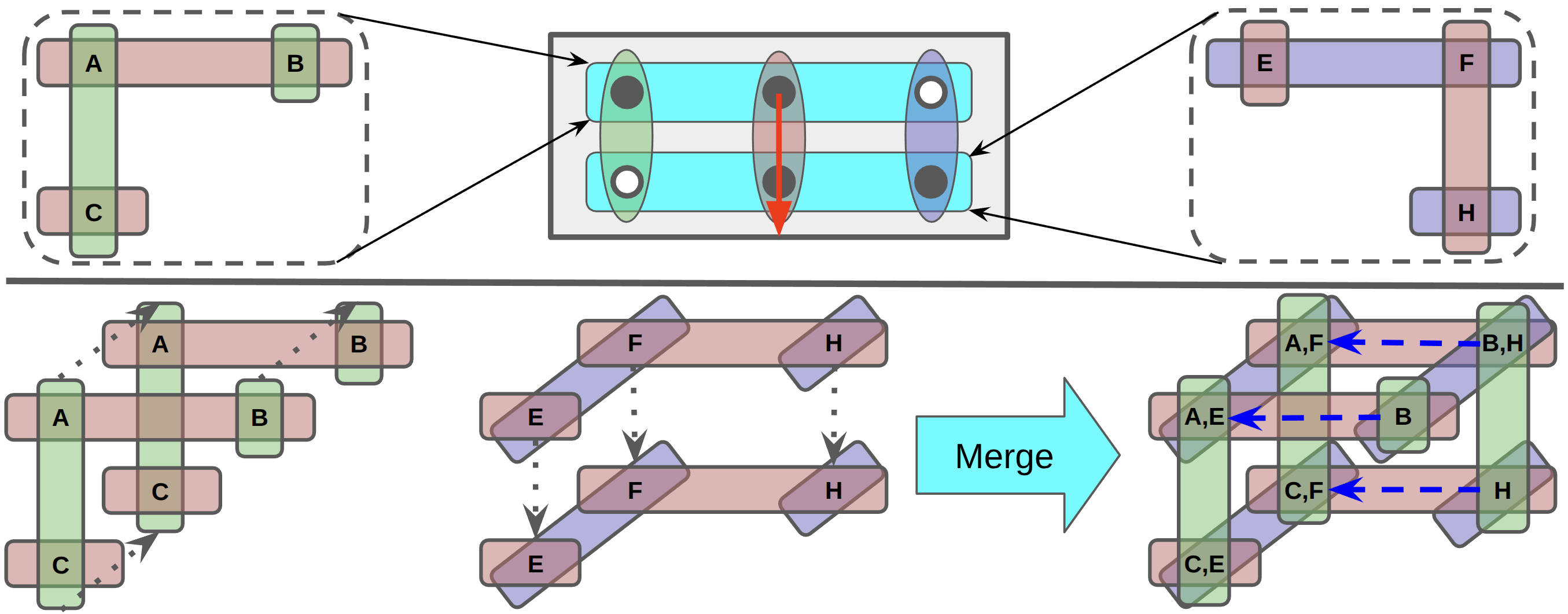}
	\caption{Visualization of the matrix multiplication of two jagged matrices. Modes are color-coded.}
	\label{fig:JaggedMatMul}
\end{figure}

\paragraph{The Fish Product.}
\begin{wrapfigure}[16]{r}{0.5\textwidth}
	\centering
	\raisebox{0pt}[\dimexpr\height-1.5\baselineskip\relax]{
		\includegraphics[width=0.5\textwidth]{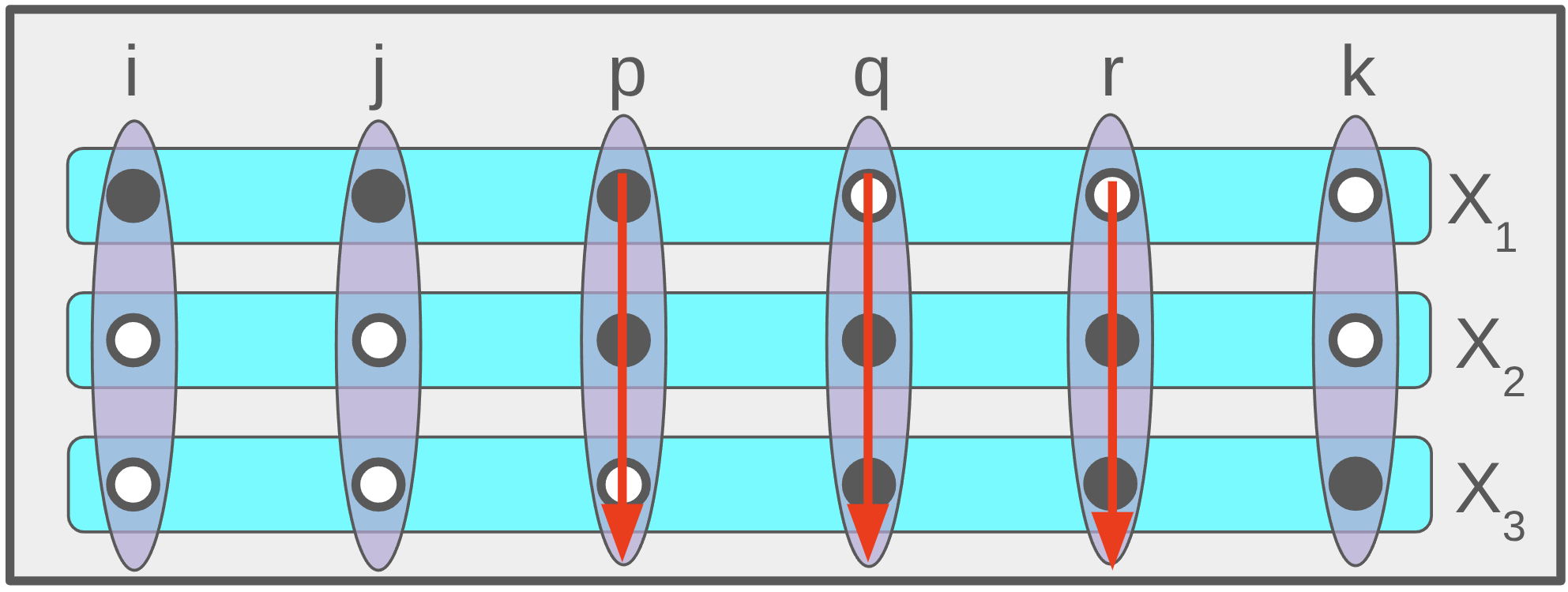}}
	\caption{Tensor operation matrix (TOM) representation of the fish product. Rows correspond to the input tensors; columns to the modes of the output. Filled/open circles indicate which modes are/are not part of which inputs. {\color{red}Red arrows} denote contractions. Observe that despite being arity $3$, this operation has coupling arity $2$.}
	\label{fig:FishProduct}
\end{wrapfigure}
The fish product is an example of ternary operation which has only binary coupling-arity. It was studied in \citenump{Plexes}. We recall the definition:
\begin{equation*}
	Y[i, j, k] = \sum_{p,q,r}X_1[i, j, p]X_2[p,q,r]X_3[q,r,k]
\end{equation*}
The tensor operation for the fish product is shown in \cref{fig:FishProduct}. Notice that despite being a ternary operation (i.e., an operation involving $3$ tensors), the coupling-arity is only $2$. This makes its decomposition into binary tensor operations (see \cref{cor:DecompIntoBinary}) particularly obvious:
\begin{equation*}
	\begin{aligned}
		Z[i,j,q,r] &= \sum_{p}X_1[i,j,p]X_2[p,q,r]\\
		Y[i, j, k] &= \sum_{q,r}Z[i,j,q,r]X_3[q,r,k]
	\end{aligned}
\end{equation*}
\clearpage
\paragraph{High Complexity Examples.}
\begin{wrapfigure}[10]{r}{0.475\textwidth}
	\centering
	\raisebox{0pt}[\dimexpr\height-0.0\baselineskip\relax]{\includegraphics[width=0.5\textwidth]{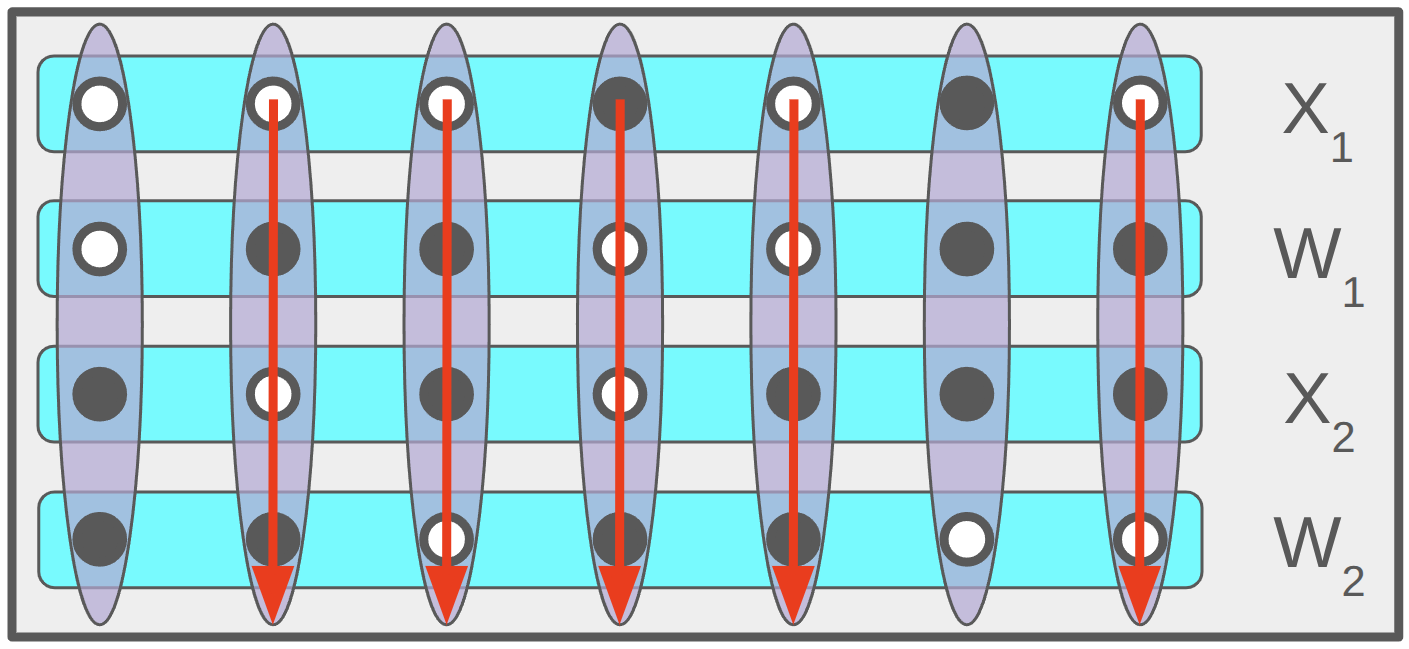}}
	\caption{Copy of \cref{fig:SpicyTOM}. A TOM representation of a high complexity operation.}
	\label{fig:SpicyTOM_Appendix}
\end{wrapfigure}
We now provide full formulae for some of the high complexity operations discovered during the dataset collection (sampling process described in full in \cref{sec:DataCollection}). We start with the operation from \cref{fig:SpicyTOM}. For convenience, the TOM for this operation is copied in \cref{fig:SpicyTOM_Appendix}. The formula for this operation is given below:
\begin{equation*}
	\begin{aligned}
		&Y[i,j,k,l,m,n,o] =\\
		&\sum_{j,k,l,m,o}X_1[l,n]W_1[j,k,n,o]X_2[i,k,m,n,o]W_2[i,j,l,m]
	\end{aligned}
\end{equation*}

Another high-complexity tensor operation from a top-performing sampled architecture is shown in \cref{fig:SpicyTOM2}. This is an example of an operation which has $\cC_A \neq \cC_{\alpha} = 4$. \begin{wrapfigure}[11]{r}{0.475\textwidth}
	\centering
	\raisebox{0pt}[\dimexpr\height-0.0\baselineskip\relax]{\includegraphics[width=0.475\textwidth]{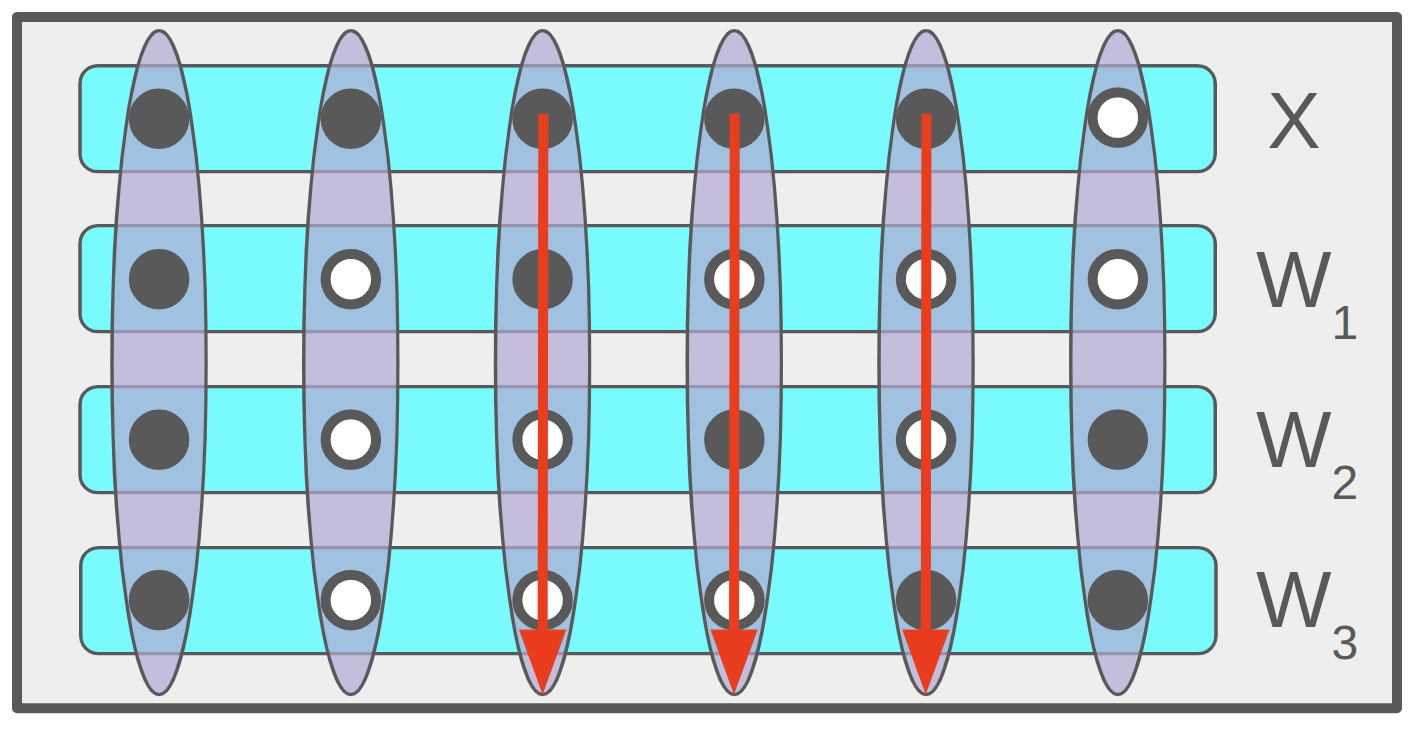}}
	\caption{TOM with $\cC_O = 6$, $\cC_{\alpha} = 4$, $\cC_A = 4$.}
	\label{fig:SpicyTOM2}
\end{wrapfigure}
For completeness, the formula defined by this TOM is given below:
\begin{equation*}
	\begin{aligned}
		&Y[i,j,k,l,m,n] =\\
		&\sum_{k,l,m}X[i,j,k,l,m]W_1[i,k]W_2[i,l,n]W_3[i,m,n]
	\end{aligned}
\end{equation*}
This operation has both arity $4$ and coupling-arity $4$, making it an example of a ``full" quad-ary tensor operation. Notice that three of the four tensors involved in this operation are learned, meaning it can be thought of as a very large fully connected layer which has been ``operand factored". This effectively enforces a particular sort of structural bias on this fully connected layer.

\begin{wrapfigure}[8]{r}{0.475\textwidth}
	\centering
	\raisebox{0pt}[\dimexpr\height-0.0\baselineskip\relax]{\includegraphics[width=0.5\textwidth]{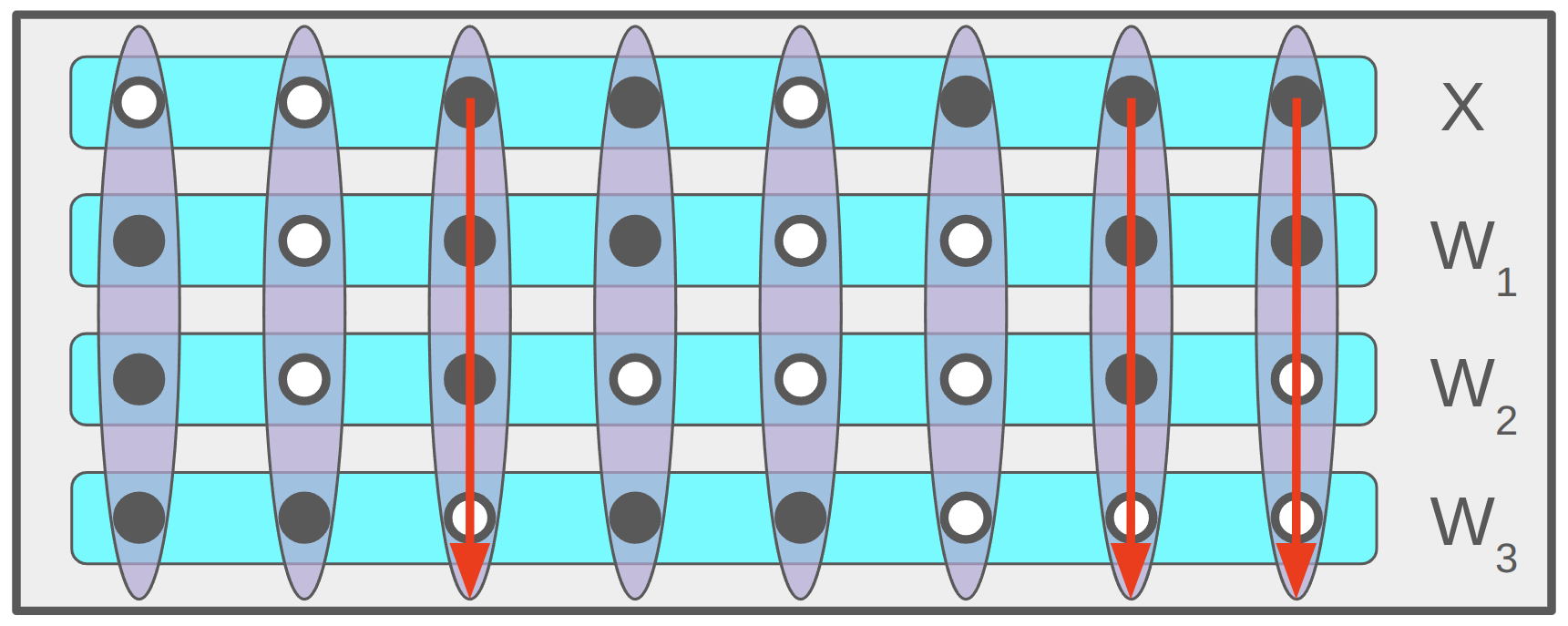}}
	\caption{TOM with $\cC_O = 8$, $\cC_{\alpha} = 4$, $\cC_A = 3$.}
	\label{fig:SpicyTOM3}
\end{wrapfigure}
A third high-complexity tensor operation from a top-performing sampled architecture is shown in \cref{fig:SpicyTOM3}. This is an example of an operation which has $\cC_A \neq \cC_{\alpha} = 4$. For completeness, the formula defined by this TOM is given below:
\begin{equation*}
	\begin{aligned}
		Y[i,j,k,l,m,n,o,p] &=\\
		\sum_{k,o,p}\Big[X[k,l,n,o,p]&W_1[i,k,l,o,p]\\&W_2[i,k,o]W_3[i,j,l,m]\Big]
	\end{aligned}
\end{equation*}
Notice that, similarly to the previous example, three of the four operands are learned, meaning that this operand can be thought of as a structurally regularized fully connected layer.

\subsubsection{\label{sec:TensorOps_Connections}Connections to Other Frameworks}

\paragraph{The Plex Formalism.} There is a natural connection between tensor operations (as we have defined them) and the plexes of \citenump{Plexes}. Recall that a plex is hypergraph whose vertices are modes of tensors and whose hyper-edges are tensors (couplings are then given as hyper-edge intersections). There is a straightforward correspondence between these representations of tensor operations: the hypergraph defined a tensor operation matrix corresponds exactly to that operation's plex diagram. Precisely, the \textit{hypergraph defined by a tensor operation matrix} $s^4$ is the structure $\langle V = \cM', E = \cX^3_T|_{s^4} \rangle$, where $\cM'$ is the set of modes modulo the couplings and $\cX^3_T|_{s^4}$ is the set of tensors contained in $s^4$. Stated another way, TOMs can be interpreted as incidence matrices of hypergraphs, where columns/rows correspond to vertices/hyper-edges. These hypergraphs are exactly the plex diagrams.

\paragraph{Copresheaf Networks.} There are several connections between our framework and the copresheaf framework introduced in \citenump{CopresheafNNs}. We recall that a \textit{copresheaf neighborhood matrix} (CNM) (Definition 7 of \citenump{CopresheafNNs}) is a linear transformation valued multidimensional array of order $2$, whose entries are determined by a neighborhood relation. Typically, these neighborhood relations are derived from the topology of a combinatorial complex (Definition 1 of \citenump{CopresheafNNs}). Given one or more copresheaf neighborhood matrices, one can then construct various types of neural network ``layers" (Definitions 9, 10, and 11 of \citenump{CopresheafNNs}).

\begin{wrapfigure}[10]{r}{0.55\textwidth}
	\centering
	\raisebox{0pt}[\dimexpr\height-1.5\baselineskip\relax]{\includegraphics[width=0.55\textwidth]{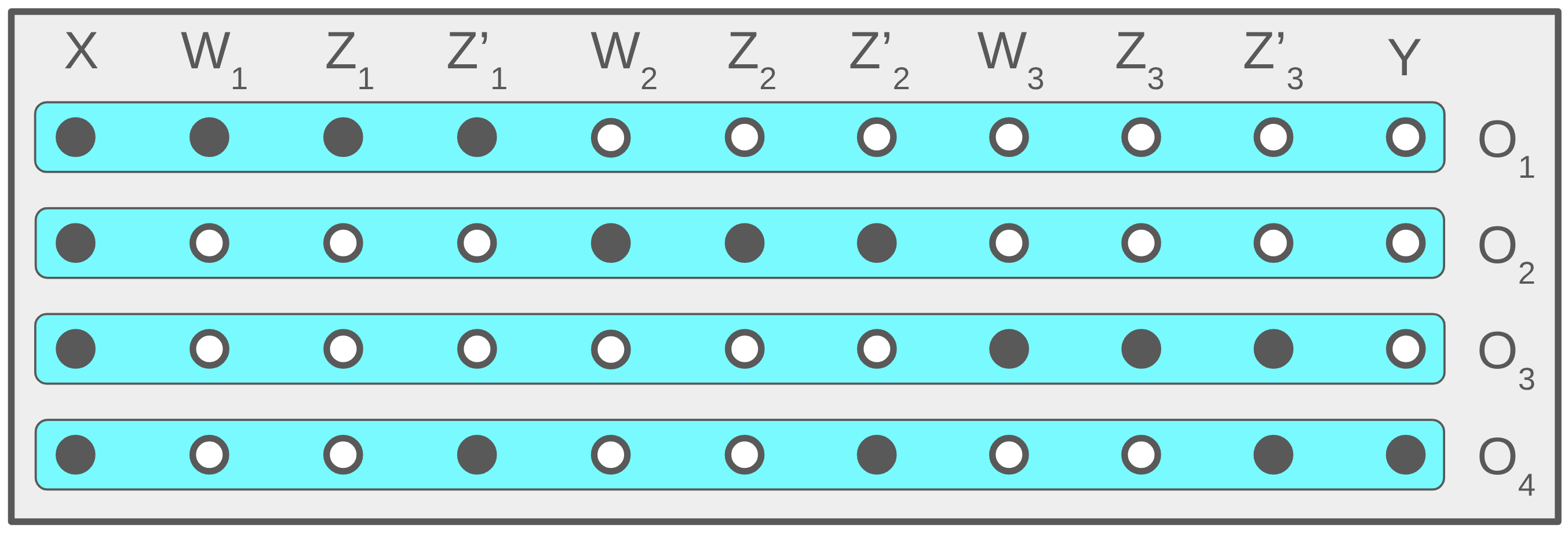}}
	\caption{TEM for Copresheaf message aggregation on a vertex neighborhood of cardinality $3$.}
	\label{fig:CopresheafTEM}
\end{wrapfigure}
Naturally, any of these copresheaf-based networks can be expressed in our framework, as they are still combinations of tensor operations and mode maps (most mode maps being trivial). However, there are deeper connections between CNMs and TEMs. For the sake of discussion, let us focus our attention to the message aggregation step which is common to all the specific networks constructed in definitions 9, 10, and 11 of \citenump{CopresheafNNs}. Let $X$ denote $h_x^{(l)}$, $Y$ denote $h_x^{(l+1)}$, $Z_i$ denote $h_{y_i}^{(l)}$ for each $y_i$ in the neighborhood of $x$, and $Z'_i$ denote the result of each application of the learnable message function.

Copresheaf message aggregation then corresponds to a particular type of TEM, as shown in \cref{fig:CopresheafTEM}. Note that we have, for simplicity, assumed that the learnable functions $\alpha$ and $\beta$ in Definition 9 of \citenump{CopresheafNNs} use compatible base operations with the linear transformations of the CNM.

It is interesting to observe that the copresheaf framework is effectively a categorical language for the description of TEMs similar to that of \cref{fig:CopresheafTEM}. An important detail is that the transformations that can be implemented in a CNM are strictly linear, meaning that in particular, they cannot leverage certain operations of $\cC_{\alpha} > 2$. This limitation highlights the key benefit of our framework, namely, it allows for the description of both TEM structure (including TEMs of the copresheaf network type) and tensor operation structure via the TOMs. Stated another way, our framework can be seen as a generalization of copresheaf networks which allows for higher complexity tensor operations to be used as the CNM functions.

\clearpage

\subsection{Architectural Derivations\label{sec:ArchDerivs}}
In this section, we discuss how the complexity measures reported in \cref{tab:ComplexityHistory} of the main paper are derived. We start with fully connected networks, and work our way left-to-right in the table. Throughout, we use the term \textit{tensor shape} to refer to the ordered tuple $(M_1, ..., M_{\cO})$. $X$ will be used for input tensors, $Z$ for intermediate tensors, and $Y$ for outputs.

\subsubsection{Computation Details.\label{sec:Disambig}}
There are often several equivalent ways to encode neural networks as HCCs. Therefore, some decisions must be made to eliminate ambiguity in the computed complexity measures. As articulated in the main paper, our focus is the study of how architectures have become more intuitively ``complicated" over the past $40$ years. We describe the disambiguating assumptions made with this objective in mind.

One of the largest sources of encoding ambiguity stems from \cref{thrm:ArityDecomp}. This is because when analyzing real-valued neural networks, it is difficult to determine the boundaries of tensor operations as they may be freely decomposed and merged to form a wide range of different arity operations. To remove this ambiguity from our analysis, we simplify encodings whenever possible by defining lower rank cells to be simpler. For example, of the two equivalent formulations of the fish product (see \cref{fig:FishProduct}), the ternary encoding is the simplest. This is because it involves fewer operations ($1$ instead of $2$) despite increasing arity complexity. We note that our set-theoretic definition of tensor operation necessarily requires that no operation contain multiple copies of the same tensor. 

We recall that tensor operations are evaluated with a choice of two \textit{fixed} base operations on $\mathbb{R}$, meaning that \cref{cor:OperationMerging} does not apply to operations evaluated with different base operations. In particular, this means that we can only merge binary tensor operations into a single ternary operation when the original operations are evaluated with the same base operations.

As we aim to analyze neural networks, we define a ``nonlinear activation" to be a unary tensor operation with base operations given by the nonlinear function. This is important, because it means that nonlinear activations (by this formal definition) clarify the boundaries between tensor operations. Specifically, no nonlinearity can occur ``in the middle of" a single ($\times$, $+$)-type of tensor operation.

For the sake of analysis, we ignore the batch mode when computing order complexity. Similarly, we do not count the ``packaging" tensors required to express some mode maps towards the final tensor complexity. In other words, tensor complexity is computed based on tensors which have distinct base sets. Yet, mode maps are nevertheless still relevant to complexity computation, as non-trivial mode maps further clarify tensor operations boundaries. Specifically, no non-identity mode map can occur ``in the middle of" a single tensor operation.

\paragraph{Fully Connected Networks.}
Fully connected networks are defined as a sequence of alternating matrix multiplications and nonlinear activations \citenump{NAGExpress}. We focus our attention to a single ``layer", which consists of a single operation of order complexity $3$ and arity $2$ (see \cref{fig:TensorOpEx,fig:JaggedMatMul}).

\paragraph{Convolutional Networks.}
\begin{wrapfigure}[8]{r}{0.5\textwidth}
	\centering
	\raisebox{0pt}[\dimexpr\height-1.5\baselineskip\relax]{\includegraphics[width=0.5\textwidth]{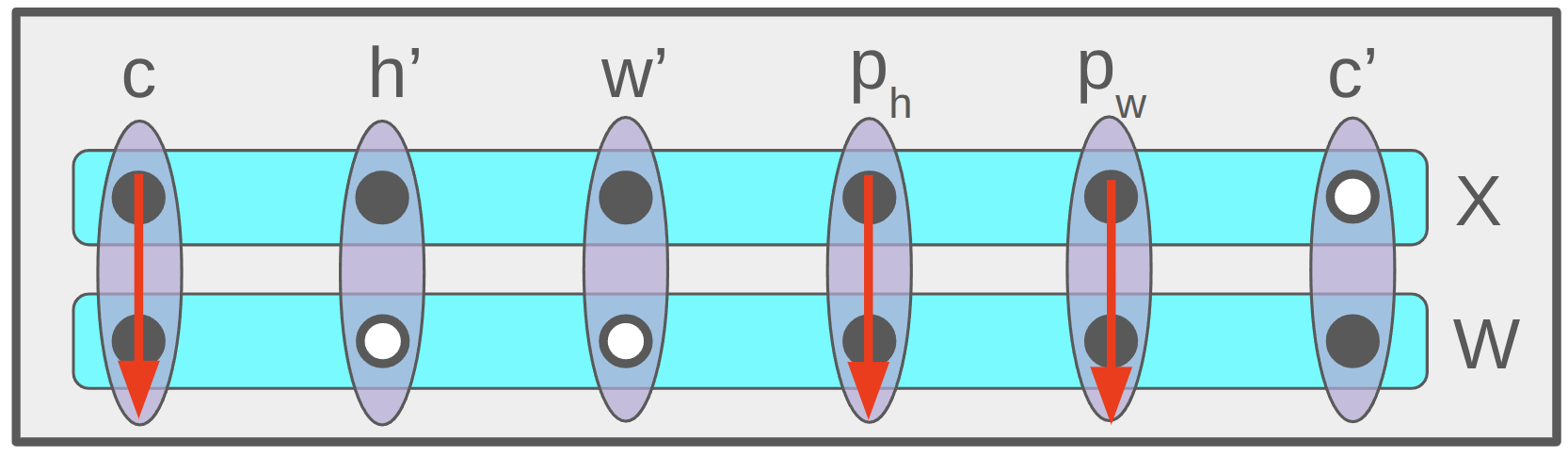}}
	\caption{TOM for 2-dimensional convolution.}
	\label{fig:Conv2TOM}
\end{wrapfigure}
We now consider purely convolutional networks such as VGG \citenump{VGG}. We focus our attention to a single convolutional ``layer", which consists of a binary operation and an unfolding mode map. The mode map is responsible for extracting patches from the input image of shape $(c, h, w)$, a process which produces an order $5$ tensor of shape $(c, h', w', p_h, p_w)$, where $p_h, p_w$ are the patch sizes for the height and width modes, and $h', w'$ are the numbers of extracted patches.

The TOM for convolution (after the mode map) is given in \cref{fig:Conv2TOM}. The weight tensor is of shape $(c, p_h, p_w, c')$, where $c'$ is the number of output channels. We observe that this operation has $\cC_O = 6$, $\cC_A = \cC_{\alpha} = 2$.

\paragraph{Residual Convolutional Networks.}
Next, we consider residual models such as ResNet \citenump{ResNet}. We focus our attention to the fundamental aspect of ResNets: the residual block (see Figure 2 of \citet{ResNet}). It is evident from this ``architecture diagram" that there are $3$ tensor operations involved. As the element-wise addition is binary and has order complexity $3$, the convolutional operations define $\cC_O$ and $\cC_{\alpha}$.

\begin{wrapfigure}[10]{r}{0.5\textwidth}
	\centering
	\raisebox{0pt}[\dimexpr\height-1.0\baselineskip\relax]{\includegraphics[width=0.5\textwidth]{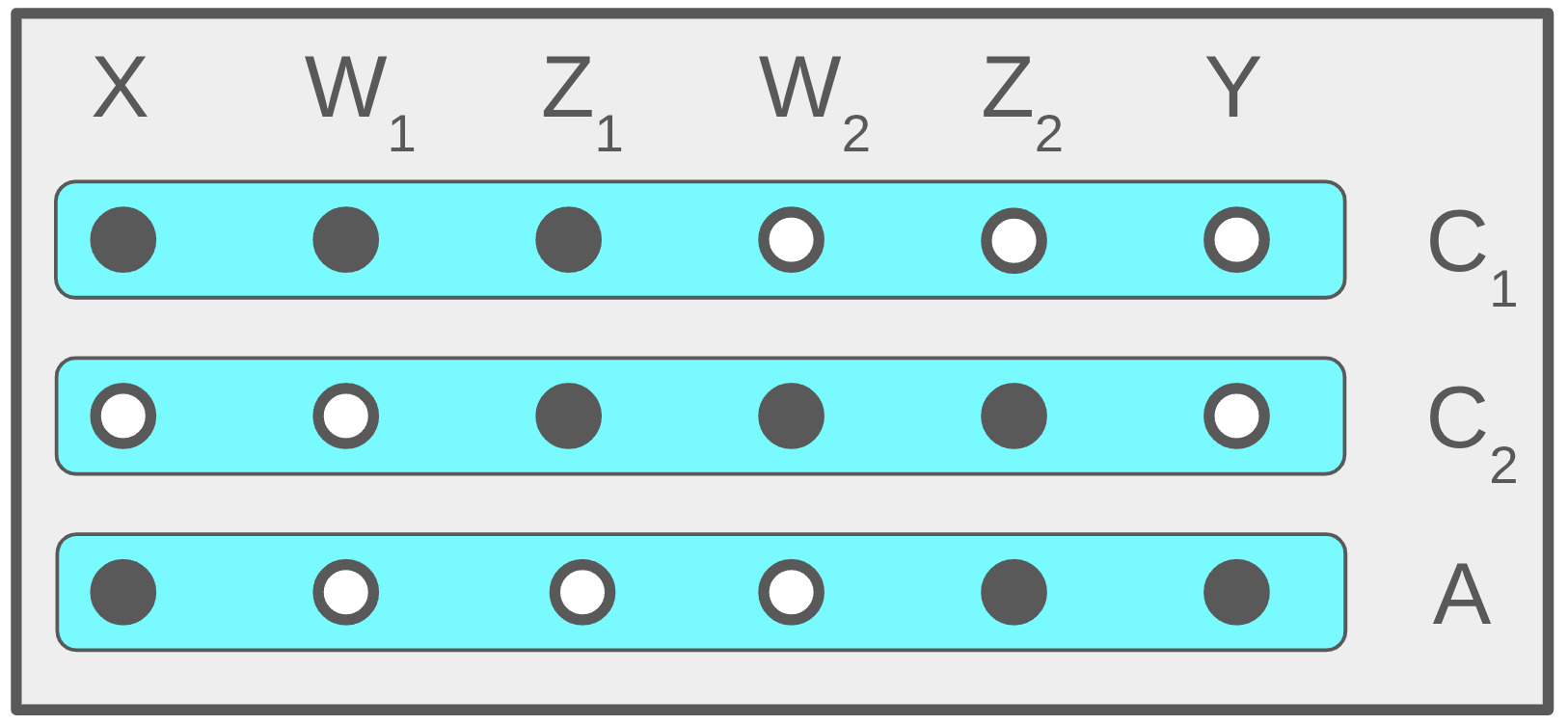}}
	\caption{TEM for a residual block.}
	\label{fig:ResNetTEM}
\end{wrapfigure}
As demonstrated in \cref{fig:ResNetTEM}, $\cC_T = 6$, and $\cC_{op} = 3$ for a residual block. Explicitly, the system of operations is given by:
\begin{equation*}
	\begin{aligned}
		Z_1 &= X \circledast W_1\\\
		Z_2 &= Z_1 \circledast W_2\\
		Y &= X \oplus Z_2
	\end{aligned}
\end{equation*}
where $\circledast$ is 2-dimensional convolution, and $\oplus$ is element-wise addition. The convolutions cannot be merged with \cref{thrm:ArityDecomp}, because there is a non-trivial mode map in between them. It is useful to note that \cref{fig:ResNetTEM} is effectively the incidence matrix for the ``tensor operation graph" (i.e., ``architecture diagram") shown in the original figure from \citenump{ResNet}. This theme will continue in the following derivations, illustrating how TEMs encode the intuitive idea of an architecture diagram.

\paragraph{Self-Attention.}
\begin{wrapfigure}[7]{r}{0.5\textwidth}
	\centering
	\raisebox{0pt}[\dimexpr\height-1.0\baselineskip\relax]{\includegraphics[width=0.5\textwidth]{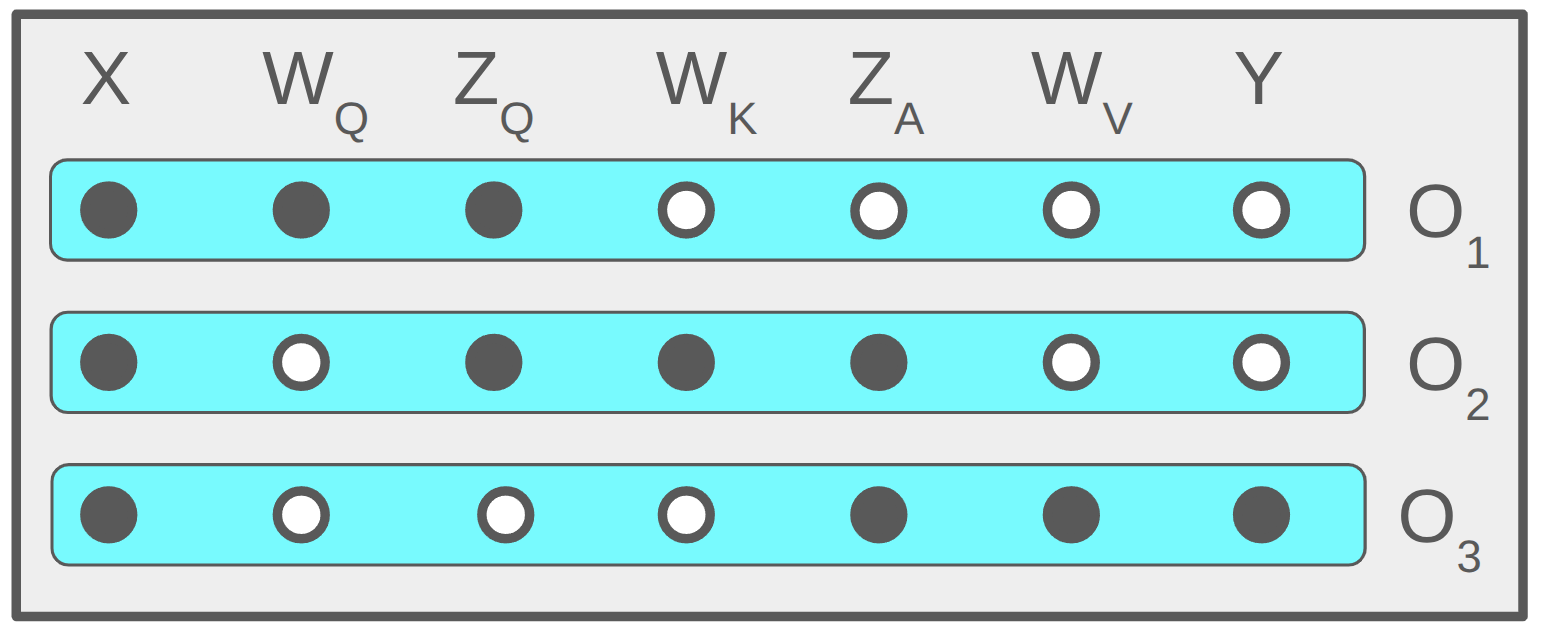}}
	\caption{TEM for self attention.}
	\label{fig:SATEM}
\end{wrapfigure}
Next, we consider transformer models such as \citenump{Transformers,ViT}. We focus our attention to a single self-attention ``operation". We start with the system of operations:
\begin{equation*}
	\begin{aligned}
		Z_Q &= X \times W_Q\\
		Z_K &= X \times W_K\\
		Z_V &= X \times W_V\\
		Z_A &= Z_Q \times Z_K^T\\
		Y &= Z_A \times Z_V
	\end{aligned}
\end{equation*}

The above would suggest that $\cC_{op} = 5$. However, as there are no activations or non-trivial mode maps between the computation of $Z_3$ and $Y$, we can combine those operations into a single operation of arity $3$ and coupling-arity $2$. Similarly, we can simplify the computation of the attention map $A$. This updated system of equations is below:
\begin{equation*}
	\begin{aligned}
		Z_Q &= X \times W_Q\\
		A &= Z_Q \times (X \times W_K)^T\\
		Y &= A \times (X \times W_V)
	\end{aligned}
\end{equation*}

This reduces $\cC_{op}$ to $3$, at the cost of increasing $\cC_{\alpha}$ to $3$. It is of course possible to swap the roles of $Z_Q$ and $Z_K$, but this does not change any of the complexity measures. $Z_Q$ and $Z_K$ cannot be simultaneously compressed without producing an operation with duplicate tensors. The TEM and TOMs for this system of operations are given in \cref{fig:SATEM,fig:SATOMs}, respectively. For ease of understanding, we include the tensor shapes for these TOMs. To recall, $X$ is of shape $(N,D)$ where $N$ is the number of tokens and $D$ is the token dimension. $D'$ and $D''$ indicate other ``hidden dimensions" which may be freely chosen.

\begin{wrapfigure}[19]{r}{0.35\textwidth}
	\centering
	\raisebox{0pt}[\dimexpr\height-0.0\baselineskip\relax]{\includegraphics[width=0.35\textwidth]{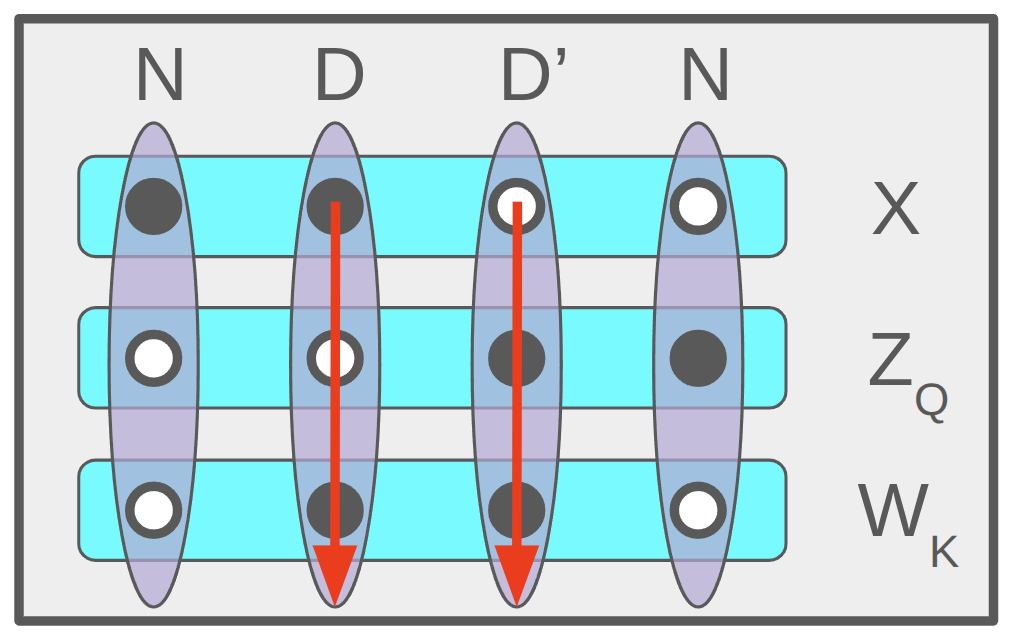}}\\
	\includegraphics[width=0.35\textwidth]{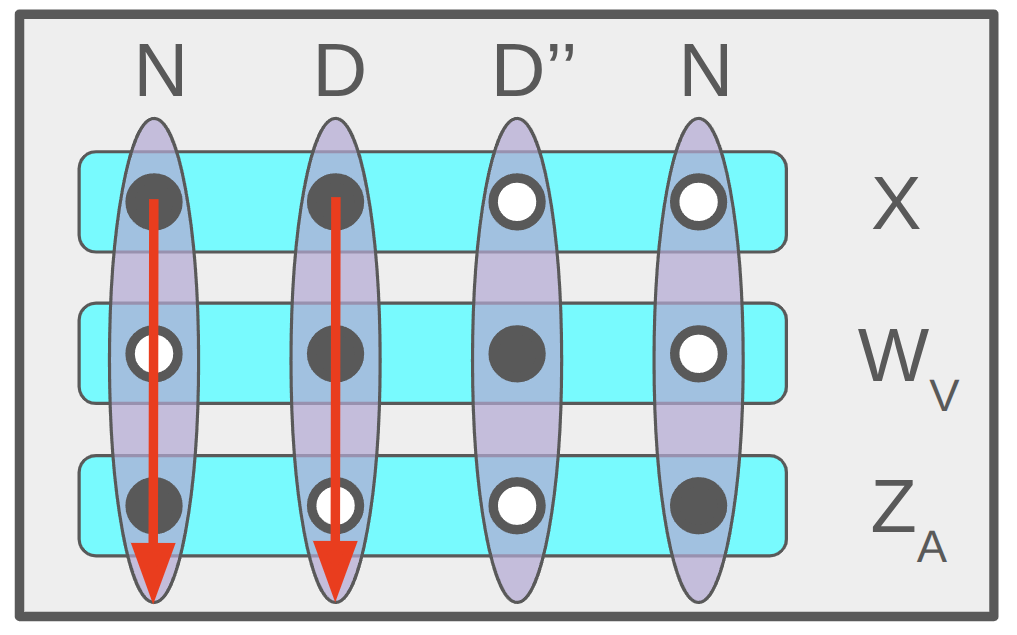}
	\caption{TOMs for operations $2$ and $3$ of self attention.}
	\label{fig:SATOMs}
\end{wrapfigure}
\cref{fig:SATOMs} demonstrates why $\cC_O = 4$, $\cC_T = 7$, $\cC_{\alpha} = 3$, and $\cC_A = 2$ for single-head self-attention. This operation is (to the best of our knowledge) the first time a higher arity core block was used to construct a deep neural network. As articulated in the main paper, this observation has gone unnoticed because the operation was originally described with only binary operations. It is only thanks to \cref{thrm:ArityDecomp} that we are able to glean such insights.

As we will show in the following derivations, the influence of self-attention on architecture design is significant. Specifically, many models developed after 2017 exhibit similar structural patterns, namely, they result in HCCs of higher arity tensor operations after simplification by \cref{thrm:ArityDecomp}.

It is also interesting to observe that multi-head self-attention has the same complexity signature up to order complexity, which is $5$ instead of $4$. This is because the parallel heads result in an extra mode of the $W_Q, W_K, W_V$ tensors and subsequent intermediate tensors.
\clearpage
\paragraph{Poly-Nets.}
\begin{wrapfigure}[12]{r}{0.45\textwidth}
	\centering
	\raisebox{0pt}[\dimexpr\height-0.0\baselineskip\relax]{\includegraphics[width=0.45\textwidth]{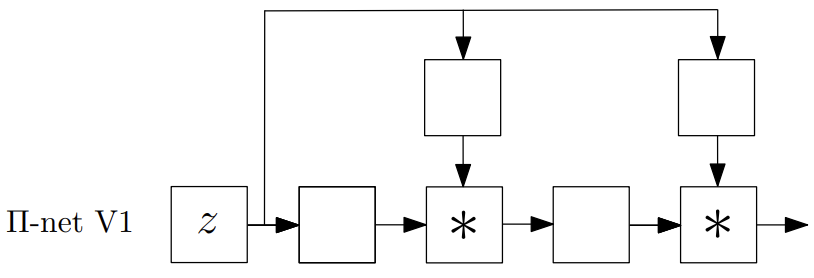}}\\
	\caption{Copy of Figure 1 (top) from \citenump{P-Nets} (reproduced with permission).}
	\label{fig:PNetFig1}
\end{wrapfigure}
Next, we consider polynomial neural networks \citenump{P-Nets}. These architectures are interesting because they are capable of learning without any non-linear activations. We focus our attention to the core building block of $\prod$-net V1 (top architecture diagram of Figure 1 from \citenump{P-Nets}). For the reader's convenience, we have reproduced this figure in \cref{fig:PNetFig1}. This architecture was originally designed for convolutional settings, so we interpret the input as an image (i.e., order $3$ tensor), and assume that all operations (boxes) are 2-dimensional convolution or Hadamard product. The resulting system of operations is given below:

\begin{equation*}
	\begin{aligned}
		Z_1 &= X \circledast W_1\\
		Z_2 &= X \circledast W_2\\
		Z_3 &= X \circledast W_3\\
		Z_4 &= Z_1 \odot Z_3\\
		Z_5 &= Z_4 \circledast W_4\\
		Y &= Z_2 \odot Z_5
	\end{aligned}
\end{equation*}

\begin{wrapfigure}[10]{r}{0.5\textwidth}
	\centering
	\raisebox{0pt}[\dimexpr\height-0.75\baselineskip\relax]{\includegraphics[width=0.5\textwidth]{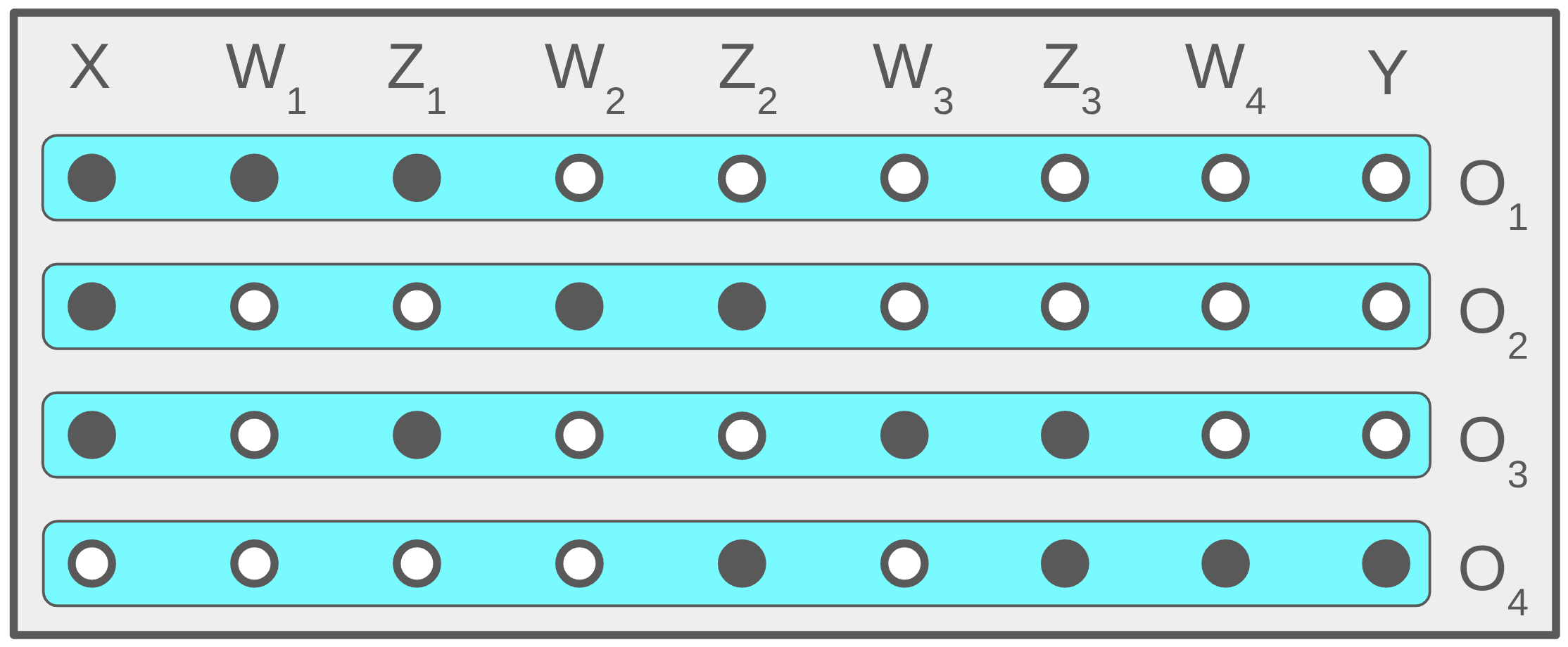}}\\
	\caption{TEM for $\prod$-net V1.}
	\label{fig:PNetTEM}
\end{wrapfigure}
Similarly to self-attention, this system of operations can be simplified by combining operations. Specifically, we can combine the computation of $Z_3$ and $Z_4$ as no non-trivial mode maps nor non-linear activations occur in between. Similarly, we can combine the computation of $Z_5$ and $Y$. This produces the following system of operations:
\begin{equation*}
	\begin{aligned}
		Z_1 &= X \circledast W_1\\
		Z_2 &= X \circledast W_2\\
		Z_3 &= Z_1 \odot (X \circledast W_3)\\
		Y &= Z_2 \odot (Z_3 \circledast W_4)
	\end{aligned}
\end{equation*}

\begin{wrapfigure}[9]{r}{0.4\textwidth}
	\centering
	\raisebox{0pt}[\dimexpr\height-1.0\baselineskip\relax]{\includegraphics[width=0.4\textwidth]{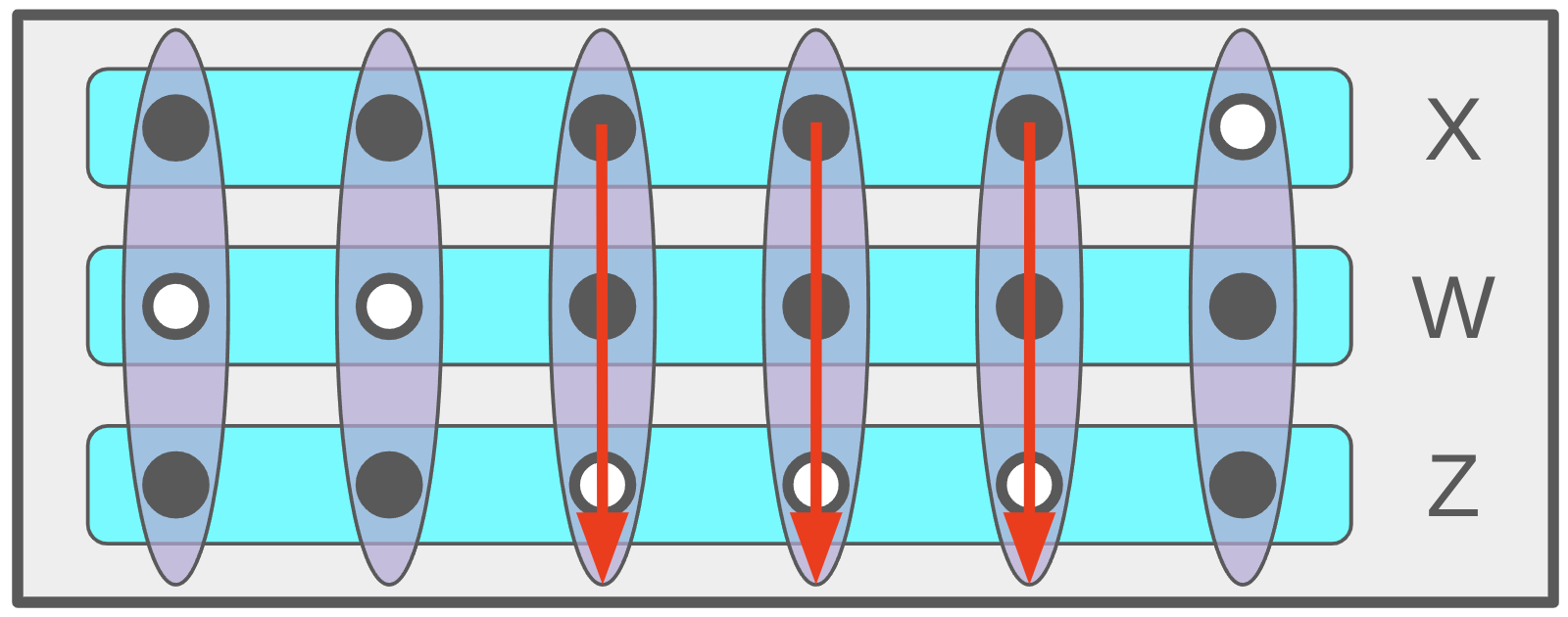}}\\
	\caption{TOM for operations $3$ and $4$ of $\prod$-net V1.}
	\label{fig:PNetTOM}
\end{wrapfigure}
The TEM for this compressed system of operations is given in \cref{fig:PNetTEM}, and the TOM for the third and fourth operations is given in \cref{fig:PNetTOM}. We conclude that $\cC_{op} = 4$, $\cC_T = 9$, $\cC_{\alpha} = 3$, and $\cC_O = 6$, while $\cC_A$ is still only $2$. It is interesting to observe that this architecture exhibits the same arity signature as self-attention. This is indicative of the larger patterns observed in this work, namely, that higher arity complexity operations exhibit promising performance characteristics.
\clearpage
\paragraph{MO-Nets.}
\begin{wrapfigure}[9]{r}{0.5\textwidth}
	\centering
	\raisebox{0pt}[\dimexpr\height-0.0\baselineskip\relax]{\includegraphics[width=0.5\textwidth]{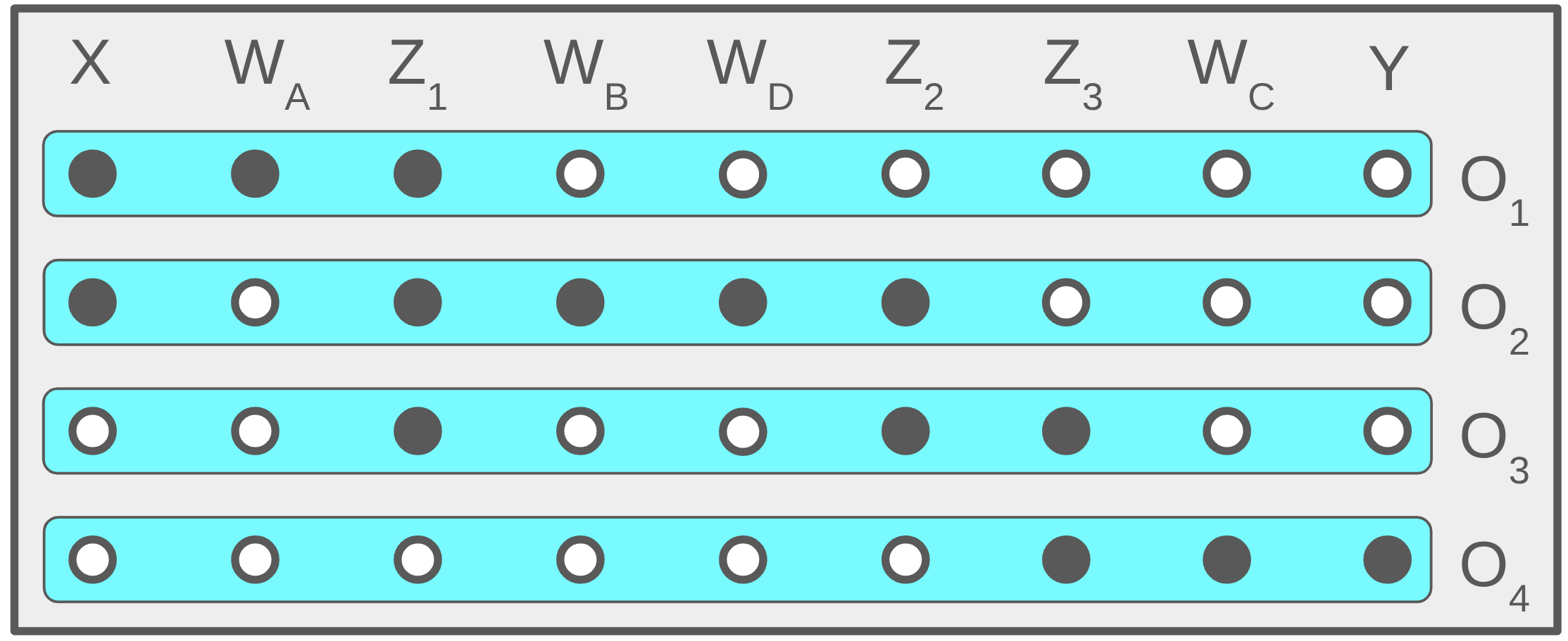}}\\
	\caption{TEM for MO-Nets.}
	\label{fig:MONetTEM}
\end{wrapfigure}
Next, we consider the multilinear operator networks (MO-Nets) of \citenump{MultiLinOpNets}. We focus our attention to the core building block defined by equation (1) of \citenump{MultiLinOpNets}. We use the tensor shapes provided in the original paper. This system of operations is given below:
\begin{equation*}
	\begin{aligned}
		Z_1 &= X \times W_A\\
		Z_2 &= X \times W_B \times W_D\\
		Z_3 &= Z_1 \odot Z_2\\
		Z_4 &= Z_3 \oplus Z_1\\
		Y &= Z_4 \times W_C
	\end{aligned}
\end{equation*}

\begin{wrapfigure}[10]{r}{0.4\textwidth}
	\centering
	\raisebox{0pt}[\dimexpr\height-1.5\baselineskip\relax]{\includegraphics[width=0.4\textwidth,height=0.25\textwidth]{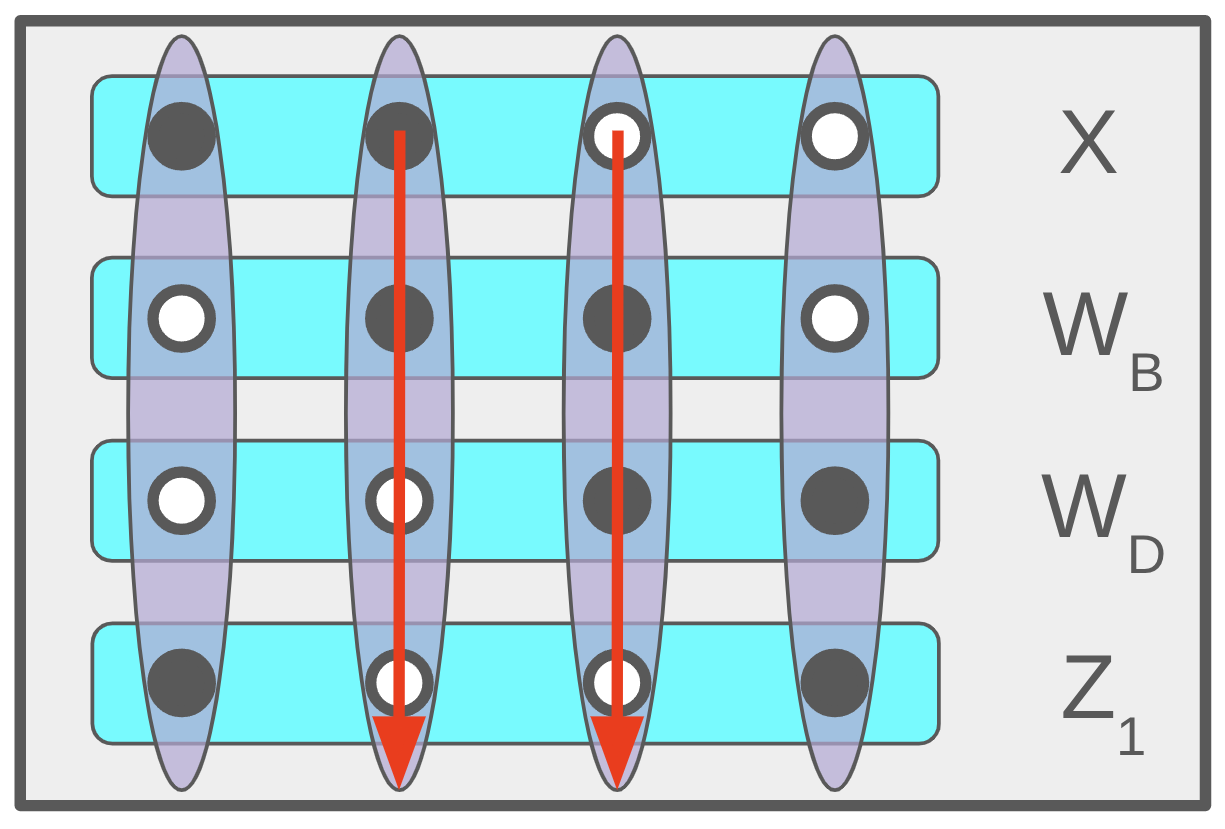}}\\
	\caption{TOM for operation $1$ of MO-Nets.}
	\label{fig:MONetTOM}
\end{wrapfigure}
Where $X$ has shape $(N, D)$, $W_A$ has shape $(D, M)$, $W_B$ has shape $(D, L)$, $W_D$ has shape $(L, M)$, and $C$ has shape $(M, O)$. Similarly to the analysis of Poly-Nets, we can combine the computation of $Z_2$ and $Z_3$. Nothing else can be simplified because the TOM for $Z_4$ uses different base operations. The simplified system of operations is given below:
\begin{equation*}
	\begin{aligned}
		Z_1 &= X \times W_A\\
		Z_2 &= Z_1 \odot (X \times W_B \times W_D)\\
		Z_3 &= Z_2 \oplus Z_1\\
		Y &= Z_3 \times W_C
	\end{aligned}
\end{equation*}
The TEM for MO-Nets is given in \cref{fig:MONetTEM} and the TOM for operation $1$ is given in \cref{fig:MONetTOM}. We conclude that $\cC_{op} = 4$, $\cC_T = 9$, $\cC_{\alpha} = 4$, $\cC_O = 4$, and $\cC_A = 2$. This is (to our knowledge) the first example of a core building block which has arity complexity $4$ despite still having coupling-arity $2$.

\paragraph{Vision Mamba.}
Next, we consider vision mamba (ViM) \citenump{VisionMamba}. This architecture is an adaptation of the mamba sequence model \citenump{Mamba} to image data. We focus our attention to the so-called ``ViM Block Process" described in Algorithm (1) of \citenump{VisionMamba}. For consistency with the other analyzed architectures, we count only a single iteration of the for-loop on line 19 (Algorithm 1, \citenump{VisionMamba}) towards the final complexity. This algorithm is very long and the analysis techniques used are identical to those used in the $3$ previous derivations, so we only highlight the more interesting observations here.

\begin{wrapfigure}[12]{r}{0.4\textwidth}
	\centering
	\raisebox{0pt}[\dimexpr\height-1.5\baselineskip\relax]{\includegraphics[width=0.4\textwidth]{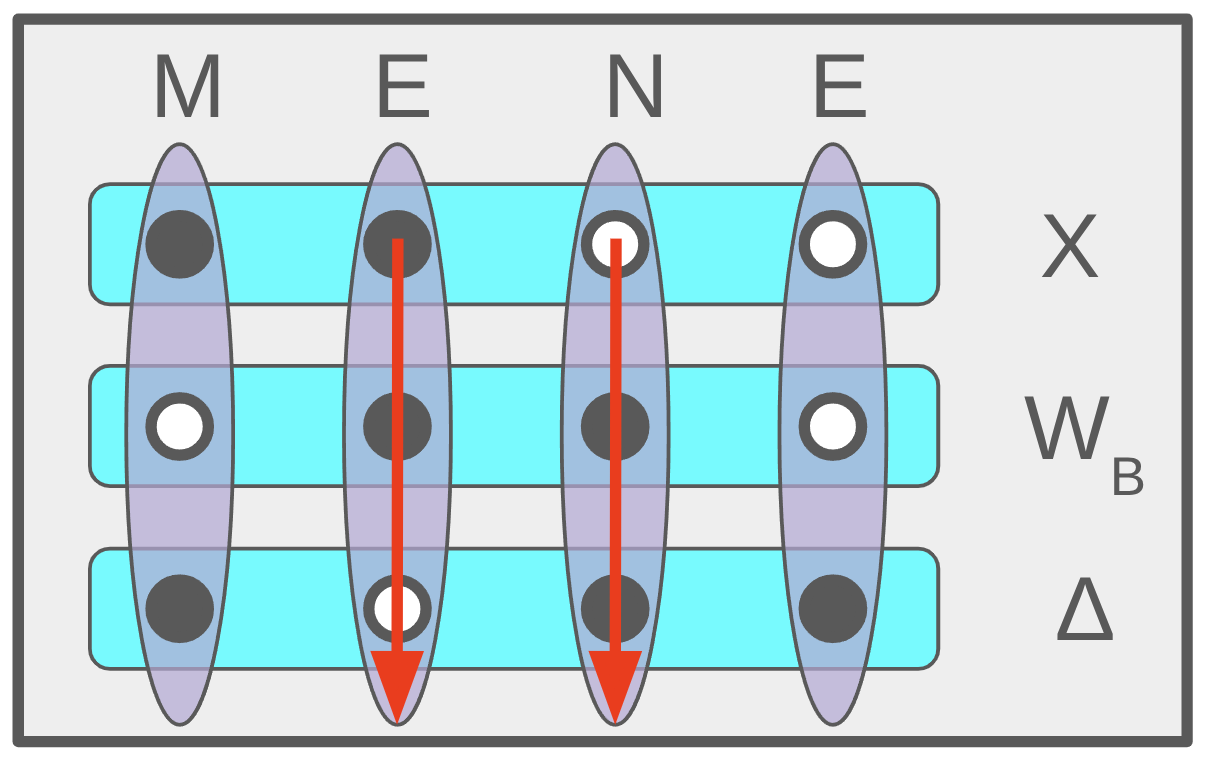}}\\
	\caption{Selected TOM from ViM.}
	\label{fig:VimTOM}
\end{wrapfigure}
First, the computation on line 8 (Algorithm 1, \citenump{VisionMamba}) can be combined with that of the second term on line 14 to produce a ternary operation. The TOM for this is shown in \cref{fig:VimTOM}. 

Second, the recurrent computation of $h$ can be tensorized by storing the value of $h$ after each loop iteration. This results in a tensor of shape $(M,E,N)$. Using this trick, we can combine the computation on line 9 with that of line 21 (Algorithm 1, \citenump{VisionMamba}) to produce another ternary operation. Intriguingly, the TOM for this operation is exactly the same as that of \cref{fig:VimTOM}, suggesting that this particular tensor operation is of foundational importance to the vision mamba architecture. We conclude that the ViM core block has $\cC_O = 4$ and $\cC_{\alpha} = 3$. Still, $\cC_A  = 2$.

The remainder of the complexity analysis for this core block is a routine counting exercise. The conclusion (leveraging the simplifications discussed above) is that a single ViM core block has $\cC_{op} = 27$, $\cC_T = 45$.

\clearpage

\paragraph{Deep Tensor Tree Networks.}
\begin{wrapfigure}[21]{r}{0.25\textwidth}
	\centering
	\raisebox{0pt}[\dimexpr\height-0.0\baselineskip\relax]{\includegraphics[width=0.25\textwidth]{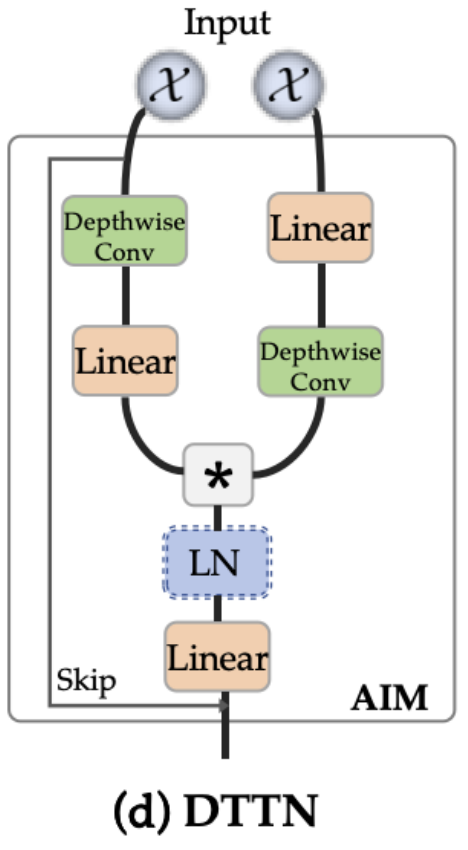}}\\
	\caption{Copy of Figure 2(d) from \citenump{DTTenNets} (reproduced with permission).}
	\label{fig:DTTNFig2d}
\end{wrapfigure}
Next, we consider the tensor tree networks (DTTNs) from \citenump{DTTenNets}. We focus our attention to the core building block of the model, the so-called asymmetric interaction module (AIM) described in Figure 2 (d) of \citenump{DTTenNets}. For the reader's convienence we have reproduced this figure in \cref{fig:DTTNFig2d}.

We start the analysis by computing the TOM for depthwise convolution. Recall that this is a variant of 2-dimensional convolution which uses separate spatial kernels for each image channel but does not change the number of channels nor contract the channel mode. The TOM is given in \cref{fig:DepthwiseConv2TOM}.

The next step in the analysis is to examine the two branches of the AIM. We recall that these branches are composed of depthwise convolution and matrix multiplication (along the channel mode). The difference is that the left branch performs depthwise convolution followed by matrix multiplication while the right branch swaps the order of these operations. This is important because the first branch can be expressed as a single tensor operation, whereas the second branch requires two operations. The reason for this asymmetric operation simplification is that non-trivial mode maps are only required for the depthwise convolution operations. Much like the previous analyses, we can integrate the Hadamard product which combines the two branches into either of the previous operations.

\begin{wrapfigure}[7]{r}{0.5\textwidth}
	\centering
	\raisebox{0pt}[\dimexpr\height-1.5\baselineskip\relax]{\includegraphics[width=0.5\textwidth]{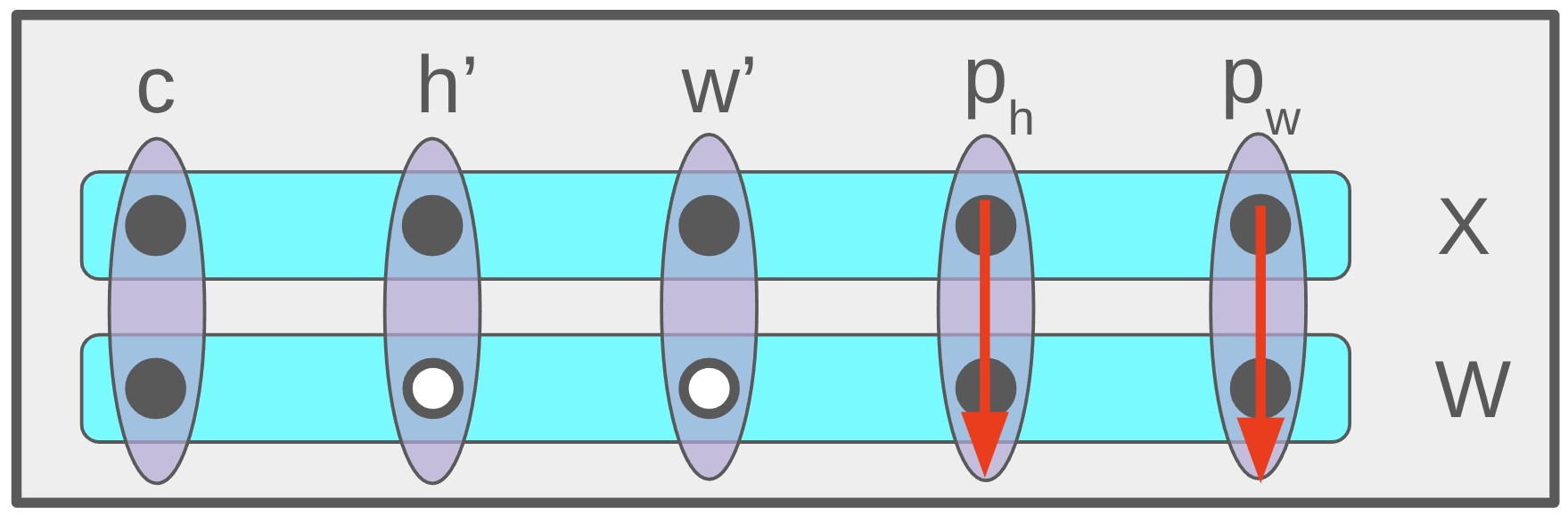}}\\
	\caption{TOM for (2-d) depthwise convolution.}
	\label{fig:DepthwiseConv2TOM}
\end{wrapfigure}
We integrate the Hadamard product into the right branch as this produces the lowest complexity encoding. The resulting simplest system of operations is given below:

\begin{equation*}
	\begin{aligned}
		Z_1 &= (X \circledast_D W_{C1}) \times_{Ch} W_{L1}\\
		Z_2 &= X \times_{Ch} W_{L2}\\
		Z_3 &= (Z_2 \circledast_D W_{C2}) \odot Z_1\\
		Z_4 &= Z_3 \times_{Ch} W_{L3}\\
		Y &= X \oplus Z_4
	\end{aligned}
\end{equation*}

\begin{wrapfigure}[11]{r}{0.5\textwidth}
	\centering
	\raisebox{0pt}[\dimexpr\height+0.5\baselineskip\relax]{\includegraphics[scale=0.0875]{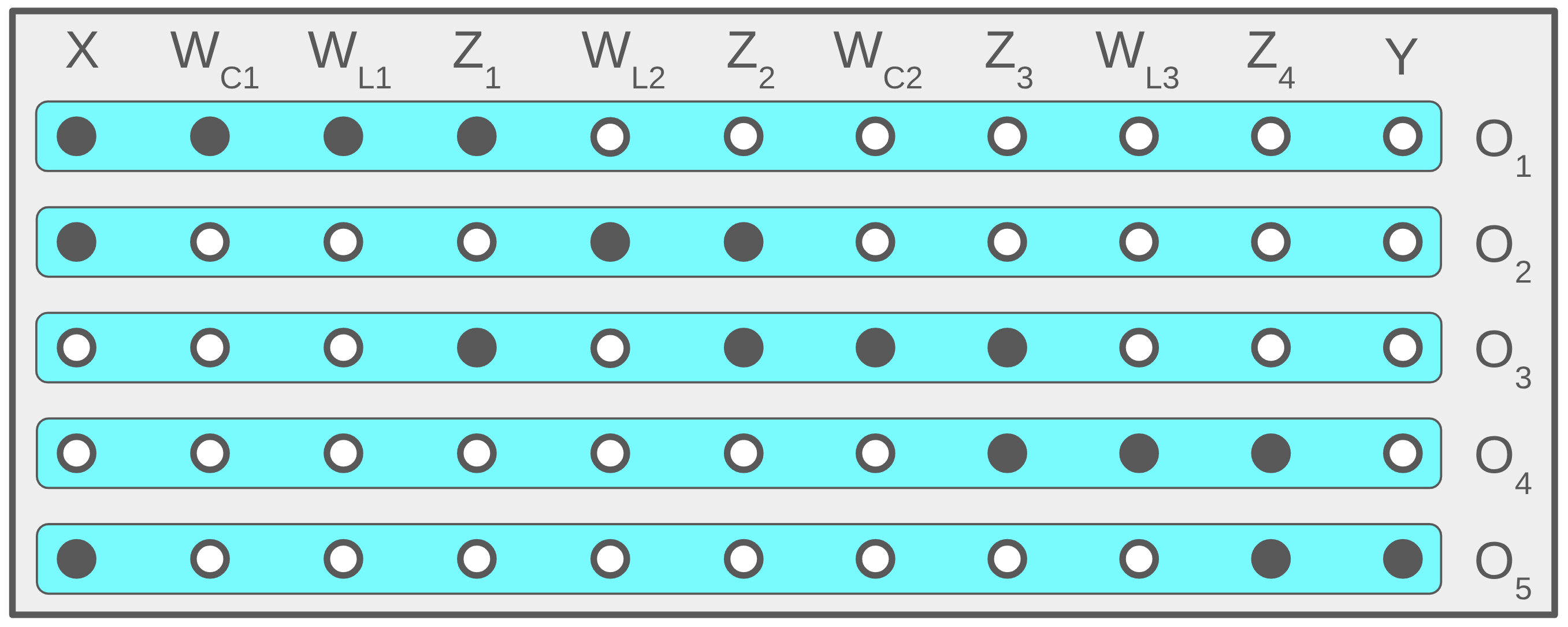}}\\
	\caption{TEM for the AIM of DTTNs.}
	\label{fig:DTTNTEM}
\end{wrapfigure}
Where $\circledast_D$ is depthwise convolution and $\times_{Ch}$ is matrix multiplication along the channel mode. The TEM is shown in \cref{fig:DTTNTEM} and the higher arity TOMs are shown in \cref{fig:DTTNTOMs}. These operations are particularly noteworthy because they are (to the best of our knowledge) the first examples of $\cC_A = 3$ operations used in a core block, at least with respect to the disambiguation assertions made in \cref{sec:Disambig}. We conclude that $\cC_{op} = 5$, $\cC_T = 11$, $\cC_{\alpha} = 3$, $\cC_O = 6$, and $\cC_A = 3$. 

\begin{figure}[b!]
	\centering
	\includegraphics[scale=0.115]{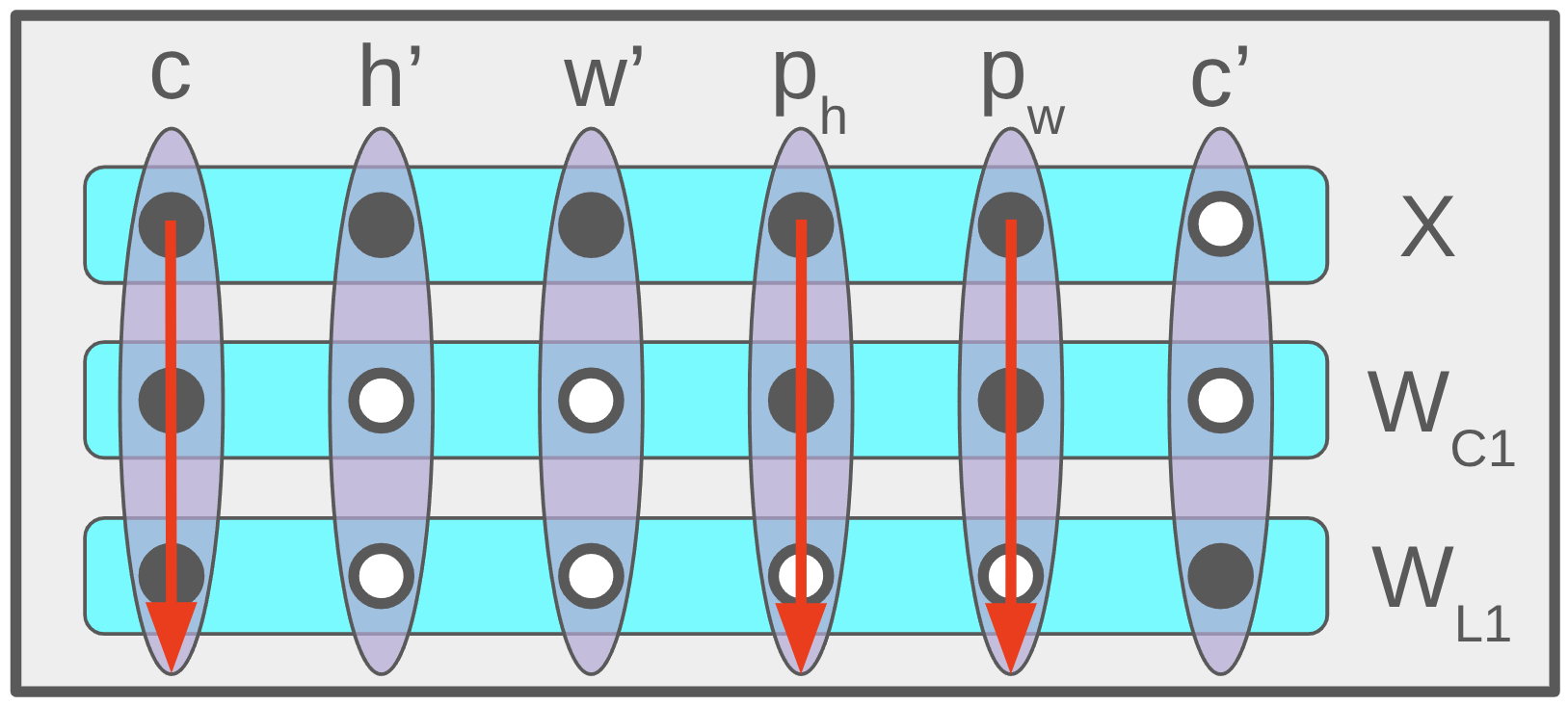}\hspace{0.75em}
	\includegraphics[scale=0.115]{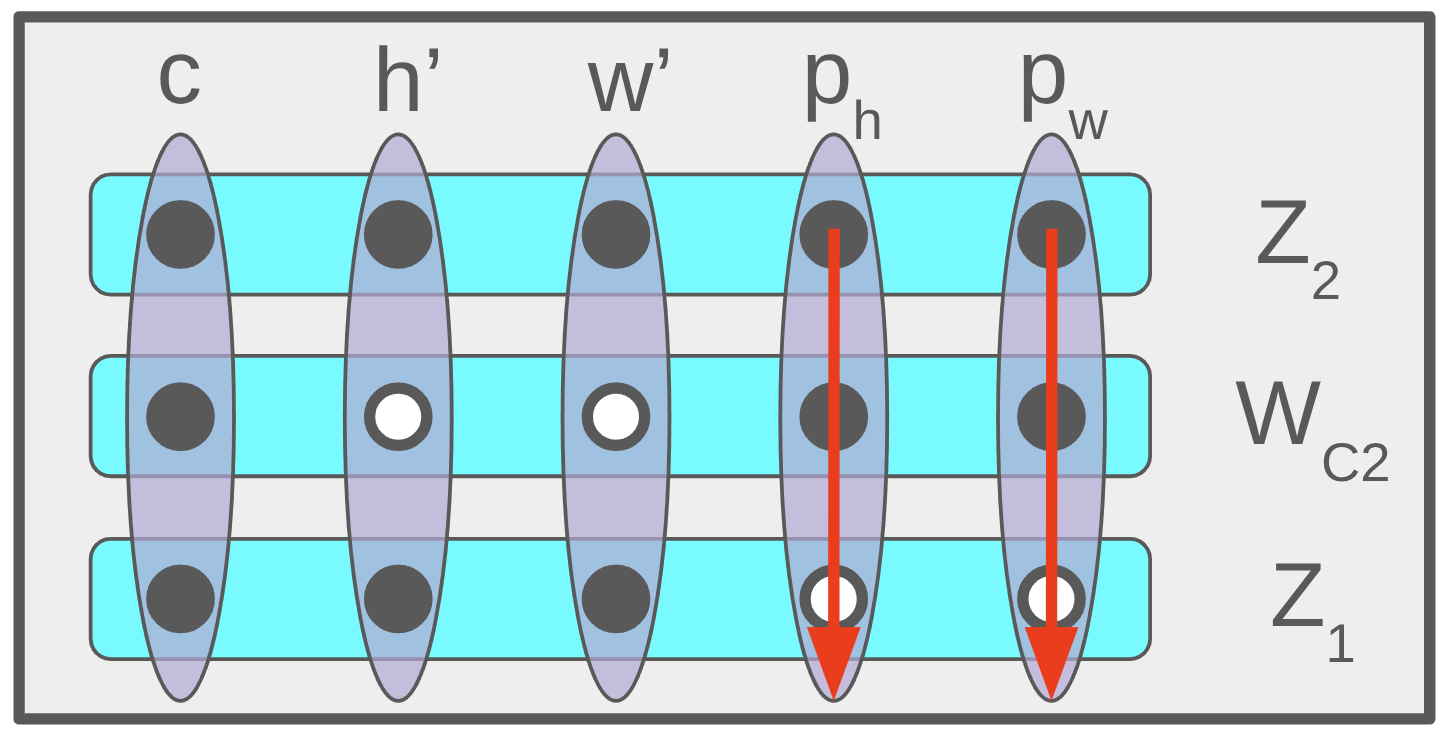}
	\caption{TOMs $1$ (left) and $3$ (right) for the AIM of DTTNs.}
	\label{fig:DTTNTOMs}
\end{figure}

\clearpage
\section{\label{sec:EmpiricalAppendix}Empirical Appendix}
In this appendix, we discuss in full detail the collection of the contributed dataset of novel architectures discussed in the main paper. 

In part 1, we describe how the dataset of architectures was collected and provide results from extra validation experiments.

In part 2, we provide statistics and information about the collected dataset. Included here are plots of all samples with respect to various different independent variables derived from the measures of architectural complexity introduced in the main paper.

In part 3, we provide full details of the highlighted ``red star'' architecture ({\color{red}$\bm{\star}$}).

\clearpage
\subsection{Dataset Collection\label{sec:DataCollection}} 
In this part, we provide complete information on how architectures were sampled, trained, and collected for our contributed dataset.

\subsubsection{\label{sec:DataCollection_Sampling}Details of Sampling Strategy}
We start by discussing the techniques used in the sampling of the architectures.

\paragraph{Stage 1 Modules.} For all settings, we include the tensor shape reduction at the end of each stage 1 module. Meaning, the first pooling layer is included.

For ResNet34, there are $3$ \textit{Basic Blocks} in the first stage. These are residual blocks of the complexity signature reported in \cref{tab:ComplexityHistory}. To mimic the design of the original model, we include an unfolding mode map prior to the sampled modules, and set the output shape to be the number of unfolded image patches, i.e., the channel dimension is removed. This results in the same \textit{ambient dimension} as the final linear layer of the complete ResNet34($\times 0.5$) architecture, which is $256$. We recall that the ambient dimension is the total dimensionality of latent space, e.g., a feature map of shape $64 \times 16 \times 16$ is of ambient dimension $16384$. Here, the ($\times 0.5$) indicates that the number of channels throughout the model is scaled by $0.5$.

For SWIN\_T, there are $2$ \textit{Swin Transformer Blocks} in the first stage. We do not include any non-trivial mode maps between the first stage and sampled modules. To mimic the design of the complete architecture, we set the output shape of the sampled modules to be the number of tokens (image patches) and average over the token dimension before the final linear layer. This results in a reduced final latent dimension of $192$ vs. the $768$ of the original model.

\paragraph{Width Scaling.} To produce the parameter vs. accuracy curves of \cref{fig:C100Results}, we scaled the hidden widths of ResNet34 and Swin\_T. Specifically, the number of channels for each stage of ResNet34 is multiplied by the same scalar. For example, the original hidden widths for ResNet34 are $(64, 128, 256, 512)$ for stages $1, 2, 3, 4$; these widths are modified proportionately, so ResNet34($\times 0.5$) uses $(32, 64, 128, 256)$ channels and so on. For Swin, an analogous width scaling strategy was applied. As Swin is a transformer, it does not have any meaningful ``channel dimension'' so we instead scale the embedding dimension used in the self attention blocks. Again, we scale this hidden dimension uniformly across all stages. 

\paragraph{TEM and TOM Sampling.} For each sampled architecture, we first select a number of operations (i.e., $\cC_{op}$) in between $2$ and $5$. Then, a TEM is sampled at random, with respect to an arity constraint of $2 \leq \cC_{\alpha} \leq 4$. Then, for each tensor operation, a TOM is sampled at random, with respect to an order complexity constraint of $\cC_O \leq 13$. The shape of the output tensor from each operation is sampled at random, by selecting integer factors from the input shape. As not all TOMs are compatible with the necessary tensor shapes, we re-sample TOMs until a compatible operation is found for each row of the TEM.

\paragraph{Non-Linear Activations.} For non-linear activation functions, a balanced coin flip is performed to determine if any will be applied after each tensor operation. For those operations selected to receive non-linear activation, a second balanced coin flip is performed to determine if a second non-linearity is applied. In other words, $50\%$ of the sampled tensor operations have no non-linear activation performed afterwards, whereas $25\%$ have a single activation, and the remaining $25\%$ have two activations. When activations are selected, they are sampled from a pool of: Leaky ReLU, ReLU6, LayerNorm, and Softmax. When LayerNorm or Softmax is sampled, a mode of the corresponding output tensor is sampled at random to determine which axis the activation is applied along.

\paragraph{Hardware Requirements.} We collected all samples using a workstation with $4$ Nvidia RTX $3090$ GPUs ($24$gb VRAM).
\clearpage
\subsubsection{\label{sec:DataCollection_DiagData}Diagnostic Data}
Here we describe the diagnostic data provided in the collected dataset.

For each sampled architecture, we record the TEM, TOMs, and non-linear activations used. Additionally, we record the total number of parameters and complete information from the training process. Specifically, we collect per-epoch cross-entropy loss and top-1 accuracy metrics on both the training and test sets. Trained model weights are also available for each sampled architecture.

\subsubsection{\label{sec:DataCollection_Optim}Optimization}
As the purpose of our contributed dataset is to collect information for the understanding on new tensor operation structures, we maximize the number of collected samples by training all models \textit{efficiently}. Specifically, we used the AdamW optimizer \citenump{Adam} and the OneCycle learning rate scheduler \citet{OneCycle}, as this scheduler dramatically reduces the training time required to achieve model convergence. Concretely, this optimization recipe allows all models to either convergence completely (most do), or achieve a high-degree of convergence in $50$ epochs. We therefore use this training duration for the dataset collection.

All models are trained with fixed hyperparameters for all datasets on the \textit{original resolution images}. We used a batch size of $128$ and set the maximum learning rate for the ResNet-based samples to $0.0075$. The Swin-based samples use a max learning rate of $0.001$. A standard data augmentation recipe was used, consisting of random offset crops (max offset of $4 \times 4$) and random horizontal flips (with probability $0.5$). All error ranges reported are the standard deviations computed over $3$ independent training runs from random initialization.

\paragraph{Training Duration Verification.} To verify that the trained models achieve a high-degree of converge in $50$ epochs, we provide the training curves for the original ResNet34 and Swin\_T architectures. Training configuration hash IDs are provided to facilitate cross-referencing the experimental results with the contributed codebase that we will release. Plots shown in \cref{fig:Diag1,fig:Diag2,fig:Diag3,fig:Diag4,fig:Diag5,fig:Diag6}.

We also verify that the collected dataset is indicative of results obtained using a longer training duration. Specifically, we spot-checked the dataset by training two randomly selected samples (one from either side of the baseline performance level) for $120$ epochs. The stage 1 baselines are re-trained in the same way. Results shown in \cref{tab:RN_Validation,tab:Swin_Validation}.

To validate the comparisons against MobileNetV2 (MNV2) made in the main paper, we also re-train the ({\color{red}$\bm{\star}$}) architecture and MNV2 for $120$ epochs. Results shown in \cref{tab:MNV2_Validation}, and diagnostic data for MNV2 is shown in \cref{fig:Diag7}.

\Cref{tab:RN_Validation,tab:Swin_Validation,tab:MNV2_Validation} and \cref{fig:Diag1,fig:Diag2,fig:Diag3,fig:Diag4,fig:Diag5,fig:Diag6,fig:Diag7} can be found on the following pages.

These validation experiments confirm that models do indeed achieve a high-degree of convergence when trained for $50$ epochs with the One Cycle scheduler. This is reflected in both the learning curves, and the low standard deviations of all baseline models trained for $50$ epochs. Moreover, we observe that the dataset samples are fully representative of the results obtained by switching to a longer training duration. Specifically, we observe \textit{zero} instances of a `bad' sample (meaning one that falls below baseline) becoming a `good' sample (meaning one that falls above baseline) when extra compute budget is introduced. Dually, there are no instances of `good' samples becoming `bad' when trained for $120$ epochs vs. $50$. We conclude that training models for $50$ epochs with the One Cycle scheduler is sufficient for producing a meaningful, diverse dataset for the understanding of tensor operation structure.

As expected, the ({\color{red}$\bm{\star}$}) architecture continues to outperform MobileNetV2 with the increased compute budget. Moreover, the diagnostic data of \cref{fig:Diag7} confirms that MNV2 has fully converged at $120$ epochs. These validation results clearly demonstrate that the ({\color{red}$\bm{\star}$}) architecture outperforms MNV2 with just $\sim5$ ``layers'' and a parameter count of 198,342. For additional context, the stage 1 baseline for this sample contains 152,000 parameters, meaning the sampled block of ({\color{red}$\bm{\star}$}) contains 46,342 parameters. Despite this small increase in model size, ({\color{red}$\bm{\star}$}) is $4.11\%$ above the baseline. A complete breakdown of this remarkably parameter efficient sample is provided in \cref{sec:RedStarArch}.

\clearpage
\begin{table}[h!]
	\centering
	\begin{tblr}{
			colspec={c|c|c||c|c|c||},
			}
			\SetCell[c=3]{c} Dataset ($\bm{\downarrow}$) / Architecture Type ($\bm{\rightarrow}$) & & & Baseline & $\Delta_+$ sample & $\Delta_-$ sample \\\hline
		\hline
		\SetCell[r=4]{c} CIFAR-10 & \SetCell[r=2]{c} $50$ Epochs & 
		Acc. ($\%$) & $88.36 \pm 0.19$ & $89.38 \pm 0.21$ & $87.38 \pm 0.88$ \\\cline{3-6}
		& & Delta & N/A & {\color{blue}$+1.02$} & {\color{red}$-0.98$} \\\cline{2-6}
		& \SetCell[r=2]{c} $120$ Epochs & 
		Acc. ($\%$) & $89.27 \pm 0.18$ & $90.39 \pm 0.28$ & $88.72 \pm 0.38$ \\\cline{3-6}
		& & Delta & N/A & {\color{blue}$+1.12$} & {\color{red}$-0.55$} \\\cline{1-6}
		\hline
		\SetCell[r=4]{c} CIFAR-100 & \SetCell[r=2]{c} $50$ Epochs & 
		Acc. ($\%$) & $60.58 \pm 0.24$ & $61.70 \pm 0.35$ & $57.02 \pm 0.33$ \\\cline{3-6}
		& & Delta & N/A & {\color{blue}$+1.12$} & {\color{red}$-3.56$} \\\cline{2-6}
		& \SetCell[r=2]{c} $120$ Epochs & 
		Acc. ($\%$) & $62.21 \pm 0.16$ & $64.52 \pm 0.16$ & $60.23 \pm 0.13$ \\\cline{3-6}
		& & Delta & N/A & {\color{blue}$+2.31$} & {\color{red}$-1.98$} \\\cline{1-6}
		\hline
		\SetCell[r=4]{c} Tiny Imagenet & \SetCell[r=2]{c} $50$ Epochs & 
		Acc. ($\%$) & $43.93 \pm 0.01$ & $46.75 \pm 0.97$ & $43.42 \pm 0.43$ \\\cline{3-6}
		& & Delta & N/A & {\color{blue}$+2.82$} & {\color{red}$-0.51$} \\\cline{2-6}
		& \SetCell[r=2]{c} $120$ Epochs & 
		Acc. ($\%$) & $45.68 \pm 0.32$ & $47.31 \pm 0.15$ & $45.59 \pm 0.23$ \\\cline{3-6}
		& & Delta & N/A & {\color{blue}$+1.63$} & {\color{red}$-0.09$} \\\cline{1-6}
		\hline
	\end{tblr}
	\vspace{0.5em}
	\caption{Training duration validation for the ResNet34-based samples. The $\Delta_+$ and $\Delta_-$ samples are selected at random from the architecture dataset. Deltas are relative to the stage 1 baseline.}
	\label{tab:RN_Validation}
\end{table}

\begin{table}[h!]
	\centering
	\begin{tblr}{
			colspec={c|c|c||c|c|c||},
		}
		\SetCell[c=3]{c} Dataset ($\bm{\downarrow}$) / Architecture Type ($\bm{\rightarrow}$) & & & Baseline & $\Delta_+$ sample & $\Delta_-$ sample \\\hline
		\hline
		\SetCell[r=4]{c} CIFAR-10 & \SetCell[r=2]{c} $50$ Epochs & 
		Acc. ($\%$) & $74.91 \pm 0.29$ & $80.87 \pm 0.81$ & $71.60 \pm 1.52$ \\\cline{3-6}
		& & Delta & N/A & {\color{blue}$+5.96$} & {\color{red}$-3.31$} \\\cline{2-6}
		& \SetCell[r=2]{c} $120$ Epochs & 
		Acc. ($\%$) & $79.16 \pm 0.22$ & $82.93 \pm 0.18$ & $76.78 \pm 2.73$ \\\cline{3-6}
		& & Delta & N/A & {\color{blue}$+3.77$} & {\color{red}$-2.38$} \\\cline{1-6}
		\hline
		\SetCell[r=4]{c} CIFAR-100 & \SetCell[r=2]{c} $50$ Epochs & 
		Acc. ($\%$) & $47.07 \pm 0.21$ & $54.95 \pm 0.38$ & $41.83 \pm 0.66$ \\\cline{3-6}
		& & Delta & N/A & {\color{blue}$+7.88$} & {\color{red}$-5.24$} \\\cline{2-6}
		& \SetCell[r=2]{c} $120$ Epochs & 
		Acc. ($\%$) & $52.75 \pm 0.29$ & $57.68 \pm 0.01$ & $46.81 \pm 0.39$ \\\cline{3-6}
		& & Delta & N/A & {\color{blue}$+4.93$} & {\color{red}$-5.94$} \\\cline{1-6}
		\hline
		\SetCell[r=4]{c} Tiny Imagenet & \SetCell[r=2]{c} $50$ Epochs & 
		Acc. ($\%$) & $37.10 \pm 0.20$ & $42.00 \pm 0.13$ & $33.65 \pm 0.26$ \\\cline{3-6}
		& & Delta & N/A & {\color{blue}$+4.90$} & {\color{red}$-3.45$} \\\cline{2-6}
		& \SetCell[r=2]{c} $120$ Epochs & 
		Acc. ($\%$) & $40.52 \pm 0.27$ & $43.81 \pm 0.21$ & $36.96 \pm 0.16$ \\\cline{3-6}
		& & Delta & N/A & {\color{blue}$+3.29$} & {\color{red}$-3.56$} \\\cline{1-6}
		\hline
	\end{tblr}
	\vspace{0.5em}
	\caption{Training duration validation for the Swin\_T-based samples. The $\Delta_+$ and $\Delta_-$ samples are selected at random from the architecture dataset. Deltas are relative to the stage 1 baseline.}
	\label{tab:Swin_Validation}
\end{table}

\newcommand{\DiagScale}{0.0905}

\begin{figure}[h!]
	\centering
	\includegraphics[scale=\DiagScale]{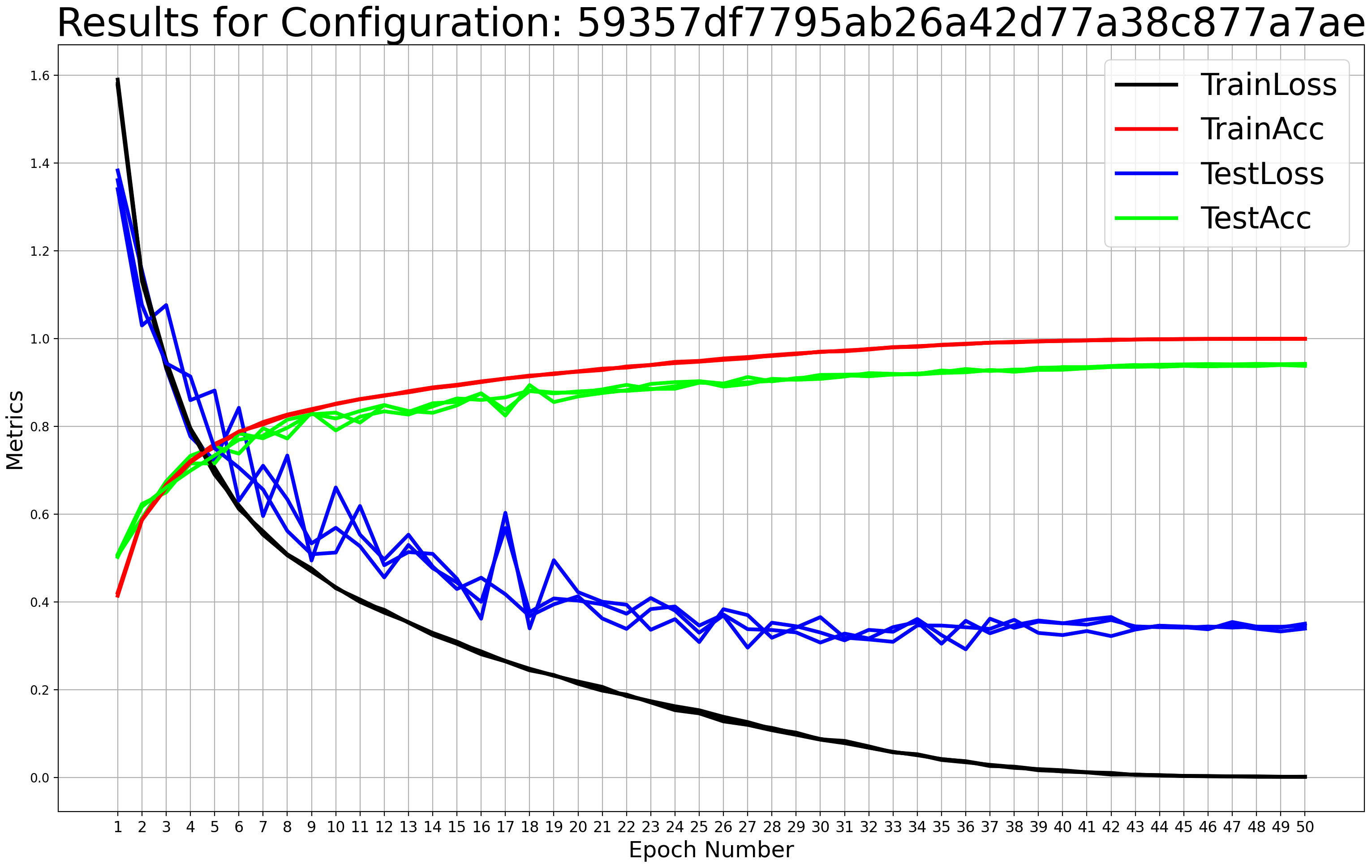}
	\caption{Diagnostic data for ResNet34 trained for $50$ epochs on CIFAR-10.}
	\label{fig:Diag1}
\end{figure}

\begin{figure}[h!]
	\centering
	\includegraphics[scale=\DiagScale]{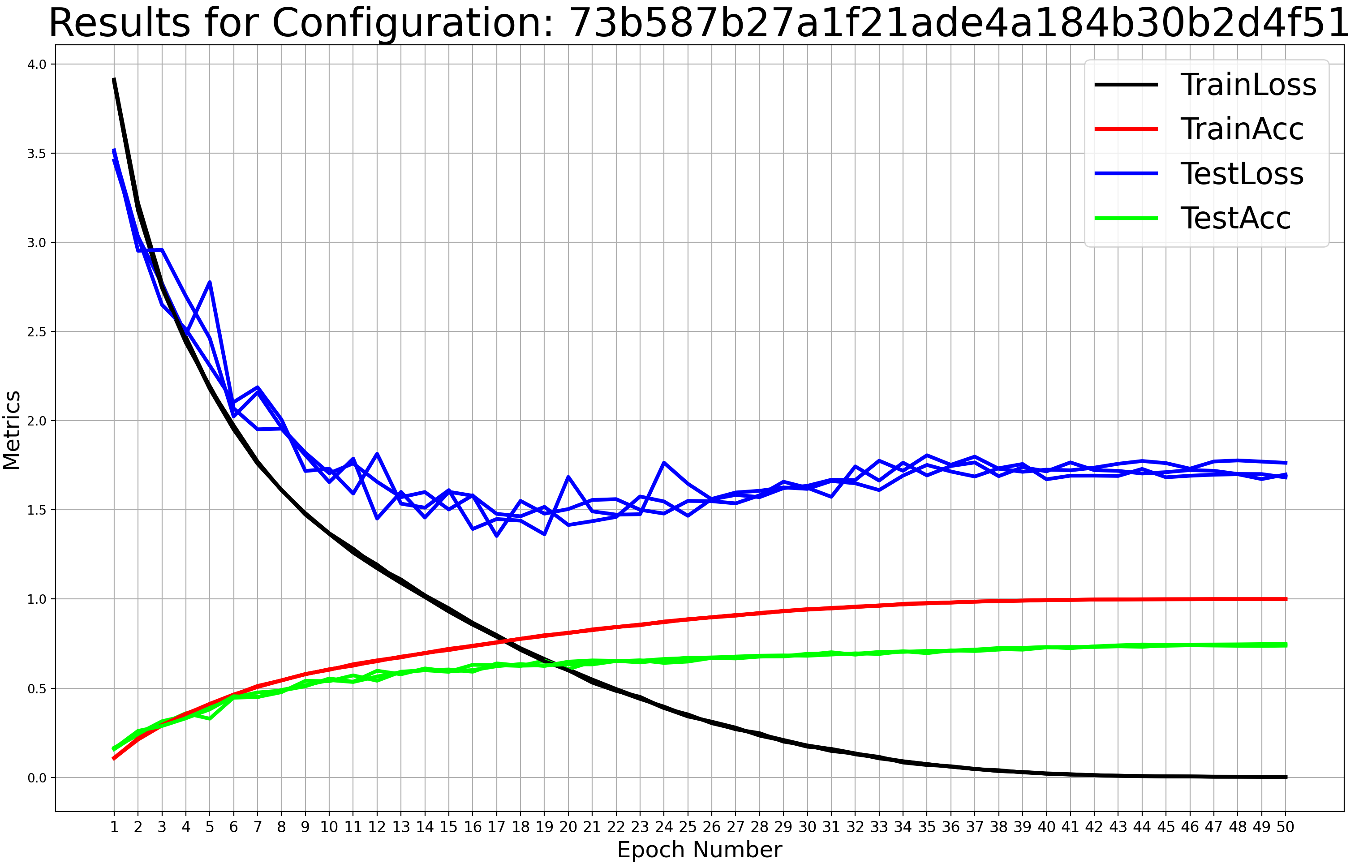}
	\caption{Diagnostic data for ResNet34 trained for $50$ epochs on CIFAR-100.}
	\label{fig:Diag2}
\end{figure}

\begin{figure}[h!]
	\centering
	\includegraphics[scale=\DiagScale]{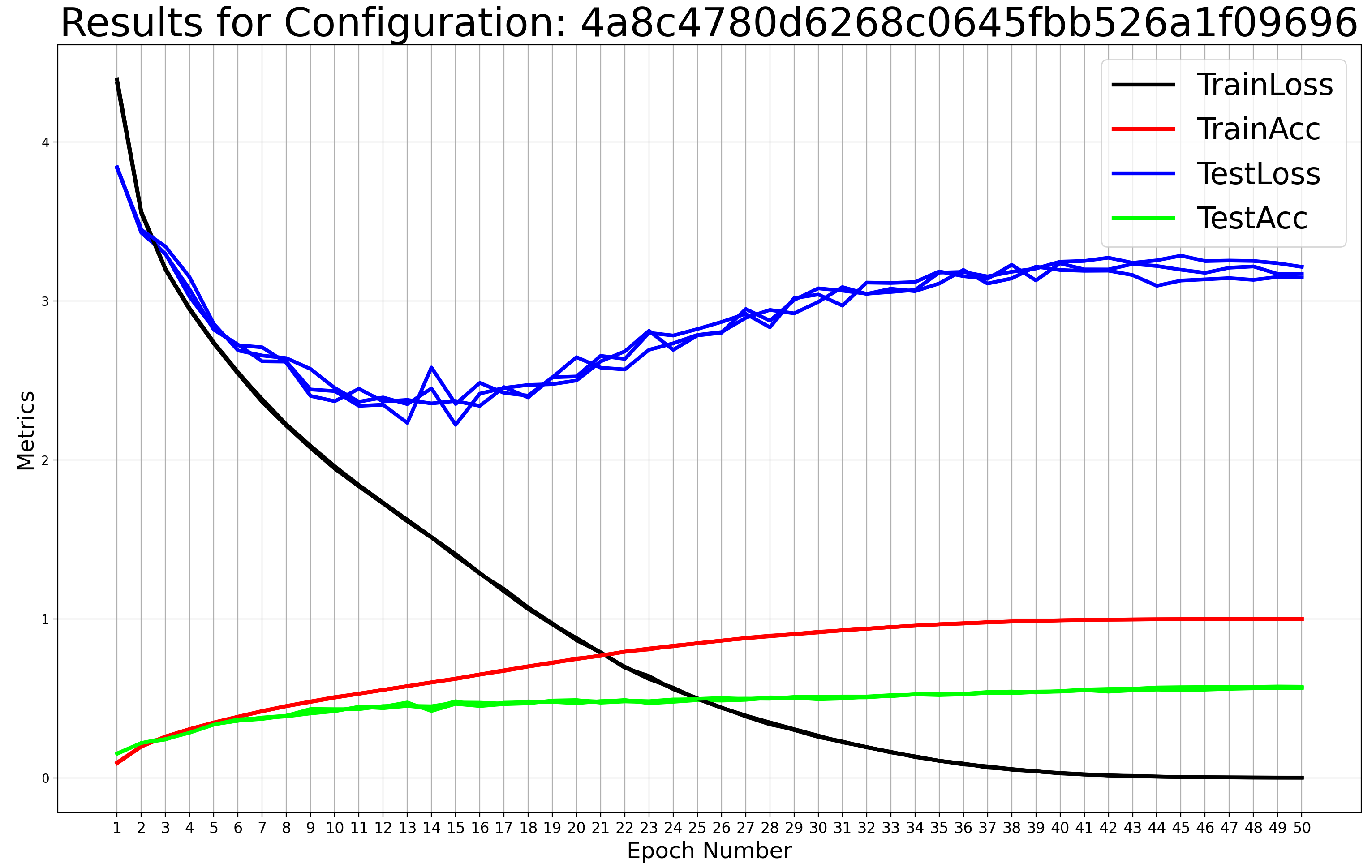}
	\caption{Diagnostic data for ResNet34 trained for $50$ epochs on Tiny Imagenet.}
	\label{fig:Diag3}
\end{figure}

\begin{figure}[h!]
	\centering
	\includegraphics[scale=\DiagScale]{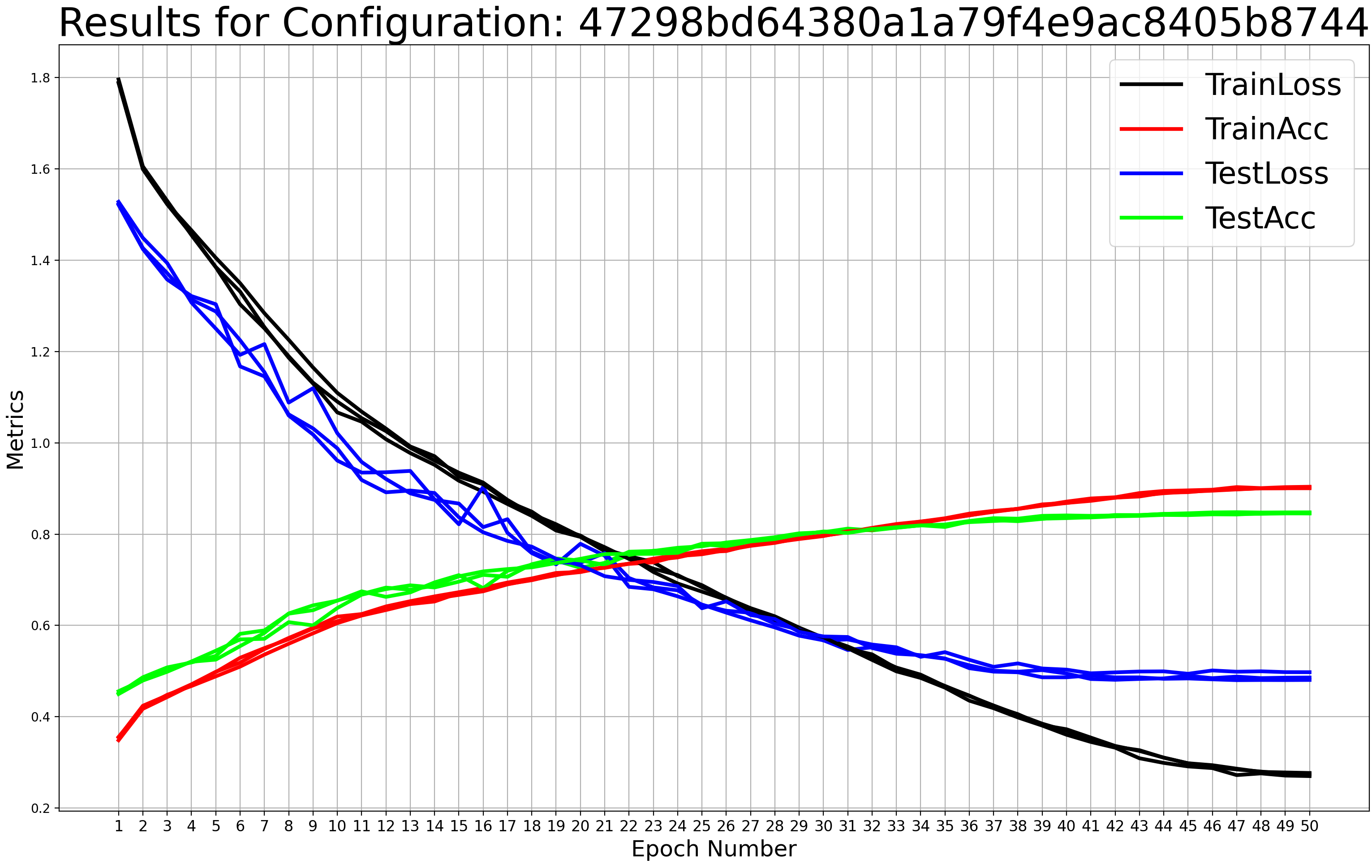}
	\caption{Diagnostic data for Swin\_T trained for $50$ epochs on CIFAR-10.}
	\label{fig:Diag4}
\end{figure}

\begin{figure}[h!]
	\centering
	\includegraphics[scale=\DiagScale]{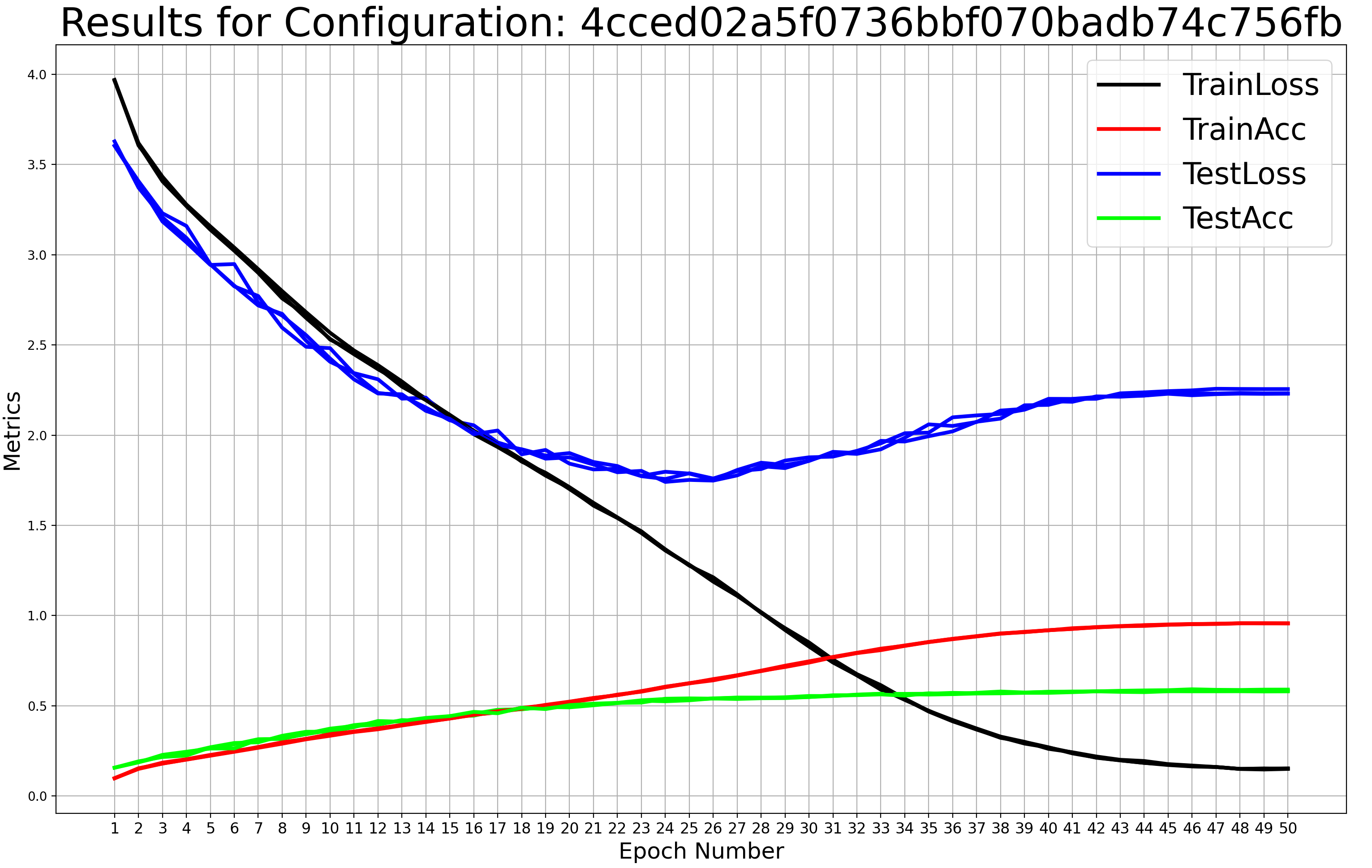}
	\caption{Diagnostic data for Swin\_T trained for $50$ epochs on CIFAR-100.}
	\label{fig:Diag5}
\end{figure}

\begin{figure}[h!]
	\centering
	\includegraphics[scale=\DiagScale]{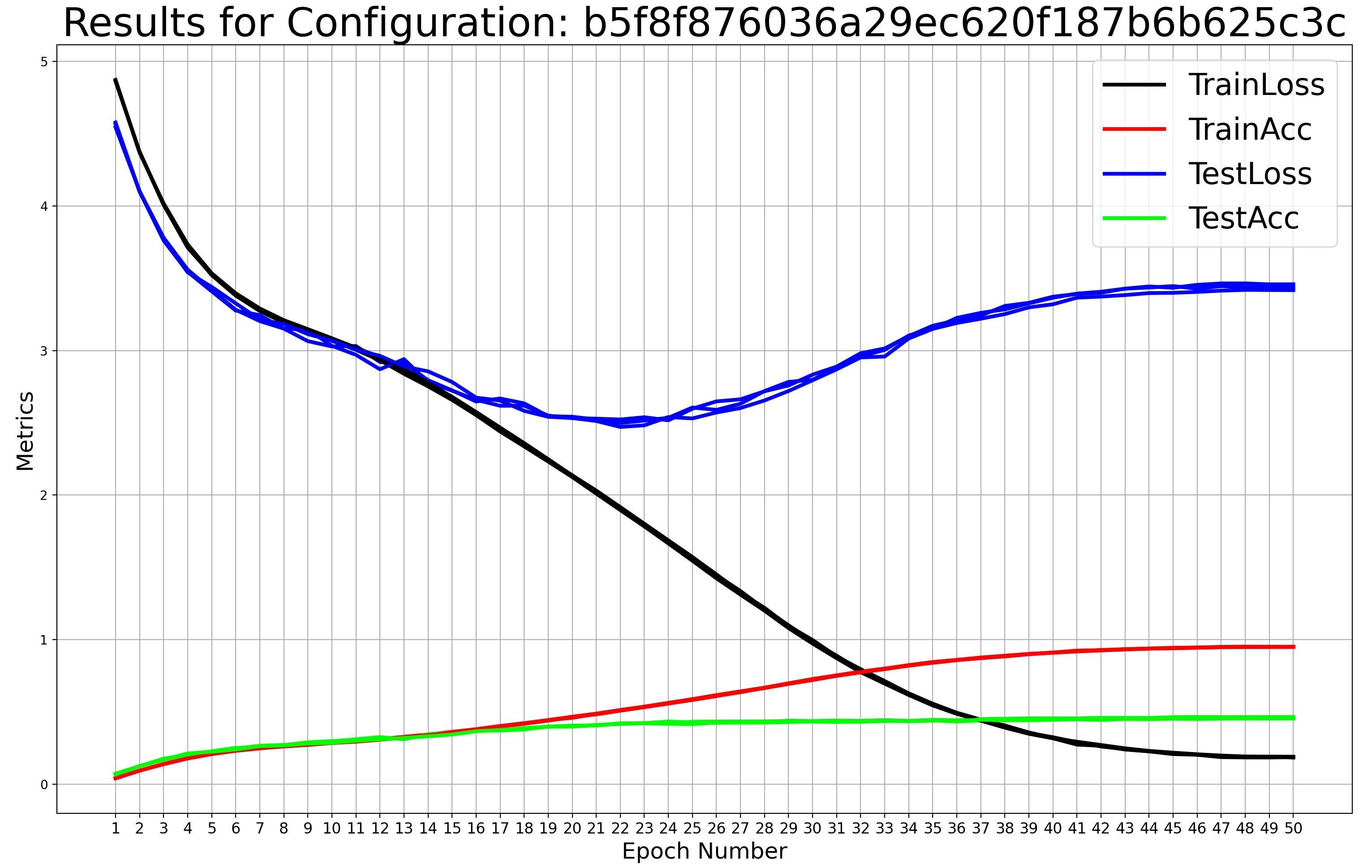}
	\caption{Diagnostic data for Swin\_T trained for $50$ epochs on Tiny Imagenet.}
	\label{fig:Diag6}
\end{figure}

\clearpage
\begin{table}[h!]
	\centering
	\begin{tblr}{
			colspec={c|c|c||c|c||},
		}
		\SetCell[c=3]{c} Dataset ($\bm{\downarrow}$) / Architecture Type ($\bm{\rightarrow}$) & & & MobileNetV2 & ({\color{red}$\bm{\star}$}) sample \\\hline
		\hline
		\SetCell[r=4]{c} CIFAR-100 & \SetCell[r=2]{c} $50$ Epochs & 
		Acc. ($\%$) & $64.29\pm 0.04$ & $65.52\pm 0.22$ \\\cline{3-5}
		& & Delta & N/A & {\color{blue}$+1.23$} \\\cline{2-5}
		& \SetCell[r=2]{c} $120$ Epochs & 
		Acc. ($\%$) & $65.69 \pm 0.32$ & $66.32 \pm 0.15$ \\\cline{3-5}
		& & Delta & N/A & {\color{blue}$+0.63$} \\\cline{1-5}
		\hline
	\end{tblr}
	\vspace{0.5em}
	\caption{Training duration validation for the MobileNetV2 vs. ({\color{red}$\bm{\star}$}) architecture comparison. Deltas are relative to MNV2.}
	\label{tab:MNV2_Validation}
\end{table}
\begin{figure}[h!]
	\centering
	\includegraphics[scale=\DiagScale]{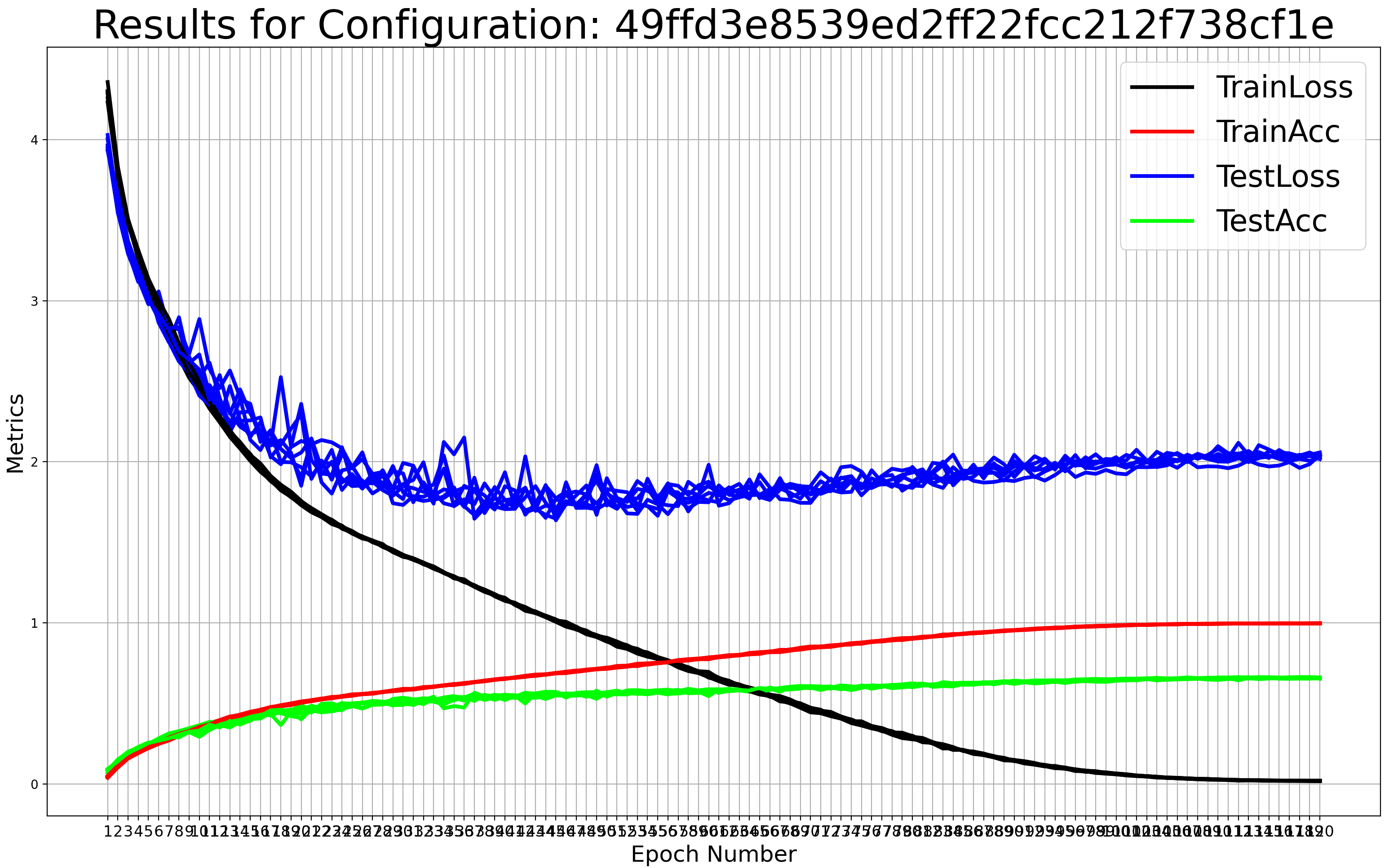}
	\caption{Diagnostic data for MobileNetV2 trained for $120$ epochs on CIFAR-100.}
	\label{fig:Diag7}
\end{figure}

\clearpage
\subsection{\label{sec:DataResults}Complete Dataset Results}
In this part, we first provide plots in the style of \cref{fig:C100Results} for all datasets. Then, we provide breakdowns of the number of samples collected along with a suite of figures visualizing the data with respect to various architectural complexity measures.

\newcommand{\PAScale}{0.1175}

\subsubsection{\label{sec:DataResults_PA}Parameter Efficiency of the Sampled Architectures.}
\vspace{-0.5em}The following figures are in the style of \cref{fig:C100Results}.
\begin{figure}[h!]
	\centering
	\includegraphics[scale=\PAScale]{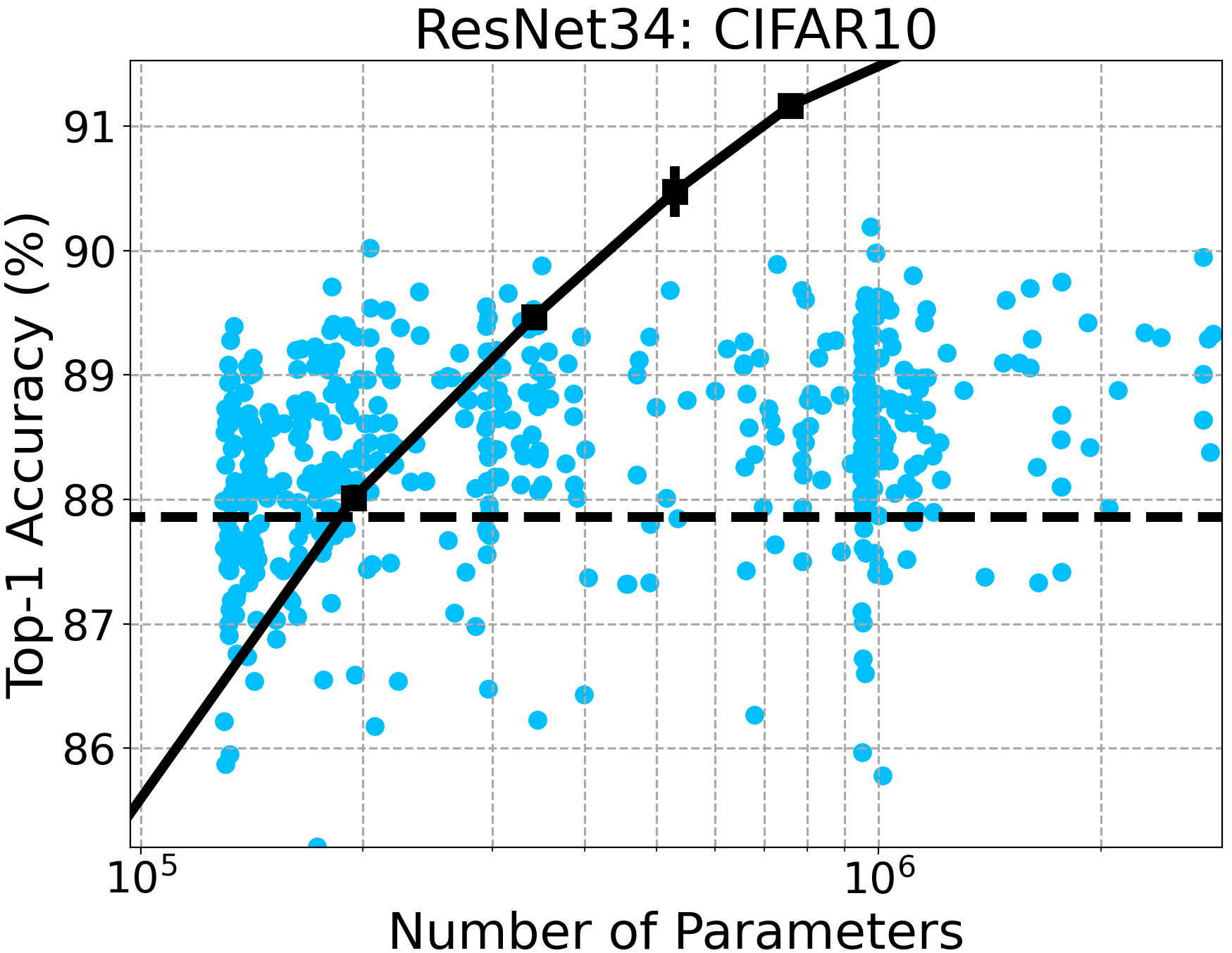}
\end{figure}

\begin{figure}[h!]
	\centering
	\includegraphics[scale=\PAScale]{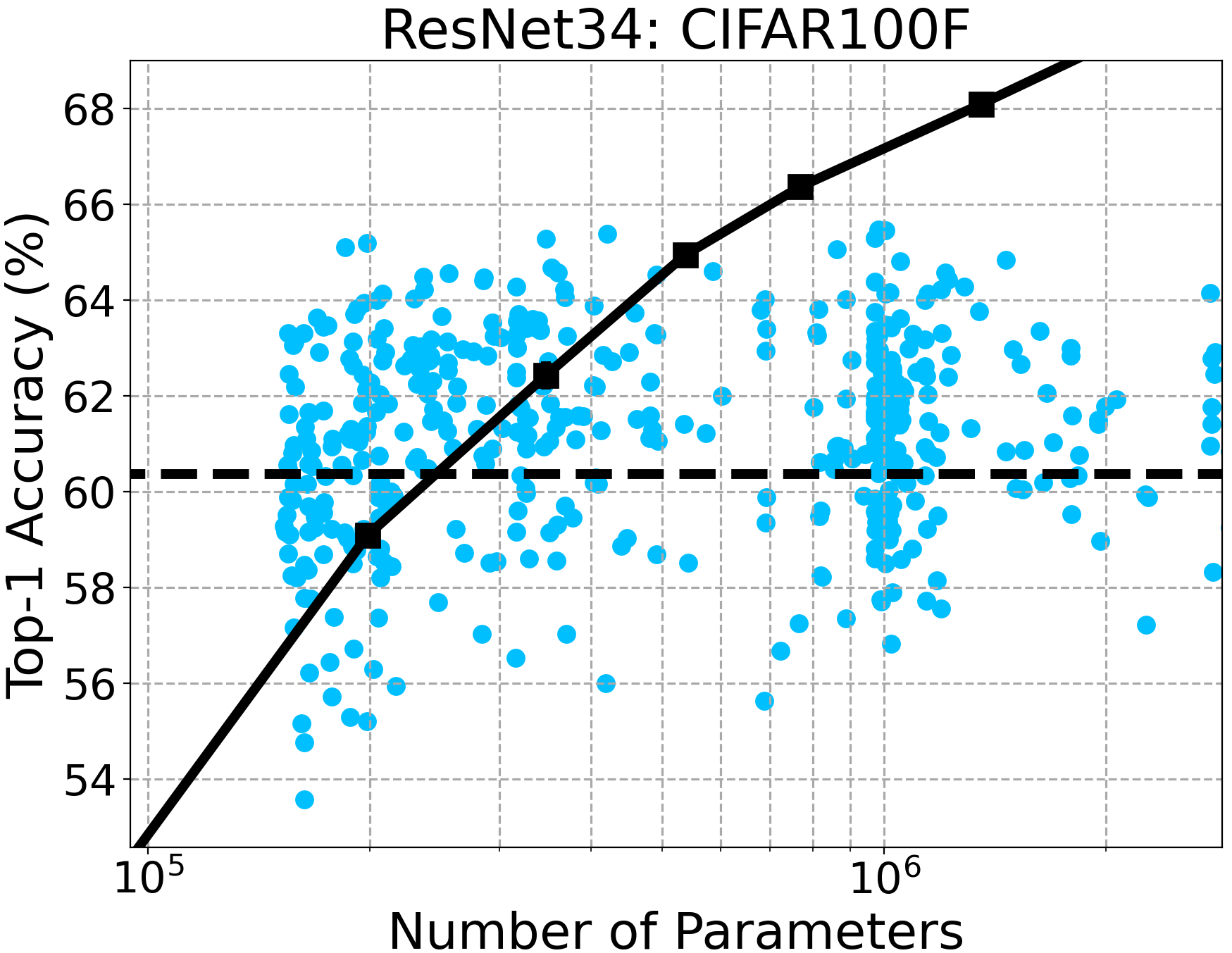}
\end{figure}

\begin{figure}[h!]
	\centering
	\includegraphics[scale=\PAScale]{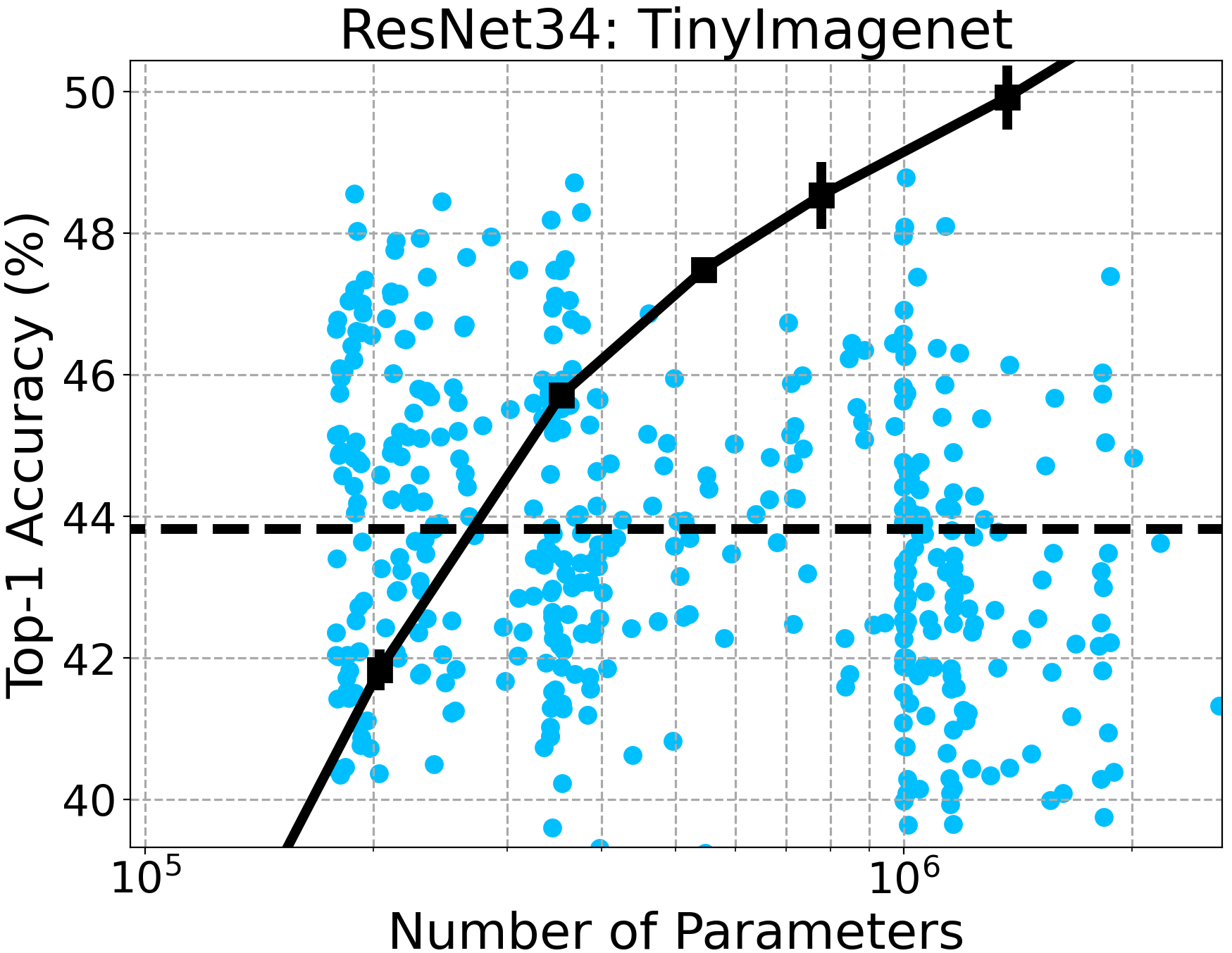}
\end{figure}

\begin{figure}[h!]
	\centering
	\includegraphics[scale=\PAScale]{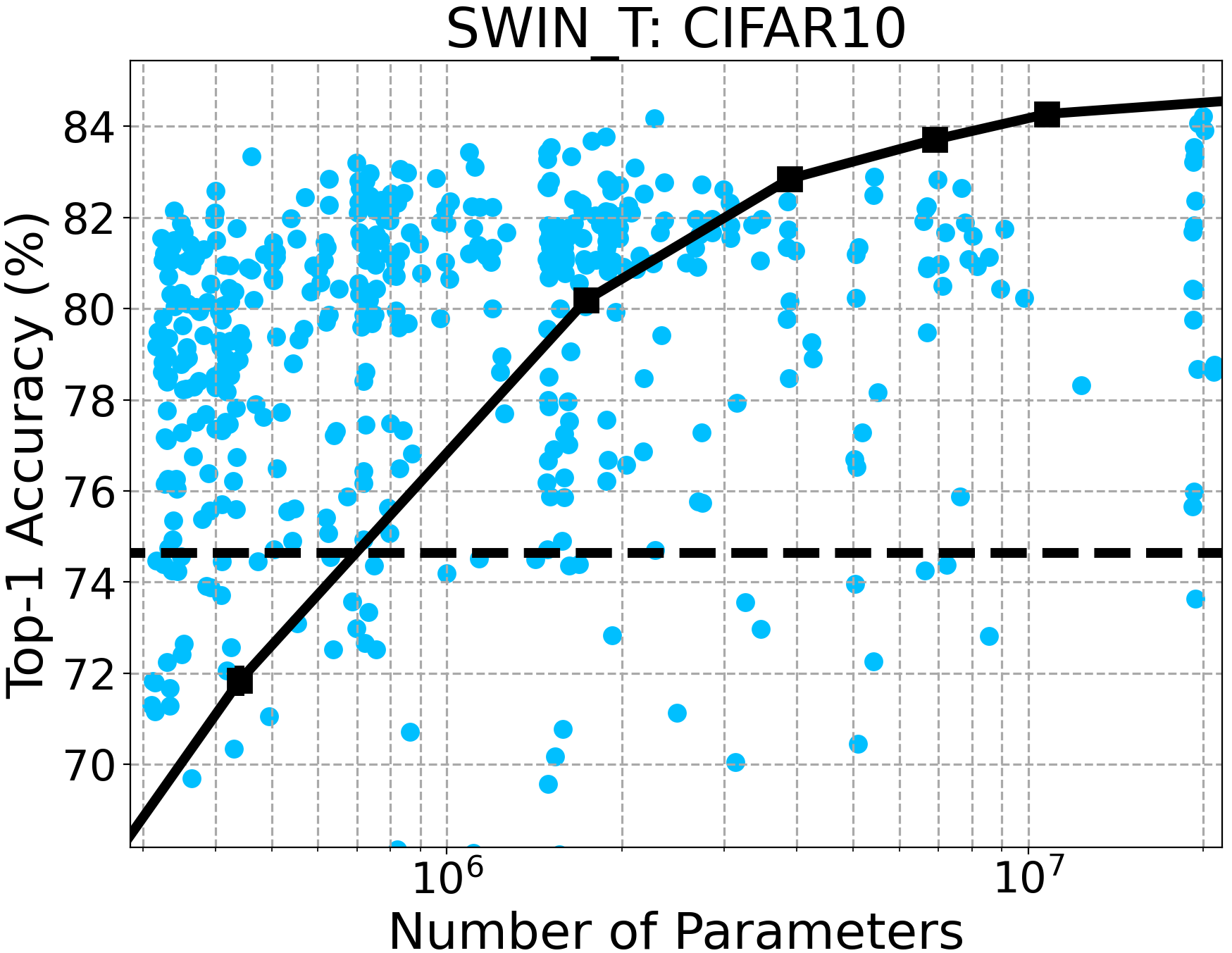}
\end{figure}

\begin{figure}[h!]
	\centering
	\includegraphics[scale=\PAScale]{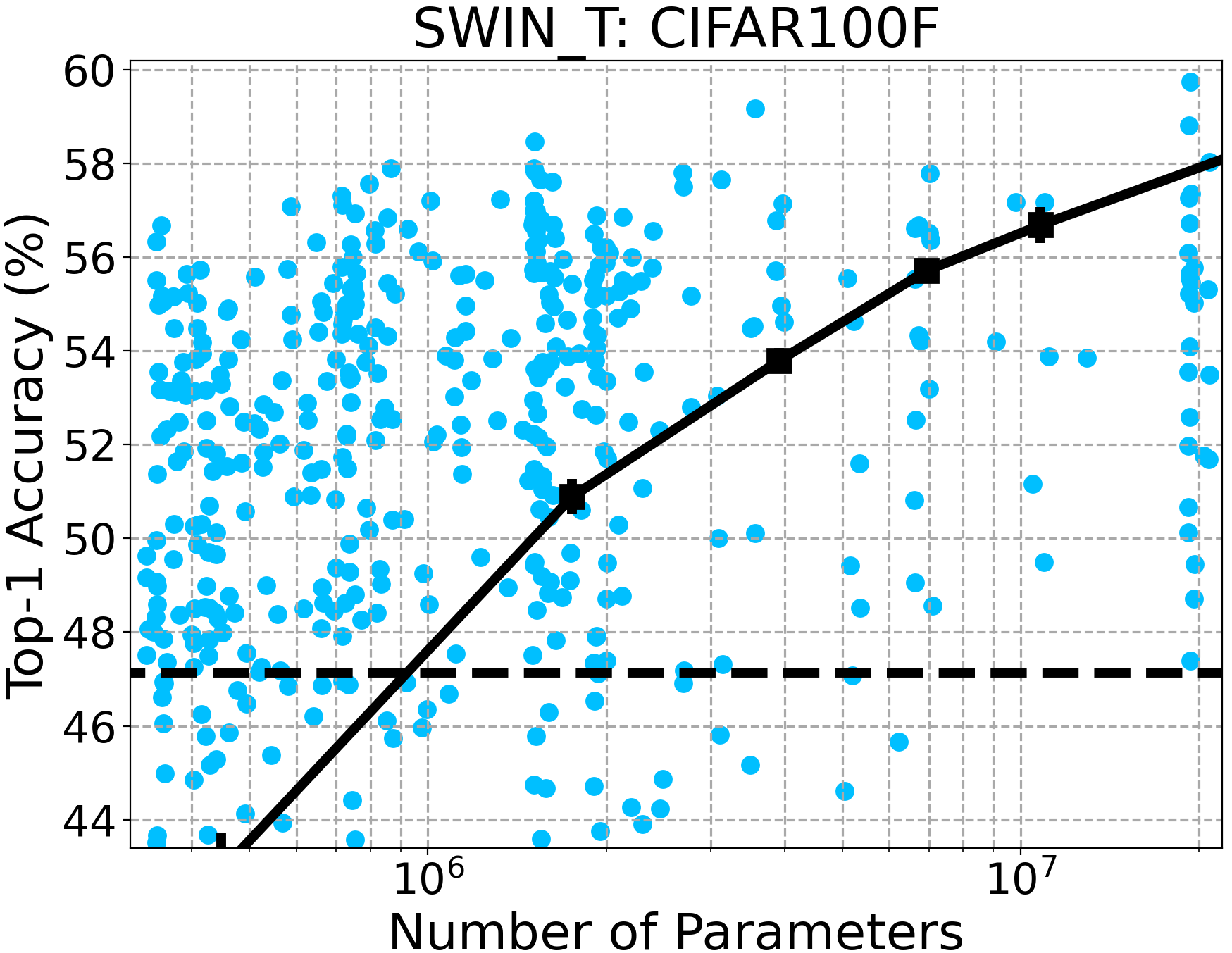}
\end{figure}

\begin{figure}[h!]
	\centering
	\includegraphics[scale=\PAScale]{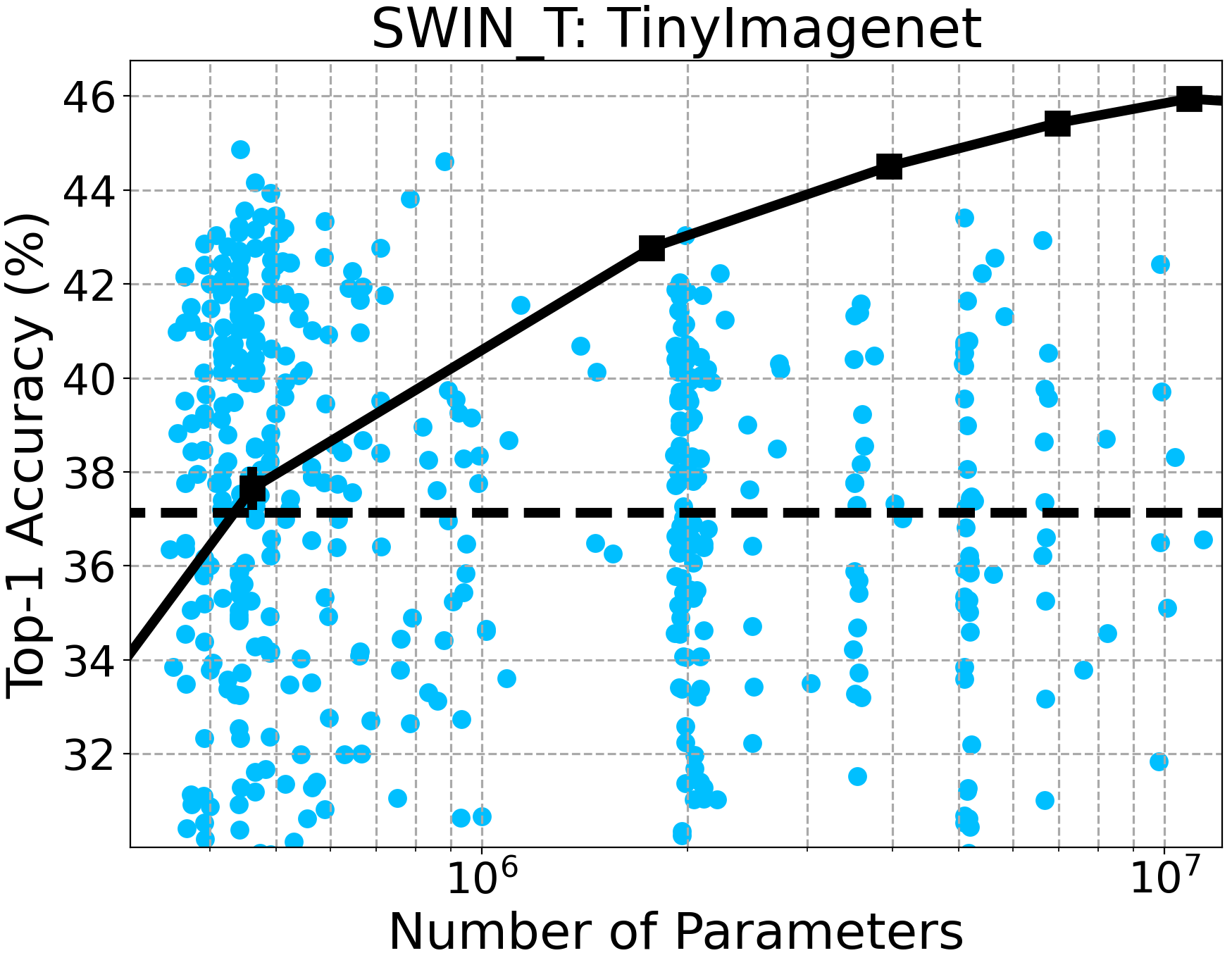}
\end{figure}

\paragraph{Discussion.} The complete results mirror the findings from CIFAR-100 reported in the main paper. To provide additional context, when trained with the same hyperparameters as the ResNet-based sampled models, MobileNetV2 attains accuracy scores of: $91.04\% \pm 0.14$, $64.29\% \pm 0.04$, and $45.49\% \pm 0.65$ on CIFAR-10, CIFAR-100, and Tiny Imagenet, respectively.

Overall, the results on CIFAR-10 and Tiny Imagenet support the conclusions articulated in the main paper. We observe that the sampled architectures compare best against both baselines on Tiny Imagenet. Additionally, the sampled blocks appear to be more impactful in the context of self attention operations, as shown by the Swin results.

\clearpage
\subsubsection{\label{sec:DataResults_Stats}Statistics of the Sampled Architectures.}
We now provide visualizations and breakdowns of the number of architectures sampled with respect to all the complexity measures introduced in the main paper.

\newcommand{\VsXScale}{0.16}

\begin{figure}[h!]
	\centering
	\includegraphics[scale=\VsXScale]{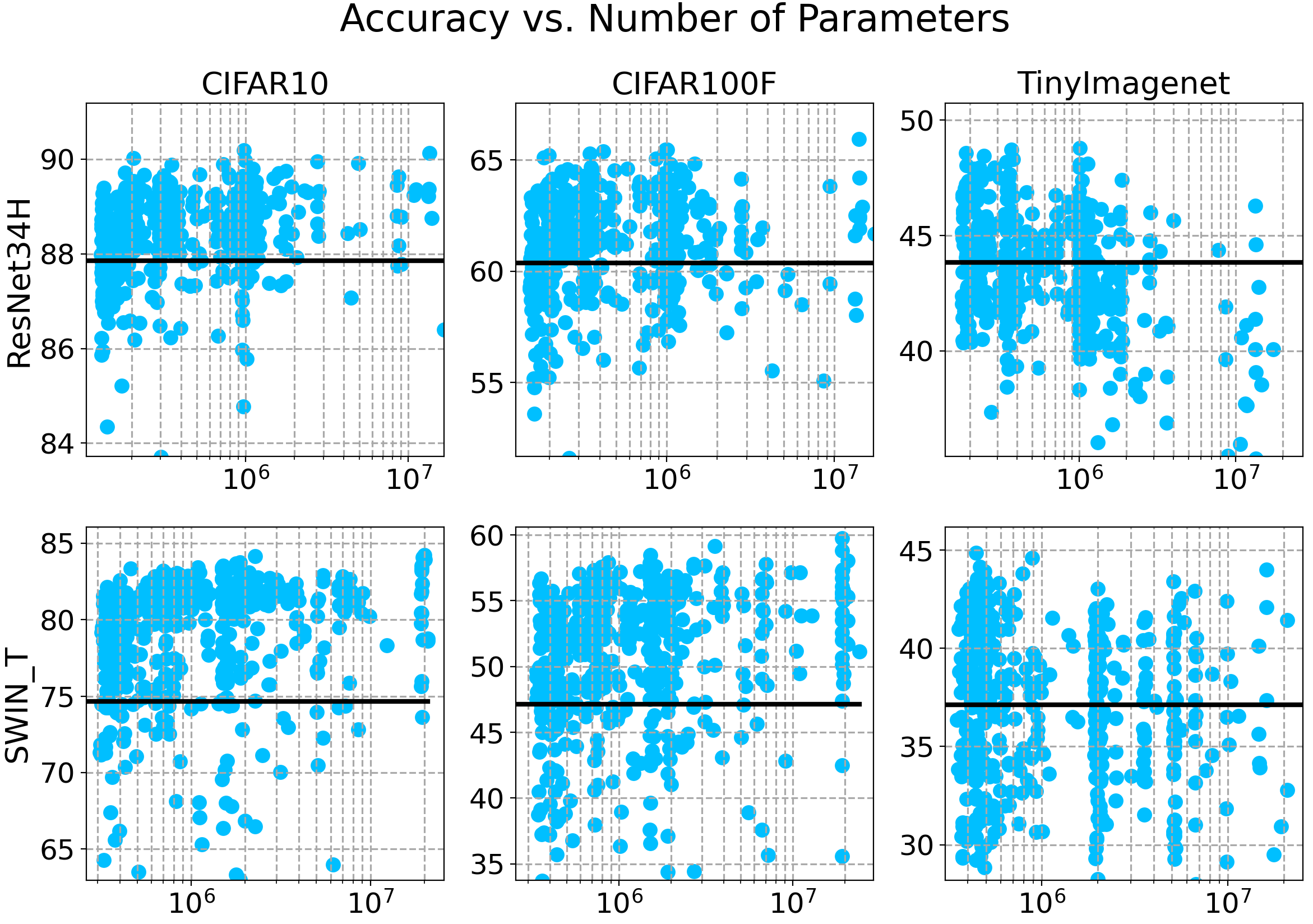}
	\caption{Distribution of sampled architectures ({\color{cyan}$\bullet$}) vs. number of parameters of the sampled block. (\textbf{---}) Stage 1 baseline accuracy.}
\end{figure}
\begin{table*}[h!]
	\centering
	\begin{tblr}{
			colspec={c||c|c||c|c||c|c||},
			column{1}={halign=r}
		}
		Dataset & \SetCell[c=2]{c} CIFAR 10 & & \SetCell[c=2]{c} CIFAR 100 & & \SetCell[c=2]{c} Tiny ImageNet & \\ \hline
		Stage 1 & ResNet34 & SWIN\_T & ResNet34 & SWIN\_T & ResNet34 & SWIN\_T \\ \hline\hline
		
		Total & 526 & 507 & 502 & 513 & 486 & 494 \\ \hline
		
	\end{tblr}
	\caption{Number of sampled architectures with respect to dataset and stage 1 model.}
	\label{tab:DatasetVsSettings}
\end{table*}

\clearpage

\begin{figure}[h!]
	\centering
	\includegraphics[scale=\VsXScale]{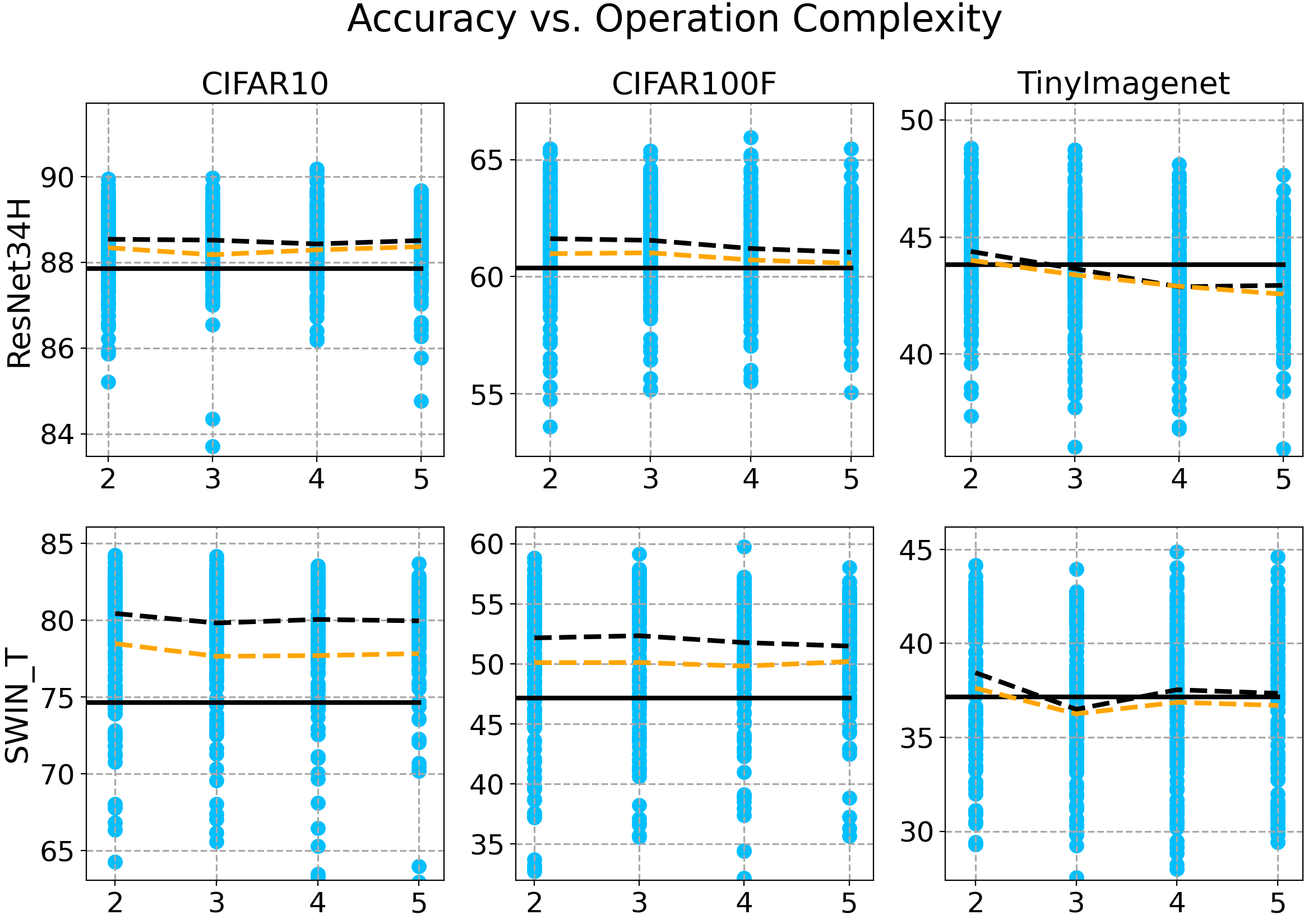}
	\caption{Distribution of sampled architectures ({\color{cyan}$\bullet$}) vs. operation complexity $\cC_{op}$ of the sampled block. ({\color{orange}$\bm{- -}$}) Mean sample accuracy. ({$\bm{- -}$}) Median sample accuracy. (\textbf{---}) Stage 1 baseline accuracy.}
\end{figure}
\begin{table*}[h!]
	\centering
	\begin{tblr}{
			colspec={c||c|c||c|c||c|c||c},
			column{1}={halign=r}
		}
		Dataset & \SetCell[c=2]{c} CIFAR 10 & & \SetCell[c=2]{c} CIFAR 100 & & \SetCell[c=2]{c} Tiny ImageNet & & \\\hline
		Stage 1 & ResNet34 & SWIN\_T & ResNet34 & SWIN\_T & ResNet34 & SWIN\_T\ & Total \\\hline\hline
		
		$\cC_{op} = 2$ & 162 & 153 & 151 & 145 & 129 & 141 & 881\\\hline
		$\cC_{op} = 3$ & 131 & 122 & 127 & 124 & 127 & 128 & 759 \\\hline
		$\cC_{op} = 4$ & 132 & 125 & 124 & 125 & 129 & 123 & 758 \\\hline
		$\cC_{op} = 5$ & 101 & 107 & 100 & 119 & 101 & 102 & 630 \\\hline
		
	\end{tblr}
	\caption{Number of sampled architectures vs. operation complexity ($\cC_{op}$) of the sampled block.}
	\label{tab:DatasetVsNumOps}
\end{table*}

\begin{figure}[h!]
	\centering
	\includegraphics[scale=\VsXScale]{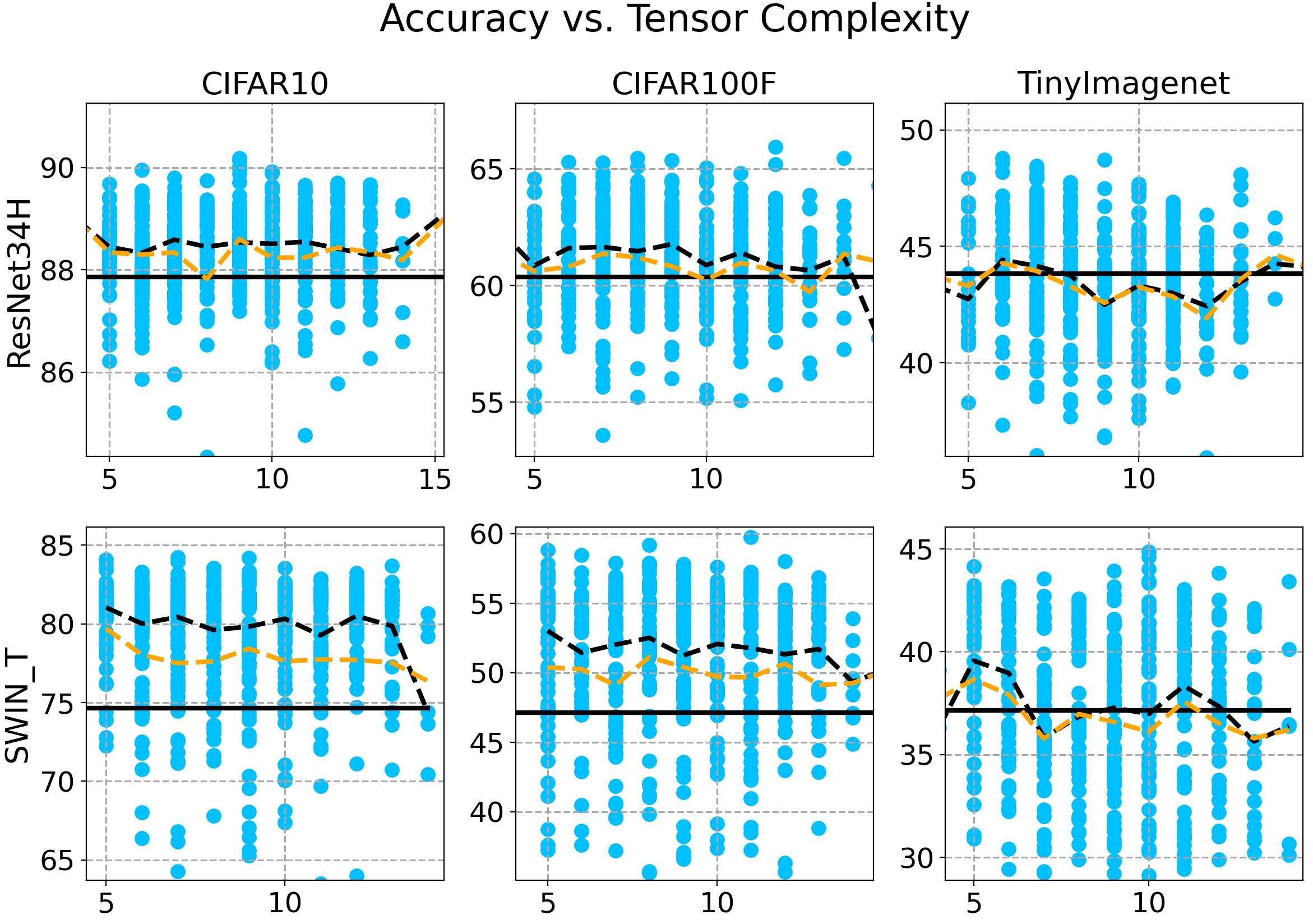}
	\caption{Distribution of sampled architectures ({\color{cyan}$\bullet$}) vs. tensor complexity $\cC_{T}$ of the sampled block. ({\color{orange}$\bm{- -}$}) Mean sample accuracy. ({$\bm{- -}$}) Median sample accuracy. (\textbf{---}) Stage 1 baseline accuracy.}
\end{figure}
\begin{table*}[h!]
	\centering
	\begin{tblr}{
			colspec={c||c|c||c|c||c|c||c},
			column{1}={halign=r}
		}
		Dataset & \SetCell[c=2]{c} CIFAR 10 & & \SetCell[c=2]{c} CIFAR 100 & & \SetCell[c=2]{c} Tiny ImageNet & & \\\hline
		Stage 1 & ResNet34 & SWIN\_T & ResNet34 & SWIN\_T & ResNet34 & SWIN\_T\ & Total \\\hline\hline
		
		$\cC_{T} = 4$  & 1   & \ES & 3   & \ES & 6   & 2   & 12 \\\hline
		$\cC_{T} = 5$  & 38  & 45  & 35  & 59  & 21  & 40  & 238 \\\hline
		$\cC_{T} = 6$  & 62  & 63  & 60  & 45  & 42  & 53  & 325 \\\hline
		$\cC_{T} = 7$  & 77  & 74  & 75  & 70  & 82  & 62  & 440 \\\hline
		$\cC_{T} = 8$  & 62  & 54  & 64  & 57  & 63  & 49  & 349 \\\hline
		$\cC_{T} = 9$  & 68  & 73  & 64  & 66  & 63  & 76  & 410 \\\hline
		$\cC_{T} = 10$ & 80  & 55  & 56  & 63  & 64  & 65  & 383 \\\hline
		$\cC_{T} = 11$ & 54  & 56  & 62  & 60  & 78  & 72  & 382 \\\hline
		$\cC_{T} = 12$ & 50  & 50  & 46  & 53  & 39  & 46  & 284 \\\hline
		$\cC_{T} = 13$ & 26  & 31  & 22  & 30  & 23  & 23  & 155 \\\hline
		$\cC_{T} = 14$ & 7   & 6   & 13  & 9   & 4   & 6   & 45 \\\hline
		$\cC_{T} = 15$ & \ES & \ES & 2   & \ES & 1   & \ES & 3 \\\hline
		$\cC_{T} = 16$ & 1   & \ES & \ES & 1   & \ES & \ES & 2 \\\hline
		
	\end{tblr}
	\caption{Number of sampled architectures vs. tensor complexity ($\cC_{T}$) of the sampled block.}
	\label{tab:DatasetVsNumTens}
\end{table*}

\begin{figure}[h!]
	\centering
	\includegraphics[scale=\VsXScale]{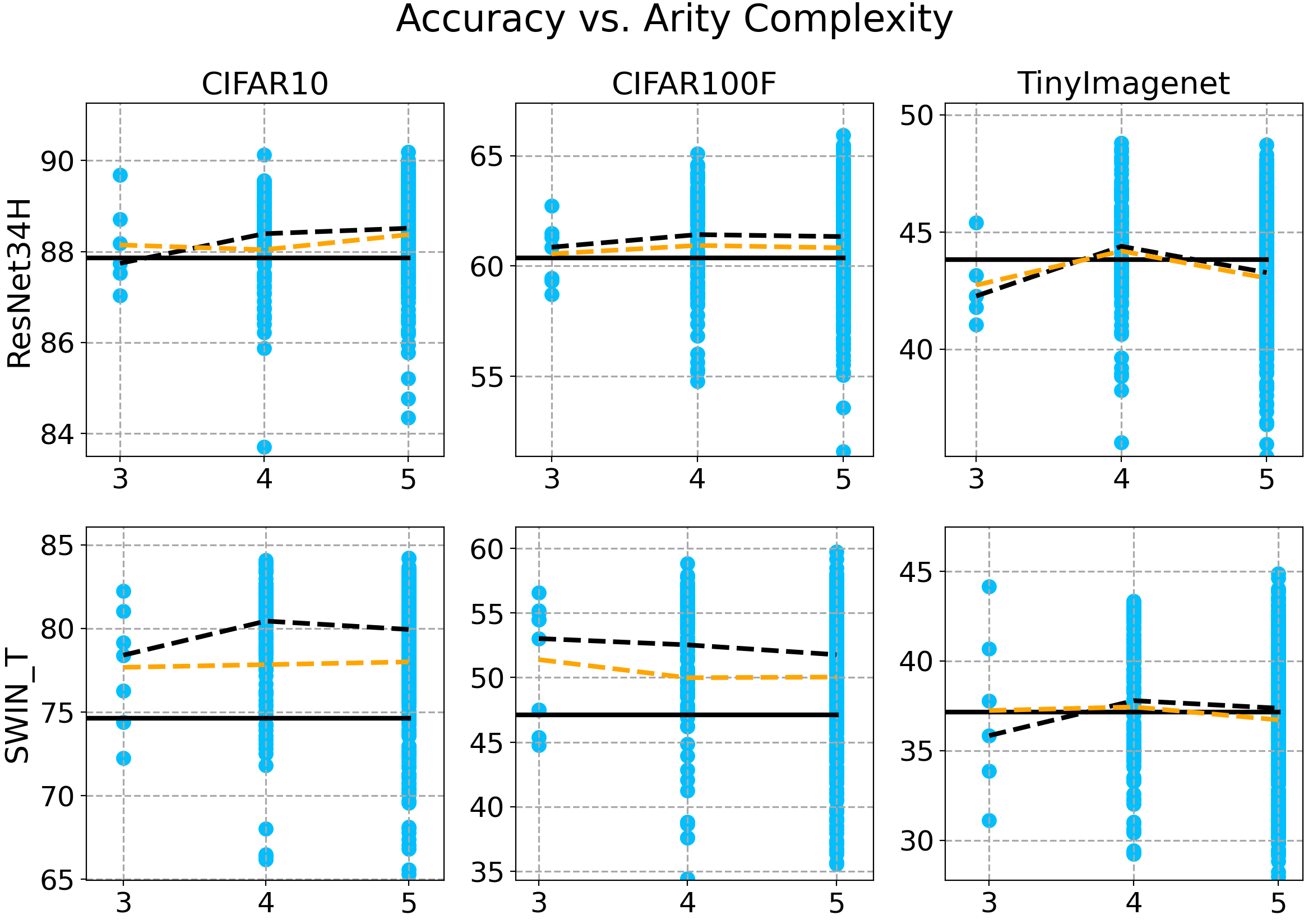}
	\caption{Distribution of sampled architectures ({\color{cyan}$\bullet$}) vs. arity complexity $\cC_{\alpha}$ of the sampled block. ({\color{orange}$\bm{- -}$}) Mean sample accuracy. ({$\bm{- -}$}) Median sample accuracy. (\textbf{---}) Stage 1 baseline accuracy.}
\end{figure}
\begin{table*}[h!]
	\centering
	\begin{tblr}{
			colspec={c||c|c||c|c||c|c||c},
			column{1}={halign=r}
		}
		Dataset & \SetCell[c=2]{c} CIFAR 10 & & \SetCell[c=2]{c} CIFAR 100 & & \SetCell[c=2]{c} Tiny ImageNet & & \\\hline
		Stage 1 & ResNet34 & SWIN\_T & ResNet34 & SWIN\_T & ResNet34 & SWIN\_T\ & Total \\\hline\hline
		
		$\cC_{\alpha} = 2$ & 6   & 7   & 7   & 10  & 5   & 6   & 41\\\hline
		$\cC_{\alpha} = 3$ & 116 & 127 & 105 & 111 & 90  & 108 & 657 \\\hline
		$\cC_{\alpha} = 4$ & 404 & 373 & 390 & 392 & 391 & 380 & 2330 \\\hline
		
	\end{tblr}
	\caption{Number of sampled architectures vs. arity complexity ($\cC_{\alpha}$) of the sampled block.}
	\label{tab:DatasetVsArity}
\end{table*}

\clearpage

\begin{figure}[h!]
	\centering
	\includegraphics[scale=\VsXScale]{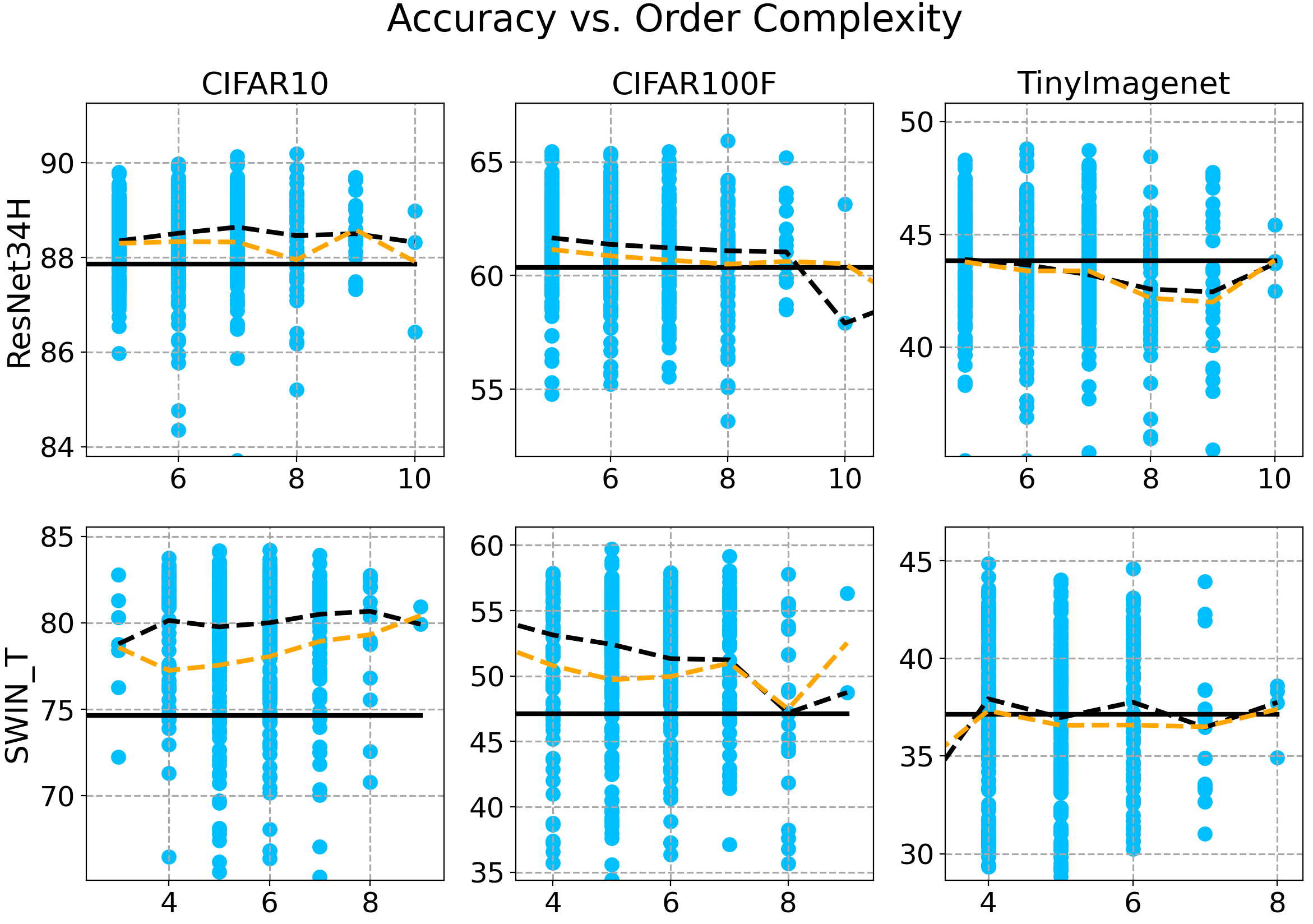}
	\caption{Distribution of sampled architectures ({\color{cyan}$\bullet$}) vs. order complexity $\cC_{O}$ of the sampled block. ({\color{orange}$\bm{- -}$}) Mean sample accuracy. ({$\bm{- -}$}) Median sample accuracy. (\textbf{---}) Stage 1 baseline accuracy.}
\end{figure}
\begin{table*}[h!]
	\centering
	\begin{tblr}{
			colspec={c||c|c||c|c||c|c||c},
			column{1}={halign=r}
		}
		Dataset & \SetCell[c=2]{c} CIFAR 10 & & \SetCell[c=2]{c} CIFAR 100 & & \SetCell[c=2]{c} Tiny ImageNet & & \\\hline
		Stage 1 & ResNet34 & SWIN\_T & ResNet34 & SWIN\_T & ResNet34 & SWIN\_T\ & Total \\\hline\hline
		
		$\cC_{O} = 3$  & \ES & 7   & \ES & 3   & \ES & 7   & 17 \\\hline
		$\cC_{O} = 4$  & \ES & 58  & \ES & 80  & \ES & 213 & 351 \\\hline
		$\cC_{O} = 5$  & 156 & 193 & 140 & 183 & 131 & 163 & 966 \\\hline
		$\cC_{O} = 6$  & 144 & 159 & 160 & 162 & 117 & 89  & 831 \\\hline
		$\cC_{O} = 7$  & 151 & 73  & 120 & 62  & 142 & 18  & 566 \\\hline
		$\cC_{O} = 8$  & 49  & 15  & 64  & 21  & 65  & 8   & 222 \\\hline
		$\cC_{O} = 9$  & 23  & 2   & 15  & 2   & 27  & \ES & 69 \\\hline
		$\cC_{O} = 10$ & 3   & \ES & 2   & \ES & 4   & \ES & 9 \\\hline
		$\cC_{O} = 11$ & \ES & \ES & 1   & \ES & \ES & \ES & 1 \\\hline
		
	\end{tblr}
	\caption{Number of sampled architectures vs. order complexity ($\cC_{O}$) of the sampled block.}
	\label{tab:DatasetVsOrderComplexity}
\end{table*}

\clearpage

\begin{figure}[h!]
	\centering
	\includegraphics[scale=\VsXScale]{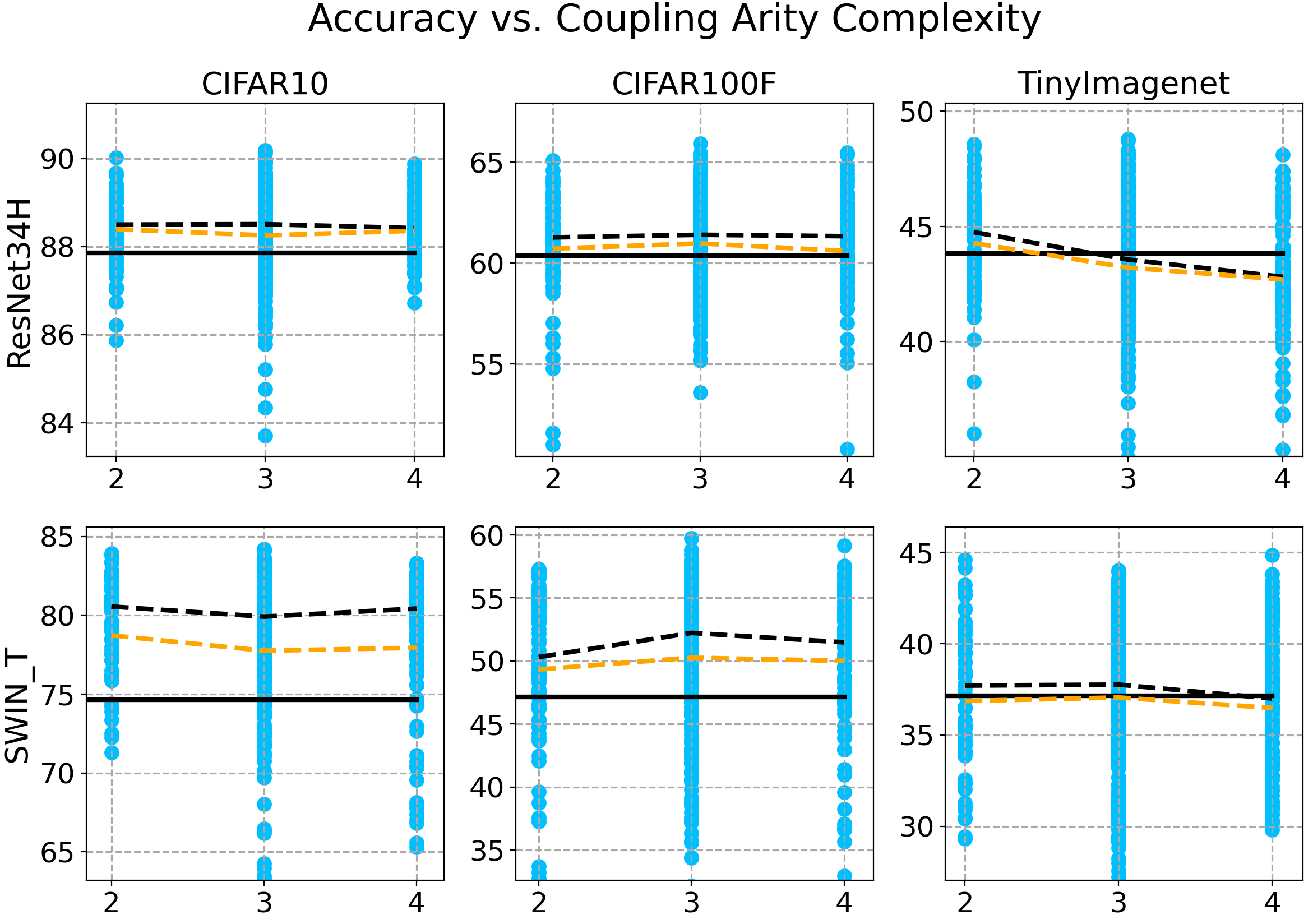}
	\caption{Distribution of sampled architectures ({\color{cyan}$\bullet$}) vs. coupling arity complexity $\cC_{A}$ of the sampled block. ({\color{orange}$\bm{- -}$}) Mean sample accuracy. ({$\bm{- -}$}) Median sample accuracy. (\textbf{---}) Stage 1 baseline accuracy.}
\end{figure}
\begin{table*}[h!]
	\centering
	\begin{tblr}{
			colspec={c||c|c||c|c||c|c||c},
			column{1}={halign=r}
		}
		Dataset & \SetCell[c=2]{c} CIFAR 10 & & \SetCell[c=2]{c} CIFAR 100 & & \SetCell[c=2]{c} Tiny ImageNet & & \\\hline
		Stage 1 & ResNet34 & SWIN\_T & ResNet34 & SWIN\_T & ResNet34 & SWIN\_T\ & Total \\\hline\hline
		
		$\cC_{A} = 2$ & 79  & 81  & 74  & 84  & 71  & 65  & 454\\\hline
		$\cC_{A} = 3$ & 357 & 316 & 303 & 319 & 300 & 293 & 1888 \\\hline
		$\cC_{A} = 4$ & 90  & 110 & 125 & 110 & 115 & 136 & 686 \\\hline
		
	\end{tblr}
	\caption{Number of sampled architectures vs. coupling arity complexity ($\cC_{A}$) of the sampled block.}
	\label{tab:DatasetVsCouplingArity}
\end{table*}

\clearpage

\begin{figure}[h!]
	\centering
	\includegraphics[scale=\VsXScale]{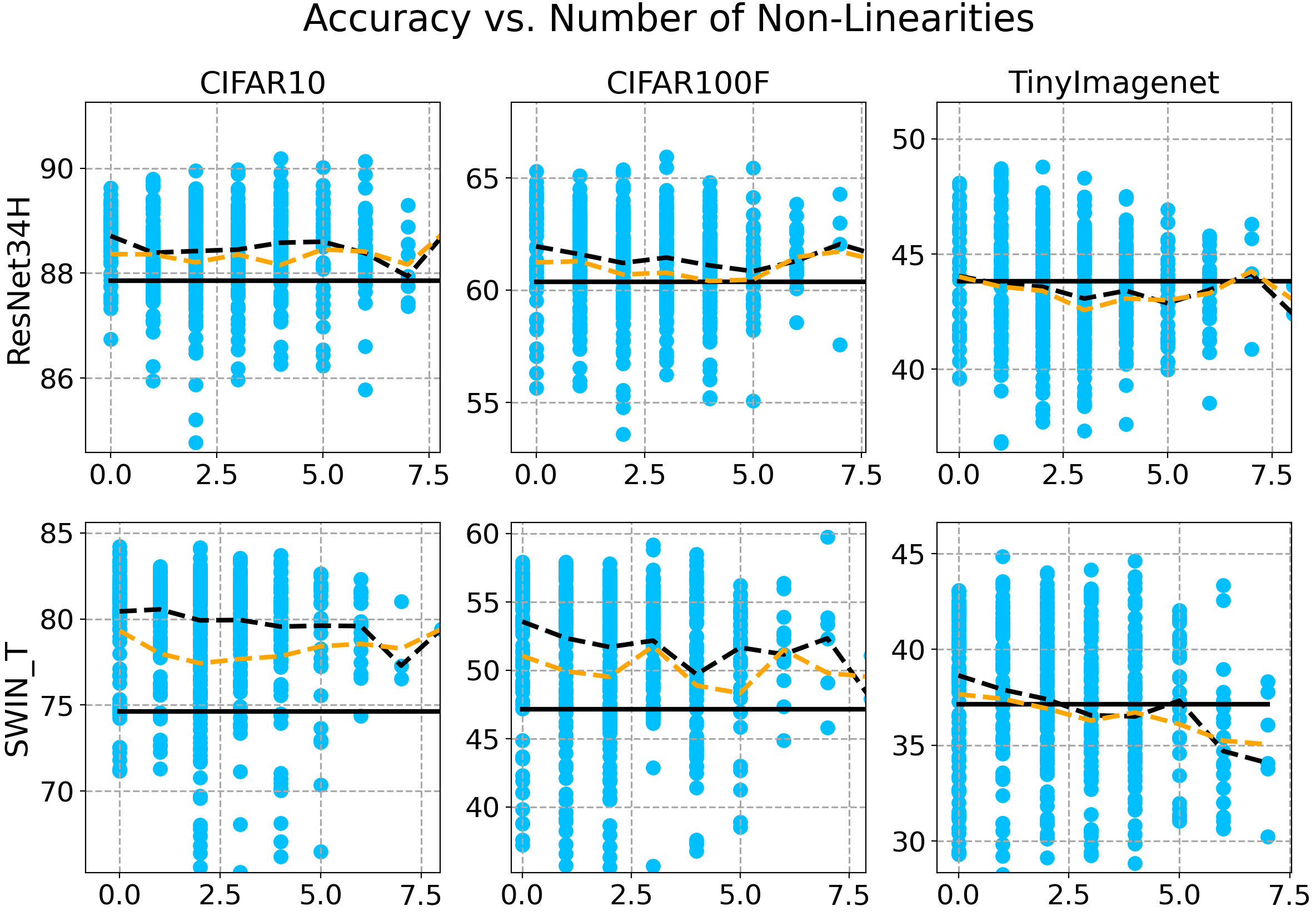}
	\caption{Distribution of sampled architectures ({\color{cyan}$\bullet$}) vs. number of non-linear activations $\cC_{NL}$ of the sampled block. ({\color{orange}$\bm{- -}$}) Mean sample accuracy. ({$\bm{- -}$}) Median sample accuracy. (\textbf{---}) Stage 1 baseline accuracy.}
\end{figure}
\begin{table*}[h!]
	\centering
	\begin{tblr}{
			colspec={c||c|c||c|c||c|c||c},
			column{1}={halign=r}
		}
		Dataset & \SetCell[c=2]{c} CIFAR 10 & & \SetCell[c=2]{c} CIFAR 100 & & \SetCell[c=2]{c} Tiny ImageNet & & \\\hline
		Stage 1 & ResNet34 & SWIN\_T & ResNet34 & SWIN\_T & ResNet34 & SWIN\_T\ & Total \\\hline\hline
		
		$\cC_{NL} = 0$ & 53  & 70  & 51  & 81  & 39  & 95  & 389\\\hline
		$\cC_{NL} = 1$ & 73  & 81  & 87  & 98  & 85  & 79  & 503 \\\hline
		$\cC_{NL} = 2$ & 135 & 146 & 121 & 120 & 123 & 122 & 767 \\\hline
		$\cC_{NL} = 3$ & 89  & 92  & 94  & 80  & 92  & 74  & 521\\\hline
		$\cC_{NL} = 4$ & 89  & 69  & 88  & 76  & 81  & 66  & 469 \\\hline
		$\cC_{NL} = 5$ & 45  & 26  & 36  & 36  & 32  & 33  & 208 \\\hline
		$\cC_{NL} = 6$ & 32  & 18  & 20  & 12  & 26  & 19  & 127 \\\hline
		$\cC_{NL} = 7$ & 9   & 3   & 4   & 8   & 4   & 6   & 34 \\\hline
		$\cC_{NL} > 7$ & 1   & 2   & 1   & 3   & 4   & \ES & 11 \\\hline
		
	\end{tblr}
	\caption{Number of sampled architectures vs. the number of non linear activations ($\cC_{NL}$) of the sampled block.}
	\label{tab:DatasetVsNumNL}
\end{table*}

\clearpage
\subsection{The Red Star Architecture \label{sec:RedStarArch}}
In this part, we provide the complete structure of the highlighted red start CIFAR-100 architecture ({\color{red}$\bm{\star}$}) discussed throughout this document.

\subsubsection{\label{sec:RedStarArch_SMs}Graphical Description}
We start with the tensor equation matrix; then, the tensor operation matrices:
\begin{figure}[h!]
	\centering
	\includegraphics[scale=\VsXScale]{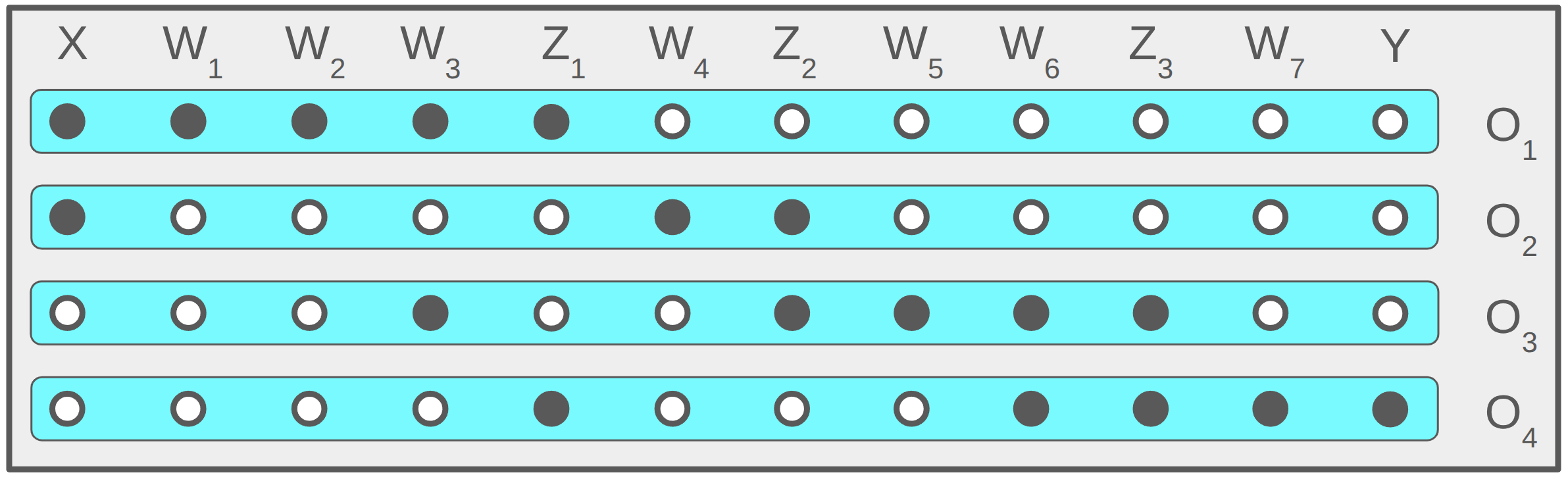}
	\caption{Tensor equation matrix for the sampled block of the {\color{red}$\bm{\star}$} architecture. $\cC_{op} = 4$, $\cC_T = 12$.}
\end{figure}

\begin{figure}[h!]
	\centering
	\includegraphics[scale=0.175]{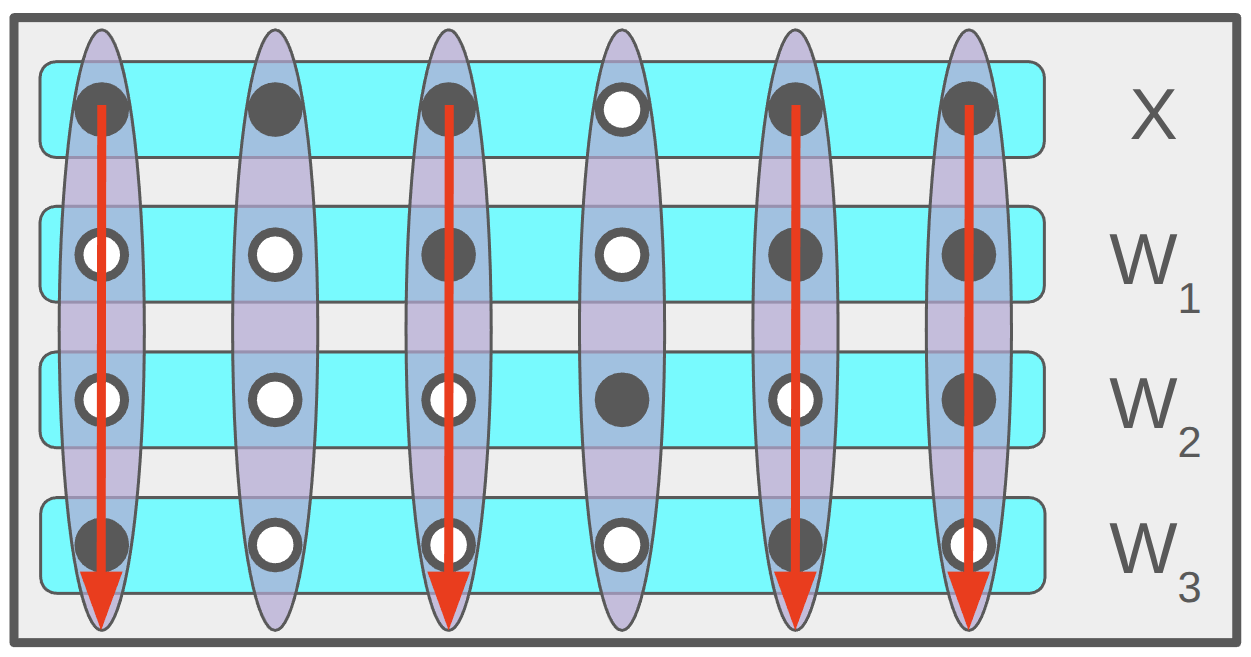}\\ \vspace{0.5em} \includegraphics[scale=0.175]{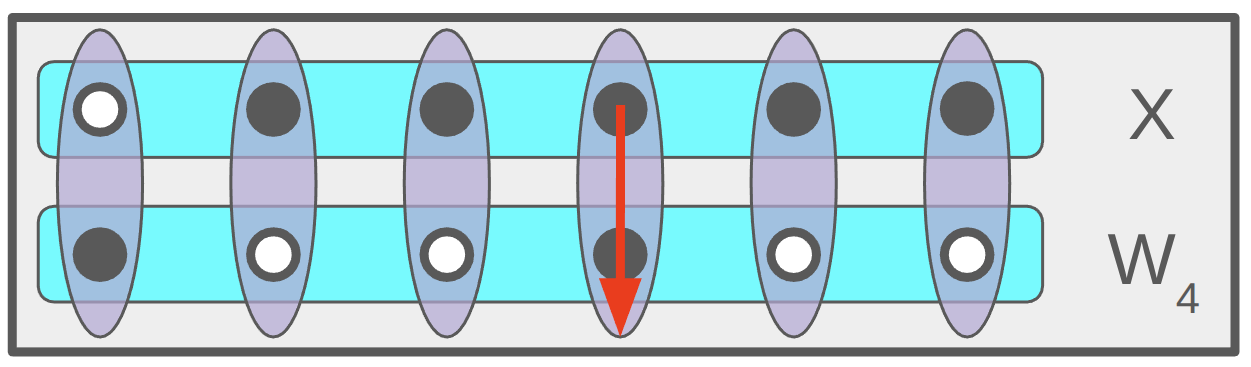}
	\caption{Tensor operation matrices for operations $1$ (top) and $2$ (bottom) used in the {\color{red}$\bm{\star}$} architecture. $\cC_{O} = 6$, $\cC_{\alpha} = 4,2$, $\cC_A = 3,2$.}
\end{figure}

\begin{figure}[h!]
	\centering
	\includegraphics[scale=0.175]{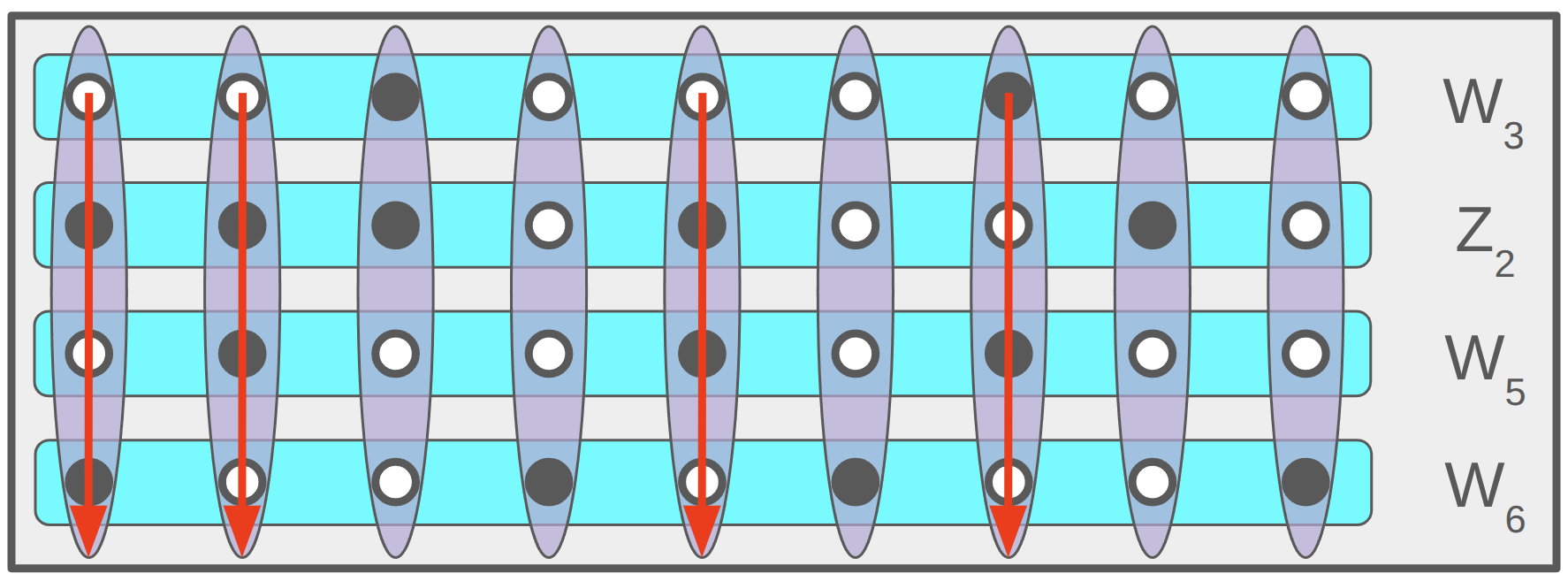}
	\caption{Tensor operation matrix for operation $3$ used in the {\color{red}$\bm{\star}$} architecture. $\cC_{O} = 9$, $\cC_{\alpha} = 4$, $\cC_A = 2$.}
\end{figure}

\begin{figure}[h!]
	\centering
	\includegraphics[scale=0.175]{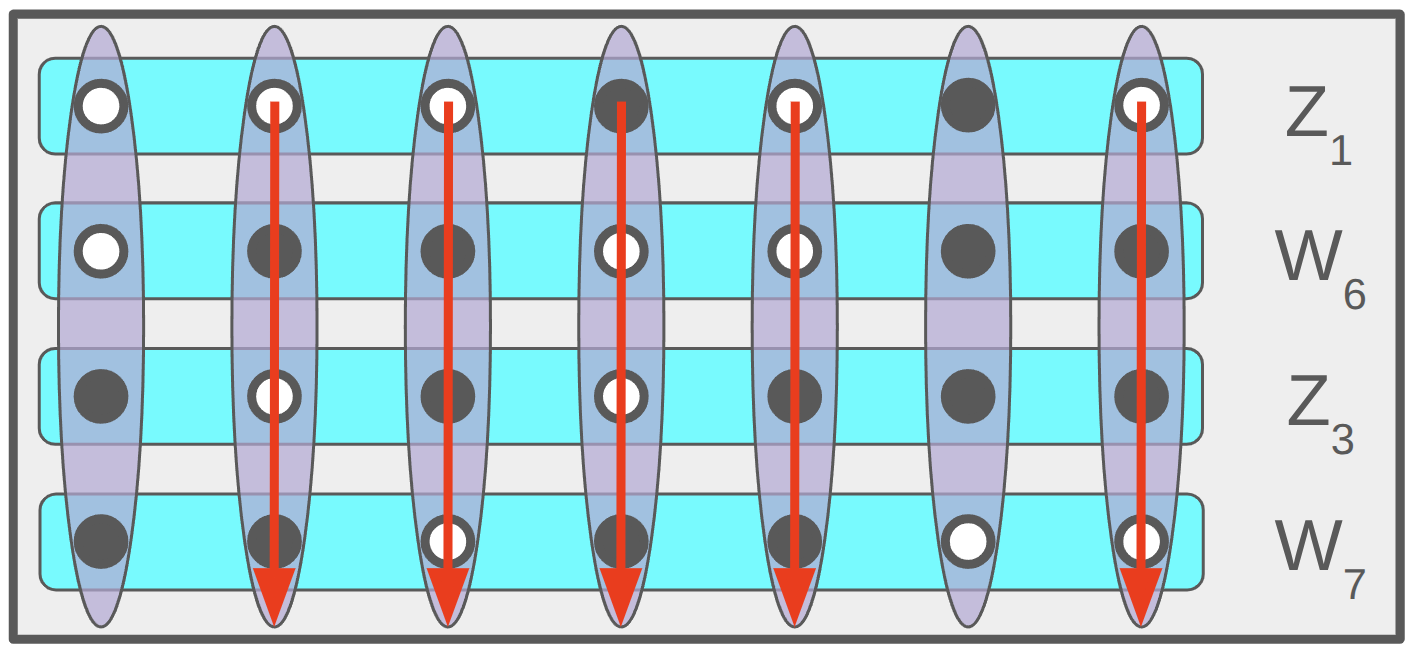}
	\caption{Tensor operation matrix for operation $4$ used in the {\color{red}$\bm{\star}$} architecture. $\cC_{O} = 7$, $\cC_{\alpha} = 4$, $\cC_A = 3$. Same TOM as in \cref{fig:SpicyTOM}.}
\end{figure}

\paragraph{Non Linear Activations.} The red star architecture uses two non linear activations. First, Layer Norm is applied across the $2^{nd}$-mode of the output from the second operation, $O_2$. Second, Layer Norm is applied across both modes of the final output from $O_4$, i.e., across both non-contracted modes of the fourth tensor operation matrix. Notably, no non-linear activations are applied to the output of the high complexity operations $O_1$ and $O_3$. 

\paragraph{Discussion.}
The tensor equation matrix involves a type of ``skip connection''. Observe how the input from stage 1, $X$, is first sent through separate operations $O_1$ and $O_2$. Later, the output from the first operation, $Z_1$, is skipped forward as input to the final operation, $O_4$. This final operation is the same high complexity example discussed in \cref{fig:SpicyTOM}, where in this case $Z_1$ and $Z_3$ are used as the inputs; $W_6$ and $W_7$ are the learned parameter tensors. Notably, ``in between'' this skip connection of $Z_1$, an even higher complexity operation, $O_3$, is performed. Interestingly, this operation takes as input a ``parameter skip connection'', meaning that $W_3$ is reused from $O_1$. This is effectively a particular type of weight tying. A similar reuse of parameters occurs between $O_3$ and $O_4$. The key takeaway is that this architecture simultaneously exhibits high architectural complexity and low parameter requirements.

\subsubsection{\label{sec:RedStarArch_Alg}Algebraic Description}
We conclude this section with the formal system of equations used to define the sampled block of the red star architecture.
\begin{equation*}
	\begin{aligned}
		Z_1[j,l] &= \sum_{i,k,m,n} X[j,k,l,m,n]W_1[k,m,n]W_2[l,n]W_3[i,m]\\
		Z_2'[i,j,k,m,n] &= \sum_{l} X[j,k,l,m,n]W_1[i,l]\\
		Z_2[i,j,k,m,n] &= \gamma_2[j] \bigg[ \frac{Z_2'[i,j,k,m,n] - E[Z_2'][j]}{\sqrt{Var[Z_2'][j] + \varepsilon}} \bigg] + \beta_2[j]\\
		Z_3[k,l,n,p,q] &= \sum_{i,j,m,o} W_3[k,o]Z_2[i,j,k,m,p]W_5[j,m,o]W_6[i,l,n,q]\\
		Z_4'[i,n] &= \sum_{j,k,l,m,o} Z_1[l,n]W_6[j,k,n,o]Z_3[i,k,m,n,o]W_7[i,j,l,m]\\
		Z_4 &= (\gamma_4 \odot Z_4') \oplus \beta_4\\
	\end{aligned}
\end{equation*}
Here $\gamma_2$, $\gamma_4$, $\beta_2$, $\beta_4$ are the learned parameters of the two layer norm activations. As the second layer norm is applied across all modes of the tensor, it reduces to a learned element-wise scale and shift.

\end{document}